\documentclass[12pt,a4paper]{report}
\let\openright=\clearpage

\usepackage[a-2u]{pdfx}

\usepackage[utf8]{inputenc}

\usepackage{lmodern}

\usepackage{amsmath}        
\usepackage{amsfonts}       
\usepackage{amssymb}        
\usepackage{amsthm}         
\usepackage{bbding}         
\usepackage{bm}             
\usepackage{graphicx}       
\usepackage{fancyvrb}       
\usepackage{natbib}         
\usepackage[nottoc]{tocbibind} 
\usepackage{dcolumn}        
\usepackage{booktabs}       
\usepackage{paralist}       
\usepackage[rgb,hyperref,rgb,hyperref]{xcolor}  
\usepackage{subcaption}
\usepackage{floatrow}
\usepackage{enumitem}   
\usepackage{algorithm}
\usepackage[noend]{algpseudocode}

\usepackage[acronym]{glossaries}
\usepackage{xcolor}
\usepackage{thmtools} 
\usepackage{thm-restate}

\usepackage[most]{tcolorbox}

\usepackage{epigraph}

\usepackage{units}          

\usepackage[bottom]{footmisc}

\DeclareMathOperator*{\argmax}{arg\,max}
\DeclareMathOperator*{\argmin}{arg\,min}

\newglossaryentry{game_tree}{name=game tree,
  description={Formal model of sequential decision making in perfect information settings}}

\newglossaryentry{constant_sum_game}{name=constant-sum game,
  description={Strictly competitive game where the reward of both players always sum to a constant}}

\newglossaryentry{zero_sum_game}{name=zero-sum game,
  description={\Gls{constant_sum_game} where the constant is zero}}

\newglossaryentry{best_response}{name=best response,
  description={Strategy maximizing the player's utility given the fixed policy of other players}}
  
\newglossaryentry{maximin_policy}{name=maximin policy,
  description={Policy that maximizes a value against the worst case opponent}}
  
\newglossaryentry{maximin_value}{name=maximin value,
  description={Unique value of the \gls{maximin_policy}}}
  
\newglossaryentry{nash_equilibrium}{name=Nash equilibrium,
  description={Strategy profile where none of the players can improve}}

\newglossaryentry{game_value}{name=game value,
  description={Unique value of the game when both players follow optimal policy}}

\newglossaryentry{minimax_algo}{name=minimax algorithm, description={Simple tree traversal algorithm producing optimal policies in perfect information games}}

\newglossaryentry{alphabeta_pruning}{name=alpha-beta pruning, description={Pruning trick that allows the \gls{minimax_algo} to potentially visit less nodes of the game tree}}

\newglossaryentry{selfplay}{name=self-play, description={A general framework where the players play sequence of games and gradually improve based on past performance}}

\newglossaryentry{selfplay_irl}{name=self-play via independent reinforcement learning, description={\Selfplay{} where the next strategy in the sequence is computed using a reinforcement learning algorithm}}

\newglossaryentry{fictitious_play}{name=fictitious play, description={Variant of \selfplay{} where the agent's next strategy in the sequence is \gls{best_response} against the averaged strategy of the opponent}}

\newglossaryentry{best_responding_sequence}{name=best responding sequence, description={Sequence of strategies where the agent's next strategy is \gls{best_response} against the the last strategy of the opponent}}

\newglossaryentry{repeated_game}{name=repeated game, description={Repeated game consists of a sequence of individual matches}}
  
\newglossaryentry{stateless_algorithm}{name=stateless algorithm, description={Online algorithm that does not use an internal state between the games}}

\newglossaryentry{stateful_algorithm}{name=stateful algorithm, description={Online algorithm that uses an internal state between the games and thus can condition the computation on the past games}}
  
\newglossaryentry{tabularized_strategy}{name=tabularized strategy, description={Offline strategy produced from online strategy}}
  
\newglossaryentry{equilibrium_selection_problem}{name=equilibrium selection problem,
  description={Selecting Nash equilibrium to play in multiplayer settings}}

\newglossaryentry{resolving}{name=re-solving,
  description={Search as a method to re-solve an optimal policy}}

\newglossaryentry{online_minimax_algo}{name=online minimax algorithm, description={Online algorithm that runs the \gls{minimax_algo} on the \subgame{} rooted in the current state}}

\newglossaryentry{full_lookahead}{name=full lookahead,
  description={Search tree consisting all the states of the \subgame{}}}

\newglossaryentry{limited_lookahead}{name=limited lookahead,
  description={Search tree consisting of some state of the \subgame{}}}

\newglossaryentry{support}{name=support,
  description={Actions with a non-zero probability}}

\newglossaryentry{expected_reward}{name=expected reward,
  description={Expected reward of a player given all players' policies}}

\newglossaryentry{imperfect_recall}{name=imperfect recall, description={grouping of states to information sets so that the agent forgets previously known information}}

\newglossaryentry{regret_matching}{name=regret matching,
  description={Regret minimization algorithm - selects actions proportionally to the positive regret}}
  
\newglossaryentry{regret_matching_plus}{name=regret matching plus,
  description={Regret minimization algorithm based on regret matching - caps the negative regrets to zero}}
  
\newglossaryentry{behavioral_strategy}{name=behavioral strategy,
    plural={behavioral strategies},
  description={Strategy defined over all the information states in \gls{efg_g}}}

\newglossaryentry{information_state}{name=information state,
  description={Decision state of a player in \gls{efg_g}}}

\newglossaryentry{matrix_games}{name=matrix games,
  description={Simultaneous move games - simple formalism for imperfect information games}}

\newglossaryentry{normal_form_games}{name=normal form games,
  description={See \gls{matrix_games}}}

\newglossaryentry{glasses}{name=Glasses,
  description={Simple variant of the graph chasing game}}

\newglossaryentry{lbr}{name=local best response,
  description={Exploitability approximating technique.}}

\newglossaryentry{consistent_states}{name=consistent states,
  description={Set of (opponent's) states consistent with a current state.}}

\newglossaryentry{consistency_operation}{name=consistency operation,
  description={Operation that given a state, returns \gls{consistent_states}}}

\newglossaryentry{common_info_set}{name=common information set,
  description={Set of states closed under the consistency operation}}

\newglossaryentry{counterfactual_best_response}{name=counterfactual best response,
  description={Best response having non-negative counterfactual regrets}}

\newglossaryentry{subgame_refinement}{name=subgame refinement,
  description={Resolving a sub-game using a finer-grained abstraction than the original policy used for resolving constraints.}}

\newglossaryentry{aivat}{name=AIVAT,
  description={Variance reduction technique for evaluating agent's performance}}

\makeglossaries

\setacronymstyle{long-short}

\newacronym{neq}{$\mathbb{NEQ}$}{set of all Nash equilibrium strategy profiles}
\newacronym{br}{$\mathbb{BR}$}{set of all best response strategies}
\newacronym{p2p}{P2P}{Peer-to-Peer}
\newacronym{rl}{RL}{reinforcement learning}

\newglossaryentry{cfr_g}{name={Counterfactual regret minimization},
    description={Regret minimization in sequential decision making (i.e on a game tree)}}

\newglossaryentry{cfr}{type=\acronymtype, name={CFR}, description={Counterfactual regret minimization}, first={Counterfactual regret minimization (CFR)\glsadd{cfr_g}}, see=[Glossary:]{cfr_g}}

\newglossaryentry{efg_g}{name={extensive form games},
    description={Formalism for imperfect information games}}

\newglossaryentry{fog_g}{name={factored observation stochastic games},
    description={Formalism for imperfect information games}}

\def\ThesisTitle{Search in Imperfect Information Games}

\def\ThesisAuthor{Martin Schmid}

\def\YearSubmitted{2021}

\def\Department{Department of Applied Mathematics}

\def\DeptType{Department}

\def\Supervisor{Milan Hladík}

\def\SupervisorsDepartment{ Department of Applied Mathematics}

\def\Advisor{Michael Bowling}

\def\AdvisorsDepartment{Department of Computing Science, University of Alberta}

\def\StudyProgramme{Computer Science}
\def\StudyBranch{Discrete Models and Algorithms}

\def\Abstract{%
From the very dawn of the field, search with value functions was a fundamental concept of computer games research.
Turing's chess algorithm from $1950$ was able to think two moves ahead \citep{copeland2004essential}, and Shannon's work on chess from $1950$ includes an extensive section on evaluation functions to be used within a search \citep{shannon1950xxii}.
Samuel's checkers program from $1959$ already combines search and value functions that are learned through self-play and bootstrapping \citep{samuel1959some}.
TD-Gammon improves upon those ideas and uses neural networks to learn those complex value functions --- only to be again used within search \citep{tesauro1995temporal}. 
The combination of decision-time search and value functions has been present in the remarkable milestones where computers bested their human counterparts in long standing challenging games --- DeepBlue for Chess \citep{campbell2002deep} and AlphaGo for Go \citep{silver2016mastering}.

Until recently, this powerful framework of search aided with (learned) value functions has been limited to {\pig}s.
As many interesting problems do not provide the agent perfect information of the environment, this was an unfortunate limitation.
This thesis introduces the reader to sound search for imperfect information games.

}

\def\Keywords{%
{game theory}, {search}, {imperfect information}, {games}, {DeepStack}
}

\hypersetup{unicode}
\hypersetup{breaklinks=true}

\makeatletter
\def\@makechapterhead#1{
  {\parindent \z@ \raggedright \normalfont
  \Huge\bfseries \thechapter. #1
  \par\nobreak
  \vskip 20\p@
}}
\def\@makeschapterhead#1{
  {\parindent \z@ \raggedright \normalfont
  \Huge\bfseries #1
  \par\nobreak
  \vskip 20\p@
}}
\makeatother

\def\chapwithtoc#1{
\chapter*{#1}
\addcontentsline{toc}{chapter}{#1}
}

\theoremstyle{plain}
\newtheorem{thm}{Theorem}
\newtheorem{lemma}[thm]{Lemma}

\newtheorem{corollary}{Corollary}[thm]
\newtheorem{observation}{Observation}[thm]

\theoremstyle{plain}
\newtheorem{defn}{Definition}

\theoremstyle{remark}

\definecolor{block-gray}{gray}{0.90}
\newtcolorbox{myquote}{colback=block-gray,grow to right by=-0mm,grow to left by=-0mm,
boxrule=0pt,boxsep=0pt,breakable}

\DefineVerbatimEnvironment{code}{Verbatim}{fontsize=\small, frame=single}

\newcommand{\R}{\mathbb{R}}

\newcommand{\T}[1]{#1^\top}

\newcommand{\NEQ}{\mathbb{NEQ}}

\newcommand{\brv}{BRV}
\newcommand{\cfbrv}{CBRV}
\newcommand{\brset}{\mathbb{BR}}
\newcommand{\cfbrset}{\mathbb{CFBR}}
\newcommand{\maximinset}{\mathbb{MAXIMIN}}
\newcommand{\nashset}{\mathbb{NEQ}}

\newcommand{\rl}{reinforcement learning}

\newcommand{\gt}{game theory}

\newcommand{\pig}{perfect information game}

\newcommand{\iig}{imperfect information game}

\newcommand{\first}[1]{{#1}}
\newcommand{\second}[1]{{#1}}

\newcommand{\lookahead}{look-ahead}

\newcommand{\Minimax}{Minimax}
\newcommand{\maximin}{maximin}
\newcommand{\Maximin}{Maximin}
\newcommand{\selfplay}{self-play}
\newcommand{\Selfplay}{Self-play}
\newcommand{\SelfPlay}{Self-Play}
\newcommand{\subgame}{sub-game}
\newcommand{\subgames}{sub-games} 
\newcommand{\Subgame}{Sub-game}
\newcommand{\Subgames}{Sub-games} 
\newcommand{\subproblem}{sub-problem}

\newcommand{\resolve}{re-solve}
\newcommand{\resolving}{re-solving}
\newcommand{\resolved}{re-solved}
\newcommand{\Resolving}{Re-solving}
\newcommand{\search}{\Omega}
\newcommand{\match}{z} 
\newcommand{\consistent}{\mathcal{C}}
\newcommand{\rps}{rock-paper-scissors}
\newcommand{\glasses}{Glasses}
\newcommand{\Glasses}{Glasses}

\DeclareMathOperator{\EX}{\mathbb{E}}

\begin{document}


\pagestyle{empty}
\hypersetup{pageanchor=false}
\begin{center}

\centerline{\mbox{\includegraphics[width=166mm]{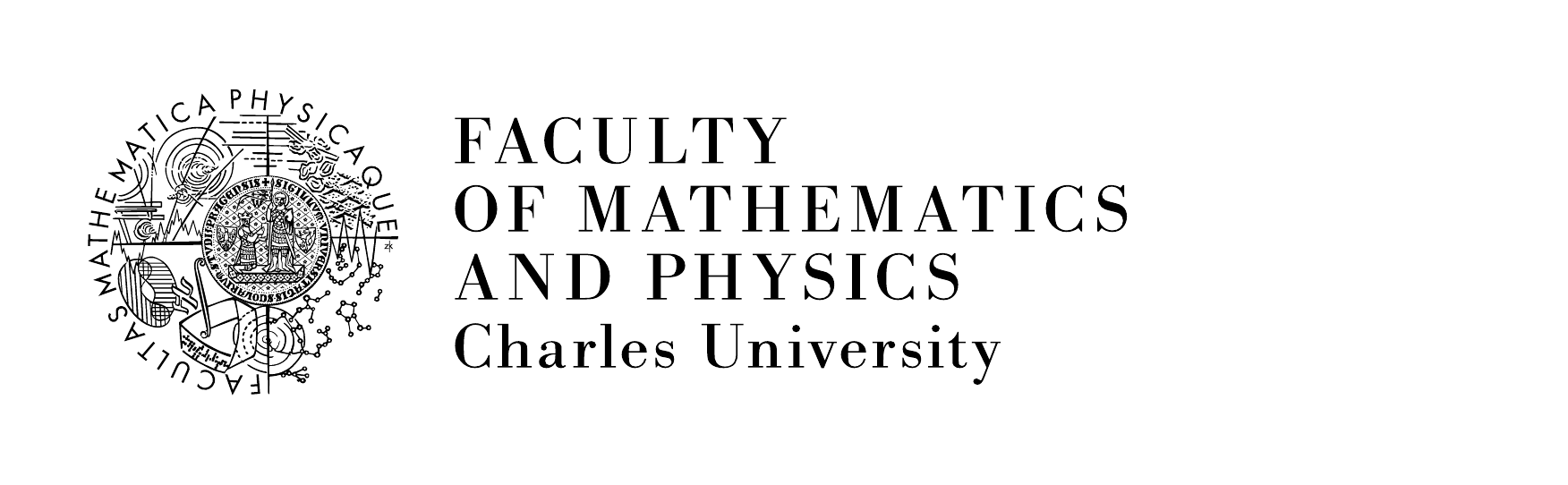}}}

\vspace{-8mm}
\vfill

{\bf\Large DOCTORAL THESIS}

\vfill

{\LARGE\ThesisAuthor}

\vspace{15mm}

{\LARGE\bfseries\ThesisTitle}

\vfill

\Department

\vfill

\begin{tabular}{rl}

Supervisor of the doctoral thesis: & \Supervisor \\
\noalign{\vspace{2mm}}
Advisor of the doctoral thesis: & \Advisor \\
\noalign{\vspace{2mm}}
Study programme: & \StudyProgramme \\
\noalign{\vspace{2mm}}
Study branch: & \StudyBranch \\
\end{tabular}

\vfill

Prague \YearSubmitted

\end{center}

\newpage



\openright
\hypersetup{pageanchor=true}
\pagestyle{plain}
\pagenumbering{roman}
\vglue 0pt plus 1fill

\noindent
I declare that I carried out this doctoral thesis independently, and only with the cited
sources, literature and other professional sources.

\medskip\noindent
I understand that my work relates to the rights and obligations under the Act No.~121/2000 Sb.,
the Copyright Act, as amended, in particular the fact that the Charles
University has the right to conclude a license agreement on the use of this
work as a school work pursuant to Section 60 subsection 1 of the Copyright Act.

\vspace{10mm}

\hbox{\hbox to 0.5\hsize{%
In ........ date ............	
\hss}\hbox to 0.5\hsize{%
signature of the author
\hss}}

\vspace{20mm}
\newpage


\openright

\noindent

\newpage


\openright

\vbox to 0.5\vsize{
\setlength\parindent{0mm}
\setlength\parskip{5mm}

Title:
\ThesisTitle

Author:
\ThesisAuthor

\DeptType:
\Department

Supervisor:
\Supervisor, \SupervisorsDepartment

Advisor:
\Advisor, \AdvisorsDepartment

Abstract:
\Abstract

Keywords:
\Keywords

\vss}

\newpage

\openright
\pagestyle{plain}
\pagenumbering{arabic}
\setcounter{page}{1}

\chapter*{Acknowledgments}
\second{
There are many people who I am deeply grateful to and for.

My supervisor Milan Hladik allowed for unprecedented level of freedom in my research.
He trusted me to pursue ideas and goals that I found interesting and relevant, always believing me to achieve the challenging dreams of mine.
He also included me in many unforgettable seminars at the society for the optimization of the consumption, experience which one can only enjoy.

I have met my advisor Michael Bowling during my very first conference and my very first travel across the ocean.
At the same conference two years later is where we discussed the exciting plan of me and Matej visiting the Edmonton office to join his team.
I vividly remember me pacing back and forth at the hotel room, pondering the idea of moving to Canada to pursue the dream of DeepStack.
Since then, my work under Mike has deeply molded me and made me a better person and better scientist.
As his motto is ``once my student, always my student'', I am proud to belong to this club.

DeepStack brings me to Matej Moravcik, with whom I have spent sleepless nights --- full of coffee, code and math (the most enjoyable combination really!), working for almost a year to finish the project.
Never will I forget the thrill of working on this amazing break-through, the ups and downs, the excitement of pushing the boundaries and the scientific joy the live is worth living for.

My scientific mind and heart was from the very early age forged by my mother.
She is the archetypal absent-minded professor who toughed me to be curious, think scientifically and work hard.
To never take yourself too seriously, to ignore the mundane standards of the masses but to do the things that excite me and make me happy.    
}

\newpage

\epigraph{Only in silence the word, only in dark the light, only in dying life: bright the hawk's flight on the empty sky.}{\textit{Ursula K. Le Guin, A Wizard of Earthsea}}


\tableofcontents

\glsunsetall

\include{preface}
\chapter*{Contributions}
\addcontentsline{toc}{chapter}{Contributions}

This thesis aims to tell a coherent story of search in imperfect information games.
My main contributions in this line of work are described below.
Each contribution has been published in a major conference or a journal (see given citations and the full list of author's publications at the end of the thesis in Section~\ref{chap:list_of_publications}).

\paragraph{Bounding Support Sizes in Games}
In order to play optimally, imperfect information games often require stochastic policies.
This is in contrast to perfect information games, where an optimal deterministic strategy always exist.
We have proven an intriguing connection between a players' level of uncertainty in state of the game and the support size (number of actions with non-zero probability) \citep{schmid2014bounding}.
Section~\ref{sec:iig-stochastic_policies} includes a brief overview of the result as most of the contribution dates to an older version of the paper, published prior to the beginning of my doctoral study \citep{schmid2013equilibrium}.

\paragraph{Factored Observation Games Formalism}
The work on new sound search methods for imperfect information games revealed that previous formalisms are ill-suited for this task.
We have introduced a new formalism particularly suited for the needs of modern search algorithms 
\citep{kovavrik2019rethinking}.
We have also proven connections to the prior formalism of extensive form games.
Finally, we formulated counterfactual regret minimization (Section~\ref{sec:iig-cfr}) and sequence form linear optimisation (Section~\ref{sec:iig-sequence_games_lp}) in this formalism --- popular algorithms from n extensive form games.
This formalism is also adopted for most of this thesis.
See also Section \ref{sec:igg-fog}.

\paragraph{Variance reduced Monte Carlo Counterfactual Regret Minimization}
Monte Carlo CFR (MCCFR) is a family of game tree sampling algorithms for imperfect information games that has become a key-building block of many recent methods.
As the algorithm does not traverse the entire game tree but rather samples  trajectories, the corresponding sampled values are inherently noisy.
While MCCFR still provides strong (probabilistic) convergence bounds, the variance introduced by sampling can greatly slow down the convergence speed.
We have introduced VR-MCCFR that decreases the variance and results in orders of magnitude faster convergence \citep{schmid2019variance, davis2020low}.
For details see Section \ref{sec:iig-vrmccfr}.

\paragraph{Refining \Subgames{} in Large Imperfect Information Games}
Prior to the dawn of search, state of the art methods were based on the abstraction framework which simply solved a smaller, abstracted game (Section~\ref{sec:iig-abstraction_methods}).
We have introduced sound and safe refinement of \subgames{}, that allows to on-line improve the strategy near the end of the game when the corresponding \subgame{} becomes (more) tractable \citep{moravcik2016refining}.
More details in Section \ref{sec:iig-subgame_refinement}.

\paragraph{Formalising $\epsilon$-Sound Search}
The key solution concept of $\epsilon$-Nash equilibrium is defined for fixed, offline strategies.
We have generalized the concept for online settings, allowing one to analyze the worst-case behavior of search algorithms \citep{vsustr2020sound, vsustr2021sound}.
Our consistency hierarchy then revealed further discrepancies between perfect and imperfect information games.
Relevant Chapters are \ref{chap:pig-online_settings} and \ref{chap:iig-online_settings}.

\paragraph{DeepStack}
The culmination of the presented thesis is the sound search technique of continual \resolving{} (Section~\ref{sec:iig-continual_resolving_vf}).
Continual \resolving{} with value functions represented by deep neural networks is the algorithm behind the DeepStack agent (\url{https://www.deepstack.ai}).
DeepStack (Chapter~\ref{chap:iig-deepstack}) became the first agent to beat professional human players in the game of no-limit Texas hold'em poker, combining search and learning and constitutes a major break-through for search in imperfect information games \citep{moravvcik2017deepstack}.

\paragraph{AIVAT}
As poker is an inherently high variance game, many matches are often required to get a statistically significantly result.
While this is possible for computer agents where computer poker competitions often required millions of data-points during evaluation, it is expensive and even more problematic when human play is involved.
We have developed AIVAT~\citep{burch2018aivat}, a provably unbiased technique for reducing the variance during evaluation. 
Note that this technique is not poker specific and can be used to evaluate any game.
AIVAT and prior variance reduction techniques is described in Section~\ref{sec:iig-aivat}

\chapter{Introduction}

\second{
Games have long served as benchmarks and marked milestones of progress in artificial
intelligence (AI), serving as ``fruit flies'' of AI research \citep{mccarthy1990chess}.
The same way lab mice allow the researchers to speed up a development of new human drugs, 
games allow us to design algorithms for challenging real world environments.
Unlike mice, games allow for an easy natural progression, as different games bring varying challenges.
Games can be single player or multiplayer, perfect or imperfect information, simultaneous moves or sequential, discrete or continuous actions, etc.

Last but not least, games are fun.
Games are fun to play, and even more fun to do research in.
``Don't worry about the overall importance of the problem; work on it if it looks interesting. I think there is a sufficient correlation between interest and importance'' (David Blackwell).
}
\section{Learning}
\second{
As we hope that games help us to design algorithms for a wide class of problems, the more general the algorithms are, the better.
If one had one algorithm for chess, one for poker an another one for security games, it would be harder to imagine how these techniques seamlessly generalize to other problems. 

A powerful idea that is appealing to incorporate into our approaches is learning.
Rather than hard-coding specific game algorithm, we can design an algorithm that learns to play the game by repeatedly interacting with the game and improving its policy.
}

\section{Games in Question}
\second{
This thesis is limited to a particular subset of games.
The class of games we will be concerned with are strictly adversarial (zero sum) two player games.
Many popular games fit these criteria (e.g. chess, poker, go etc.).
Search for this relatively large class of games is still a step forward from perfect information games, as imperfect information games generalize perfect information games.

Why this particular class?
First, the optimal policies provide strong and appealing properties.
Second, these optimal policies are computable in polynomial time (in the size of the game).
Third, it is not clear what policies would be desirable in the larger class of multiplayer games.
The common solution concepts lose the appealing guarantees and thus their appeal.
Overall, the theoretical results for multiplayer games are rather unfortunate: equilibrium policies are harder to compute and they lack the same appealing guarantees as in two-player case.
}

\section{Optimal Policies}
\second{
In single agent environments, the common concept of optimal policy is that of a policy that maximizes the return.
In the case of multiple agents, the reward depends not just on the actions of our agent, but also on the actions of all the other players in the game.
As we do not get to know the policy of the opponents, the question of optimality is complicated.
We still want to maximize return, but how can we do that as we do not know what the opponent will do?
}

\subsection{Worst Case Opponent --- \Maximin{}}
\label{sec:worst_case_reasoning}
\second{
One option is to consider the worst case opponent.
This is rather natural if we think of how one reasons in chess.
Consider an overly simplistic situation where we simply need to maximize chess material.
When making a chess move, would you rather make:
\begin{itemize}
    \item A move that no matter what counter-move the opponent does, you will win a pawn.
    \item A move that for almost all of the counter-moves of the opponent, you will win a rook. But there is one counter-move of the opponent that loses you a queen.
\end{itemize}

Optimizing the reward against a worst-case opponent will lead to the popular \maximin{} solution concept.
This solution concept is guaranteed to exist in both perfect and imperfect information games, with some appealing properties.
}

\subsection{Nash Equilibrium}
\second{
We will also see another solution concept, Nash equilibrium.
Unlike the worst case reasoning, Nash equilibrium is defined as a joint policy (strategy profile) for all the players where none of the players benefit from deviating from this profile.
Crucially, we will learn that for the games in question --- two player zero sum perfect and imperfect information games --- the concepts of minmax and Nash equilibria collapse.
We thus often use the term optimal policy to refer to both solution concepts interchangeably.
}

\subsection{The Appeal of Optimal Policies}
\second{
Arguably the most appealing property of optimal policies in the context of two-player zero-sum games is that if we get to play both positions (e.g. white and black in chess, small and big blind in poker), we can not lose on expectation against any opponent.
This allows us to focus on algorithms that scale to large problems and produce provably near-optimal strategies, rather than dealing with the distributions of the opponent we could be facing.
}

\section{Offline Algorithms}
\second{
Offline algorithms compute the optimal policy prior to the actual game play, in other words, solve the game.
The result of the computation then describes what policy to follow in each state of the game, and
is then simply retrieved during the game-play.
}

\subsection{Tabular vs Implicit Representation}
\second{
An important distinction has to be made between tabular methods and methods that store the strategy in some implicit way.
A natural representation of the computed strategy is a tabular one --- that is,
storing explicitly how to act in every state $s \in \mathcal{S}$.
While such a representation might be necessary for an optimal strategy, it 
can be intractable for large games.

It is worth mentioning that in some cases, it is possible to represent provably optimal policies without the need to the store strategy for each state of the game.
Consider the idea of alpha-beta pruning that allows one to prune parts of the game tree in a sound way \citep{knuth1975analysis}.
The resulting policy can then be thought of as a small certificate and one can verify the validity of such a certificate in time linear in the size of the certificate (in contrast to the size of the full game).
Interestingly, this idea can be generalized to Nash equilibrium certificates in imperfect information games \citep{zhang2020small}.

Rather than explicit representation of the policy, one can opt for an implicit representation, for example by a parametric function representation (a state to a strategy). 
If we represent the states by a single number $s \in \{ 1 \hdots |S| \}$ and the actions in all 
the states are $\{a, b\}$,
we can represent a policy by letting $p(a | s) \propto \frac{1}{s}$.
While such representation is rather silly, it illustrates the idea.
In practice, the function would be some more complex parametric function
to produce $p(a | s) = f_{\Theta}(a, s)$, maybe represented by a deep
neural network.
}

\section{Search}
\second{
Search is a particular form of an online algorithm, similar to the way humans play chess \citep{newell1964example}.
Unlike offline algorithms, online algorithms compute the policy for the required/observed state during game-play.
Given a state of a game, search algorithms reason about potential future states using a \lookahead{} tree.
As the number of all future states can be too large, this \lookahead{} is often truncated and the terminal states assigned a heuristics value.
This evaluation ideally would closely approximate the true value of these states under the assumption of optimality.
That is, the value of a future state is the value of the sub-problem rooted in that state, given that both players play optimally in that sub-problem.
Search algorithms then reason about optimal play within the \lookahead{} tree using these sub-problem values, resulting in an optimal policy for the state in question.
Search is thus typically a combination of i) sub-problem decomposition, ii) (learning of) value functions iii) local (\lookahead{}) reasoning method.
}

\subsection{Benefits}
\second{
There are multiple reasons search is an appealing method.

First, we get to use compute resources during the game-play and can focus these resources on the part of the problem that is immediately relevant.
Search algorithms compute a strategy for the current state to play and do not have to reason about the states irrelevant to the current situation.
Purely offline algorithms do not know what states will be visited during the play, and thus have to compute the strategy for all of them.

Second, search methods get to leverage a model of the environment.
This allows it to exactly reason about the future states and the estimated values  (or even exact values in the case of terminal states).
While the model can be potentially learned \citep{schrittwieser2020mastering}, we limit this work to settings with an exact game model.

While these reason sound appealing, what matters most is empirical performance and search methods do in fact excel in practice.
All the top computer agents in chess, go, shogi, Arimaa and many other games share one thing: they all use search methods.
And now the same argument of performance holds for imperfect information games such as poker.

}

\section{Imperfect Information Games}

\begin{figure}[ht]
    \centering
    \begin{subfigure}[b]{0.3\textwidth}
        \includegraphics[width=\textwidth]{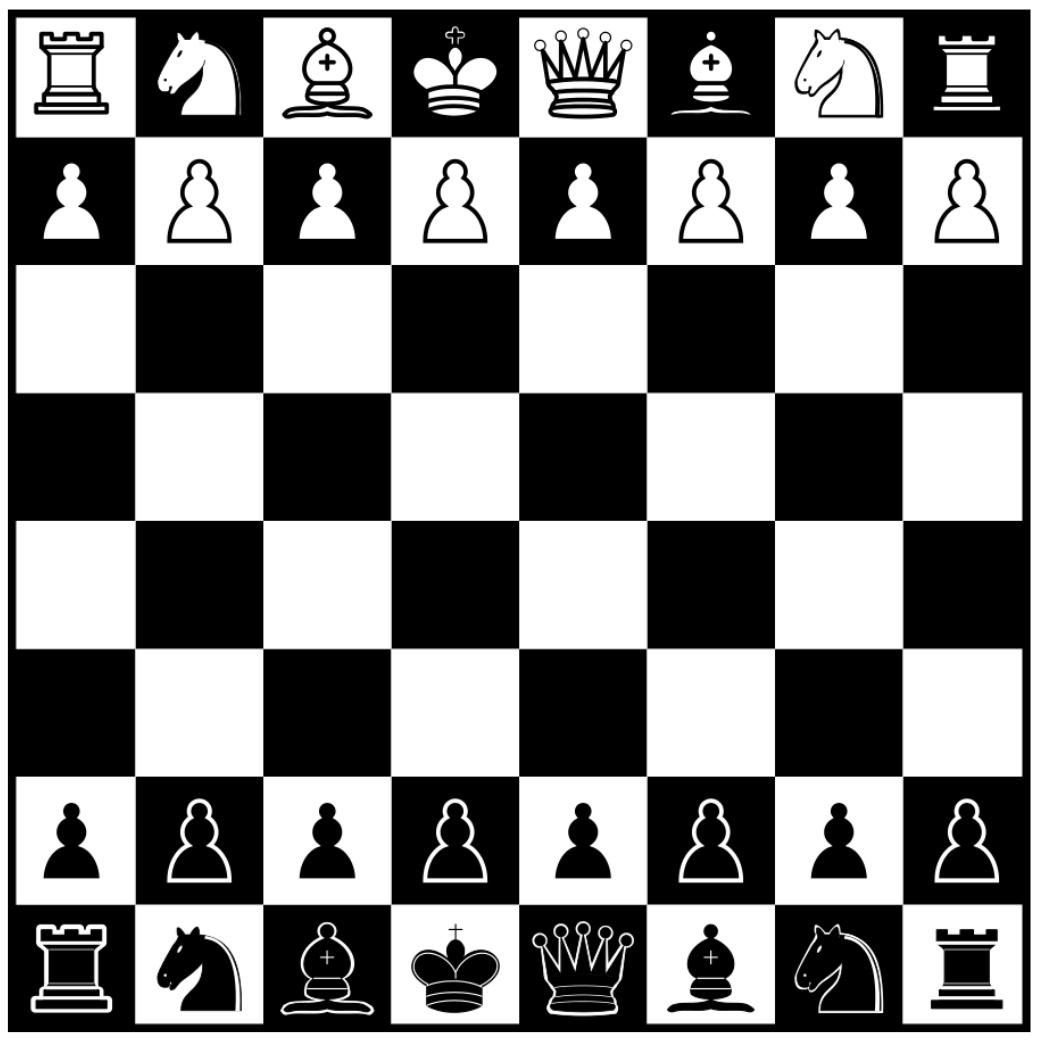}
        \caption{}
        \label{fig:intro:chess_perfect}
    \end{subfigure}
    ~ 
    \begin{subfigure}[b]{0.3\textwidth}
        \includegraphics[width=\textwidth]{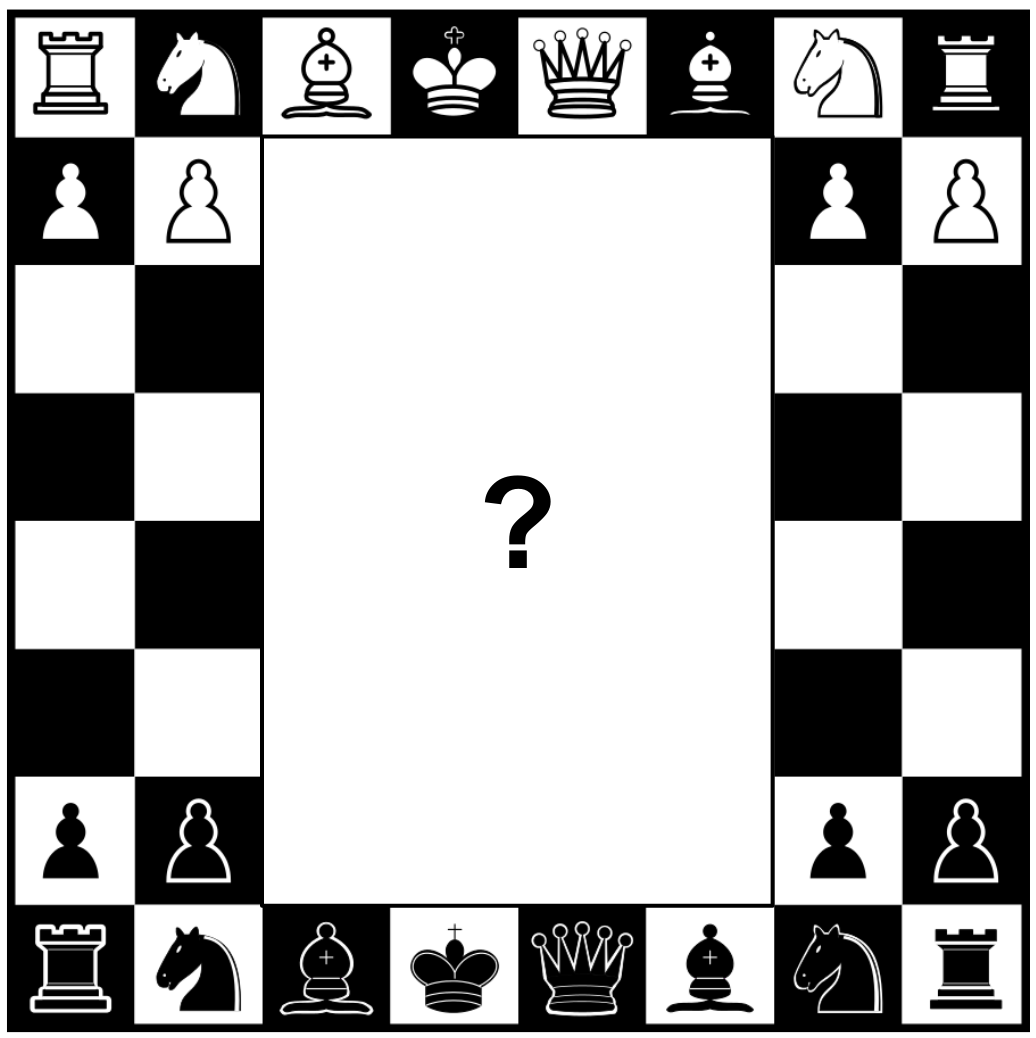}
        \caption{}
        \label{fig:intro:chess_imperfect}
    \end{subfigure}
    ~ 
    \begin{subfigure}[b]{0.3\textwidth}
        \includegraphics[width=\textwidth]{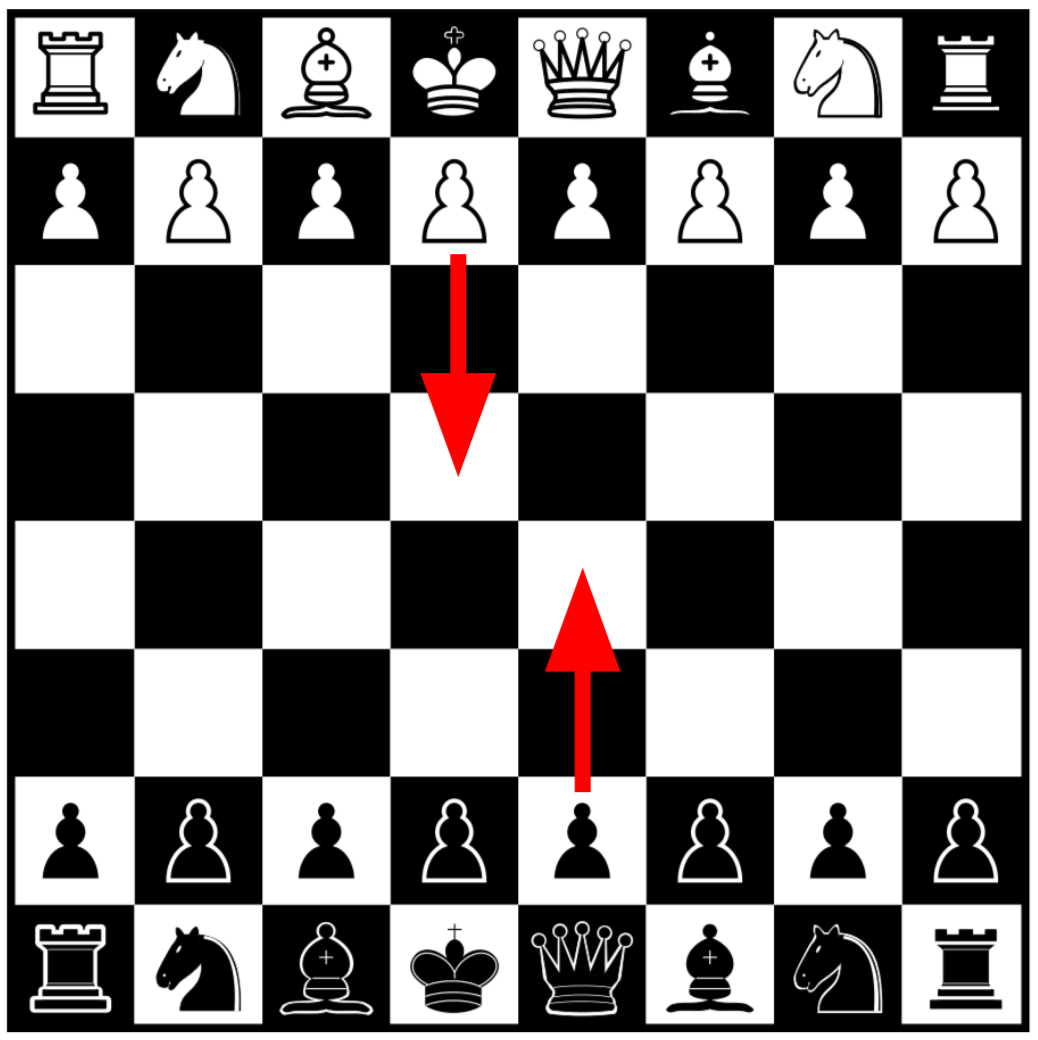}
        \caption{}
        \label{fig:intro:chess_sim_moves}
    \end{subfigure}
    \caption[Caption]{
    a) Perfect information chess.
    b) Imperfect information chess.
    c) Simultaneous moves chess.
    Chessboard image under the creative commons license\footnotemark}
    \label{fig:intro:chess_perfect_imperfect}
\end{figure}

\footnotetext{
    \href{https://commons.wikimedia.org/wiki/File:AAA\_SVG\_Chessboard\_and\_chess\_pieces\_06.svg}{AAA SVG Chessboard and chess pieces06}
    from  \href{https://commons.wikimedia.org/wiki/User:ILA-boy}{ILA-boy},
    licensed under \href{https://creativecommons.org/licenses/by-sa/3.0/deed.en}{CC BY-SA 3.0}
    }

\second{
This thesis deals with search in imperfect information games.
The defining property is that unlike in perfect information games (e.g. chess, go), agents do not have a full (or symmetric) information about the state of the game.
But why is this imperfect information property of any importance?
Consider the following three chess variants (where the first one is perfect information game, while the other two are imperfect information games).
}

\begin{enumerate}
    \itemsep0em 
    \item  Standard, perfect information chess (Figure \ref{fig:intro:chess_perfect}).
    \item Imperfect information variant, where the players do not see the opponent's pieces in the the shaded squares (Figure \ref{fig:intro:chess_imperfect}).
    \item Simultaneous moves chess, where the players make move at the same time (Figure \ref{fig:intro:chess_sim_moves}).
\end{enumerate}

\subsection{Optimal Policies}
\second{
All these chess variants make no difference when it comes to the existence or guarantees of optimal policies.
We will see that just like in perfect information chess, optimal policies still exist and have all the same desirable properties.
One distinction of the optimal policies in imperfect information games is that it often has to be stochastic.
While there always is a deterministic optimal policy in perfect information games, this is not the case for the other variants.
It turns out that the number of actions we need to randomize between is closely related to the level of uncertainty in the game \citep{schmid2014bounding}.

From the algorithmic perspective, the crucial distinction is that finding optimal policies is more challenging.
Many offline algorithms that provably work in perfect information, fail to find optimal policies in imperfect information games, most notably many popular\footnote{(Deep) Q-learning, SARSA, REINFORCE, PPO etc.} reinforcement learning algorithms.
}
\subsection{Search}
\second{
As offline algorithms for imperfect information games are more challenging, so are the online algorithms that do search.
Search consists of three fundamental components:
i) local reasoning ii) sub-problem decomposition iii) value functions.
As we will see, all of these components are substantially more complex in imperfect information settings.
While some popular perfect information based methods (e.g. Monte Carlo tree search) can be used in imperfect information \citep{whitehouse2014monte}, these methods are not sound and fail to produce optimal policies even in very small games.
For a long time, sound search has been thought to be impossible in imperfect information settings \citep{frank1998finding, lanctot2014search}.

This thesis tries to coherently introduce the reader to a history of computational game theory research
leading to this important milestone: sound search in imperfect information.
}

\section{Reinforcement Learning or Game Theory}
\label{intro:rl_vs_gt}
\second{
There are two main disciplines and corresponding communities that study sequential decision-making in multi-agent environments using games as benchmarks for progress: \rl{} and \gt{}.
In \rl{}, search and learning of value functions quickly became a common ingredient of many algorithms \citep{sutton1998introduction}.

On the other hand, \gt{} typically focused on solution concepts, complexity results and algorithms with strong theoretical guarantees - producing provably optimal policies for {\iig{}}s \citep{shoham2008multiagent}.

Are the techniques presented reinforcement learning or game theory techniques?
It is the opinion of the authors that they are both.
The techniques combine value functions, learning, search and strong theoretical convergence guarantees in both perfect and imperfect information games.
Furthermore, the line between both communities is blurring and many challenging environments  will require insights of both communities. 

}

\subsection{Dictionary}
\second{
While both communities deal with optimal behavior of agent(s) in an environment and thus use the same concepts, the language even for the basic concepts often differs.
This Babylonian confusion of tongues has mostly historical reasons.
But if we want to finish the tower of general algorithms, we need the work to be accessible to both communities.
Table \ref{tab:intro-biscuits} lists some of the basic concepts and we purposely use words from both \rl{} and \gt{} throughout this book.
}

\begin{table*}[ht]
\centering
\renewcommand{\arraystretch}{1.3}
\begin{tabular}{@{}ll@{}}
\toprule
Reinforcement Learning  & Game Theory            \\
\midrule

  Agent          &  Player                       \\ 
  Environment    &  Game                         \\ 
  Reward         &  Utility                      \\ 
  State          &  InfoSet / InfoState / State  \\
  Policy         &  Strategy  \\
\bottomrule
\end{tabular}
\caption{Reinforcement learning and game theory use different words for analogous concepts. }
\label{tab:intro-biscuits}
\end{table*}

\section{Structure}
\second{
The thesis is structured into two parts. First --- perfect information games --- looks at
search in perfect information games from a slightly unusual perspective.
Looking at search from this perspective creates basic building blocks,
concepts and ideas to be expanded to imperfect information games.

The second part --- imperfect information games --- is the core of the contributions of this thesis.
Without the first part, the ideas and concepts would be accessible only to a very small community.
}
\part{Perfect Information}

\chapter{Introduction}
\second{
The first part of this thesis deals with perfect information games.
The point of this section is not to introduce any new results or insights into search in perfect information games.
Rather, it is to look closely at the individual search components --- locality, decomposition and value functions.
We will look at search from a particular perspective.
And while this perspective is arguably not very useful for perfect information games, it will allow us to introduce search in imperfect information in a much more comprehensible way.
It will also allow us to appreciate the different complications in the fundamental search components, as we will be extending these concepts in the second section of this thesis.
}

\section{Chapters Structure}

\paragraph{Formal Models} 
\second{Chapter \ref{chap:pig-formalisms} introduces a simple game-tree formalism that is sufficient for the games in question.}

\paragraph{Offline Policies}
\second{
Chapter \ref{chap:pig-optimal_policies} motivates optimal policies.
Introduces concepts of best response, \maximin{} and Nash equilibria and discusses different ways a policy can be evaluated.
It also shows the equivalency of \maximin{} and Nash for two player zero-sum games.
We finish with some dis-concerning notes on multiplayer settings.
}

\paragraph{\Subgames{}}
\second{
Chapter \ref{chap:pig-subgames} presents a key decomposition concept of \subgames{}.
}

\paragraph{Offline Solve}
\second{
Chapter \ref{chap:pig-offline_solve} discusses offline algorithms for producing optimal policies.
We also introduce self-play style methods, including the popular approach o self-play via independent reinforcement learning.
}

\paragraph{Online Settings}
\second{
Chapter \ref{chap:pig-online_settings} shows that algorithms operating in online settings need to be thought of and evaluated differently to offline algorithms.
We show some counter-intuitive examples of an online algorithm's worst case performance.
We then introduce an appropriate formalism and metric for the online setting, including a framework connecting the performance of online and offline algorithms.
}

\paragraph{Search}
\second{
Chapter \ref{chapter-search} discusses basic building blocks of search and how they come together.
It introduces online minmax algorithm and includes well-known limited depth minmax search algorithm.
}

\paragraph{Examples}
\second{
Chapter \ref{sec:pig-examples} lists some well known examples of successful utilization of search in perfect information games.
}




\chapter{Formalisms}
\label{chap:pig-formalisms}

\second{
For perfect information games, we will settle with a single formalism for sequential decision making --- theperfect information game tree.
}

\section{Perfect Information Game Tree}
\second{
A model that fits many popular perfect information games (e.g. chess, go, checkers) is a simple \gls{game_tree} (Figure \ref{fig:simple_tree}).
In this tree, the nodes correspond to game states (e.g. chess position) while the edges correspond to the legal actions (e.g. moving a queen).
}
\begin{defn}
\label{defn:game_tree}
A perfect information game tree consists of:

\begin{itemize}
    \item Set of players $\mathcal{N} = \{0, 1, 2\}$, where the player $0$ 
    is the chance player.
    \item Set of states (tree nodes) $\mathcal{S}$.
    \item Set of terminal states $\mathcal{Z} \subseteq \mathcal{S}$.
    \item Player to act in a state $s \in (\mathcal{S} \setminus \mathcal{Z})$ denoted as
    $p(s) \in \mathcal{N}$. As a single player acts in a state, we also define $\mathcal{S}_i = \{s \in S \, | \, p(s) = i\}$.
    \item The set of available actions (tree edges) for a state $s$, denoted as $\mathcal{A}(s)$.
    \item Mapping of state-action pair to a next state (child) $c(s, a): \mathcal{S} \times \mathcal{A} \rightarrow \mathcal{S}$.
    \item Fixed strategy of a chance player  $\pi_c : s \in \mathcal{S}_0 \rightarrow \Delta(\mathcal{A}(s))$.
    \item Return for a player $i$ a terminal state $R_i : \mathcal{Z} \rightarrow \mathcal{R}$.
\end{itemize}

\end{defn}

\begin{figure}[h!]
  \centering
  \includegraphics[width=0.6\textwidth]{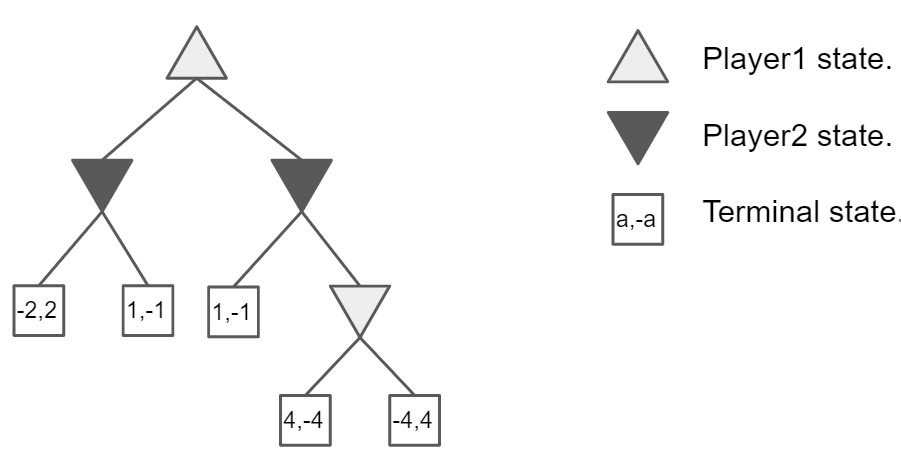}
  \caption{Simple game tree.}
  \label{fig:simple_tree}
\end{figure}

\subsection{Derived Concepts}
\second{
We denote the child state after taking the action $a \in \mathcal{A}(s)$ in a parent $s$ as $sa$, and analogically for longer sequences $sa_1a_2 \hdots a_n $.
We also use $a_i \sqsubseteq a_j$ to denote that state $a_i$ is a predecessor of $a_j$.
Furthermore, $\Delta_{R}$ is the delta of maximim/minimum utilities in the game: $\Delta = \max_{z \in \mathcal{Z}}R_i(z) - \min_{z \in \mathcal{Z}}R_i(z)$
}

\subsection{Policy}
\second{
A behavioral policy is a mapping from a state to a distribution over the available actions $s \rightarrow \Delta(\mathcal{A}(s))$.
We denote a policy of a player $i$ as $\pi_i$ and the policy in a state $s \in \mathcal{S}$ as $\pi_i(s)$.
The set of all policies of player $i$ is then $\Pi_i$.
As a single player acts in a state, we often use simply  $\pi(s)$.
Furthermore, the probability mass for an action $a \in \mathcal{A}(s)$ is $\pi_i(s,a)$.
A policy profile consists of strategies of both player $\pi = (\pi_1, \pi_2)$.
Given a policy profile, we also define a reach probability of a state $s \in \mathcal{S}$ as $P^{\pi}(s) =  \prod_{s'a \sqsubseteq s} \pi(s', a)$.
Finally, useful concept is the support (Definition~\ref{defn:pig-support})
}

\begin{defn}[Support]
\label{defn:pig-support}
Support of a policy $\pi_i$ is the set of actions with non-zero probability under that policy $supp^{\pi_i}(s) = \{ a \in \mathcal{A} \, | \, \pi(s, a) > 0 \}$.
\end{defn}

\subsection{Expected Reward}
\second{
Unlike in single agent settings, we need to know both players' policies to compute the reward.
And since both the policies and the environment (chance player) can be stochastic, we care about the expected reward.
As the reward is defined over the terminal states $\mathcal{Z}$, the expected reward of player $i$ given the strategy profile is simply (\ref{eq:pig-expected_reward}).

\begin{align}
\label{eq:pig-expected_reward}
    R_i(\pi) = R_i(\pi_1, \pi_2) = \sum_{z \in \mathcal{Z}} P^{\pi}(z) R_i(z)
\end{align}
}

\subsection{Averaging Policies}
\label{sec:pig-strategy_averaging}
\second{
One must exercise caution during any averaging in the space of behavioral strategies.
A common mistake is to simply average the strategy locally, in each individual state.
This is incorrect as for $\pi^c_1 = 0.5 \pi^a_1 + 0.5 \pi^b_1$, we expect $R_1(\pi^c_1, \pi_2) = 0.5 R_1(\pi^a_1, \pi_2) + 0.5 R_1(\pi^b_1, \pi_2)$.
Figure \ref{fig:pig_policy_averaging} shows a counterexample to this expectation for naive per-state averaging of the policy.

The issues is because the dynamics are inherently sequential --- the distribution over the states $P^\pi(z)$ depends on the full sequence.
Proper averaging in behavioral state must thus take into account the reach probability of the state (Figure \ref{fig:strat_averaging_d}).
Section \ref{sec:iig:sequence_form} will introduce sequence form representation that allows for convenient linear operations directly on this strategy representation.

\begin{figure}
    \centering
    \begin{subfigure}[b]{0.22\textwidth}
        \includegraphics[width=\textwidth]{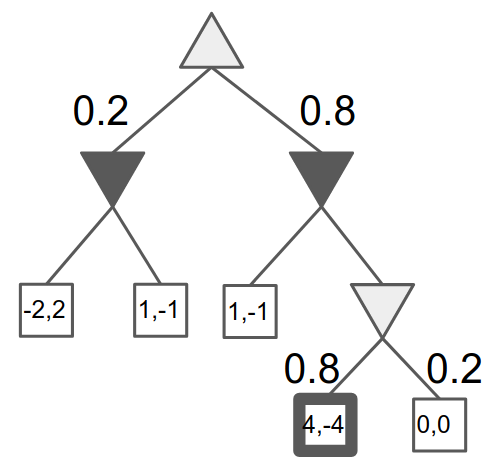}
        \caption{$\pi^a_1$}
        \label{fig:strat_averaging_a}
    \end{subfigure}
    \begin{subfigure}[b]{0.22\textwidth}
        \includegraphics[width=\textwidth]{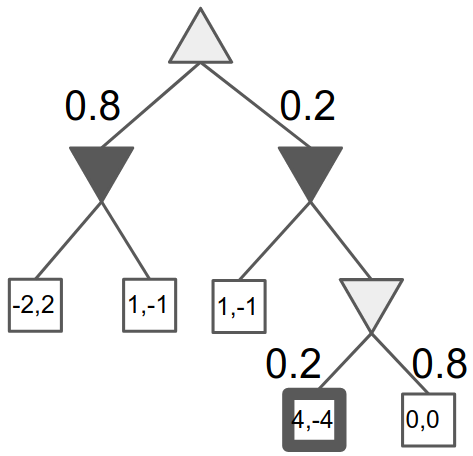}
        \caption{$\pi^b_1$}
        \label{fig:strat_averaging_b}
    \end{subfigure}
    \begin{subfigure}[b]{0.25\textwidth}
        \includegraphics[width=\textwidth]{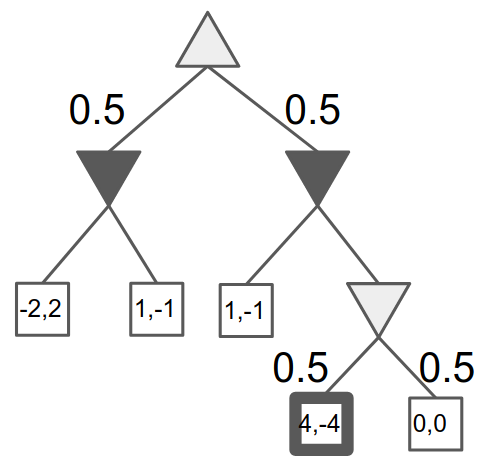}
        \caption{$\pi^c_1$ - naive average}
        \label{fig:strat_averaging_c}
    \end{subfigure}
     \begin{subfigure}[b]{0.26\textwidth}
        \includegraphics[width=\textwidth]{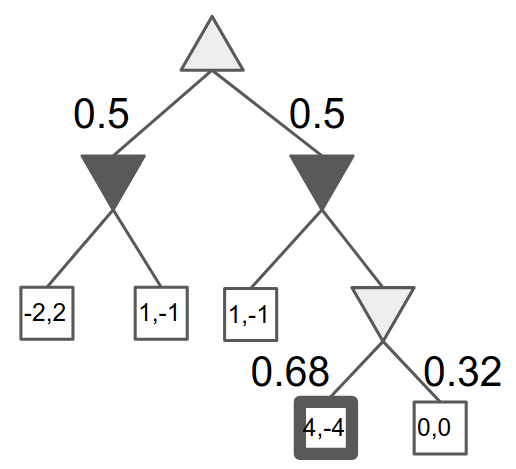}
        \caption{$\pi^c_1$ - proper average}
        \label{fig:strat_averaging_d}
    \end{subfigure}
    \caption{Consider the highlighted state $s$ against an opponent who plays to reach that state, and its reach probability $P^{\pi}(s)$.
    a) $P^{\pi^a_1}(s) = 0.64$
    b) $P^{\pi^b_1}(s) = 0.04$
    c) Naive per-state averaging of the strategy simply averages the state strategy in isolation: $P^{\pi^c_1}(s) = 0.25 \neq  0.5 \cdot 0.64 + 0.5 \cdot 0.04 $
    d) Proper averaging of the strategy takes reach account into consideration:
    }
\label{fig:pig_policy_averaging}
\end{figure}
}

\subsection{Constant-Sum and Zero-Sum Games}
\second{
For a \gls{constant_sum_game}, the reward is strictly competitive --- gain of one agent is the loss of the other.
The rewards of both players in each terminal sums to a constant (Definition \ref{def:pig-constant_sum}).
This is a property of all win/lose/draw games (e.g. chess) as well as most games where money exchange hands (e.g. poker).

\begin{defn}[Constant Sum Game]
\label{def:pig-constant_sum}
\begin{align*}
    \exists c \in \mathcal{R} \, : \, \forall z \in \mathcal{Z} \, : \, \sum_{i \in \mathcal{N}} R_i(z) = c
\end{align*}
\end{defn}

A \gls{zero_sum_game} is then simply a game for which $c=0$.
There is practically no difference between constant-sum and zero-sum games, as the reward can simply be shifted to zero with no strategic or algorithmic implications.
It is thus common to use the term zero-sum game even if the presented results hold true for any constant.
}
\chapter{Optimal Policies}
\label{chap:pig-optimal_policies}
\second{
Formal definitions for the games and policies in hand, we now come back to the question of optimal policies.
This time, we will be able to formally define these concepts and important properties.
Namely, we start with a simple concept of best response that maximizes utility given a fixed opponent.
Next we introduce a \maximin{} --- policies that optimize against the worst case scenario.
The final concept is Nash equilibrium --- a stationary strategy profile where no player has incentive to deviate.
We then discuss some connections between these concepts.
Importantly, we show that \maximin{} policies and Nash equilibrium collapse for two player zero sum games.
We finish with some notes and observations regarding multiplayer settings, where the concept of \maximin{} and Nash equilibrium differ.
}

\section{Best Response}
\label{sec:pig-best_response}
\second{
Best response is a concept that will be used throughout this book.
It is simply a policy that maximizes return against a fixed policy.

\begin{restatable}[Best Response]{defn}{restdefbestresponse}
\label{defn:best_response}
     Best response against a policy $\pi_i$ is:
     \begin{equation*}
         \argmax_{\pi_{-i} \in \Pi_{-i}} R_{-i}(\pi_i, \pi_{-i})
     \end{equation*}
\end{restatable}

We use $\brset{}(\pi_i)$ to denote the set of best response policies against the policy $\pi_i$.
Note that for zero-sum games, opponent maximizing their reward is equivalent to opponent minimizing our reward.

\begin{equation*}
    \argmax_{\pi_{-i}} R_{-i}(\pi_i, \pi_{-i}) = \argmin_{\pi_{-i}} R_i(\pi_i, \pi_{-i})
\end{equation*}

As this means the player's value against any best-response strategy is unique, we denote this unique value as $BRV_i(\pi_i)$.
\begin{align*}
    BRV_i(\pi_i) = \min_{\pi_{-i}}  R_i(\pi_i, \pi_{-i}) = -\max_{\pi_{-i}} R_{-i}(\pi_i, \pi_{-i})
\end{align*}

}

\subsection{Properties}
\second{
How hard is it to compute a best response?
The first look at the definition is not necessarily helpful from the computation perspective --- we are taking an $\argmin$ over all possible policies.
Fortunately, computing a best response strategy can be done via a single bottom-up pass traversal of the game tree (Algorithm \ref{alg:pig_best_response}).
As the algorithm produces deterministic policies, we immediately get the Lemma~\ref{lem:pig_deterministic_br}
}

\begin{lemma}
\label{lem:pig_deterministic_br}
There is always a deterministic best response.
\end{lemma}

\begin{algorithm}
\caption{Best Response}
\label{alg:pig_best_response}
\begin{algorithmic}[1]

\Function{BestResponseDFS}{$s \in S, i \in N$ }
 \State \Comment{Terminal state, return terminal utility.}
  \If{$s \in \mathcal{Z}$}
    \State \Return ${R_i(s)}$
  \EndIf
  \State \Comment{Compute the best response values of the children states.}
  \For{$a : \mathcal{A}(s)$}
    \State $v[sa]$ = BestResponseDFS($sa$, $i$)
  \EndFor

  \State
  \If{p(s) = i} \Comment{Best response decision, taking max.}
    \State $\pi(s) = \argmax_{a \in \mathcal{A}(a)} v[sa]$
    \State \Return $\max_{a \in \mathcal{A}(a)} v[sa]$
  \Else \Comment{Other player's decision, use their strategy}
    \State \Return $\sum_{a \in \mathcal{A}(a)} \pi(s, a) v[sa]$
\EndIf

\EndFunction
\State
\Function{BestResponse}{}
\Comment{Start the DFS computation in the initial state}
\State BestResponseDFS($s_0$, $i$)
\EndFunction
\end{algorithmic}
\end{algorithm}

\section{Maximin}
\second{
The worst-case reasoning (Section \ref{sec:worst_case_reasoning}) motivates a policy that maximizes the reward against the worst case opponent.
We are now able to formally introduce the \gls{maximin_policy} (Definition \ref{defn:maximin_policy}).

\begin{restatable}[Maximin Policy]{defn}{restdefminmaxpolicy}
\label{defn:maximin_policy}
    Maximin policy of a player $i$ is:
    \begin{equation}
        \argmax_{\pi_i \in \Pi_i} \min_{\pi_{-i} \in \Pi_{-i}} R_i(\pi_i, \pi_{-i}) = \argmax_{\pi_i \in \Pi_i} BRV_i(\pi_i)
    \label{eq:pig-maxmin-policy}
    \end{equation}
\end{restatable}

We denote the set of all such \maximin{} policies for a player $i$ as $\maximinset{}_i$.
\begin{align}
    \maximinset{}_i = \{\pi_i \, | \, BRV_i(\pi_i) = \max_{\pi'_i} BRV_i(\pi'_i) \}
\label{eq:pig-minmax_maxmin_set}
\end{align}

Furthermore, the corresponding value of the maximin policy is the \gls{maximin_value}, denoted as $\underline{v_i}$ (Definition \ref{defn:maximin_value})

\begin{restatable}[Maximin Value]{defn}{restdefminmaxvalue}
\label{defn:maximin_value}
    \Maximin{} value is the value of the \maximin{} policy:
    \begin{equation}
        \underline{v_i} = \max_{\pi_i \in \Pi_i} \min_{\pi_{-i} \in \Pi_{-i}} R_i(\pi_i, \pi_{-i}) = \max_{\pi_i \in \Pi_i} BRV_i(\pi_i)
    \label{eq:pig-maxmin-value}
    \end{equation}
\end{restatable}




\subsection{\Minimax{} Theorem}
We will now investigate the important relation between the players' respective \maximin{} values $\underline{v_i}$ and $\underline{v_{-i}}$.

\begin{align}
    \underline{v_i} &= \max_{\pi_i} \min_{\pi_{-i}} R_i(\pi_i, \pi_{-i}) \label{eq:pig-maxmin_minmax_1} \\
    \underline{v_{-i}} &= \max_{\pi_{-i}} \min_{\pi_{i}} R_{-i}(\pi_i, \pi_i) \label{eq:pig-maxmin_minmax_2}
\end{align}

First, we will re-phrase equation (\ref{eq:pig-maxmin_minmax_2}) in terms of $R_i$:

\begin{align}
    \underline{v_{-i}} &= \max_{\pi_{-i}} \min_{\pi_{i}} R_{-i}(\pi_i, \pi_i) \label{eq:pig-maxmin_minmax_3} \\
                       &= \max_{\pi_{-i}} \min_{\pi_{i}} -R_{i}(\pi_i, \pi_i) \label{eq:pig-maxmin_minmax_4} \\
                       &= \max_{\pi_{-i}} - \max_{\pi_{i}} R_{i}(\pi_i, \pi_i) \label{eq:pig-maxmin_minmax_5} \\
                       &= - \min_{\pi_{-i}} \max_{\pi_{i}} R_{i}(\pi_i, \pi_i) \label{eq:pig-maxmin_minmax_6}
\end{align}

Theorem \ref{thm:pig-minmax} then states a critical result --- the maximin values are in balance $\underline{v_i} = -\underline{v_{-i}}$.
We refer to this unique value as the \gls{game_value} and denote it as $GV_i$.

\begin{restatable}[Minimax Theorem]{thm}{restminmaxtheorem}
\begin{align}
    \max_{\pi_i} \min_{\pi_{-i}} R_i(\pi_i, \pi_{-i}) = \min_{\pi_{-i}} \max_{\pi_i} R_i(\pi_{i}, \pi_{i})
\label{eq:pig-minmax_maxming}
\end{align}
\label{thm:pig-minmax}
\end{restatable}

The minimax theorem, proven by John Von Neumann in $1928$ \citep{neumann1928theorie} has dramatic consequences for two player zero sum games (the proof is for both perfect and imperfect information games).
Von Neumann himself later wrote ``As far as I can see, there could be no theory of games on these bases without that theorem'' \citep{von1953communication}.
We will prove the minimax theorem for the more general case of imperfect information in Section \ref{sec:iig-minmax_theorem}.
}

\section{Nash equilibrium}
\second{
Another key solution concept --- \gls{nash_equilibrium} --- is based on a stationary property.
The idea is that if for a strategy profile $(\pi_i, \pi_{-i})$, none of the players benefit by deviating from their policy, the profile forms an equilibrium (Definition~\ref{def:iig-nash}).

\begin{restatable}[Nash Equilibrium]{defn}{restdefnash}
\label{def:iig-nash}
Strategy profile $(\pi_i, \pi_{-i})$ forms a Nash equilibrium if none of the players benefit by deviating from their policy.
\begin{equation*}
   \forall i \in N, \forall \pi'_i \, : \, R_i(\pi_i, \pi_{-i}) \geq  R_i(\pi'_i, \pi_{-i})
\end{equation*}
\end{restatable}

We denote the set of all such strategy profiles as \gls{neq} --- strategy profile $(\pi_1, \pi_2)$ forms a Nash equilibrium iff $(\pi_1, \pi_2) \in \nashset$.
Note that the $\nashset$ can also be formulated using the best response (Lemma~\ref{lem:pig-nash_best_response})

\begin{lemma}
\label{lem:pig-nash_best_response}
Strategy profile $(\pi_i, \pi_{-i})$ forms a Nash equilibrium iff
\begin{equation*}
    \forall i \in N \, : \, \pi_i \in \brset(\pi_{-i})
\end{equation*}
\end{lemma}
\begin{proof}
Follows directly from the definition of $\nashset$ and $\brset$.
\end{proof}

}

\section{Nash Equilibrium vs \Maximin{}}
\label{section-Nash_minmax}
\second{
It might be unclear why we introduced both \maximin{} and Nash equilibria as solution concepts.
While the properties discussed do sound intriguing, which solution concept should we prefer: the stationary property of Nash Equilibria, or the worst case reasoning of \maximin{}?
While the \maximin{} strategy is a single strategy that optimizes against the worst case, Nash equilibria was introduced as a strategy profile (that is, a pair of strategies).

It turns out the these solution concepts in two player zero sum games are the same!

\begin{thm}[Nash is \maximin{}]
For two player zero-sum  games:
\begin{equation*}
    \label{thm-Nash_implies_minmax}
    (\pi_1, \pi_2) \in \nashset \implies \pi_1 \in \maximinset_1 \land \pi_2 \in \maximinset_2
\end{equation*}
\label{thm:pig-nash_implies_minmax}
\end{thm}

\begin{proof}
By contradiction. 
Denote the worst case value of $\pi_i$ as $a=BRV_i(\pi_i)$. 
Note that since the $\pi_{-i}$ is a best response (Lemma \ref{lem:pig-nash_best_response}) and  $R_i = -R_{-i}$ (zero-sum game),
it must be the case that $a = R_i(\pi_i, \pi_{-i})$.
Let there be another policy $\pi^*_i$ that does strictly better in worst case, that is, 
 $b=BRV_i(\pi^*_i)$ and $b > a$. 
 Then the player $i$ could improve by switching to policy
 $\pi^*_i$ : $R_1(\pi^*_i, \pi_{-i}) > R_1(\pi_i, \pi_{-i})$, contradicting the definition of Nash equilibrium.

\end{proof}

\begin{thm}[\Minimax{} is Nash]
For two player, zero-sum  games:
\begin{equation*}
   \pi^*_1 \in \maximinset_1 \land  \pi^*_2 \in \maximinset_2 \implies (\pi^*_1, \pi^*_2) \in \nashset
\end{equation*}
\label{thm:pig-minmax_implies_nash}
\end{thm}

\begin{proof}
Let $\underline{v_i}, \underline{v_{-i}}$ be the $\max\min$ values of the respective players.
We will use the fact that $\underline{v_i} = -\underline{v_{-i}}$ (Theorem \ref{thm:pig-minmax}) to show a contradiction.
Assume that the strategy profile is not Nash and there exists $\pi_1'$ such that
$u_1(\pi_1', \pi^*_2) > u_1(\pi^*_1, \pi^*_2)$:
\begin{align*}
    u_1(\pi'_1, \pi^*_2) & > u_1(\pi^*_1, \pi^*_2) \geq \underline{v_i} \\
    -u_2(\pi'_1, \pi^*_2) = u_1(\pi'_1, \pi^*_2) & > u_1(\pi^*_1, \pi^*_2) \geq \underline{v_i} \\
    \underline{v_i} = -\underline{v_{-i}} \geq -u_2(\pi'_1, \pi^*_2) = u_1(\pi'_1, \pi^*_2) & > u_1(\pi^*_1, \pi^*_2) \geq \underline{v_i} \\
\end{align*}

Which concludes the contradiction.
\end{proof}

Putting Theorem~\ref{thm:pig-nash_implies_minmax} and Theorem~\ref{thm:pig-minmax_implies_nash} together results in Corollary~\ref{thm:pig-minmax_is_nash}.

\begin{corollary}[\Maximin{} and Nash Collapse]
For two player zero sum games, solution concepts of maxmin and Nash equilibria are the same.
\begin{align}
    \pi_1 \in \mathbb{MAXMIN}_1 \land  \pi_2 \in \mathbb{MAXMIN}_2 \iff (\pi_1, \pi_2) \in \NEQ
\end{align}
\label{thm:pig-minmax_is_nash}
\end{corollary}

}

\section{Properties of Optimal Policies}
\second{
We now have formal understanding of the \maximin{} and Nash equilibria solutions concepts.
We also know that in our settings, these concepts are the same.
But what motivation do we have to follow such policies?
}

\subsection{Guaranteed Value}
\second{
Following an optimal policy guarantees the \maximin{} game value against any opponent.
Furthermore, Theorem \ref{thm:pig-minmax} tells us that the \maximin{} value of a player is in balance with the \maximin{} value of the opponent $\max_{\pi_i}\min_{\pi_{-i}} R_i(\pi_i, \pi_{-i}) = -\max_{\pi_i}\min_{\pi_{-i}} R_{-i}(\pi_i, \pi_{-i})$.
We thus directly get the important Corollary \ref{cor:pig-both_sides_guaranteed_value}.
This appealing property of optimal policies is of a great importance, and easily motivates such policies.

\begin{corollary}
\label{cor:pig-both_sides_guaranteed_value}
    If we follow an optimal policy when playing as either player, the expected utility against any opponent is greater or equal to zero:
    \begin{align*}
        (\pi_i, \pi_{-i}) \in \nashset{} \, : \, R_i(\pi_i, \pi'_{-i}) + R_{-i}(\pi'_i, \pi_{-i}) \geq 0 \,  \forall \pi'_i, \pi'_{-i}
    \end{align*}
\end{corollary}

Corollary \ref{cor:pig-both_sides_guaranteed_value} essentially tells us that if we follow an optimal policy, we can't lose.
If we face an optimal opponent, we draw.
If we face a sub-optimal opponent, we might win (i.e. expect a positive reward), although we are certainly not guaranteed to.
This is because even if the opponent follows a highly exploitable strategy, the optimal policy we follow might not take advantage of any of the mistakes.
Even though this can easily happen in perfect information games, we defer examples and more discussion of this behavior to imperfect information games (Section \ref{sec:iig-suboptimal_vs_opimal}).
}

\subsection{Convexity}
\second{
The optimal policy guarantees us the \gls{game_value} regardless of the opponent we face.
And while there might be multiple optimal strategies, it does not matter which one we follow --- we still enjoy the same guarantees.
But how many optimal policies are there?
What structure does the set of all optimal policies have?
It is easy to verify that any convex combination\footnote{One must be careful not to naively combine the behavioral strategies, proper combination as described in Section \ref{sec:pig-strategy_averaging} is needed} of optimal policies results in an optimal policy.

\begin{lemma}
\label{thm:nash_convex_set}
The set of all Nash equilibria $\NEQ$ is convex.
\end{lemma}
\begin{proof}
We defer this result to the more general case of imperfect information games, proven in Section~\ref{sec:iig-matrix_games_lp}, Theorem~\ref{thm:iig-matrix_lp}.
\end{proof}

\begin{corollary}
\label{crl:nash_one_or_infinite}
There is either one unique Nash equilibrium, or infinitely many.
\end{corollary}
\begin{proof}
Follows directly from Lemma~\ref{thm:nash_convex_set}.
\end{proof}
}

\section{Multiplayer Games}
\label{sec:pig-multiplayer_games}
\second{
While the properties of Nash equilibria do sound very intriguing, they are lost once we leave the comfort of two player zero-sum games.
First, observe that introducing more players is essentially the same as removing the zero-sum property.
For non-zero sum games, one can always add a dummy player with no actions with the surplus reward.
}

\subsection{Nash is not \Maximin{}}
\second{
Unlike in two player zero sum games, the solution concepts of Nash and \maximin{} are different.
See Figure \ref{fig:nash_not_minmax} for simple counterexamples.
}

\begin{figure}
    \centering
    \begin{subfigure}[b]{0.3\textwidth}
        \includegraphics[width=\textwidth]{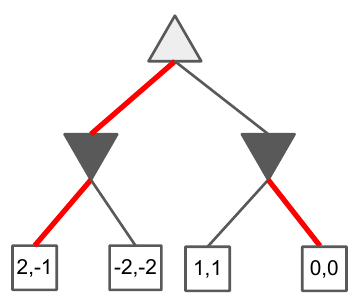}
        \caption{}
        \label{fig:nash_not_minmax_a}
    \end{subfigure}
    \hspace{0.1\textwidth} 
    \begin{subfigure}[b]{0.3\textwidth}
        \includegraphics[width=\textwidth]{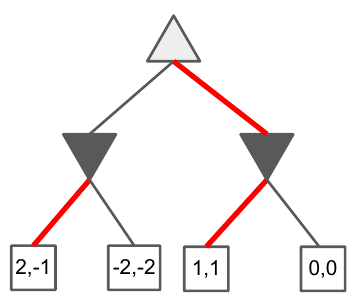}
        \caption{}
        \label{fig:nash_not_minmax_b}
    \end{subfigure}
    \caption{In multiplayer games, Nash equilibria and \maximin{} solution concepts differ. Action selected by a policy are denoted by the bold red line. a) Nash strategy profile, where the individual strategies are not \maximin{}. b) Pair of \maximin{} strategies that do not for Nash strategy profile.}
    \label{fig:nash_not_minmax}
\end{figure}

\subsection{Existence}
\second{
Under what circumstances do optimal policies exist?
It is easy to believe that \maximin{} strategy always exists as it is simply maximizing against a best response value.
But is it also the case for Nash?
It turns our that the existence is guaranteed for a very wide class of games --- even for games with multiple players.
Essentially, one can show that Nash equilibrium is guaranteed to exist for any finite game \citep{nash1950equilibrium, nash1951non}.
}

\subsection{Unique Value}
\label{sec:pig-unique_value}
\second{
There is no longer a unique value of the game, i.e. the minimax theorem no longer holds.
While in two player settings, any Nash equilibrium guarantees the same game value and thus the players could arbitrarily select any Nash to follow, this does not hold in multiplayer settings.
Even worse, it could very well be that $(\pi^a_1, \pi^a_2) \in \nashset,  (\pi^b_1, \pi^b_2) \in \nashset$ and both strategy profiles produce high reward for the players, but if they do not coordinate on which equilibrium to follow and they select different ones --- they end up with $(\pi^a_1, \pi^b_2)$ where the value can be arbitrarily bad for both.
This issue is also referred to as the  \gls{equilibrium_selection_problem}, and we defer an example of this behavior to Section \ref{sec:iig-multiplayer_games}.
}

\subsection{Complexity}
\second{
One can easily argue that \maximin{} is as hard in multiplayer settings as it is for two players.
We can just let a single opponent control all the other players.

For Nash equilibria, the question of complexity in multiplayer settings is more complicated.
We can not use the same argument as for \maximin{}, since Nash equilibria requires that none of the players benefits by deviating from the strategy profile --- that is, each player individually.
But the players could benefit from deviating if multiple agents deviate at once.
In general, the problem is known to be PPAD-complete \citep{papadimitriou1994complexity, chen2006settling, daskalakis2009complexity}.
}

\subsection{What to Do}
\second{
But maybe it is just the case that there is another, simple and intriguing
solution concept for multiplayer games?
Unfortunately, it seems that in general, things are quite gloomy.
There seems to be no consensus of what solution concept would that be,
and the results are mostly negative.
The most promising line of work now focuses on online settings and coordinated strategies rather than on following a fixed strategy.
}

\section{Evaluation}
\second{
Given a sub-optimal policy $\pi_i$, we often want to evaluate how ``close'' to an optimal policy it is.
As the worst-case value is no more than the game value: $\delta(\pi_i) = GV_i - BRV_i(\pi_i)$, common measure then is $NASHCONV(\pi) = \sum_i \delta_i$ and exploitability = 
$\frac{NASHCONV}{|\mathcal{N}|}$.
Further popular concept is $\epsilon$-Nash equilibrium, where the players receive at most $\epsilon$ by deviating from the strategy profile: $\max_i \delta_i(\pi_i) \leq \epsilon$.
Exploitability and $\epsilon$-optimal policies allow to objectively and quantitatively evaluate a policy, and contrasts its worst-case performance to the worst-case performance of an optimal policy.

But exploitability itself does not tell the full story about the strength of an agent. 
Strong chess agent that is beatable by a particularly clever line of play is a arguably a better chess player than an agent that always resigns. 
Yet, both of these agents have the same exploitability. 

Sometimes, we are not interested in the worst-case performance, but rather in the qualitative performance of the agent in head to head evaluation against a specific pool of the opponents.
The issue is that such performance strongly depends the players in the pool, but there are also inherent intransitivies in head-to-head evaluation.
Given three players $\pi^a, \pi^b, \pi^c$, it could very well be that $\pi^a$ beats $\pi^b$, $\pi^b$ beats $\pi^c$ and $\pi^c$ beats $\pi^a$.
}

\subsection{Approximate Evaluation}
\label{pig:approx_evaluation}
\second{
To compute exploitability or nashconv, we need to compute the exact best response.
This requires a full tree traversal (Algorithm \ref{alg:pig_best_response}), which can quickly become intractable in large games.
By fixing policy of one of the players, the game collapses into a single agent environment \citep{bowling2003multiagent}, where the optimal policy corresponds to a best response.
One can thus opt for any standard reinforcement learning methods to learn (or approximate) this optimal best-responding policy \citep{greenwald2017solving} and use the resulting approximate policy to define a corresponding metric --- approximate best response \citep{timbers2020approximate}.
Local best response (Section~\ref{sec:iig-lbr}) is then another method approximating the best response, combining search and heuristic value functions \citep{lisy2017eqilibrium}.
}

\chapter{\Subgames{}}
\label{chap:pig-subgames}

\begin{figure}[ht]
    \centering
    \begin{subfigure}[b]{0.45\textwidth}
        \includegraphics[width=\textwidth]{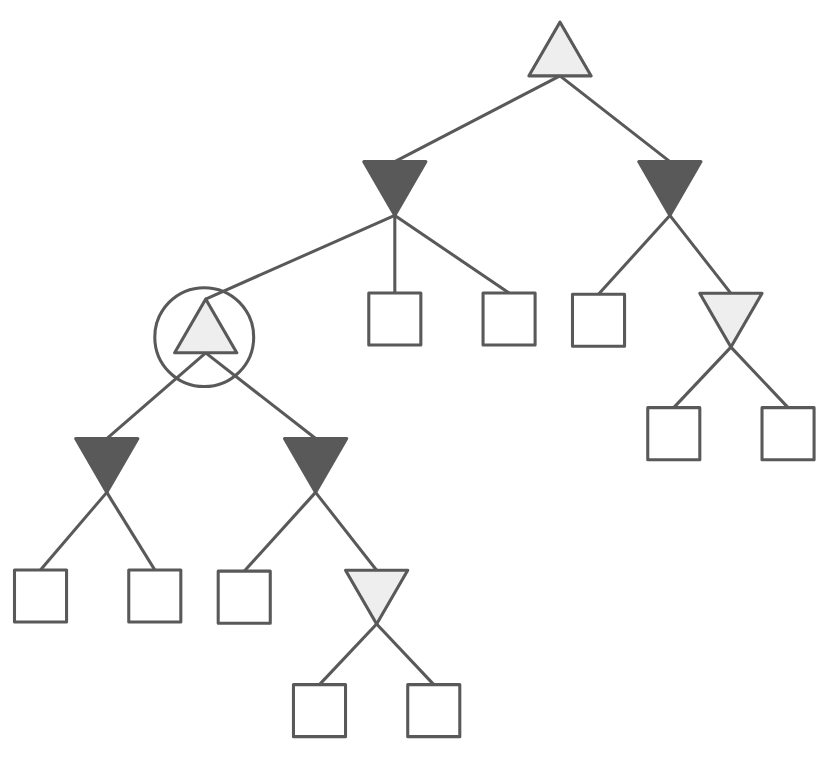}
        \caption{}
    \end{subfigure}
    ~ 
    \begin{subfigure}[b]{0.45\textwidth}
        \includegraphics[width=\textwidth]{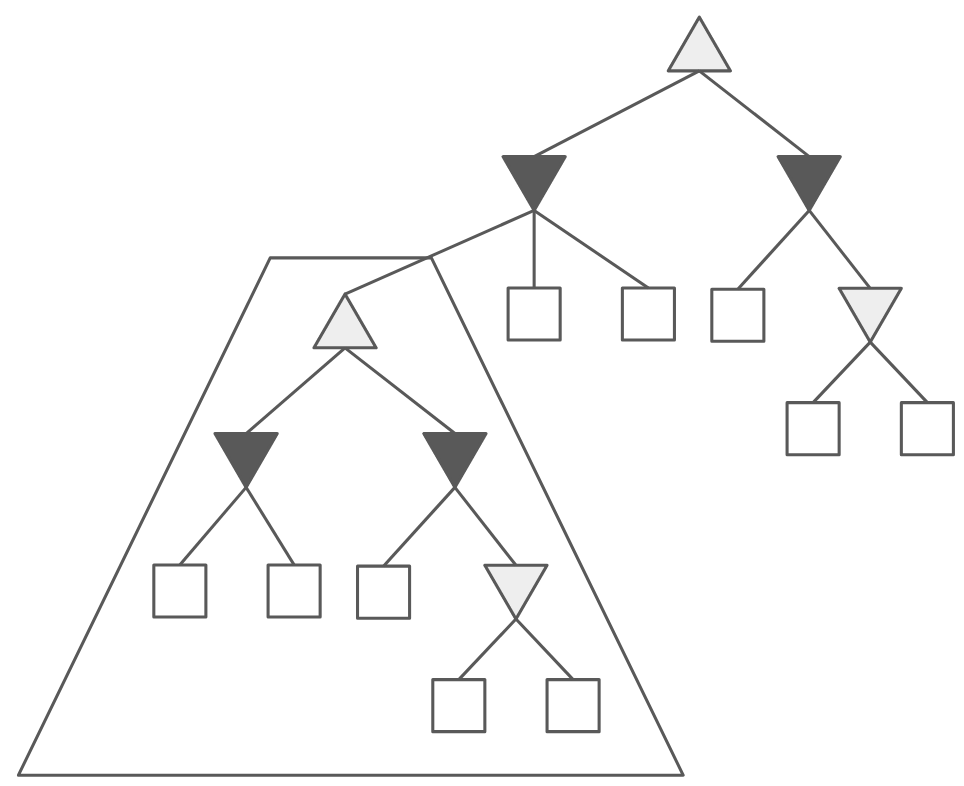}
        \caption{}
    \end{subfigure}
    \caption{a) Current state to reason about.
    b) All the states relevant for the current decision form a \subgame{}.}
    \label{fig:pig_subproblem_reasoning}
\end{figure}

\second{
As the name suggests, a \subgame{} is a \subproblem{} of a game.
\Subgames{} are a fundamental concept for online search methods, and there are two critical ways the are used.
While both cases essentially decompose the problem and combine the solutions, they are very much semantically distinct.

First, \subgames{} allow us to reason about a small \subproblem{} in isolation.
Search methods reason about \subgames{}, namely about the \subgames{} defined by the currently observed state.

Second, \subgames{} are the basis of value functions.
If the \subgame{} defined by the current state is still too large, it allows for further decomposition.
Rather than reasoning about all future states in the \subgame{}, we can look forward only some number of steps and use the value of the future states (future \subgames{}).
}

\section{Perfect Information \Subgame{}}
\label{sec:pig-subgame}
\second{
Given a current state, what is the smallest set of states relevant to the current decision?
In perfect information games, it is simply the state and all the future states, a sub-tree rooted in that state (Figure \ref{fig:pig_subproblem_reasoning}).
To verify that is the case, note that a policy for these states is sufficient to compute a value for the current state, as the policies in all other states are irrelevant.

The defining property of \subgame{} is that --- as the name suggests --- it is a game.
It needs to be well-defined game to be reasoned about in isolation.
Fortunately, this is again straightforward in perfect information games formalism.
All we have to do is to take a sub-tree rooted in a state (Definition \ref{defn:pig_subgame}).

\begin{defn}
\label{defn:pig_subgame}
\Subgame{} of a states $s \in \mathcal{S}$ is a game formed by the sub-tree of $s$.
\end{defn}

}

\section{\Subgame{} Values}
\second{
As \subgame{} forms a well-defined game, we can readily define the value of a \subgame{} (Definition \ref{def:pig-subgame_values}).
}

\begin{defn}[\Subgame{} Value]
\label{def:pig-subgame_values}
\Subgame{} value is the game value of the corresponding game.
\end{defn}

\second{
\Subgame{} values will turn out to be very important for decomposition and search methods, as they will allow allow us to reason about optimal policies in a state by reasoning about future states and the respective \subgame{} values.
}

\section{\Subgame{} Perfect Optimal Strategy}
\label{sec:pig-subgame_perfect}

\begin{figure}[ht]
    \centering
    \begin{subfigure}[b]{0.3\textwidth}
        \includegraphics[width=\textwidth]{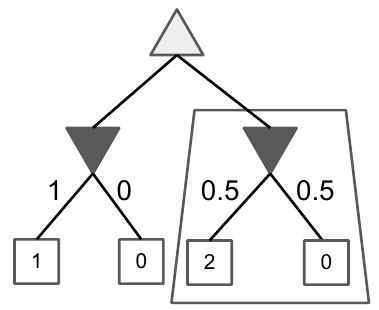}
        \caption{}
    \end{subfigure}
    \hspace{0.2\textwidth} 
    \begin{subfigure}[b]{0.3\textwidth}
        \includegraphics[width=\textwidth]{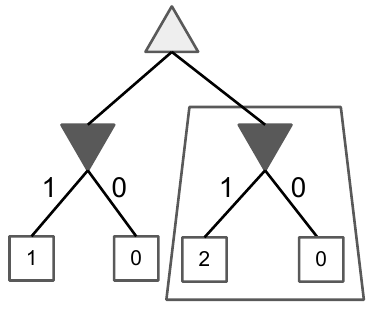}
        \caption{}
    \end{subfigure}
    \caption{a) Optimal policy that guarantees a game value of $1$, but is not \subgame{} perfect as it does not play optimally in the highlighted \subgame{}.
    b) \Subgame{} perfect optimal policy.} 
    \label{fig:pig-subgame_perfect}
\end{figure}

\second{
With the notion of \subgames{} in hand, we can define a useful refinement of the optimal policies.
A strategy in the full game can be optimal and guarantee the game value while playing sub-optimally in some of the \subgames{}.
An important refinement of the optimality concept is then \subgame{} perfect strategy, which is optimal in each \subgame{} (Definition~\ref{def:pig-subgame_perfect}).
See Figure~\ref{fig:pig-subgame_perfect} for an example.

\begin{defn}[\Subgame{} Perfect Strategy]
\label{def:pig-subgame_perfect}
A \subgame{} perfect strategy is optimal for each \subgame{} $s \in \mathcal{S}$.
\end{defn}
}

\chapter{Offline Solving}
\label{chap:pig-offline_solve}
\second{
As previous chapters motivated the policies we are interested in, we will
now have a look at methods that allow us to find the desired strategies.
The setting we start with is that of offline solving.
That is, the algorithm computes the strategy for all the possible decision points prior to game play.
While we will be mostly concerned with tabular algorithms and representations, we also discuss some approximation methods.
}

\section{Minimax Algorithm}
\label{section-minmax}

\begin{figure}[ht]
\floatbox[{\capbeside\thisfloatsetup{capbesideposition={right,top},capbesidewidth=0.5\textwidth}}]{figure}[\FBwidth]
{\caption{Minimax algorithm traverses the game tree bottom-up, producing optimal policies in the \subgames{} and sending the \subgame{} values up the tree.}\label{fig:pig-minimax}}
{\includegraphics[width=0.45\textwidth]{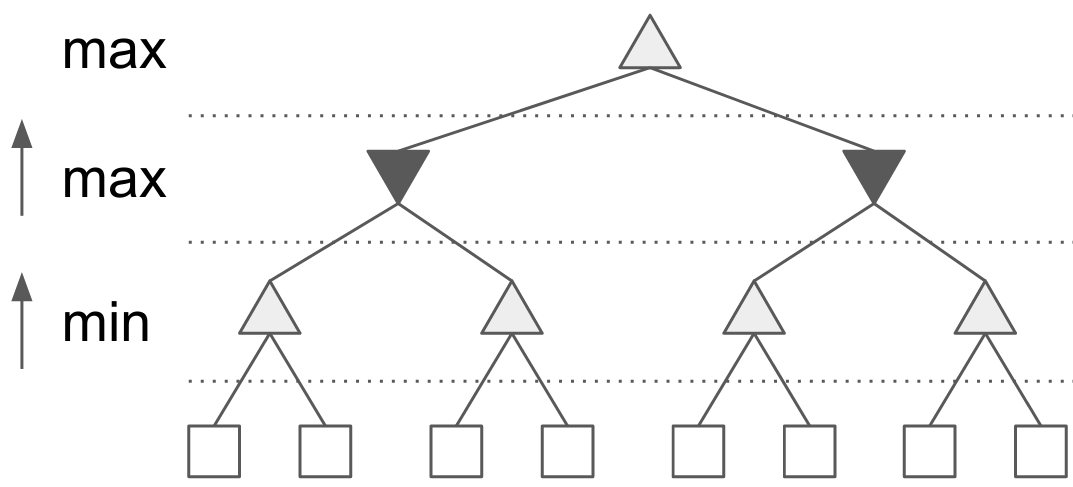}}
\end{figure}


\second{
The \gls{minimax_algo} is a very simple method of computing an optimal policy for a game tree.
This algorithm is often introduced as producing ``optimal policies'' without any deeper elaboration on what that really means.
Previous chapters allow us to get more insights into this algorithm, as we now formally understand optimal policies and their properties (Chapter~\ref{chap:pig-optimal_policies}) as well as \subgames{} (Chapter~\ref{chap:pig-subgames}).

The \gls{minimax_algo} traverses the game tree bottom-up, producing optimal policies in the \subgames{} and sending the \subgame{} values up the tree. These \subgame{} values are then turned into optimal policies for larger \subgames{} until the process reaches the root state (Figure \ref{fig:pig-minimax} and Algorithm \ref{alg:pig-minmax}).
The \gls{minimax_algo} is a great example of decomposition, as it inherently relies on the notion of \subgames{} and \subgame{} values.

Note that as we are computing the policies for both players in the bottom-up fashion, we are not using any knowledge of how the players play further up the tree.
This is because in perfect information, \subgames{} require no such information as the \subgame{} is defined simply as a sub-tree.
This is yet another property that does not hold in imperfect information as we will see in Chapter~\ref{chap:iig_subgames}.
}

\subsection{Resulting Policy}
\second{
To verify that Algorithm \ref{alg:pig-minmax} produces the desired solution, notice that the resulting policies best-respond to each other (see the best response Algorithm~\ref{alg:pig_best_response}).
Lemma~\ref{lem:pig-nash_best_response} then guarantees that such policy pair forms a Nash equilibrium.
Finally thanks to Corollary~\ref{thm:pig-minmax_is_nash}, Nash equilibrium and \maximin{} collapse in our settings.
}

\subsection{AlphaBeta Pruning}
\label{sec:iig-alphabeta}
\second{
While the \gls{minimax_algo} traverses the entire tree, there is a way to potentially traverse only part of the tree while keeping the same guarantees.
\Gls{alphabeta_pruning} allows for potentially visiting only $\sqrt{|\mathcal{S}|}$ of nodes rather than $|\mathcal{S}|$ \citep{knuth1975analysis}.
}

\begin{algorithm}[ht]
\caption{Minimax}
\label{alg:pig-minmax}
\begin{algorithmic}[1]
\State
\Function{Minimax}{$s \in S$}

\If{$s \in Z$} \Return{$R_1(s)$} \Comment{Terminal state.} \EndIf
\State
    \For{$a \in \mathcal{A}$}
        \State $v[sa]$ = \Call{Minimax}{$sa$}
    \EndFor
\State
\If{p(s) = 1} \Comment{Maximizing player to act.}
    \State $\pi(s) = \argmax_{a \in \mathcal{A}(s)}v[sa]$
    \State \Return{$\max_{a \in \mathcal{A}(s)}v[sa]$}
\ElsIf{p(s) = 2} \Comment{Minimizing player to act.}
    \State $\pi(s) = \argmin_{a \in \mathcal{A}}v[sa]$
    \State \Return{$\min_{a \in \mathcal{A}(s)}v[sa]$}
\Else \Comment{Chance player to act.}
    \State \Return{$\sum_{a \in \mathcal{A}(s)} \pi(s, a) v[sa]$}
\EndIf

\EndFunction

\end{algorithmic}
\end{algorithm}

\section{\SelfPlay{} Style Methods}
\label{sec:pig-selfplay_methods}
\second{
\Gls{selfplay} is a powerful idea where agents improve their policies in a self reinforcing loop (Figure \ref{fig:pig-selfplay}).
The hope is that as the strategies keep improving, they will eventually converge to the optimal ones.
In the \selfplay{} framework, agents produce a sequence of strategies $(\pi^t_1, \pi^t_2)$ and algorithms differ in important details.
First is how the algorithm produces the next strategy in the sequence.
Second is which strategy is expected to converge (e.g. the average one or the current one).
}

\begin{figure}[ht]
  \centering
  \includegraphics[width=\textwidth]{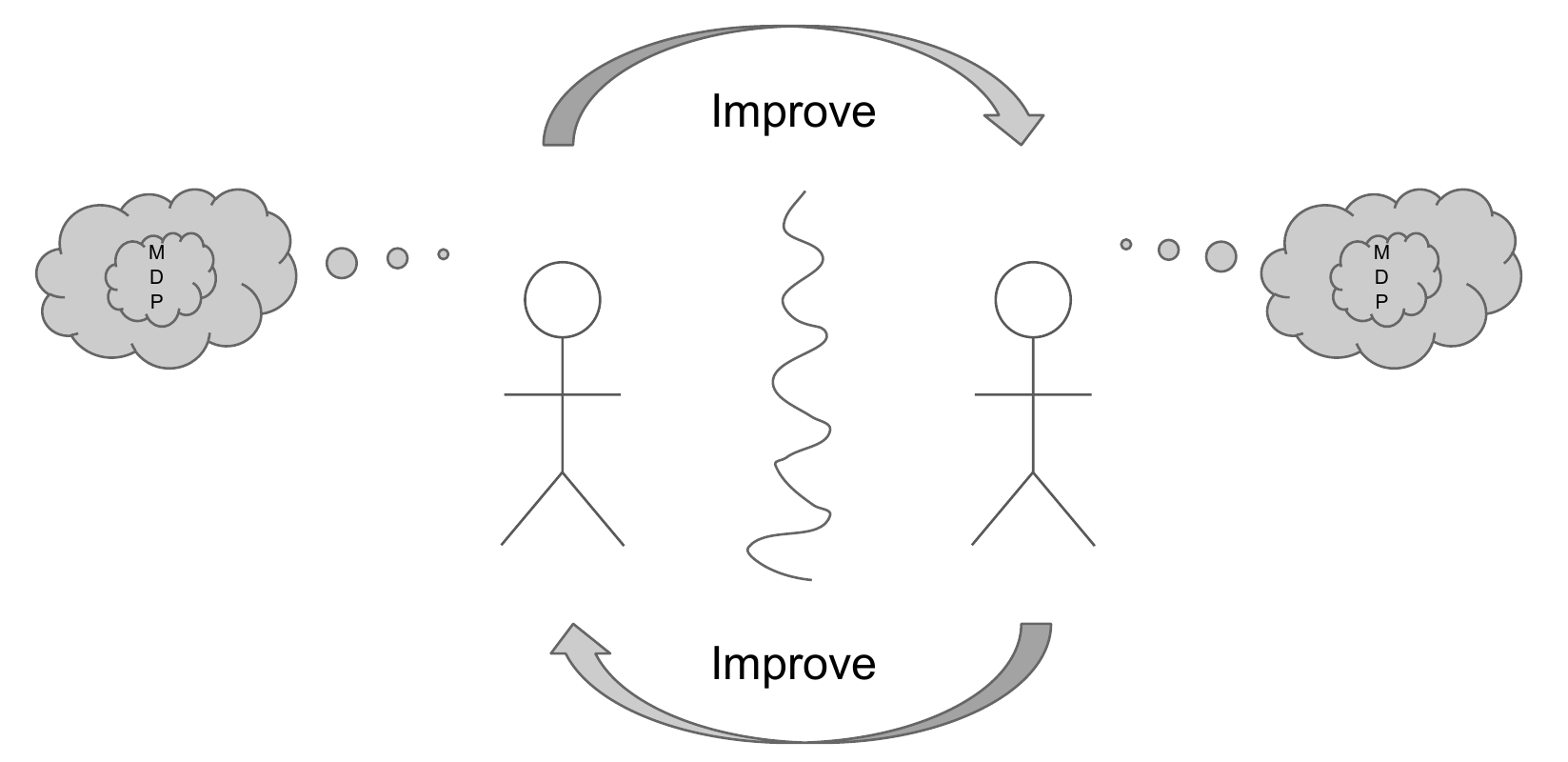}
  \caption{
  \Selfplay{} methods produce a sequence of strategies that improve against each other.
  As fixing one agent makes the environment single-agent, the environment becomes MDP during each improvement step.
  }
\label{fig:pig-selfplay}
\end{figure}

\subsection{Self Play via Independent Reinforcement Learning}

\begin{figure}
\floatbox[{\capbeside\thisfloatsetup{capbesideposition={right,top},capbesidewidth=0.5\textwidth}}]{figure}[\FBwidth]
{\caption{Self play methods that simply best respond to the last policy of the opponent do not necessarily converge. }\label{fig:pig-seflplay_counterexample}}
{\includegraphics[width=0.4\textwidth]{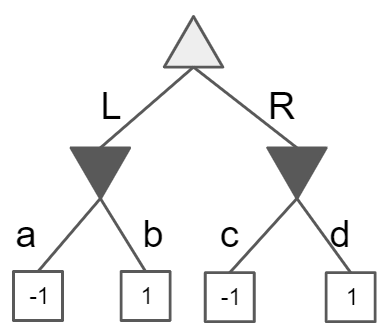}}
\end{figure}


\second{
\Gls{selfplay_irl} is a popular method coming from the reinforcement learning community.
This approach simply uses a single-agent reinforcement learning method to improve the agent's policy against the opponent.
The naive approach is to simply run a single agent \rl{} method of choice independently for both players, each improving against the last strategy of the opponent.
This is indeed at the very heart of many well-known success stories of reinforcement learning methods in two player games \citep{tesauro1995temporal, silver2017mastering_alphagozero}.
The reasoning sounds appealing --- if we keep improving, surely we will end up with an optimal policy?
}

\subsubsection{Convergence}
\second{
What are the convergence guarantees of this method?
To simplify the analysis, we can let the reinforcement learning method fully converge during each self play episode.
We know from the Section \ref{pig:approx_evaluation} that this produces the best response.
We thus end up in a self-play sequence where at each time step, the policy best-responds to the previous policy of the opponent: $\pi^t_i \in \brset_i(\pi^{t-1}_{-i})$.
We refer to this sequence as the \gls{best_responding_sequence}.

Lemma \ref{lem:pig-selfplay_convergence_implies_optimality} shows that if such sequence converges, the resulting strategy profile is optimal.
}

\begin{lemma}
\label{lem:pig-selfplay_convergence_implies_optimality}
If the best responding sequence converges $\exists t \forall t' > t : \pi^{t'} = \pi^t $, the strategy profile $\pi^t$ forms a Nash equilibrium.
\end{lemma}
\begin{proof}
Since the sequence is best-responding $\pi^{t'}_i \in \brset{}(\pi^{t'-1}_{-i})$ and due to convergence $\pi^{t'}_i=\pi^{t'-1}_i$.
Thus the policies $\pi^{t'}_i$ and $\pi^{t'}_{-i}$ best respond to each other, forming a Nash equilibrium (Lemma~\ref{lem:pig-nash_best_response}).
\end{proof}

\second{
But does this method necessarily converge?
While it is tempting to say yes, the answer is no --- even in perfect information games!
Consider a game in Figure \ref{fig:pig-seflplay_counterexample} and the following sequence:

\begin{equation*}
    (\pi^t_1, \pi^t_2) =
    \begin{cases}
       (L, ad), & \text{if} \, t\ \% \ 2 = 0 \\
      (R, bc), & \text{otherwise}
    \end{cases}
    \label{eq:pig-seflplay_counterexample}
\end{equation*}.
 
The sequence satisfies the best-responding property, but it never converges.
This would not be bad if the sequence would simply alternate between different optimal policies.
But at each time step, $\pi^t_2$ is highly exploitable.
This sequence does not converge, and produces a highly exploitable current strategy as well as the average strategy.
}
\second{
Further analysis of the sequence reveals that a necessary ingredient in this counter-example is that the best response of the second player is never \subgame{} perfect (Section \ref{sec:pig-subgame_perfect}).
A best responding sequence is indeed guaranteed to produce optimal policies if the best response is \subgame{} perfect.
Note that this is true only in perfect information games and we will see in Section~\ref{sec:iig-selfplay_irl} that these approaches fail in imperfect information.
}

\chapter{Online Settings}
\label{chap:pig-online_settings}

\second{
Hopefully by this point, the reader is convinced of the importance of Nash equilibria.
We have learned about its properties and introduced methods producing this optimal policy.
So far, we dealt with offline policies --- prior to playing the  game, we would compute a policy for each state and then store it.
During the actual game-play, we would simply follow this policy.

But search algorithms operate in fundamentally different settings, producing a strategy for a state only once it is visited.
This online setting requires a different and careful analysis, as the offline concepts do not apply.
Unlike offline policies, online algorithms can condition the computation on the past experience and games, allowing for opponent adaptation \citep{bard2016online} or re-using parts of the previous computation \citep{silver2016mastering}.
In general, an online algorithm can produce game dynamics that is not consistent with any offline algorithm (Lemma \ref{thm:pig-online_more_general}).

We will introduce analogous concepts to $\epsilon$-Nash equilibria, $\epsilon$-Soundness, generalizing the concept of Nash equilibria to online settings.
Sound algorithms guarantee (in expectation) the game value against arbitrary opponent.
While $\epsilon$-soundness is the proper measure of worst-case performance, it can be hard to compute for some online algorithms as we might need to construct an exponentially large response game.
We therefore introduce a consistency hierarchy that allows us to formalize when on online algorithm plays consistently with some Nash equilibrium.
The strongest of such presented hierarchies, strong global consistency, then allows one to compute exploitability in a particularly easy way using tabularization.

The results presented in this Chapter (and the following results for imperfect information case in Chapter~\ref{chap:iig-online_settings}
) are one of the contributions of this thesis.
}

\section{Cautionary Example}
\label{sec:pig-online_cautionary_example}

\second{
To further illustrate the deceiving nature of online settings, consider the following algorithm.
The algorithm internally uses an offline algorithm that produces provably optimal strategy.
During play, the online algorithm i) runs the offline algorithm for some more time ii) retrieves a strategy $\pi_i(s)$ and takes an action $a \sim \pi_i(s) $.
As the algorithm always ``follows'' an equilibrium, we would naturally expect this algorithm to be sound.
This is not necessarily the case and such an algorithm can potentially be highly exploitable (Theorem~\ref{thm:pig-local_consistency_unsound}).
See Section~\ref{sec:pig-local_consistency} for a concrete example.
Imperfect information games (Chapter~\ref{chap:iig-online_settings}
) will then bring even more challenging examples.
}

\section{Repeated Game}
\second{
To properly analyze the performance of an algorithm, we need to introduce the \gls{repeated_game}.
A repeated game consists of a sequence of individual matches (e.g. playing some number of chess games against the world champion).
As a match progresses, the algorithm produces a strategy for a visited state on-line, once it actually observes the state.
We are then interested in the accumulated reward of the agent during the span of the repeated game.
Of particular interest will be the expected reward against a worst-case adversary.

Formally, the repeated game $p$ consists of a finite sequence of $k$ individual matches $p = (z_1,\,z_2,\, \hdots,\, z_k)$, where each match $z_i \in \mathcal{Z}$ corresponds to the sequence of states and actions leading to that terminal state.
}

\section{Online Algorithm}
\label{sec:pig-online_algorithm}
\second{
An online algorithm $\search$ maps a state visited during a repeated  game to a strategy, while possibly using and updating its internal state (Def.~\ref{defn:pig-online_algorithm}).
As the state of the algorithm is a function of previously visited (queried) states, we can also use $\search^{(z_1,\,\hdots,\, z_{k-1})}(s)$ to denote its output in a particular internal state corresponding to past experience.
}

\begin{restatable}{defn}{onlinealgorithm}
\label{defn:pig-online_algorithm}
\citep[Definition 2]{vsustr2020sound}
Online algorithm $\search$ is a function $s \times \Theta \mapsto \Delta(\mathcal{A}_i(s)) \times \Theta$ that maps an information state $s \in \mathcal{S}$ to a strategy $\Delta(\mathcal{A}_i(s))$, 
while possibly making use of algorithm's state $\theta \in \Theta$ and updating it.
We denote the algorithm's initial state as $\theta_0$.
\end{restatable}

\subsection{Stateless vs Stateful}
\label{sec:pig-online_stateful_vs_stateles}
\second{
A special case of an online algorithm is a \gls{stateless_algorithm}, where the output of the function is independent of the algorithm's state.
If the output depends on the algorithm's state, we say the algorithm is \gls{stateful_algorithm}.
The distinction between stateful and stateless algorithms has some important consequences, as the state is the ``memory'' of the algorithm allowing any conditioning on the past.
}

\subsection{Reward}
\second{
Given two online players $\search_1, \search_2$, we use $P^k_{\search_1, \search_2}$ to denote the distribution over all the possible repeated games $p$ of length $k$ when these two players face each other.
The average reward of $p$ is then $R_i(p) = \frac{1}{k} \sum_{j=1}^k u_i(z_j)$.
Finally, we use $\EX_{p \sim P^k_{\search_1, \search_2}}[R_i(p)]$ to denote the expected average reward when the players play $k$ matches.

While nothing stops the online algorithm from simply following a fixed policy profile, online settings allow for more general dynamics (Lemma~\ref{thm:pig-online_more_general}).
}

\begin{lemma}
\label{thm:pig-online_more_general}
An online algorithm $\search$ can produce game dynamics $P^k_{\search}$ that no offline strategy can.
\end{lemma}
\begin{proof}
Consider a simple game with a single decision state $s_1$, $\mathcal{A}(s_1) = \{L, R\}$ and the following stateful online algorithm with state $\theta \in \mathbb{Z}$: 

\begin{align*}
    \search(s_1, \theta) = \begin{cases} 
                        (L, \theta + 1), & \text{if } \theta \, \% \, 2 = 0\\
                        (R, \theta + 1), & \text{otherwise}
                    \end{cases}
\end{align*}

\end{proof}

\section{Soundness}
\second{
We are now ready to formalize the online concept analogous to Nash equilibrium.
Exploitability / $\epsilon$-equilibrium considers the expected utility of a fixed strategy against a worst-case adversary in a single match.
The analogous concept for the online settings in a repeated game is $\epsilon$-soundness (Definition \ref{defn:pig-epsilon_sound}).
Intuitively, an online algorithm is $\epsilon$-sound if and only if it is guaranteed the same reward as if it followed a fixed $\epsilon$-equilibrium.

\begin{restatable}{defn}{soundness}
\label{defn:pig-epsilon_sound}
\citep[Definition 4]{vsustr2020sound}
For an $\epsilon$-sound online algorithm $\search$, the expected average reward against any opponent is at least as good as if it followed an 
$\epsilon$-Nash equilibrium fixed strategy $\pi$ for any number of matches $k$:
\begin{align}
\label{eq:epsilon_sound}
    \forall k \forall \search_2 \, : \, \EX_{p \sim P^k_{\search, \search_2}}[R(p)] \geq \EX_{p \sim P^k_{\pi, \search_2}}[R(p)].
\end{align}
\end{restatable}
}

\section{Response Game}

\begin{figure}[ht]
  \centering
  \includegraphics[width=\textwidth]{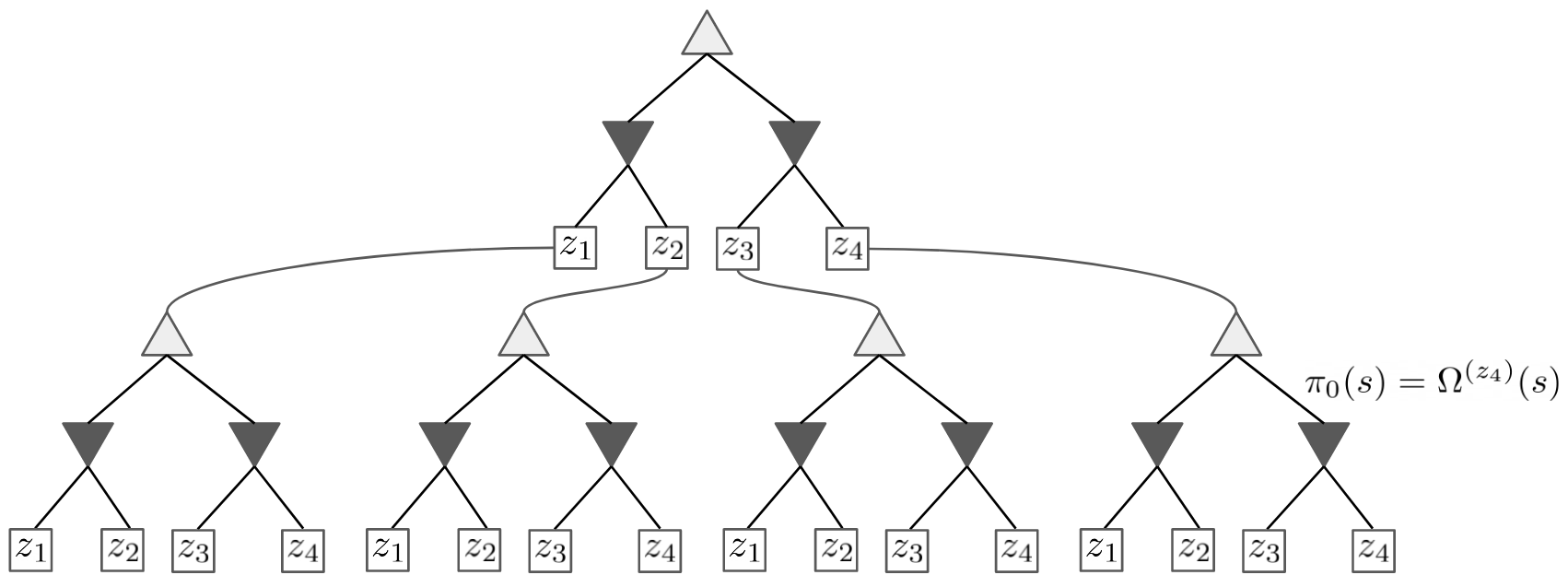}
  \caption{
  To compute the $\epsilon$-soundness, we construct a $k$-repeated game where we replace the decisions of the online algorithm $\search$ with a fixed chance policy $\pi_0$, resulting in a single-agent game of exponential size.
  In the presented example, $k=2$.
  }
\label{fig:pig-kresponse}
\end{figure}

\second{
To compute the $\epsilon$-soundness as in Definition~\ref{defn:pig-epsilon_sound}, we need construct a repeated game, where we replace the decisions of the online algorithm with stochastic (chance) transitions.
As we allow the online algorithm to be stateful and thus produce strategies depending on the game trajectory, the response game must also reflect this possibility.
The resulting game is thus exponential in size as it reflects all possible trajectories of $k$ matches.
The chance policy $\pi_0$ for a state corresponding to past experience $p = (z_1, z_2, \hdots)$ and current state $s$ is then $\pi_0(s^{p.s}) = \search^p(s)$.
We call this single-player game a $k$-step response game (Figure~\ref{fig:pig-kresponse}).
}

\section{Tabularized Offline Strategy}
\label{sec:pig-tabularized_policy}
\second{
When an online algorithm is stateless and thus produces the same strategy for an information state regardless of the previous matches, there is no need for the $k$-response game.
A fixed strategy $\pi$ sufficiently describes the behavior of the algorithm and the notions of exploitability and soundness collapse.
To compute the corresponding fixed strategy, one simply queries the online algorithm for all the information states in the game $\pi(s) = \search(s) \forall s \in \mathcal{S}$.
We refer to this strategy as the \gls{tabularized_strategy}.
}

\section{Search Consistency}
\second{
To prove that an online search algorithm is $\epsilon$-sound, we often want to formally state that the online algorithm plays ``consistently'' with an $\epsilon$-equilibrium.
This allows one to directly bound the $\epsilon$-soundness of the online algorithm.

We introduce three levels of consistency, with varying connections of how closely the online algorithm plays ``just like'' an $\epsilon$-equilibrium.
}

\subsection{Local Consistency}
\label{sec:pig-local_consistency}

\begin{figure}
\floatbox[{\capbeside\thisfloatsetup{capbesideposition={right,top},capbesidewidth=0.75\textwidth}}]{figure}[\FBwidth]
{\caption{Locally consistent algorithm can be highly exploitable.
Consider two Nash equilibria: $\pi_1 = \{s_1:L, s_2:Y \}, \pi_2 = \{s_1:R, s_2:X \}$.
Online algorithm $\search = \{s_1:L, s_2:Y \}$ is locally consistent but highly exploitable.}\label{fig:pig-local_consistency}}
{\includegraphics[width=0.2\textwidth]{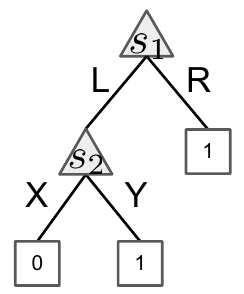}}
\end{figure}


\second{
The weakest of the connections, local consistency simply guarantees that every time we query the online algorithm, there is an $\epsilon$-equilibrium consistent with the produced behavioral strategy for that state (Definition \ref{def:pig-local_consistency}).

\begin{defn}
\label{def:pig-local_consistency}
\citep[Definition 6]{vsustr2020sound}
Algorithm $\search$ is locally consistent with $\epsilon$-equilibria if
\begin{align}
  \forall p=(z_1, \, z_2, \, \hdots,\, z_k) \,  \forall s \sqsubset z_k \,  \exists \pi \in \NEQ \, : \, 
  \search^{(z_1,\,\hdots,\, z_{k-1})}(s) = \pi(s).
\end{align}
\end{defn}

It might seem that local consistency is sufficient as it plays just like an equilibrium.
Figure \ref{fig:pig-local_consistency} presents a counterexample to this intuition. 
The example game includes two states ($s_1, s_2$) for the first player and gives an example of two Nash equilibria ($\pi^1_1, \pi^2_1$) that are in a way ``incompatible''.
If an online algorithm follows $\pi^1_1$ and $\pi^2_1$  in $s_1$ and $s_2$ respectively, the resulting strategy can be highly exploitable (Theorem~\ref{thm:pig-local_consistency_unsound}).

\begin{thm}
\label{thm:pig-local_consistency_unsound}
\citep[Theorem 7]{vsustr2020sound}
An algorithm that is locally consistent might not be sound.
\end{thm}

Looking at the counterexample in more detail, one can notice that it relies on the fact that one of the Nash strategies is not \subgame{} perfect.
Local consistency is indeed sufficient in perfect information games if the algorithm is consistent with a \subgame{} perfect equilibrium (Theorem~\ref{thm:pig-local_consistency_subgame_perfect}).
Theorem~\ref{thm:pig-local_consistency_subgame_perfect} does not hold in imperfect information settings and we will present a counter-example in Section~\ref{sec:iig-search_consistency}.


\begin{restatable}{thm}{localsubgameperfect}
\label{thm:pig-local_consistency_subgame_perfect}
\citep[Theorem 8]{vsustr2020sound}
In perfect information games, an algorithm that is locally consistent with a subgame perfect equilibrium is sound.
\end{restatable}
}

\subsection{Global Consistency}
\second{
Local consistency guarantees consistency for individual states in isolation.
The problem we have seen is that the combination of behavioral strategies of individual states can yield highly exploitable strategy as there might be no equilibrium consistent with the resulting tabularized strategy.
A natural extension is then to guarantee consistency for all the visited states in combination.
Global consistency guarantees the existence of an equilibrium consistent with all the states visited during the gameplay (Definition~\ref{def:pig-global_consistency}).
Unfortunately, even this stronger notion of consistency does not guarantee soundness (Theorem~\ref{thm:global_consistency_iig}).

\begin{defn}[Global Consistency]
\label{def:pig-global_consistency}
\citep[Definition 9]{vsustr2021sound}
Algorithm $\search$ is globally consistent with $\epsilon$-equilibria if
\begin{align}
  \forall p=(z_1, \, z_2, \, \hdots,\, z_k) \, \exists \pi \in \NEQ \, \forall s \sqsubset z_k \,   : \, 
  \search^{(z_1,\,\hdots,\, z_{k-1})}(s) = \pi(s).
\end{align}
\end{defn}

\begin{restatable}{thm}{globallyconsistentnotsound}
\label{thm:global_consistency_iig}
\citep[Theorem 10]{vsustr2020sound}
An algorithm that is globally consistent with an $\epsilon$-equilibria might not be $\epsilon$-sound.
\end{restatable}

But what if the algorithm keeps on playing the repeated game?
While the global consistency with an equilibrium does not guarantee soundness, it guarantees that the average reward eventually converges to the game value in the limit (Theorem~\ref{thm:global_consistency_iig_limit}).

\begin{restatable}{thm}{globalconsistencyiiglimit}
\label{thm:global_consistency_iig_limit}
\citep[Theorem 11]{vsustr2020sound}
For an algorithm $\Omega$ that is globally consistent with an $\epsilon$-equilibria:
\begin{align}
    \forall k \ \forall \search_2 \, : \, \EX_{p \sim P^k_{\search, \search_2}}[R(p)] \geq GV - \epsilon - \frac{ \bigl| \mathcal{S}_1 \bigr| \Delta_R}{k},
\end{align}
\end{restatable}

\begin{figure}
    \centering
    \begin{subfigure}[b]{0.3\textwidth}
        \includegraphics[width=\textwidth]{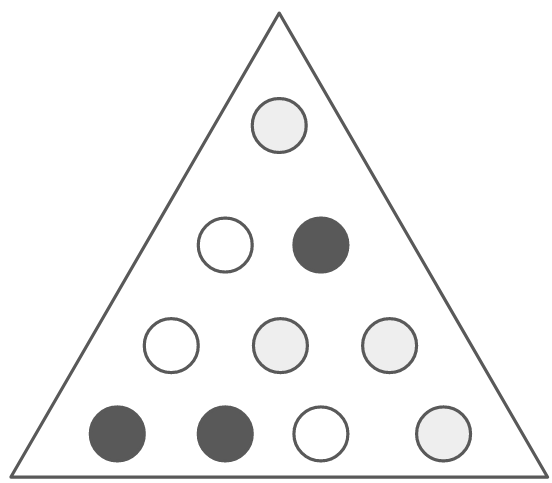}
        \caption{}
        \label{fig:pig-resolving_nonconsistent}
    \end{subfigure}
    ~ 
    \begin{subfigure}[b]{0.3\textwidth}
        \includegraphics[width=\textwidth]{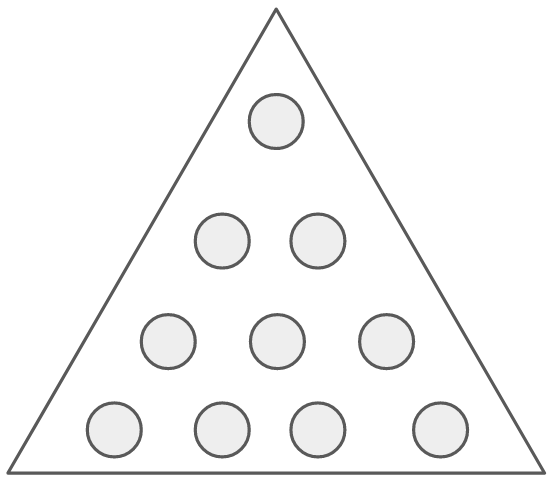}
        \caption{}
        \label{fig:pig-resolving_consistent}
    \end{subfigure}
    \caption{
    (a) Local consistency --- solving \subgames{} independently can lead to solution where policies in all individual states are
    consistent with an optimal policy, but each with possibly a different one.
    For this figure, the different colors represent different optimal policies for the full game.
    (b) Strong global consistency.
    }
    \label{fig:animals}
\end{figure}

}

\subsection{Strong Global Consistency}
\second{
Essentially, the problem with global consistency is that it guarantees the existence of a consistent equilibrium \emph{after} the game-play is generated.
Strong global consistency additionally guarantees that the game-play \emph{itself} is generated consistently with an equilibrium.
In other words, the online algorithm simply exactly follows a predefined equilibrium.

\begin{defn}
\label{def:pig-strong_global_consistency}
\citep[Strong Global Consistency]{vsustr2020sound}
Algorithm $\search$ is strongly globally consistent with $\epsilon$-equilibria if
\begin{align}
     \exists \pi \in \NEQ \,  \forall p=(z_1, \, z_2, \, \hdots,\, z_k) \,\forall s \sqsubset z_k \, : \, \search^{(z_1,\,\hdots,\, z_{k-1})}(s) = \pi(s).
\end{align}
\end{defn}

Strong global consistency guarantees that the algorithm can be tabularized, and the exploitability of the tabularized strategy matches $\epsilon$-soundness of the online algorithm.
}

\section{Search as Re-Solving}
\label{sec:pig-search_as_resolving}
\begin{figure}
    \centering
    \begin{subfigure}[b]{0.3\textwidth}
        \includegraphics[width=\textwidth]{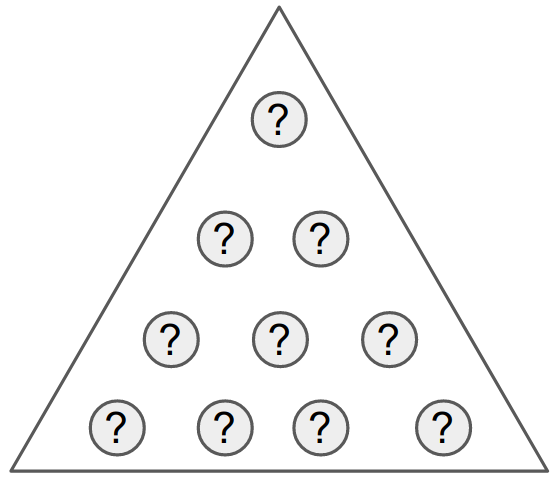}
        \caption{}
    \end{subfigure}
    ~ 
    \begin{subfigure}[b]{0.3\textwidth}
        \includegraphics[width=\textwidth]{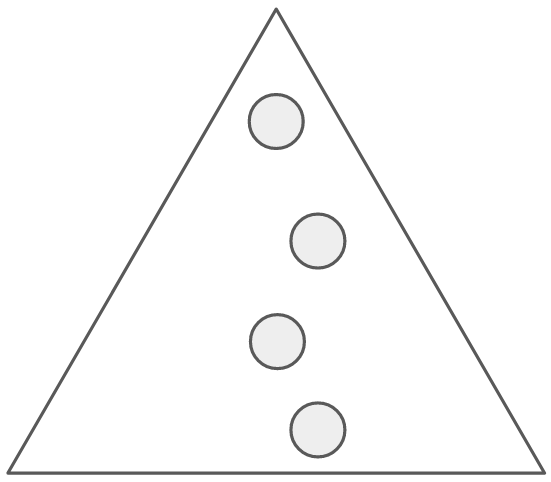}
        \caption{}
    \end{subfigure}
    \caption{
    (a) Possibly unknown offline policy $\pi$ the online search is \resolving{}.
    (b) Online search that is strong globally consistent with $\pi$.
    }
    \label{fig:pig-resolving}
\end{figure}

\second{
The easiest way to argue that an online algorithm is sound is to make sure it is strongly globally consistent with some policy $\pi$.
Such a search method then plays just like an offline policy $\pi$ would, except there is no explicit representation of $\pi$.
Given a state $s$, we do not just solve for an optimal policy for that state.
We rather \resolve{} a policy for that state, making sure it matches the ``original'' solve $\pi$.
Search is then \resolving{} this policy step by step for all the visited states (Figure \ref{fig:pig-resolving}).
}
\chapter{Search}
\label{chapter-search}
\second{
We will now look into a particular variant of online algorithms --- search.
The algorithm is online as it computes the strategy for the current state $s \in \mathcal{S}$ only once the state is observed during game play.
It then computes the strategy for the current state by searching forward, reasoning about the current \subgame{}.

We start with full lookahead settings, where the online algorithm constructs the full \subgame{} to be reasoned about.
We describe the online minimax algorithm that computes a policy for the current state by running the \gls{minimax_algo} (Section~\ref{section-minmax}) on the \subgame{} rooted in the current state.
We then extend this algorithm to limited lookahead, where the search is truncated after some number of moves and the game value of the future \subgames{} is used as a substitute for continued search.
Section~\ref{sec:pig-search_as_resolving} then discusses how to view these search algorithms in the context of \resolving{}, consistency framework an soundness.

We finish with Section~\ref{sec:pig-approximate_value_functions} that includes several approaches for representing (approximate) value functions and Section~\ref{sec:pig-lookahead_tree_size} that includes discussion on the importance of size and structure of the lookahead tree.
}

\section{Full Lookahead Online Minimax Algorithm}
\label{sec:pig-online_minmax}
\second{
Full lookahead \gls{online_minimax_algo} is a stateless online methods that produces a strategy for the current state by searching forward and reasoning about the \subgame{} rooted in the current state.
This allows search to focus on the relevant part for the current decision, as \subgames{} are the smallest sub-problem we can reason about in isolation (Section \ref{chap:pig-subgames}).
Full lookahead \gls{online_minimax_algo} simply runs the minimax algorithm on the full \subgame{} rooted in the current decision $s \in \mathcal{S}$.
It then produces a strategy for the current state $\search(s)$ and disregards the strategies of the other states of that \subgame{}.

The upside of this search algorithm is that if the current state is near the end of the game, we only need to reason about a small \subgame{}.
The downside is that if the state is near the beginning of the game, we have to traverse the entire game tree anyway (Figure~\ref{fig:pig-subgame_full_lookahead}).
Thus the algorithm might be no better than the offline minimax algorithm, although it can avoid the memory needed to store the complete policy.
While the algorithm is arguably not very practical, it is an important building block for the future algorithms and methods presented in this thesis.
}

\section{Limited Lookahead Online Minimax Algorithm}
\label{sec:pig-minimax_limited_lookahead}

\second{
The full lookahead \gls{online_minimax_algo} considers all the states in the \subgame{}, looking forward until the end of the game.
The idea of \gls{limited_lookahead} is to rather look forward only some number of steps, truncate the search and use the value (function) of the future \subgames{} (future \subgames{} of the current \subgame{}, Figure \ref{fig:pig_limited_lookahead}) in place of the continued search.
The core of the idea is to realize that the minimax algorithm only uses values of the future states to derive a strategy.
Given the game values (value under an optimal policy) of \subgames{}, we are thus able to still derive an optimal policy for the current state.

The limited lookahead version of the \gls{online_minimax_algo} builds a limited lookahead tree and runs the minmax algorithm within that tree.
As we still need to know the value of the states at the end of the lookahead, the algorithm evaluates these states using a value function (Algorithm \ref{alg:pig-minmax_limited_lookahead}).
}

\begin{figure}[ht]
    \centering
    \begin{subfigure}[b]{0.3\textwidth}
        \includegraphics[width=\textwidth]{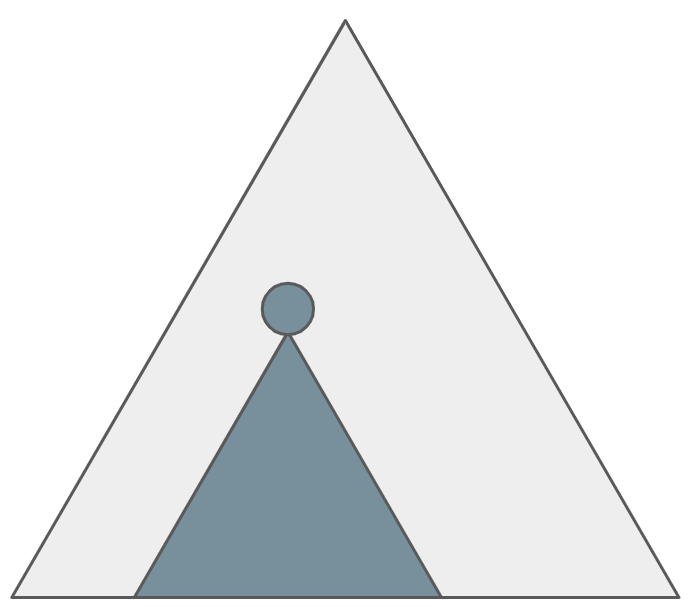}
        \caption{Full lookahead}
        \label{fig:pig-subgame_full_lookahead}
    \end{subfigure}
    \begin{subfigure}[b]{0.3\textwidth}
        \includegraphics[width=\textwidth]{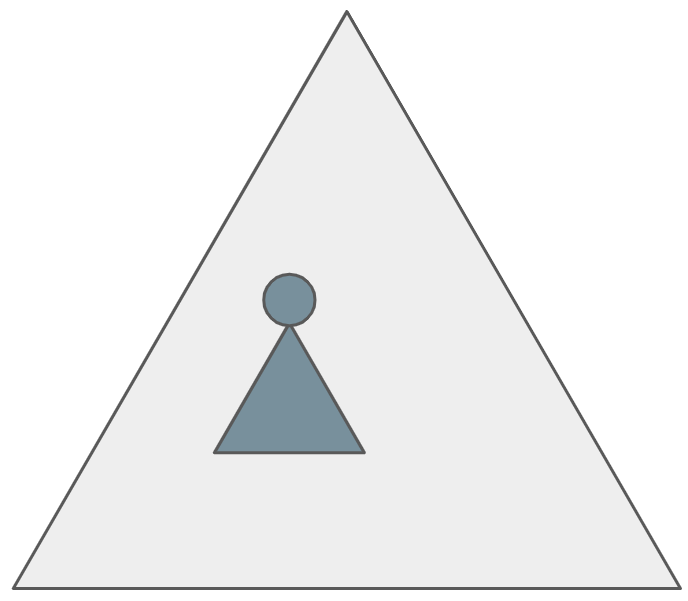}
        \caption{Limited lookahead}
        \label{fig:pig-subgame_limited_lookahead_a}
    \end{subfigure}
    \begin{subfigure}[b]{0.3\textwidth}
        \includegraphics[width=\textwidth]{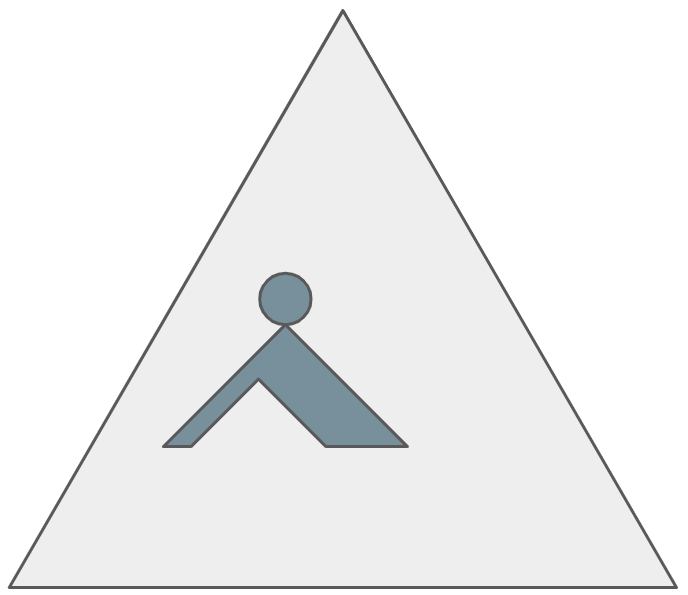}
        \caption{Limited lookahead}
        \label{fig:pig-subgame_limited_lookahead_b}
    \end{subfigure}
    \caption{
    (a) Full lookahead consists of the full subgame rooted in the current state.
    (b) and (c) Limited lookahead consists only of some of the future state.
    Size and shape of the lookahead can vary.
    }
    \label{fig:pig_limited_lookahead}
\end{figure}

\begin{algorithm}[ht]
\caption{Min Max Limited Lookahead}
\label{alg:pig-minmax_limited_lookahead}
\begin{algorithmic}[1]
\State
\Function{MinMaxLookaheadDFS}{$s \in \mathcal{S}$, $V: \mathcal{S} \rightarrow \mathcal{R}$}

\If{$s \in Z$} \Return{$R_1(s)$} \Comment{Terminal state.} \EndIf
\If{$lookahead\_end(s)$}
\State \Return{V(s)} \Comment{Evaluate the state using value the function.}
\EndIf
\State
    \For{$a \in \mathcal{A}$}
        \State $v[sa]$ = \Call{MinMaxLookaheadDFS}{$sa$, $V$}
    \EndFor
\State
\If{p(s) = 1} \Comment{Maximizing player to act.}
    \State $\pi(s) = \argmax_{a \in \mathcal{A}(s)}v[sa]$
    \State \Return{$\max_{a \in \mathcal{A}(s)}v[sa]$}
\ElsIf{p(s) = 2} \Comment{Minimizing player to act.}
    \State $\pi(s) = \argmin_{a \in \mathcal{A}}v[sa]$
    \State \Return{$\min_{a \in \mathcal{A}(s)}v[sa]$}
\Else \Comment{Chance player to act.}
    \State \Return{$\sum_{a \in \mathcal{A}(s)} \pi(s, a) v[sa]$}
\EndIf

\EndFunction

\Function{MinMaxLookahead}{$s \in \mathcal{S}$, $V: \mathcal{S} \rightarrow \mathcal{R}$}
\State \Call{MinMaxLookaheadDFS}{$s$, $V$}
\State \Return{$a \sim \pi(s)$}
\EndFunction

\end{algorithmic}
\end{algorithm}

\subsection{Decomposition}
\second{
There are two places where we just used the decomposition to \subgames{}.
First, the search methods produce a strategy for the current state by reasoning about the \subgame{} rooted in the current state.
Second, \subgames{} are the basis of the value functions in the case of limited lookaheads.
Conceptually similar decomposition will then take place in imperfect information.
}

\section{Re-Solving and Consistency}
\second{
Chapter~\ref{chap:pig-online_settings} detailed the importance of non-locality.
While search reasons inherently locally, it needs to make sure the resulting strategy is strongly globally consistent with an optimal strategy for the full game.
Fortunately, this is particularly easy in this case, as the algorithm produces exactly the same behavioral strategies as the offline minmax algorithm.
Another possible argument is to use Theorem~\ref{thm:pig-local_consistency_subgame_perfect}.

Section~\ref{sec:pig-search_as_resolving} then discussed that in the case of strong global consistency, an online algorithm can be thought of as if it was \resolving{} an offline policy.
In the case of \gls{online_minimax_algo}, the policy being \resolved{} is the policy that the offline minimax algorithm (Algorithm~\ref{alg:pig-minmax}) would produce if we were to run it on the full game (Figure~\ref{fig:pig-online_minmax_resolving}).
}

\begin{figure}[ht]
    \centering
    \begin{subfigure}[b]{0.32\textwidth}
        \includegraphics[width=\textwidth]{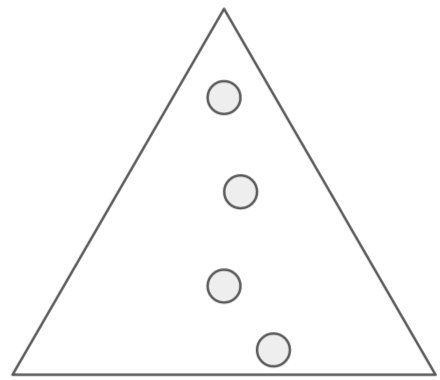}
        \caption{}
        \label{fig:pig-search_minmax_policy}
    \end{subfigure}
    \begin{subfigure}[b]{0.32\textwidth}
        \includegraphics[width=\textwidth]{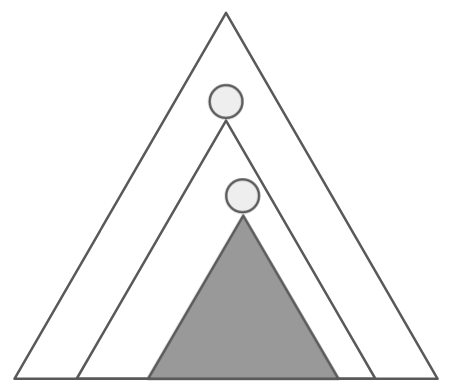}
        \caption{}
        \label{fig:pig-search_minmax_full_lh}
    \end{subfigure}
    \begin{subfigure}[b]{0.32\textwidth}
        \includegraphics[width=\textwidth]{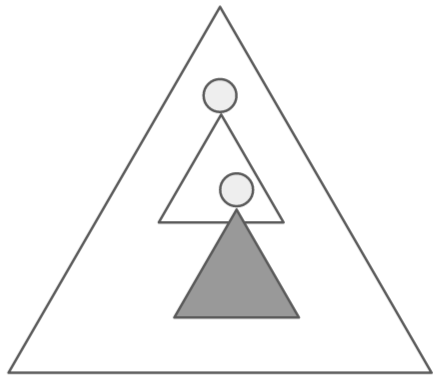}
        \caption{}
        \label{fig:pig-search_minmax_limited_lh}
    \end{subfigure}
    \caption{a) \Resolving{} policy for the visited states. 
            b) Online minimax full lookahead.
            c) Online minimax limited lookahead.
             }
    \label{fig:pig-online_minmax_resolving}
\end{figure}

\section{Approximate Value Functions}
\label{sec:pig-approximate_value_functions}
\second{
Exact value functions allow the search procedure to reason in a limited lookahead, as they provide the exact game value for the relevant future states at then end of the lookahead.
But to compute the game value, we need to solve \subgames{} corresponding to these states anyway.
Maybe instead of using the exact value, we can estimate the value of the future states.

Historically, the first approaches for approximating values were based on game specific heuristics (e.g. material in chess).
Another popular method is to use sampling and estimate the value of a state by sampling some trajectories \citep{kocsis2006bandit, browne2012survey}.

Modern approaches learn the evaluation function through the machine learning paradigm as the nature of the problem allows to generate large data to train on using self-play methods.
Recently, deep learning \citep{lecun2015deep, schmidhuber2015deep} has proven a great fit for such a task as they provide excellent empirical generalization.
}

\section{Lookahead Tree Size}
\label{sec:pig-lookahead_tree_size}
\second{
With an exact value function, the size of the lookahead tree is irrelevant as even a single one-step lookahead is guaranteed to produce an optimal policy.
But once we have only approximate values, one would expect larger trees to result in better policy.
Indeed in practice, the size of the tree does matter.
While the are good reasons for this to be the case, there is no guarantee for this behavior \citep{nau1982investigation}.

Furthermore, size is not the only parameter that matters in practice and structure is often also critical.
We often want to search deeper where the current policy is likely to play or where we are uncertain about the values (similarly to how humans perform search in games).
For example in chess, professional players search very deep where they believe the situation to be interesting.
It is thus also common to dynamically expand the lookahead tree during the search rather than to construct fixed-sized lookaheads.
}

\chapter{Examples}
\label{sec:pig-examples}

\second{
We now list some important milestones, where the combination of search and value functions led to historical achievements in perfect information games.
While the search methods and value functions evolved over the years, the high level ideas can be traced to the very dawn of games research.
}

\section{Samuel's Checkers}
\second{
Samuel's checkers is an exciting application that dates all the way back to $1959$ \citep{samuel1959some}.
It is amazing to see how many of the ideas considered as new and modern already appeared in this program.
While the strength of the program was relatively weak, one must realize how limited both hardware and software were at that time.
And while search algorithms for perfect information games have made a great deal of progress both in strength and generality since then,
the core principles of search combined with self-play learned value functions are already present here.
}

\paragraph{Search Method}
\second{
The search algorithm was a variant of minmax with AlphaBeta pruning (Section~\ref{sec:iig-alphabeta}).
The lookahead tree was dynamically grown using rules based on the depth as well as some game specific properties.
}

\paragraph{Value Functions}
\second{
Perhaps the most impressive part were the value functions.
The value was estimated via a polynomial combinations of hand-crafted game specific features.
Not only were the weights adjusted through self-play --- the learning method even used bootstrapping\footnote{The update/target for the value function being based on the value function estimate of the future states, see e.g. \cite{sutton2018reinforcement}.} techniques.
Finally, the paper also describes a method of learning from prior human games.
}

\section{Deep Blue}
\second{
Being one of the most popular board game in the world, chess playing algorithms have always been of great interest to computer scientists.
Shannon's work from $1950$ already describes a simple variant of minimax search and even includes a section on approximate value function computed as a combination of game specific features \citep{shannon1950xxii}.

Almost $50$ years after that --- in $1997$ --- Deep Blue defeated the world champion Garry Kasparov ($3.5$–$2.5$).
For the first time in history, machines defeated the human world champion in a regulation match \citep{campbell2002deep}.
In $1996$, the previous version of Deep Blue lost to Kasparov by $2$-$4$, making the second match particularly interesting.
}

\paragraph{Search Method}
\second{
Deep Blue used a massively parallel search that combined both hardware and software search.
The search was a variant of AlphaBeta search with dynamically grown tree designed to be highly non-uniform.
The software/hardware search allowed it to search up to $330$ million chess positions per second.
}

\paragraph{Value functions}
\second{
Deep Blue's evaluation was implemented in hardware. This made the evaluation particularly fast, but hard to modify.
The evaluation function consisted of about $8,000$ different chess specific features (``passed pawns'', ``bishop pair'') that were combined to produce the final value estimate.
Making such a complicated game specific value function is a hard and time-consuming task.
Authors themselves stated that ``...we spent the vast majority of our time between the two matches designing, testing, and tuning the new evaluation function.''.
}

\section{TD-Gammon}
\second{
TD-Gammon showed how powerful a combination of self-play learning and neural nets as value function approximation can be \citep{tesauro1995temporal}.
It was the first successful application of this combination for large games, achieving impressive performance.
}

\paragraph{Search Tree}
\second{
The search tree was a fairly small $2$-ply lookahead, with larger trees suggested as a possible future improvement.
}

\paragraph{Value functions}
\second{
TD-Gammon used a feed-forward fully conected neural network for its value function.
The feature representation initially encoded raw board representation, and further game-specific hand-coded features were later added to improve the performance.
}

\section{Alpha Go}
\second{
AlphaGo further showed the power of combining deep learning and self-play.
Go is a prime example of game where hand-crafting value functions or features is particularly difficult, and thus state of the art programs relied on rollout simulations to estimate the state values.

But the strength of such agents was far below professional players, and many predicted that it would take further decades for computers to match the top humans.
AlphaGo shocked many by reaching this milestone in $2016$ \citep{silver2016mastering}.
}

\paragraph{Search Method}
\second{
The search used is a variant of MCTS with a p-UCB formula to guide the simulations.
}

\paragraph{Value functions}
\second{
Value functions are represented with deep convolutional neural network \citep{lecun2015deep, schmidhuber2015deep}, and were pre-trained on human data through supervised training.
Self-play was then used for further improvement.
While some game specific features were still present in this version, further improvements made the agent both more general and stronger.
}

\section{From AlphaGo to MuZero}
\second{
Even though AlphaGo was a huge achievement, there were still some limitations.
It could only play go, it needed human data for the initial supervised training and the value function represented by a neural network received some hand-crafted go features.
Later work removed all of these limitations.
AlphaGoZero uses no human data and removes hand-crafted features, while achieving even stronger performance \citep{silver2017mastering_alphagozero}.
AlphaZero then achieves master level perfomance on go, chess, and shogi --- using a single algorithm \citep{silver2017mastering_alphazero}.
This work has also now been reproduced by Leela Zero, one of the strongest and opensourced chess engines \citep{leelazero}.

The latest representative of this lineage, MuZero makes the approach even more general, as it does not need the rules of the game to construct a game tree for search.
The rules are learned during the self-play and the search uses only the learned model during the planning.
This means that MuZero does not need to access the environment during its search, and only does so when it is taking its action \citep{schrittwieser2020mastering}.
}

\chapter{Summary}
\second{
In the first part of this book --- perfect information games --- we build a framework that allows us to build a particular view on search in games.
First, we introduced optimal policies and their properties making them desirable to be followed.
We then presented offline algorithms producing optimal policies, i.e. algorithms that solve games.
Next, we described how sound online search algorithms follow these optimal offline strategies.

In the second part of this book --- imperfect information games --- we will build upon the formalism, intuition and framework.
We will see that a combination of search and value functions is indeed possible and is also a powerful method in imperfect information.
But we will also see that sound search is substantially harder in the imperfect information case.
}
\part{Imperfect Information}

\chapter{Introduction}

\second{
In the second part of this thesis, we will revisit the concepts introduced in perfect information settings.
We will see that these concepts need to be generalized in order to be either well-defined or to produce the desired policies in imperfect information settings.
But after we do so, we will indeed be able to construct sound search methods for imperfect information games, resulting in agents that can outplay even the best human players.
}

\section{Structure}
\second{
We start again with a quick overview of the chapters.
}

\paragraph{Example Games}
\second{
Chapter~\ref{sec:iig-games} introduces the games used in the following chapters.
}

\paragraph{Formalisms}
\second{
Chapter~\ref{chap:iig-formalisms} includes formal models of imperfect information environments.
Unlike in perfect information, we present multiple models and discuss their connections.
Namely, we introduce factored-observation stochastic games --- one of the contributions of this thesis.
}

\paragraph{Offline Policies}
\second{
Chapter~\ref{chap:iig-policies} argues that the same worst-case reasoning holds in imperfect information.
The concepts of minmax and Nash equilibrium are still guaranteed to exist, with all the same desirable properties as in perfect information.
But unlike in perfect information, the policies might have to be stochastic.
We present some examples as to why the randomization is necessary, and introduce the link between the level of uncertainty in the game and the amount of randomization.
}

\paragraph{\Subgames{}}
\second{
Chapter~\ref{chap:iig_subgames} generalizes the notion of \subgames{}.
Unlike in perfect information, a single state is not sufficient to construct a well-defined sub-problem to be reasoned about.
A distribution over the consistent states of both players is required as the agents can have different observations.
We show how the notion of public information makes the construction of these consistent states particularly simple.
Finally, we also generalize the connected concept of value functions to imperfect information \subgames{}.
}

\paragraph{Offline Solving I}
\second{
Chapter \ref{chap:iig-offline_solve} presents multiple offline algorithms producing optimal policies.
These include direct optimization against the best-responding opponent, self-play via independent reinforcement learning, and fictitious play.
}

\paragraph{Offline Solving II - Regret Minimization}
\second{
Chapter~\ref{sec:iig-offline_solve_regret} deals with the regret minimization framework and its connections to optimal policies.
Importantly, we introduce the powerful idea of counterfactual regret minimization that allows the use of regret minimization in sequential decision making by decomposing the regret to partial regrets in individual information states.
Counterfactual regret minimization will be used as key building blocks for the online search settings in the following chapters.
We also include the family of sampling Monte Carlo variants and the latest addition to this family --- methods that speed up the convergence by lowering the per-iteration variance.
These low-variance methods are another contribution of this thesis.
}

\paragraph{Offline Solve III - Approximate and Abstraction Methods}
\second{
Chapter~\ref{chap:iig-abstraction} discusses methods that do not employ the idea of online search, but still allow one to deal with large games where tabular methods are not feasible.
}

\paragraph{Online Settings}
\second{
Chapter~\ref{chap:iig-online_settings} revisits the online setting, with some minor modifications to the required definitions to match the formalism of imperfect information.
It includes some interesting counter-examples that are only possible in imperfect information.
}

\paragraph{Search}
\second{
Chapter~\ref{chap:iig-search} is arguably the key chapter of this thesis, the culmination of the building blocks presented up to this point.
First, we show that many properties we depended on in perfect information settings simply break in imperfect information.
We then present the building blocks of re-solving and gadget, including another contribution of this thesis --- safe refinement of \subgames{}.
Step by step, we then put the building blocks carefully together into the final algorithm --- continual resolving.
Continual resolving allows for sound search in imperfect information, and was a crucial contribution of DeepStack.
Chapter~\ref{chap:iig-search} and DeepStack are the main contributions of this thesis.
}

\second{
\paragraph{DeepStack}
Chapter~\ref{chap:iig-deepstack} presents DeepStack --- the first computer agent to introduce the combination of sound search and learned value functions for poker, beating professional human players in no-limit Texas hold'em poker.
As human evaluation in no-limit poker is inherently noisy, we also describe some techniques that allow for provably unbiased variance reduction.
The latest of such techniques is AIVAT --- our final contribution that resulted in an impressive $85\%$ reduction in standard deviation when evaluating the DeepStack results.
We finish with a list of some of the latest search algorithms.
}

\chapter{Example Games}
\label{sec:iig-games}
\second{
This chapter briefly introduces some imperfect information games used in the second part of the book.
The quintessential example is poker, where we do not get to see opponent's cards.
Consider the groundbreaking publication of Oskar Morgenstern and John von Neumann ``Theory of Games and Economic Behavior'' \citep{morgenstern1953theory}.
This book laid the ground of modern-day game theory, and contains a full section (section $19$, over $30$ pages) dedicated to Poker.

The importance of poker resulted in substantial body of research where poker games were the only domain used for evaluating the algorithms.
To emphasize the generality of the algorithms presented here, we decided to also include a game with structure drastically different to poker.
}

\section{Kuhn Poker}
\second{
Kuhn poker is a very simple poker game with only $3$ cards in the deck and a single betting round \citep{kuhn1950simplified}.
The deck includes only three cards: (J)ack, (Q)ueen and (K)ing.
Both players are dealt a single card out of a deck, they each put one chip the table (ante) and a simple betting round follows.
First player either (b)ets a chip (adding one more chip on the table) or (c)hecks.
If the first player bets, the second player can either (c)all or (f)old --- both actions terminating the game.
If the first player checks, the second player can either bet one chip or check.
Check terminates the game, while for bet, the first player acts again and can either check or fold --- both actions terminating the game.
If any player folded, they lose the chips on the table.
Otherwise, the player holding the higher card wins these chips.
We will present figures of small poker games as soon as we introduce the necessary formalism 
in Section~\ref{sec:iig-fog_examples} and~\ref{sec:iig-fog_examples}.
}

\section{Leduc Poker}
\second{
Leduc poker is a slightly larger poker game often used as a test domain in imperfect information games.
The card deck consists of $6$ cards (two suits and three ranks), and there are two (limited) betting rounds with a single public board card \citep{southey2005bayes}.
Note that while the game is larger than Kuhn poker, it is still very small.
}

\section{Limit nad No-limit Texas Hold'em Poker}
\second{
Texas hold'em poker is played with the standard $52$ card deck.
The game progresses in four betting rounds:  i) pre-flop ii) flop iii) turn and iv) river.
Cards are dealt at the beginning of each round.
In the pre-flop, each player is dealt two private cards (hand). 
In the later rounds, public (facing up and observable to all players) cards are dealt on the table: three on the flop, one on the turn and one on the river.
During each betting round, the players alternate taking one of the three action types: i) fold ii) check/call and iii) bet/raise.
By folding, the player gives up and loses the money wagered up to this point (pot).
Calling matches the bet of the opponent, and check is the initial call action when the player is facing no bet/raise.
Raise increases a player's wager, and bet is the initial raise when there is no opponent's bet to be raised.
The round ends when both players match their bet, and the game terminates when either of the player folds or when the last round (river) ends.
When neither player folds before the end of the river, both players reveal their hand and the player with the strongest combination of the private and public cards wins the pot.
For the ranking of the possible card combinations and more detailed rules see e.g. \cite{harroch2007poker}.

Finally, the difference between limit and no-limit poker is that in limit poker, there is a single bet size the player can make.
In no-limit, the player is free to make a bet up to their stack size, and all-in bet then refers to agent betting their entire stack
}


\section{Graph Chase}
\second{
Graph chase games are a simple, security-like games.
Players get to move their stones from node to node, alternating their moves.
One player (evader) controls a single invisible stone and is trying to escape the opponent.
The second player (chaser) controls multiple stones (the location of these stones is publicly observable) and is trying to catch the invisible stone (by moving a stone on to the location of the invisible stone).
The evader wins if they gets to survive a specified number of moves, and loses otherwise.
While the rules can surely be made more complicated, this simple variant still makes an interesting game.

Figure~\ref{fig:iig-glasses} shows a small instantiation of this game, with two stones of the chaser and a small graph.
We refer to this game as the \Gls{glasses} due to the shape of the graph.
}

\subsection{Contrast to Poker}
\second{
There are some important conceptual difference to poker.
In poker, all actions are publicly observable, making some of the decomposition constructions easier.
In graph chase games, the chaser does not get to see the exact action (edge) the evader took.
Furthermore, the number of states the player can be in varies as the game progresses.
In poker, this corresponds to the possible hands of a player, which is fixed in individual betting rounds.
These distinctions are important as we need our algorithms to handle this general case.
}

\begin{figure}[ht]
  \centering
  \includegraphics[width=\textwidth]{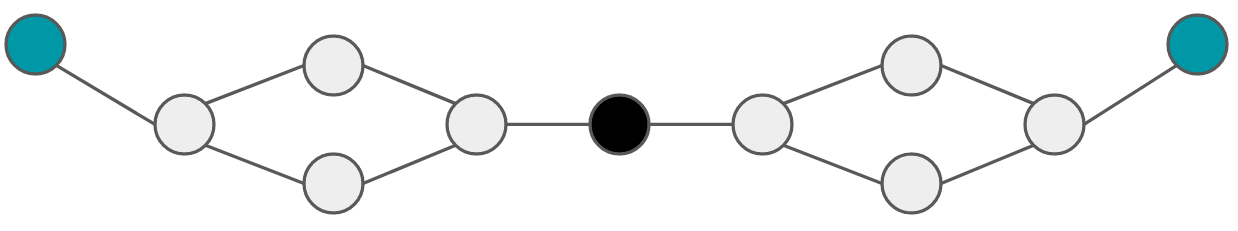}
  \caption{
      Glasses: graph chasing game on a graph with only $11$ nodes.
      Evader starts in the middle (black stone) while the chaser has two (blue) stones at the edges.
      The evader wins if they get to survive for $3$ turns.}
\label{fig:iig-glasses}
\end{figure}

\chapter{Formalisms}
\label{chap:iig-formalisms}

\second{
In this chapter, we introduce three distinct formal models of imperfect information games.
The first one --- \gls{matrix_games} --- is a simple yet powerful model for simultaneous decision making, where the players get to act at the same time.
The second model --- extensive form games - is the most popular formal model of sequential decision making in imperfect information.
It is a natural generalization of the game tree from perfect information setting.
Unfortunately, some design choices of this model make it ill-suited for formal analysis of the search algorithms, motivating the next model.

The third model --- Factored Observation Stochastic Games (FOSG) --- is a recent model build with the needs of search in mind and also one of the contributions of the thesis.
Just like extensive form games, it models sequential decision making in imperfect information.
The defining properties that make it suitable for modern search methods are i) factoring observations to private and public components, and ii) the notion of an agent's state even if they are not to act.

Note that any sub-problem decomposition and search only really makes sense in sequential decision making.
Thus while we use matrix games to introduce some of the essential theory, sub-problems and search methods will be only discussed for FOSGs.
}

\section{Matrix Games}
\second{
It turns out that the simplest formalism for imperfect information games is
based on simultaneous moves.
That is, both players get to make their action at the same time.
At first, it might not be clear why this has anything to do with imperfect information.
The reason is simple --- we can think of such interactions as one of the players acting first, while hiding the action from the opponent.
The opponent then gets to act without knowing what the first player did. 
}

\second{
\Gls{matrix_games} (also referred to as \gls{normal_form_games}) is a simple formalism, where the first player (the row player) chooses a row, while the second player (the column player)  chooses a column (Definition \ref{defn:iig-matrix_game}).
The reward of the the first player is then computed as $x X \T{y}$.

\begin{defn}
\label{defn:iig-matrix_game}
Matrix game consists of:
\begin{itemize}
    \item $\mathcal{N} = \{1, 2\}$ is the set of players.
    \item $\mathcal{A}_1, \mathcal{A}_2$ are the player's set of legal actions.
    \item Player strategies $x \in \Delta(\mathcal{A}_1)$, $y \in \Delta(\mathcal{A}_2)$.
    \item The reward matrix $\underset{|\mathcal{A}_1| \times |\mathcal{A}_2|}{X}$.
    
\end{itemize}
\end{defn}
}

\subsection{Example}
\second{
Table \ref{tab:iig:game_rps} includes the matrix game for the \rps{} game.
In this case, both row and column player have the same set of actions, $A_1 = A_2 = \{ R, P, S \}$.
As a sanity check, the row's player utility under the following strategy profile $(\pi_1 = (0.2, 0.2, 0.6), \pi_2 = (0.4, 0.2, 0.4))$ is
$\pi_1 A \pi_2' = 0.08$.
}

\begin{table}[ht]
\centering
\begin{tabular}{lllll}
  & R & P & S  &  \\
  \cline{2-4}
R & 0                         & 1                         & -1                         &  \\
R & -1                        & 0                         & 1                          &  \\
R & 1                         & -1                        & 0                          & 
\end{tabular}
\caption{Matrix game representation of the (R)ock-(P)aper-(S)cissors game.}
\label{tab:iig:game_rps}
\end{table}

\section{Extensive Form Games}
\second{
\Gls{efg_g} are a natural extension of game trees to imperfect information environments, allowing for sequential interactions.
In imperfect information, players do not directly observe the exact state of the game and thus there might be mutliple world states the players can not tell apart.
Extensive form games achieve this by grouping such states into so-called information sets (or states).
The agent's policy is then defined within these information states.

Extensive form games date all the way back to Von Neumann \citep{morgenstern1953theory}, and the model has been widely adopted by the game theory community.
The importance of the original publication is illustrated by the fact that the commemorative edition was published more than
50 years after the original version \citep{von2007theory}.
Interestingly enough, the book contains a full section (section $19$, over $30$ pages) dedicated to Poker.
And in $2015$, a limit Texas hold'em poker become the largest imperfect information game played by humans to be solved \citep{bowling2015heads} - modeled by this very formalism.
}

\second{
\begin{defn}
\label{defn:iig-extensive_form_game}
\citep[Definition 200.1]{osborne1994course}
An extensive form game is a tuple $(\mathcal{H}, \mathcal{Z}, \mathcal{A}, \mathcal N, p, \pi_c,  u, \mathcal I)$, where:
\begin{itemize}
	\item $\mathcal{H}$ is the set of histories, representing sequences of actions.
	\item $\mathcal{Z}$ is the set of terminal histories (those $z\in \mathcal{H}$ which are not a prefix of any other history).
	We use $g \sqsubseteq h$ to denote the fact that $g$ is equal to or a prefix of $h$.
    \item For a non-terminal history $h\in \mathcal{H}\setminus \mathcal{Z}$, $\mathcal{A}(h) := \{ a \, | \ ha \in \mathcal{H} \}$ is the set of actions available at h.
	\item $\mathcal{N} = \{1, \dots, N\}$ is the player set. In addition,  $c$ is a special player, called ``chance'' or ``nature''.
	\item $p : \mathcal{H} \setminus \mathcal{Z} \rightarrow \mathcal{N} \cup \{c\}$ is the player function partitioning non-terminal histories into $\mathcal{H}_i$, depending on which player acts at $h$. 
    \item The strategy of chance is a fixed probability distribution $\pi_c$ over actions in $\mathcal{H}_c$, $\pi_c(h) \in \Delta(\mathcal{A}(h))$.
    \item The utility function $u=(u_i)_{i\in \mathcal{N}}$ assigns to each terminal history $z$ a reward $u_i(z)\in \R$ received by player $i$ upon reaching $z$.
	\item The information-partition $\mathcal{I} = (\mathcal{I}_i)_{i\in \mathcal{N}}$ captures the imperfect information of $G$.
	For each player, $\mathcal{I}_i$ is a partition of $\mathcal{H}_i$. If $g,h\in \mathcal{H}_i$ belong to the same $I\in \mathcal{I}_i$ then $i$ cannot distinguish between them. For each $I \in \mathcal{I}_i$, the available actions $\mathcal{A}(h)$ must the same for each $h \in I$, and we overload $\mathcal{A}(\cdot)$ as $\mathcal{A}(I) = \mathcal{A}(h)$.
\end{itemize}

\end{defn}
}

\subsection{Strategies}
\second{
\Glspl{behavioral_strategy} prescribe behavior in all the individual information sets (states).
It is a mapping from an information state to a distribution over the actions $\pi_i: I \in \mathcal{I}_i \rightarrow \Delta{A(I)}$.
Section \ref{sec:iig:sequence_form} will present another option for strategy representation --- the sequence form.
}

\subsection{Example}
\label{sec:iig-efg_examples}

\begin{figure}[ht]
  \centering
  \includegraphics[width=6cm]{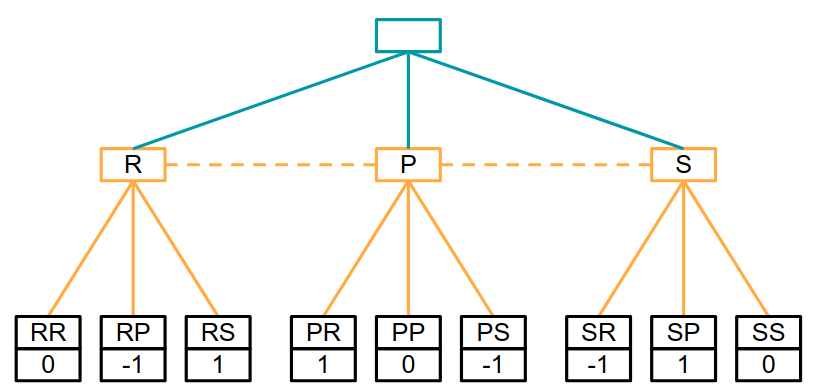}
  \caption{Extensive form game for Rock-Paper-Scissors. 
  Individual states (histories) are labeled with the corresponding history (there are $13$ histories in total, $9$ of which are terminal).
  Below the terminal states, we also include the player $1$ terminal utility.
  Finally, the dotted lines connect histories into information sets/states.
  In this case, both players have a single information set.
  }
\label{fig:iig_efg_rps}
\end{figure}

\begin{figure}[ht]
  \centering
  \includegraphics[width=14.5cm]{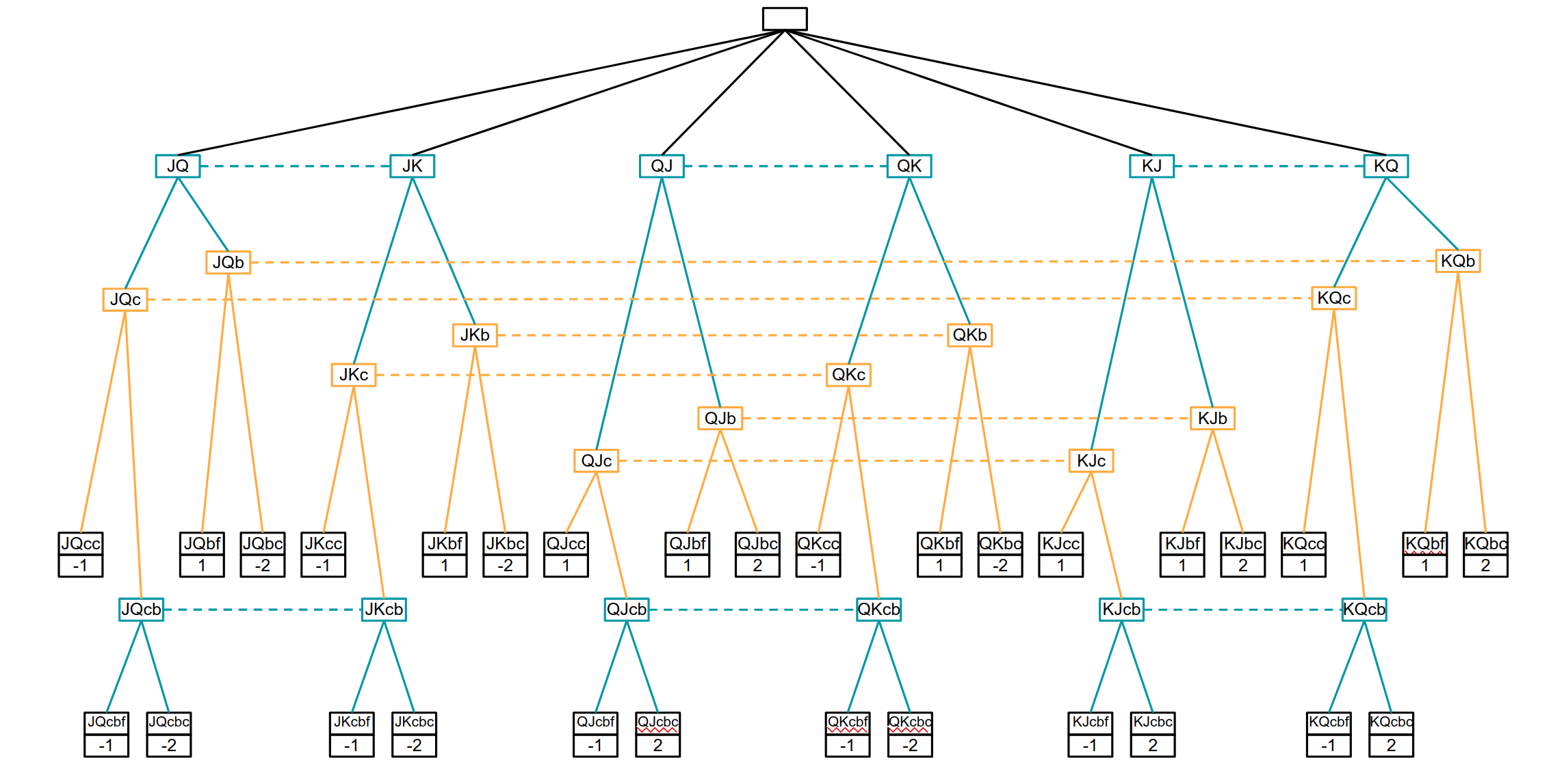}
  \caption{Extensive form game for the Kuhn Poker.
  Dashed lines connect histories grouped within an information set.
  }
\label{fig:iig_efg_kuhn}
\end{figure}

\second{
The first example illustrates how to capture the simultaneous move game of \rps{} in the inherently sequential model of extensive form games (Figure~\ref{fig:iig_efg_rps}).
While illustrated on \rps{}, it is easy to imagine the general construction for an arbitrary matrix game (Corollary~\ref{col:iig:matrix_game_to_efg}).
The idea is to simply let one player to act first, group all the subsequent histories into a single information set so that the opponent does not know what action the first player made \citep{osborne1994course}.

The second example is then the sequential game of Kuhn poker, illustrated in Figure~\ref{fig:iig_efg_kuhn}.

\begin{corollary}
\label{col:iig:matrix_game_to_efg}
Any matrix game can be converted to an equivalent extensive form game having (up to a constant factor) the same size.
\end{corollary}
}

\subsection{Extensive Form Games to Matrix Games}
\label{sec:igg-converting_formalism}
\second{
We have already seen how to convert any matrix game to an extensive form game (Corollary~\ref{col:iig:matrix_game_to_efg}).
Lemma~\ref{lem:iig-extensive_form_to_normal_form} then shows that the opposite conversion is also possible under the perfect recall (Section~\ref{sec:iig-efg_limitations}) assumption.
Unfortunately, the resulting matrix game can be exponentially larger than the original sequential representation of extensive form games.
See Figure \ref{fig:iig-efg_to_matrix} for an example transformation (taken from \cite{schmid2013game}).

\begin{lemma}
\label{lem:iig-extensive_form_to_normal_form}
\citep{osborne1994course}
Given any two-player extensive form game with perfect recall, it’s possible to create an equivalent normal form game.
\end{lemma}
}

\begin{figure}[ht]
    \centering
    \begin{subfigure}[b]{0.48\textwidth}
        \includegraphics[width=\textwidth]{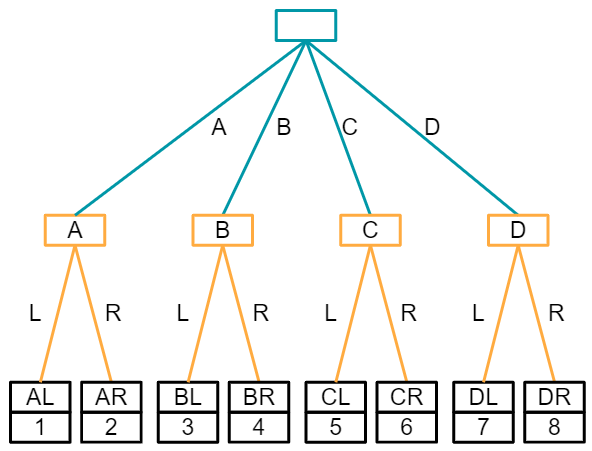}
        \caption{}
        \label{fig:iig-efg_to_matrix_1}
    \end{subfigure}
    \begin{subfigure}[b]{0.48\textwidth}
        \begin{tabular}{lllll}
  &  A & B &  C &  D  \\  
\cline{2-5}
AL-BL-CL-DL& 1 & 3 & 5 & 7\\
AL-BL-CL-DR& 1 & 3 & 5 & 8\\
AL-BL-CR-DL& 1 & 3 & 6 & 7\\
AL-BL-CR-DR& 1 & 3 & 6 & 8\\
AL-BR-CL-DL& 1 & 4 & 5 & 7\\
AL-BR-CL-DR& 1 & 4 & 5 & 8\\
AL-BR-CR-DL& 1 & 4 & 6 & 7\\
AL-BR-CR-DR& 1 & 4 & 6 & 8\\
AR-BL-CL-DL& 2 & 3 & 5 & 7\\
AR-BL-CL-DR& 2 & 3 & 5 & 8\\
AR-BL-CR-DL& 2 & 3 & 6 & 7\\
AR-BL-CR-DR& 2 & 3 & 6 & 8\\
AR-BR-CL-DL& 2 & 4 & 5 & 7\\
AR-BR-CL-DR& 2 & 4 & 5 & 8\\
AR-BR-CR-DL& 2 & 4 & 6 & 7\\
AR-BR-CR-DR& 2 & 4 & 6 & 8\\
\end{tabular}
        \caption{}
        \label{fig:iig-efg_to_matrix_2}
    \end{subfigure}
    \caption{
    (a) Extensive form game game.
    (b) Corresponding matrix game.
    }
    \label{fig:iig-efg_to_matrix}
\end{figure}

\section{Factored-Observation Stochastic Games}
\label{sec:igg-fog}

\second{
\Gls{fog_g} (FOSG) is a recently introduced formalism, and also one of the contributions of the thesis \citep{kovavrik2019rethinking}.
We first discuss the limitations of extensive form games and motivate the new formalism that builds on the notion of observations.
We then formally introduce FOSG and finish with connections between the FOSG and EFG.
}

\subsection{Limitations of EFG Formalism}

\second{
\label{sec:iig-efg_limitations}
The defining distinction of EFGs compared to perfect information game trees is the information partition forming information sets (states).
Information states group histories (world states) that a player cannot tell apart.  
While this allows for a very general class of games, it turns out that grouping arbitrarily states of the acting player is both too general and loses important information.

It is too general as it allows for \gls{imperfect_recall}, which forces the agent to forget previously known information (e.g. past actions taken).
This happens when information states group history with its prefix (i.e. grouping a node with a previous one).
Imperfect recall is troubling as it not only is unrealistic, it has unfortunate complexity consequences.
Many basic concepts including best response become complicated \citep{piccione1996absent, piccione1997interpretation}, and even the existence of Nash equilibria becomes NP-hard \citep{hansen2007finding}.
It is thus common to restrict the set of games to only perfect recall.
Recently, non-timeability has been suggested as another unrealistic property allowed by this model \citep{jakobsen2016timeability}.
The observation is that since the passage of time is observable to agents, some extensive form games can not be implemented.
A game is said to be timeable if it has an exact deterministic timing and $1$-timeable if each label is exactly one higher than its parent’s label  (Definition~\ref{defn:iig-fog_timeability}).
Similarly to perfect recall, authors argue that timeability should be a common assumption.

EFGs lose crucial information inherently present in many environments --- the notion of observations, specifying what and when they were observed by the agents.
Information states represent information available to the acting player, but there is no notion of information state for the non-acting player.
This is problematic as when we are to play, we might need to reason about the states the opponent can be in.
Furthermore, EFGs make no explicit distinction between private and public information.
But the knowledge about what information is available to the non-acting player as well as the factoring of information into private and public are critical concepts for sound search method.
While it is common to recover those concepts in specific cases (e.g. if we are building a poker agent) \citep{burch2014solving}, it turns out that it is impossible to do properly in general \citep{kovavrik2019problems}.

\begin{defn}
\label{defn:iig-fog_timeability}
\citep[Timeability]{jakobsen2016timeability}
For an extensive-form game, a deterministic timing is a labelling of the nodes  with non-negative real numbers such that the label of any node is at
least one higher than the label of its parent. 
A deterministic timing is exact if any two nodes in the same information set have the same label.
\end{defn}
}

\subsection{Augmented EFG}
\second{
Recovering the information lost by the EFG formalism essentially led to the introduction of augmented information sets and public tree \citep{johanson2011accelerating, burch2014solving}.
These additional concepts then result in the augmented extensive form games (Definition~\ref{defn:iig-augmented_extensive_form_game}).

\begin{defn}
\label{defn:iig-augmented_extensive_form_game}
\citep[Definition 3.5]{kovavrik2019rethinking}
Augmented extensive form game is a tuple $(\mathcal{H}, \mathcal{Z}, \mathcal{A}, \mathcal N, p, \pi_c,  u, \mathcal{I})$ where all objects are as in extensive form game except for the information set partitioning $\mathcal{I}$:
\begin{itemize}
	\item $\mathcal{I} = (\mathcal{I}_1, \mathcal{I}_2, \mathcal{I}_{pub})$ is a collection of partitions of $\mathcal{H}$ where each $I_i$ is a refinement of $\mathcal{I}_{pub}$
\end{itemize}

\end{defn}

Note that there are two distinctions to the standard information sets.
First, the information sets partition $\mathcal{H}$ rather than just $\mathcal{H}_i$ and information sets are thus defined even when the player is not acting.
Second, $\mathcal{I}_{pub}$ defines public states and a corresponding public tree where the individual player's partitions are a refinement of this public partitioning.

But this recovery is only possible in specific cases, using game-specific concepts (e.g. dealing cards in poker).
Furthermore, on top of the usual restrictions of perfect recall and timeability, additional restriction is then assumed.
Augmented EFG is said to not have thick public sets if no element of $\mathcal{I}_{pub}$ (and hence of $\mathcal{I}_i$)
contains both some $h$ and $h'$ for $h' \sqsubset h$.
This additional technicality is essentially the perfect recall analog for augmented information sets.
}

\subsection{FOSG Definition}
\second{
Rather than using complex combination of restrictions and augmentations of extensive form games, we argue that a simple observation-based model naturally describes the domain and preserves the necessary information.
Motivated by the issues of extensive form games (Section \ref{sec:iig-efg_limitations}), \gls{fog_g} is a recently proposed formalism of multi-agent imperfect information environments \citep{kovavrik2019rethinking}.

The formalism is a variant of partially observable stochastic games \citep{hansen2004dynamic}, where the game consists of underlying world states, and the probabilistic transition to a next world state is a function of the actions taken by all the agents.
As the game moves from one world state to another, players do not get to directly observe the underlying state.
Rather, they receive an observation that is factored to a private observation and a public observation.
The notions of transitions, world states and observations also make this formalism more familiar to the the reinforcement learning community.
}

\second{
\begin{defn}
\label{defn:iig-factored_observations_games}
\citep{kovavrik2019rethinking}
Factored Observation Stochastic Game is a tuple $G = ( \mathcal N, \mathcal W, w^o, \mathcal A, \mathcal T, \mathcal R, \mathcal O, \mathbb O )$

\begin{itemize}
    \item $\mathcal N = \{1,\dots,N\}$ is the player set.
    \item $\mathcal W$ is the set of world states and $w^0 \in \mathcal W$ is a designated initial state.
    \item $\mathcal A = \mathcal A_1 \times \dots \times \mathcal A_N$ is the space of joint actions.
    \begin{itemize}
        \item The subsets $\mathcal A_i(w) \subset \mathcal A_i$ and $\mathcal A(w) = \mathcal A_1(w) \times \dots \times \mathcal A_N(w)$ specify the (joint) actions legal at $w$.
        \item For $a\in \mathcal A$, we write $a=(a_1,\dots,a_N)$.
        \item $\mathcal A_i(w)$ for $i \in \mathcal N$ are either all non-empty or all empty. A world state with no legal actions is terminal.
    \end{itemize}
    \item After taking a (legal) joint action $a$ at $w$, the transition function $\mathcal T$ determines the next $w' \sim \mathcal T(w,a) \in \Delta (\mathcal W)$.
    \item We write $R = (R_1,\dots, R_N)$, where $ R_i(w,a)$ is the reward $i$ receives when a joint action $a$ is taken at $w$.
    \item $\mathcal O = (\mathcal O_{\textnormal{priv}(1)},\dots,\mathcal O_{\textnormal{priv}(N)}, \mathcal O_{\textnormal{pub}} )$ is the observation function, and  $\mathbb O = (\mathbb O_{priv(1)}, \mathbb O_{priv(2)}, \mathbb O_{pub})$ are the observation sets.
    The observation function $\mathcal O_{(\cdot)} : \mathcal W \times \mathcal A \times \mathcal W \to \mathbb O_{(\cdot)} $ specifies the private observation that $i$ receives, resp. the public observation that everybody receives, upon transitioning from world state $w$ to $w'$ via some $a$.
    \begin{itemize}
        \item For each $i$, we write $\mathcal O_i(w,a,w') = \left( \mathcal O_{\textnormal{priv}(i)}(w,a,w'), {\mathcal O}_{\textnormal{pub}} (w,a,w') \right) \in \mathbb O_i := \mathbb O_{\textnormal{priv}(i)} \times \mathbb O_{\textnormal{pub}}$.
        \item We assume that at the start of the game, each player receives some $\mathcal O_i(w^0)$.
    \end{itemize}
\end{itemize}
\end{defn}
}

\second{
The game then starts in the initial state $w^0$ and follows in turns.
In each turn, each player $i$ player select an action $a_i \in \mathcal A_i(w)$, resulting in a joint action $a=(a_i)_{i \in \mathcal{N}}$.
The game then transitions to a new state $w'~\mathcal T(w,a)$.
This transition generates an observation $\mathcal O(w,a,w')$, from which each player receives $\mathcal O_i(w,a,w') = ( \mathcal O_{\textnormal{priv}(i)}(w,a,w'), {\mathcal O}_{\textnormal{pub}} (w,a,w') )$ (i.e., the public observation together with their private observation).
Finally, each player is assigned the reward $\mathcal R_i(w,a)$.
This process repeats until a terminal state is reached.
}

\first{
As an example, consider the \rps{}.
The set of world states is $\mathcal{W} = \{0, r, p, s, rr, rp, rs, pr, pp, ps, sr, sp, ss \}$, with the initial state $w^0 = 0$.
As only a single player acts at each time step, the non-acting player has a single dummy action $\mathcal{A}_2(r) = \{r, p, s \}, \mathcal{A}_1(r) = \{-\}$.
The transition function is fully deterministic, e.g. $\mathcal{T}(r, (r,-)) = rr$.
Finally, the observation sets in \rps{} are empty.
In poker games, the private observations would encode the private cards, while the public observations provide information about the public (board) cards as well as about the actions taken (fold/call/bet) as the actions Texas hold'em poker are publicly observable.
}

\subsection{Properties}
\second{
As the agents receive observation during the game-play, the full action observation history is necessarily perfect recall.
The formalism does not force agent into imperfect recall (and there are extensions allowing for imperfect recall).

As all the agents act all the time, there is no issue with ill-defined information (states) of the non-acting agents.
This has additional benefit, as this implies a notion of ``tick/time step'' inferred by the number of actions taken, and thus non-timeable games are not possible.

Finally, the factorization to private information and public information allows for public states and public trees.
}

\subsection{Connection to EFG}
\first{
There are clear benefits of the observation based model of FOSGs --- it is relatively simple, encodes the information otherwise lost by EFG, and it contains only timeable perfect-recall games.
We now show that there is no loss of expressive power as any perfect recall timeable EFG can be represented.

First, observe that any FOSG corresponds to some augmented extensive form representation through the $\texttt{FOGToAugEFG}$ transformation (the formal process can be found in \citep{kovavrik2019rethinking}).
Next transformation $\texttt{AugEFGToEFG}$ simply forgets the $\mathcal{I}_{pub}$ and restricts $\mathcal{I}_{i}$ to $\mathcal{H}_{i}$.
Finally, we define $\texttt{FOGtoEFG} = \texttt{AugEFGToEFG} \circ \texttt{FOGToAugEFG}$.
With these transformations in hand, we can state the crucial connections.

\begin{thm}
\citep[Theorem 3.6]{kovavrik2019rethinking}
(FOSGs as augmented EFGs)
Every FOSG G corresponds to an augmented EFG E = \texttt{FOGToAEFG}(G) with perfect-recall and no thick
public sets.
Moreover, any perfect-recall augmented EFG with no thick public sets can be obtained this way.
\end{thm}

\begin{thm}
\citep[Theorem 3.10]{kovavrik2019rethinking}
(FOSGs are timeable EFGs)
For any FOSG G, the game E = \texttt{AugEFGToEFG} $\circ$ \texttt{FOGToAugEFG} (G) is a $1$-timeable perfect-recall EFG. 
Moreover, any $1$-timeable perfect-recall EFG can be obtained this way.
\end{thm}

Finally, it is possible to show that $1$-timeable EFG is no less general than any timeable game (Lemma~\ref{lemma:iig-fog_1timeable_quadratic}).

\begin{lemma}
\label{lemma:iig-fog_1timeable_quadratic}
\citep[Lemma 3.9]{kovavrik2019rethinking}
Any classical timeable EFG can be made $1$-timeable by adding chance nodes with a single noop action, where the resulting $1$-timeable game is no more than quadratically larger.
\end{lemma}

}

\subsection{Derived Objects}

\second{
We can now readily derive the trees of histories, cumulative rewards and information states, similar to the objects used in the extensive form games literature. 
}

\second{
\begin{itemize}
    \item History (trajectory) is a finite sequence $h = (w^0, a^0, w^1, a^1, \dots , w^t)$, where $w^i \in \mathcal W$, $a^i \in \mathcal A(w^i)$, and $w^{i+1} \in \mathcal W$ is in the support\footnote{For finite $\mathcal W$, being in support of some probability measure $\mu\in \Delta(W)$ is equivalent to having a non-zero probability under $\mu$.} of $\mathcal T(w^i, a^i)$.
    We denote the set of all legal histories by $\mathcal H$\footnote{To simplify the discussion, we assume that $\mathcal H$ is finite. This can always be enforced by using finite horizon $T$, but some domains satisfy this assumption naturally. Moreover, our results generalize into the standard $\gamma$-discounted rewards setting.}, and the set of terminal histories as $\mathcal{Z}$.
    \begin{itemize}
        \item Since the last state in each $h\in \mathcal H$ is uniquely defined, the notation for $\mathcal W$ can be overloaded to work with $\mathcal H$ (for example $\mathcal A(h) = \mathcal A(w^t)$).
        \item We use $g \sqsubset h$ to denote the fact that $g$ is a prefix of $h$, and $g \sqsubseteq h$ to denote $g$ is a prefix of $h$ or equal to $h$.
    \end{itemize}
    \item The cumulative reward (return or utility) of $i$ at $h$ is the reward accumulated up to this point $R_i(h) = \sum_{i=0}^t R_i(w^i, a^i)$.
\end{itemize}

The combination of private and public observations induce the corresponding information states and trees, and we then use these to define the notion of a strategy.

\begin{itemize}
    \item Player $i$'s action-observation history at $h$ is the sequence of the observations visible to that player (private and public) and the actions taken: $s_i(h) = (o^0_i, a^0_i, o^1_i, a^1_i, \dots, o^t_i)$, where $o^k_i = \mathcal O_i(w^{k-1}, a^{k-1}, w^k)$.
    The space $\mathcal S_i$ of all such sequences can be viewed as the information state space of $i$.
        \begin{itemize}
            \item We use $\mathcal{H}(s)$ to denote the set of compatible histories: \\ $\mathcal{H}(s) = \{ h \in \mathcal{H} | s_i(h) = s\} $
            \item We assume that each $s_i \in \mathcal S_i$ determines which $i$'s actions $\mathcal A_i(s_i)$ are legal: $(\forall h\in \mathcal{H}_i(s_i)): \mathcal A_i(h) = \mathcal A_i(s_i)$
            \item We also use $sao$ to denote a state after taking action $a$ in $s$ and observing $o$.
        \end{itemize}
    \item A policy $\pi_i : s_i \in \mathcal S_i \mapsto \pi_i(s_i) \in \Delta(\mathcal A_i(s_i))$ is a mapping from an information state to the probability distribution over the actions in that state, and $\Pi_i$ is the set of all possible policies.
\end{itemize}

Finally, the factorization of each observation into the private and public parts allows us to define the tree of public states.

\begin{itemize}
    \item The public state corresponding to $h = (w^0, a^0, w^1, a^1, \dots , w^t)$ is the sequence $s_{\textnormal{pub}}(h) := s_{\textnormal{pub}} ( s_i(h) ) := \left( O^0_{\textnormal{pub}}, O^1_{\textnormal{pub}}, \dots, O^{t}_{\textnormal{pub}} \right)$ of all public observations corresponding to $h$.
    \begin{itemize}
        \item The space $\mathcal S_{\textnormal{pub}}$ of all public states is called the public tree.
    \end{itemize}
\end{itemize}
}

\first{
It is often useful to model the stochasticity in the game as being caused by a chance player $c$ rather than via the stochastic transition function $\mathcal{T}$.
The chance player receives observations and takes actions just like the other players, but its policy is fixed. 
The addition of $c$ causes no loss of generality since their actions can always be merged back into the transition function $\mathcal{T}$, but this viewpoint can be in some cases conceptually and formally more convenient.
Finally, we also use $\Delta_R$ to denote the delta of maximim/minimum utilities in the game: $\Delta_R = \max_{z \in \mathcal{Z}}R_i(z) - \min_{z \in \mathcal{Z}}R_i(z)$.
}

\subsection{State Notation and Views}
\label{sec:iig-fog_examples}
\first{

Factored observation games naturally allow for different observer views.
Each world state contains i) player $1$ private information ii) player $2$ private information iii) public information.
Different agents then observe different subsets of that available information, corresponding 

\begin{itemize}
    \item Game state view: i) player $1$ private information ii) player $2$ private information iii) public information
    \item Player $1$ view: i) player $1$ private information iii) public information
    \item Player $2$ view: ii) player $2$ private information iii) public information)
    \item Public observer view: iii) public information
\end{itemize}

Thanks to the perfect recall property, all these views imply a well-formed trees about how the game progresses from the perspective of the respective viewers.
Furthermore, this allows for a convenient notation for identifying states.
We will use $[O_{pub}|O_1|O_2]$, $[O_{pub}|O_1|]$, $[O_{pub}||O_2]$, $[O_{pub}||]$ to identify a concrete world-state, player $1$ state, player $2$ state and public state respectively (see Table~\ref{tab:iig-fog_player_views}).

\begin{table*}\centering
\renewcommand{\arraystretch}{1.3}
\begin{tabular}{@{}llllll@{}}\toprule
 & Notation & $O_{pub}$  & $ O_{1}$ & $ O_{2}$ & Kuhn Example \\  \midrule
$h \in \mathcal{H}$ & $[O_{pub}|O_{1}|O_{2}]$ & \checkmark & \checkmark & \checkmark & $[b|Q|J]$ \\
$s \in \mathcal{S}_1$ & $[O_{pub}|O_{1}|]$ & \checkmark & \checkmark &           & $[b|Q|]$ \\
$s \in \mathcal{S}_2$ & $[O_{pub}||O_{2}]$ & \checkmark &            & \checkmark & $[b||J]$ \\
$s_{pub} \in \mathcal{S}_{pub}$ & $[O_{pub}||]$ & \checkmark     &            &            & $[b||]$  \\
\bottomrule
\end{tabular}
\caption{
Different states/views are defined by different subsets of the information/observation.
Kuhn example shows a state where the first player is dealt (Q)ueen, the second player is dealt (J)ack and the first player made a publicly observable action (b)et (see also Figure~\ref{fig:iig_iig_kuhn}).
}
\label{tab:iig-fog_player_views}
\end{table*}

}

\subsection{Example}
\second{
Figure \ref{fig:iig_iig_kuhn} shows the different views for the game of Kuhn poker.
See that as the game view, player's view and the public view contain decreasingly less information, we are ``zooming in and out'' of the respective states.
Table~\ref{tab:iig-fog_player_views} then shows state notation for a concrete Kuhn state  and Table~\ref{tab:iig-game_size} reports number of world states, infostates and public states for the imperfect information games used in this thesis.

\begin{figure}[ht]
  \centering
  \includegraphics[width=1.0\textwidth]{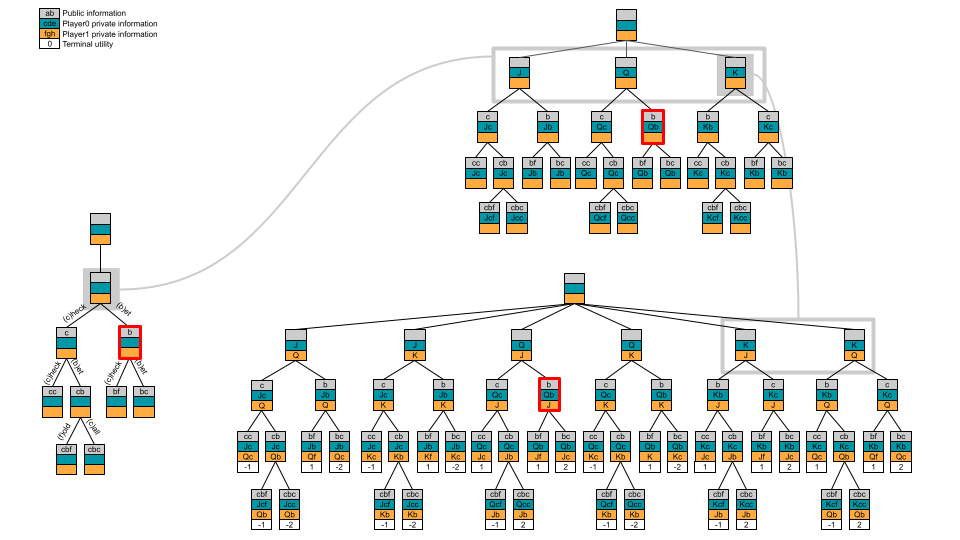}
  \caption{Different levels of information for the Kuhn poker. 
  Top left: Public tree depicting the publicly observable information (i.e. the public view). 
  Top right: Action-observation sequence tree of player 1 (i.e., player $1$ view). 
  Bottom: The world-state tree of the game (i.e. the game view).
  Red: Example states used in Table~\ref{tab:iig-fog_player_views}.
  }
\label{fig:iig_iig_kuhn}
\end{figure}
}

\begin{table*}
\centering
\renewcommand{\arraystretch}{1.3}
\begin{tabular}{@{}llll@{}}\toprule
  & Histories &  Infostates &  Public States \\  \midrule
Kuhn poker & $64$ & $20$ & $12$ \\
Leduc poker & $9,487$ & $1,398$ & $468$ \\
Glasses & $53,907$ & $11,699$ & $6,559$ \\
Limit Texas Hold'em & $3 \cdot 10^{17}$ & $3 \cdot 10^{14}$  & $3 \cdot 10^{11}$ \\
No-limit Texas Hold'em & $6 \cdot 10^{164}$ & $6 \cdot 10^{161}$  & $6 \cdot 10^{158}$ \\
\bottomrule
\end{tabular}
\caption{
The no-limit poker variant is the one used at the Annual Computer Poker Competition (ACPC) since $2010$ and its size was computed in \citep{johanson2013measuring}.
}
\label{tab:iig-game_size}
\end{table*}

\section{History of Formalisms}
\first{
As discussed in Section \ref{intro:rl_vs_gt}, there are two main communities that study sequential decision-making in multi-agent environments --- game theory and reinforcement learning.
We include a brief history of the formalisms used by the communities.

Extensive Form Games are a natural extension of game trees to imperfect information environments, dating back to 1953 \citep{morgenstern1953theory}.
Factored Observation Stochastic Games are a modern formalism that reflects the lessons learned from working with extensive form games as well as having search methods in mind.
The formalism is based on the stochastic games, and aims to bring the communities closer as the next research frontiers require insights and collaboration of both communities.

Reinforcement learning formalisms can be traced back to Markov decision processes (MDP) \citep{bellman1957markovian} and stochastic games \citep{shapley1953stochastic}. 
While the MDP formalism describes single-agent perfect information domains, researchers soon generalized MDPs into partially observable POMDPs \citep{astrom1965optimal}. 
\citet{kaelbling1998planning} then used POMDPs to present novel algorithms for planning and acting, using value functions and sub-problems as crucial concepts.
While stochastic games (SGs) were introduced prior to MDPs, they can be well thought of as ``multiagent MDPs'' (SG is a FOSG with empty observation sets $\mathbb{O}_{priv(i)} = \mathbb{O}_{pub} = \emptyset$).
Indeed, this connection allowed \citep{littman1994markov} to generalize the Q-learning algorithm from MDPs to stochastic games.
When it comes to combining multiple agents and imperfect information, researchers usually use partially observable stochastic games \citep{POSGs2004hansen} (FOSG with no notion of public observations $\mathbb{O}_{pub} = \emptyset)$) and Decentralized POMDPs (POMDP where the agents' goal is shared: $R_i=R_j \, \forall i,j \in \mathcal{N}$) \citep{bernstein2002complexity}.
}

\chapter{Offline Policies}
\label{chap:iig-policies}

\second{
As the imperfect information games formalisms are strictly more general then their perfect information counterparts, we will revisit the concepts of optimal policies introduced in
Chapter \ref{chap:pig-optimal_policies} and see if and how these concepts transfer to imperfect information.

Do the optimality results from perfect information settings hold in imperfect information?
As a reminder, the settings we are dealing with are two-player, zero sum, perfect recall and finite-horizon games.

\begin{itemize}
    \item Is a \maximin{} policy still guaranteed to exist?
    \item Is a Nash equilibrium policy still guaranteed to exist?
	\item Do \maximin{} and Nash equilibria collapse?
	\item Are there still the same worst-case garantees (as in Corollary \ref{cor:pig-both_sides_guaranteed_value})?
	\item Are the optimal policies still deterministic?
\end{itemize}
}

\section{Best Response}
\label{sec:iig-best_response}
\second{
The definition of best response in imperfect information games remains unchanged the perfect information (Section \ref{sec:pig-best_response}).

\restdefbestresponse*

And just as in perfect information, best response can always be a deterministic policy (Lemma \ref{lem:iig-deterministic_br}).
This follows from the fact that fixing the opponent turns the game into a perfect information single-agent environment \citep{bowling2003multiagent}.
To compute the best response, we use the fixed policy of the opponent to compute the transition probabilities of the agent's infoset tree.
Using those transition probabilities, we then simply traverse the tree bottom-up and greedily, selecting the action with the highest expected utility (Algorithm \ref{alg:iig_best_response}).

\begin{lemma}
\label{lem:iig-deterministic_br}
There is always a deterministic best response.
\end{lemma}
}
\subsection{Distinction to Perfect Information}
\label{sec:iig-perfect_info_distinction}
\second{
There is a very important distinction to perfect information in the way the best response is calculated.
Previously, we needed to consider the opponent's strategy only in all the future states.
This is because this provided a sufficient information about the transition probabilities between these future states.
There was no need to consider how the opponent plays prior to a state and thus a truly bottom-up pass was possible in a perfect information best response (Algorithm \ref{alg:pig_best_response}).

\begin{figure}[ht]
  \centering
  \includegraphics[width=6cm]{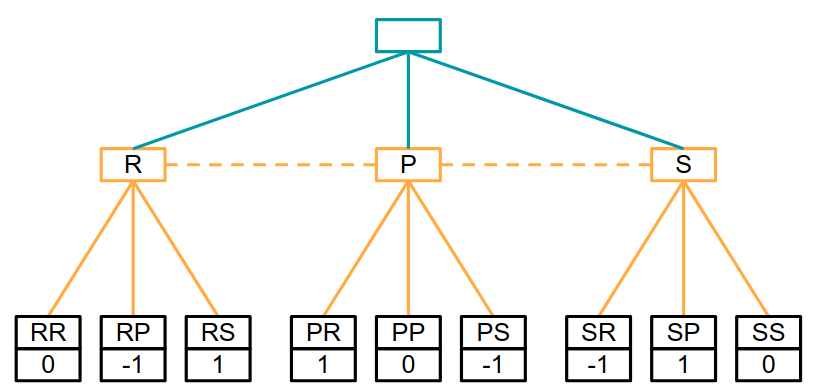}
  \caption{Best response computation for the second player in rock-paper-scissors.
  Single infostate of the player consists of three different histories, and the utility of the policy depends on distribution over these histories.
  That in turn depends on the strategy of the opponent in previous states.
  }
\label{fig:iig_best_response_distribution}
\end{figure}

This is not the case in imperfect information.
The issue is that the state-action-state transition dynamic depends on the distribution over the individual histories grouped within the infostate.
That distribution in turn depends on the past behavior of the opponent.
Consider rock-paper-scissors from the perspective of the second player (i.e. computing best response against a fixed strategy of the first player).
To compute the probability of reaching the terminal state with utility $1$ after the action rock, we need the distribution over the three possible histories grouped in the infostate.
And these three histories correspond to the first player taking rock, paper or scissors (Figure \ref{fig:iig_best_response_distribution}).

See that the Algorithm \ref{alg:iig_best_response} indeed first needs to compute the reach probabilities for all the histories, propagating them down the tree.
We will see important consequences of this property on the imperfect information \subgames{} in Section \ref{chap:iig_subgames}.
}

\begin{algorithm}
\caption{Imperfect Information Game Best Response}
\label{alg:iig_best_response}
\begin{algorithmic}[1]

\Function{ComputeReach}{$s \in S_i, i \in \mathcal{N}$ }
  \For{$h : \mathcal{S}_i(h)$}
    \For{$h' = (haw)$ : $\mathcal{H}$}
      \State \Comment{haw corresponds to some sas'}
      \State h\_reach\_prob[$h'$] = h\_reach\_prob[$h$] $\cdot \mathcal{T}(h, a, h') \cdot \pi_{-i}(a)$
    \EndFor
  \EndFor
  \For{$sao' : \mathcal{S}_i$}
    \State s\_reach\_prob[$sao'$] = $\sum_{h \in \mathcal{H}(sao')}$h\_reach\_prob[$h$]
    \State ComputeReach($sas'$, i)
  \EndFor
\EndFunction
\State \Comment{Return best response value of the state $s$.}
\Function{BestResponseDFS}{$s \in \mathcal{S}_i, i \in \mathcal{N}$}
  \State \Comment{Terminal state, return terminal utility.}
  \If{$\mathcal{A}_i(s) = \emptyset$}
    \State \Return ${R_i(s)}$
  \EndIf
  \State \Comment{Action value is the reach-weighted sum of possible child states}.
  \For{$a : \mathcal{A}(s)$}
    \State sa\_values[$sa$] = $\sum_{sao \in \mathcal{S}_i}$ BestResponseDFS($sao$) $\cdot$ s\_reach\_prob[$sao$]
  \EndFor
\State
\State \Comment{Choosing an action with the best value}
\State $\pi(s) = argmax_{sa}$ sa\_values[sa]
\State \Return $max_{sa}$ sa\_values[sa] 
\EndFunction
\State
\Function{BestResponse}{$s \in \mathcal{S}_i, i \in \mathcal{N}$ }
\State s\_reach\_prob[$s$] = 1
\For{$h: \mathcal{H}(s)$}
  \State h\_reach\_prob[$h$] = 1
\EndFor
\State \Call{ComputeReach}{$s$, $i$}
\State \Call{BestResponseDFS}{$s$, $i$}
\EndFunction
\end{algorithmic}
\end{algorithm}

\section{\Maximin{}}
\second{
Same as in perfect information, the \maximin{} strategy maximizes the agent's value against the worst case opponent.

\restdefminmaxpolicy*
}

\section{Nash Equilibrium}
\second{
The concept and definition of Nash equilibria also remains unchanged.
}

\restdefnash*

\section{Nash Equilibrium vs \Maximin{}}
\second{
There was an important connection in perfect information between the solution concepts of Nash equilibria and \maximin{}.
While these concepts are at a first glance very different, in the case of zero-sum two player games, they collapsed (Theorem \ref{thm:pig-minmax_is_nash}).

This connection crucially relied on the two-player zero-sum property and imperfect information brings no change to these results.
First, the minmax theorem (Theorem \ref{thm:pig-minmax}) still holds as the original referenced proof is for matrix games.
Proofs of Theorem \ref{thm:pig-nash_implies_minmax} and Theorem \ref{thm:pig-minmax_implies_nash} then remain the same as there were no assumptions specific to perfect information.

}





\section{Difference To Perfect Information}
\label{sec:iig-suboptimal_vs_opimal}
\second{
Imperfect information brought little to no change to the concepts of best response, \maximin{} and Nash equilibrium.
The same concepts, motivations and properties of the optimal policies hold in imperfect information.
Optimal policies are still guaranteed to exist and provide the same appealing guarantees (e.g. it guarantees a positive expected reward against any opponent if we get to play both positions, Corollary \ref{cor:pig-both_sides_guaranteed_value})
The reasons to follow optimal policies in imperfect information are thus as appealing as in perfect information.
We just need more powerful algorithms to compute the policies.
This is true for offline algorithms, but even more so for online search algorithms.
}

\subsection{Deterministic or Stochastic Policies}
\label{sec:iig-stochastic_policies}
\second{
One of the reasons why computing an optimal policy is more challenging in imperfect information is the need for careful randomization.
In perfect information, an optimal policy can always be deterministic.
But consider the imperfect information game \rps{}.
It is easy to verify that any deterministic policy receives a worst-case utility of $-1$.
On the other hand, the best response value of the uniform policy is $0$.
Deterministic policies thus might not be sufficient to play optimally in imperfect information settings (Corollary \ref{cor:iig-minmax_stochastic}).

\begin{corollary}
\label{cor:iig-minmax_stochastic}
An optimal policy in imperfect information games might have to be stochastic.
\end{corollary}
}

\subsubsection{Level of Uncertainty and the Support Size}
\second{
There is a rather interesting connection between the level of uncertainty and the support size (Definition~\ref{defn:pig-support}).
In perfect information games, there is only a single state of the opponent consistent with the player's state.
In other words, there is no uncertainty about the opponent, the opponent can only be in one state and thus there's always an optimal strategy with support size one.

In imperfect information, there is uncertainty about the state of the opponent; there are multiple states the opponent can be in given a player's state.
It turns out there is direct connection between the level of  uncertainty and the number of actions required under an optimal policy \citep{schmid2014bounding}.
}

\section{Multiplayer Games}
\label{sec:iig-multiplayer_games}
\second{
While Section \ref{sec:pig-multiplayer_games} already discussed that leaving the comfort of two-player zero sum settings greatly complicates matters, Section \ref{sec:pig-unique_value} deferred an example of the \gls{equilibrium_selection_problem}.
This is because there is a particularly nice matrix game that illustrates this issue --- the Game of Chicken \citep{rapoport1966game}.
It has been inspired by another famous game ---- Prisoner's dilemma \citep{rapoport1965prisoner}, and is now often re-branded as the ``crossroad game'' where the both players are allowed to either (S)top at the crossroad or (G)o.
If both players stop, they both receive zero reward.
If one of the players stop and the other goes, the cautious player receive a zero reward while the daring player receives positive reward.
But if both players go at the same time, they crash and both receive a large negative reward.
See Table \ref{tab:iig-game_chicken} for the corresponding matrix game.
}


\begin{table}[h!]
\centering
\begin{tabular}{@{}lllllll@{}}
& (S)top & (G)o & \hspace{0.1\textwidth} & & (S)top & (G)o  \\ 
\cline{2-3} \cline{6-7}
(S)top & 0 & 0 && (S)top & 0 & 1  \\
(G)o & 1 & -100 && (G)o & 0 & -100  \\
\end{tabular}
\caption{Chicken game. 
Left: Player $1$ utility. 
Right: Player $2$ utility.
}
\label{tab:iig-game_chicken}
\end{table}

\second{
There are three Nash equilibria in this game:

\begin{enumerate}
    \item Row player S, column player G
    \item Row player G, column player S
    \item Both players share the same strategy: $p(S) = \frac{100}{101}, p(G) = \frac{1}{101}$
\end{enumerate}

The equilibrium selection problem is that the column player can select the first equilibrium, while the row player selects the second one.
In this case, both players select the (G)o action and crash, which violates reward guarantee, and is not a Nash equilibrium.
}
\chapter{\Subgames{}}
\label{chap:iig_subgames}
\second{
This chapter generalizes sub-games in imperfect information settings.
Just as in perfect information, a \subgame{} is going to be a well-defined game, a sub-problem of the original one based on the notion of sub-tree.

But while a \subgame{} in perfect information was simply defined by a sub-tree rooted in the current state (Section \ref{sec:pig-subgame}), things are more complicated in imperfect information.
As the players have different views of the game, the possible states of the players differs and so do their respective (infostate) trees.
There are two conceptual generalizations for \subgames{} in imperfect information games.

\begin{itemize}
    \item Set of states we need to reason about.
    \item Distribution over the possible initial states.
\end{itemize}
}

\section{Set of States to Reason About}
\first{
Given a player's state, what is the set of states we need to reason about?
We must consider all the states the other player can potentially be in.
In imperfect information games, there might be multiple such consistent states consistent with the current observation.
And to reason about any of these states, the opponent has to now consider all the possible states the first player could be in.
This recursive reasoning continues until we find the set of all relevant states for the current situation.
}

\subsection{Consistent States}
\second{
To formalize the set of states we need to reason about, we first define \gls{consistent_states} --- the set of (opponent's) states consistent with the current state (Definition \ref{def:igg-consistent_states}).
The corresponding operation $\consistent : \mathcal{S} \rightarrow 2^{\mathcal{S}}$ is then referred to as the \gls{consistency_operation}.

\begin{defn}[Consistent States]
\label{def:igg-consistent_states}
Given a state $s \in \mathcal{S}_i$, the set of (opponent's) consistent states $\consistent(s)$ is:
\begin{align}
\label{eq:igg-consistent_states}
    \consistent(s) = \{ s' \in \mathcal{S}_{-i} \, | \, \mathcal H(s) \cup \mathcal H(s') \neq \emptyset \}
\end{align}
\end{defn}
}

\subsection{Common Information Set}
\second{
The \gls{consistency_operation} represents a single step of the recursive process.
We need to repeat the operation until we find a closure --- set of states closed under the \gls{consistency_operation}.
This particular closure is referred to as the \gls{common_info_set} (Definition \ref{def:igg-common_info_set}).
The \gls{common_info_set} represents the minimal set of relevant states.

\begin{defn}[Common Information States]
\label{def:igg-common_info_set}
Given a state $s$, the common information set is the closure under the consistency operation: $cl_{\consistent}(s)$.
\end{defn}
}

\subsection{Common vs Public Information}
\second{
While common information defines the minimal set of relevant states, there is an easier way to find a set closed under the consistency operation without the need for the recursive construction process (closure).
We just need to consider all the states sharing public information.

As the common information set is a subset of the public state set: $cl_{\consistent}(s) \subseteq \mathcal{S}_{pub}(s)$, the set of states sharing a public information is closed under the consistency operation $\consistent$.
The only downside of the public state is that it can potentially be strictly larger.
In many games though (e.g. in all example games), $cl_{\consistent}(s)$ and $\mathcal{S}_{pub}(s)$ are exactly the same and thus the notion of public states is preferable due to its simplicity.

}

\subsection{Future States}
\second{
Introduced sets described the necessary states relevant for the current situation.
In a \subgame{}, we must reason about the current states as well as all the reachable future states.
Given state $s$, we need to reason about $\mathcal{S}_{pub}(s)$ and all the future states of this set $\{f | s' \sqsubset f , s' \in \mathcal{S}_{pub}(s) \}$.
But do we need to repeat the construction of consistent states for these future states?
Given a future state $f$, are there any more relevant states that are not already included?
Theorem \ref{thm:iig-consistent_future_states} guarantees that this is not the case.

\begin{thm}
\label{thm:iig-consistent_future_states}
For any future states $f$ of a state $s$:
\begin{equation*}
 s \sqsubset f, f' \in \mathcal{S}_{pub}(f)  \implies \exists s' \in  \mathcal{S}_{pub}(s) : s' \sqsubset f'
\end{equation*}
\end{thm}
\begin{proof}
Follows from the perfect recall property --- there is a unique trajectory from a root of the game to each infostate (infostates form a directed tree).
\end{proof}

All the states we need to reason about are thus the states in a public state and all its children --- a public sub-tree.
}

\subsection{Public Subtree}
\second{
We have now finished the first essential piece of the subgame generalization.
In perfect information, the set of relevant states was simply the sub-tree rooted in the current state.
In imperfect information, we need to use the public state sub-tree.

Note that in the case where common information would be preferable to public information (the corresponding set being substantially smaller), all the results apply to the \subgames{} formed by common information.
}

\subsection{Walk-through Example}
\second{
To get some more intuition about the notions of consistency, common and public information sets as well as the sub-trees, we will illustrate these concepts on a simple poker game with asymmetrical dealing of cards presented in the Figure~\ref{fig:iig_common_info1}.
}

\begin{figure}[ht]
  \centering
  \includegraphics[width=14cm]{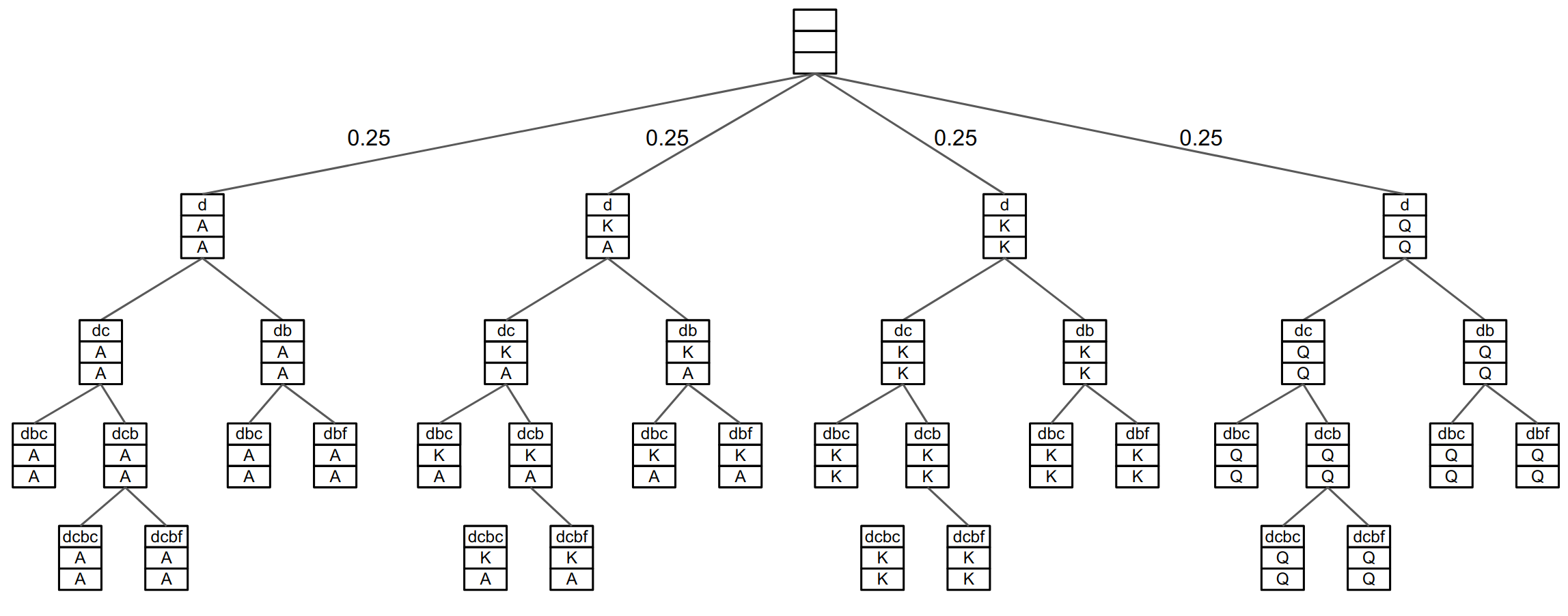}
  \caption{
    Simple poker game with a three card deck \{Q, K A\}, but the possible deals for the first and second player are \{(A, A), (K, A), (K, K), (Q, Q)\}.
    This asymmetrical deal distribution helps to illustrate the details of common and public information sets.
  }
\label{fig:iig_common_info1}
\end{figure}

\subsubsection{Common Information States}
\second{
We start with the concept of consistent states and the corresponding common information set, and we will consider two different initial states: the first player being dealt Queen or Ace.
}

\paragraph{First Player Dealt Queen}
\second{
The state we start our recursive construction is $s = [d|Q|]$ - see also Figure \ref{fig-iig_subgame_game_23}.
}

\begin{enumerate}
    \item Common information states start with $cl_{\consistent}^0(s) = \{[d|Q|]\}$.
    \item $\consistent([d|Q|]) = \{[d||Q]\}$, expanding the set $cl_{\consistent}^1(s) = \{[d|Q|], [d||Q]\}$.
    \item $\consistent([d||Q]) = \{[d|Q|]\}$, ending our recurrent reasoning.
    \item $cl_{\consistent}(s) = \{[d||Q], [d|Q|]\}$
\end{enumerate}

\begin{figure}
\begin{subfigure}{1.0\textwidth}
  \centering
  \includegraphics[width=14cm]{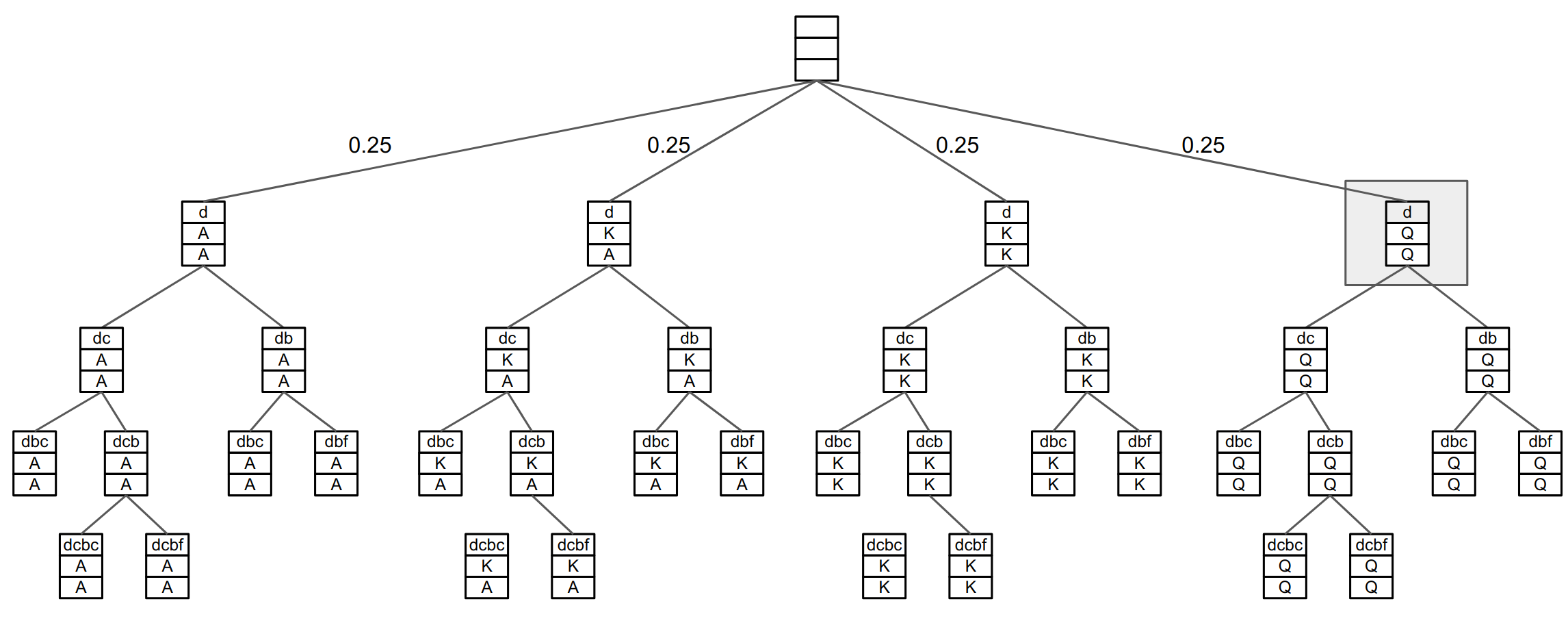}
  \caption{Infoset $[d|Q]$ of the first player.}
  \label{fig-iig_subgame_game_2}
\end{subfigure}
\newline
\begin{subfigure}{1.0\textwidth}
  \centering
  \includegraphics[width=14cm]{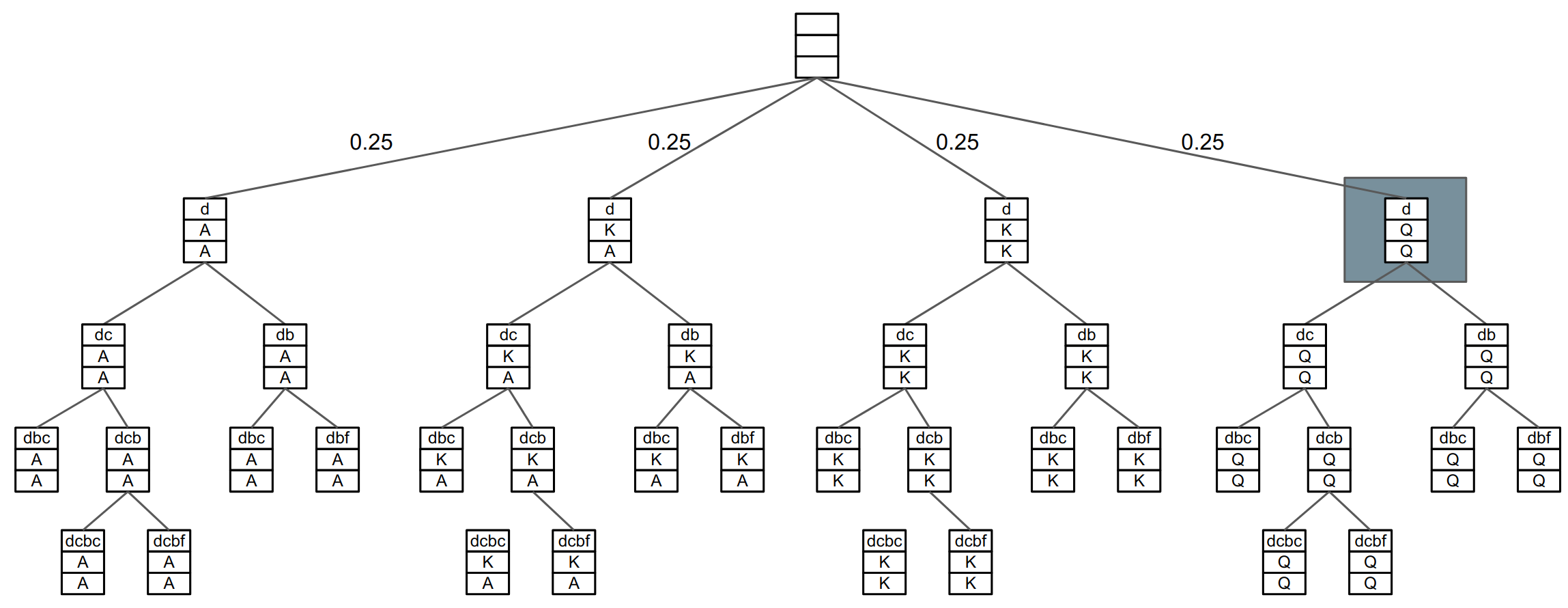}
  \caption{Infoset $[d|Q]$ of the second player.}
  \label{fig-iig_subgame_game_3}
\end{subfigure}
\caption{(a) When reasoning about the player's state $[d|Q]$, we need to reason about all the possible current and future states of both players - and continue this reasoning process until we get a closure.
(b) First level of such recursive reasoning.}
\label{fig-iig_subgame_game_23}
\end{figure}

\paragraph{First Player Dealt Ace}
\second{
The state we start our recursive construction is $s = [d|A|]$ - see also Figure \ref{fig-iig_subgame_game_A1A4}.
}
\begin{enumerate}
    \item Common information states start with $cl_{\consistent}^0(s) = \{[d|A|]\}$.
    \item $\consistent([d|A|]) = \{[d||A]\}$, expanding the set $cl_{\consistent}^1(s) = \{[d|A|], [d||A]\}$.
    \item $\consistent([d||A]) = \{[d|A|], [d|K|]\}$, and $cl_{\consistent}^2(s) = \{[d|A|], [d||A], [d|K|]\}$.
    \item $\consistent([d|K|]) = \{[d||K]\}$, and $cl_{\consistent}^3(s) = \{[d|A|], [d||A], [d|K|], [d||K]\}$.
    \item $\consistent([d|K|]) = \{[d||K]\}$, ending our recurrent reasoning.
    \item $cl_{\consistent}(s) = \{[d|A|], [d||A], [d|K|], [d||K]\}$.
\end{enumerate}

\begin{figure}
\begin{subfigure}{0.94\textwidth}
  \centering
  \includegraphics[width=\textwidth]{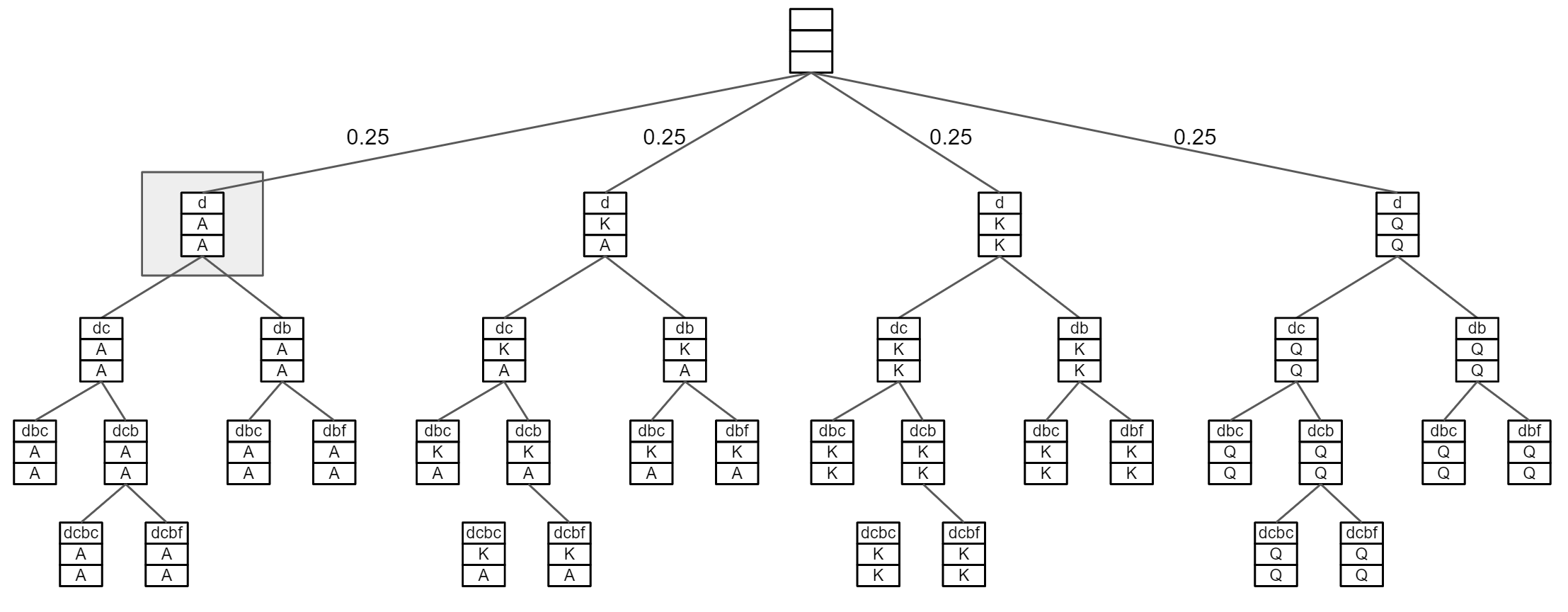}
  \caption{$cl_{\consistent}^0(s) = \{[d|A|]\}$ }
  \label{fig-iig_subgame_game_A1}
\end{subfigure}
\newline
\begin{subfigure}{0.9\textwidth}
  \centering
  \includegraphics[width=0.9\textwidth]{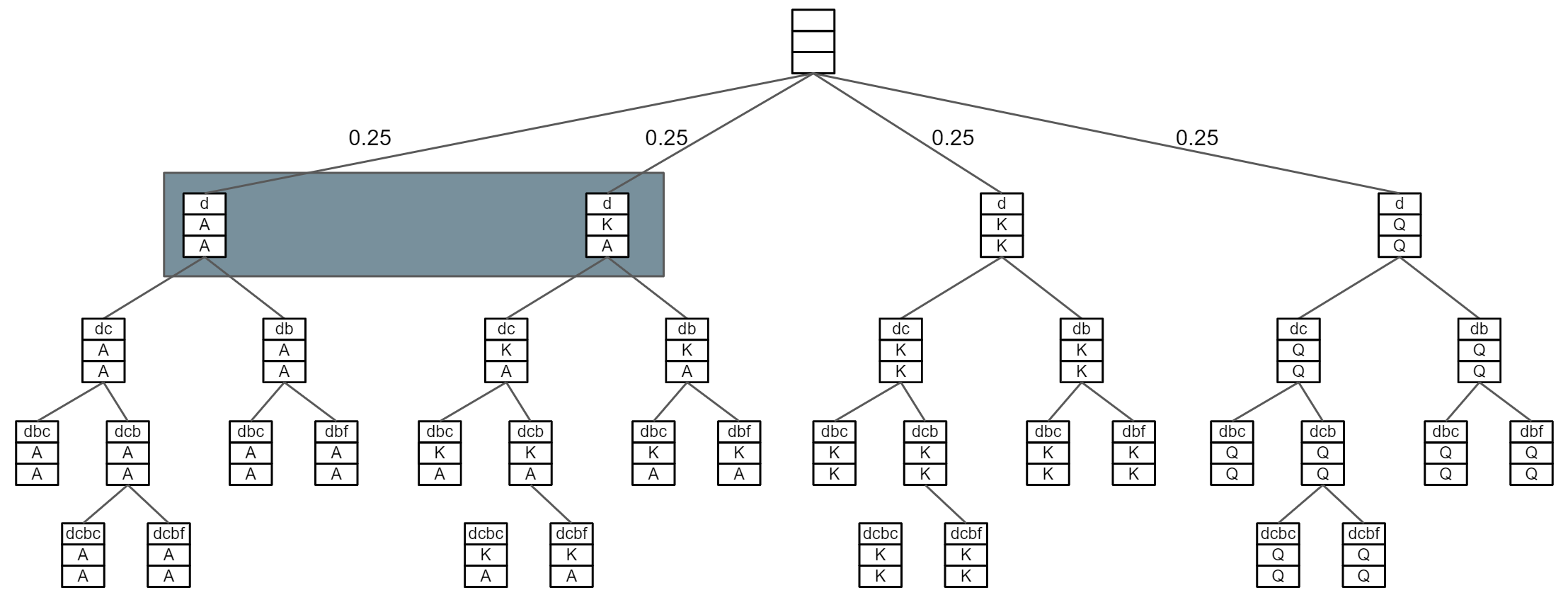}
  \caption{$\consistent([d|A|]) = \{[d||A])\}$, $cl_{\consistent}^1(s) = \{[d|A|]), [d||A]\}$}
  \label{fig-iig_subgame_game_A2}
\end{subfigure}
\begin{subfigure}{0.9\textwidth}
  \centering
  \includegraphics[width=0.9\textwidth]{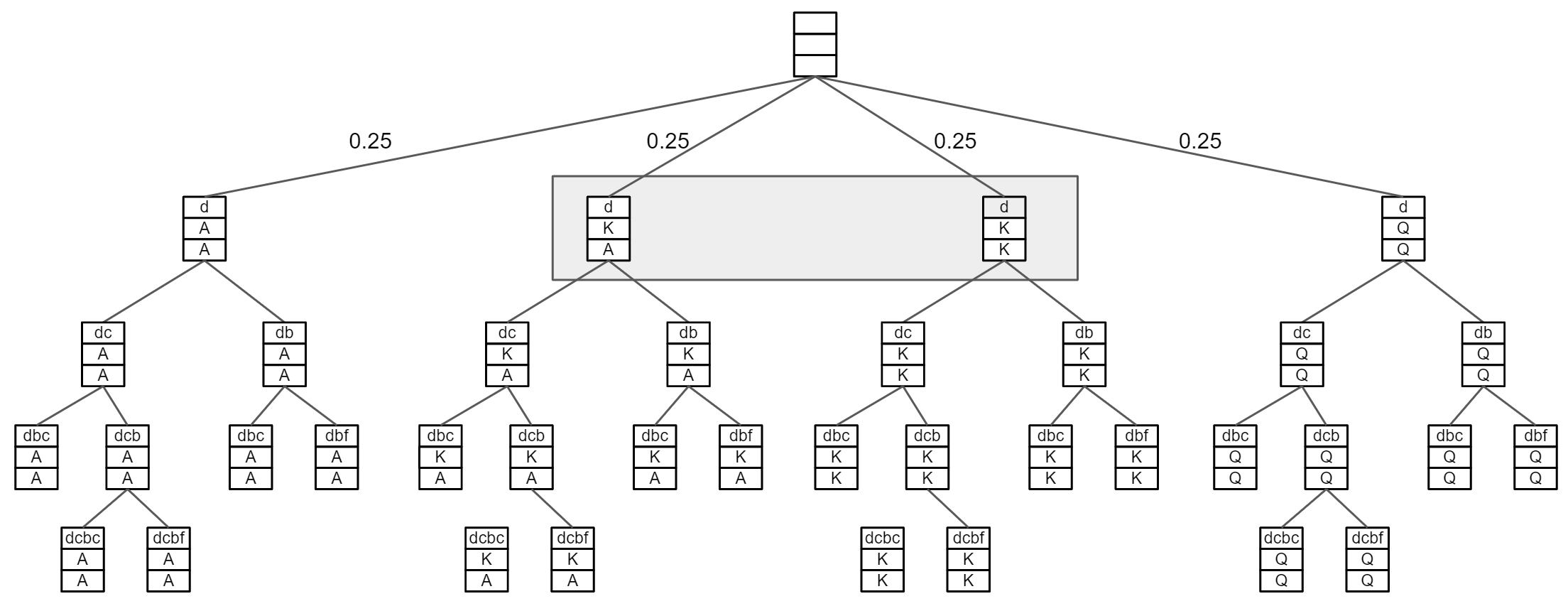}
  \caption{$\consistent([d||A]) = \{[d|K|])\}$, $cl_{\consistent}^2(s) = \{[d|A|]), [d||A], [d|K|]\}$}
  \label{fig-iig_subgame_game_A3}
\end{subfigure}
\begin{subfigure}{0.94\textwidth}
  \centering
  \includegraphics[width=0.94\textwidth]{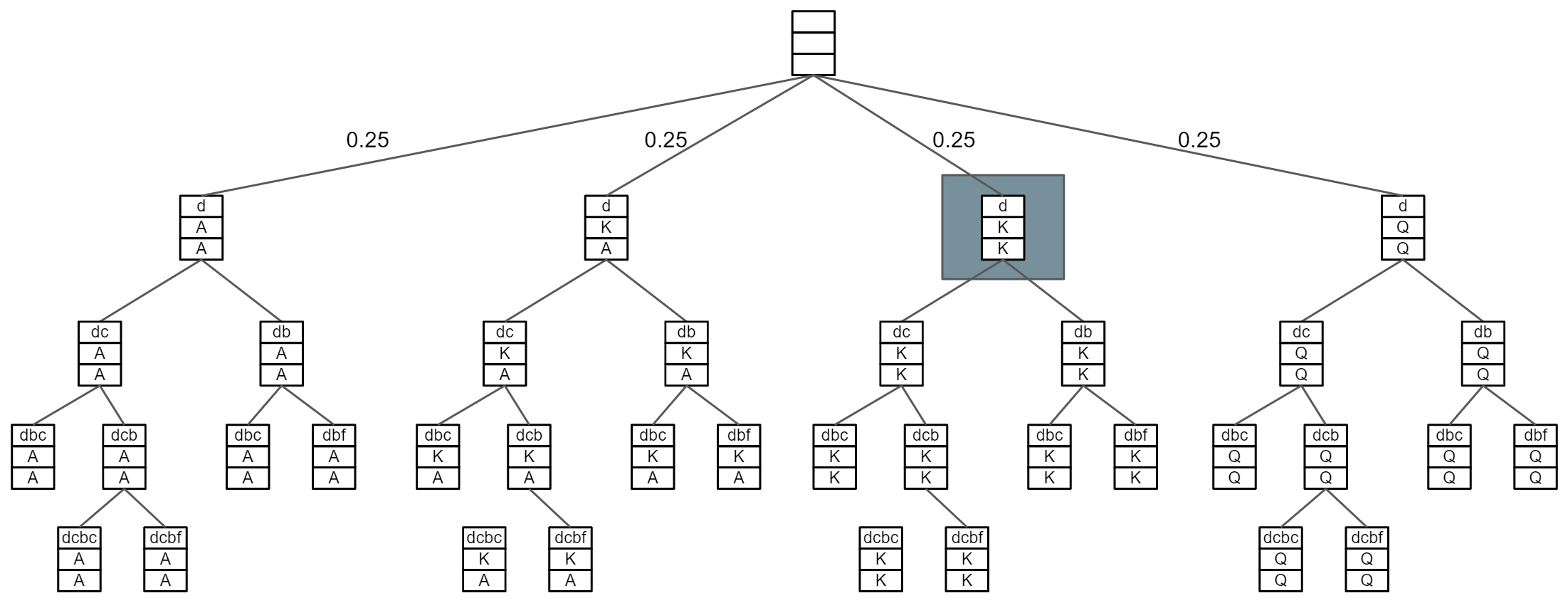}
  \caption{$\consistent([d|K|]) = \{[d||K])\}$, $cl_{\consistent}^2(s) = \{[d|A|]), [d||A], [d|K|],  [d||K]\}$}
  \label{fig-iig_subgame_game_A4}
\end{subfigure}
\caption{$s=[d|A|]$}
\label{fig-iig_subgame_game_A1A4}
\end{figure}

\subsubsection{Public States}
\second{
There is no need for the recurrent construction when public states are used, we simply use the corresponding public state $P(s)$ of the current state.
In this particular examples, both initial states we considered (player dealt Queen or Ace) are within the same public state (Figure \ref{fig:iig-subgame_public_state}) and the public state is larger than the common information sets.
This is because of the particular choice of the asymmetrical deal.
In typical poker variants, the sets based on common and public information are identical.
}

\begin{figure}[ht]
  \centering
  \includegraphics[width=14cm]{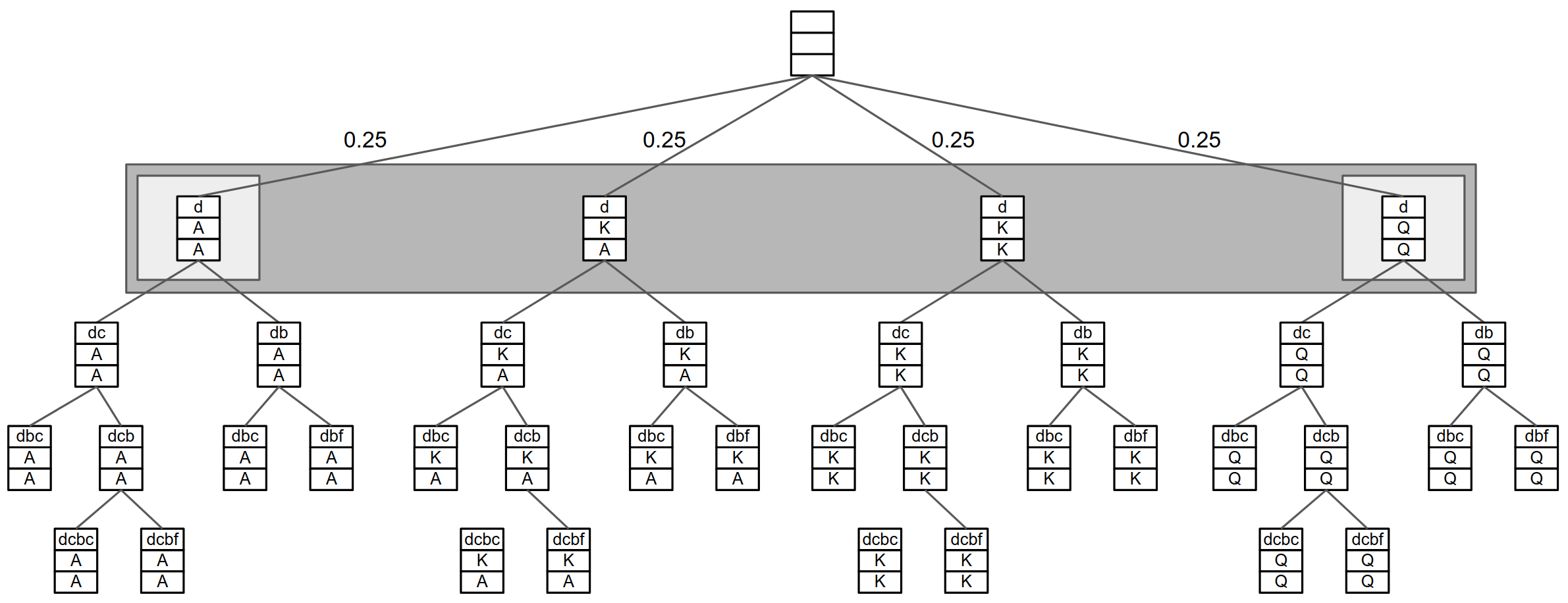}
  \caption{
  Two initial states $s_1 = [d|A|]$, $s = [d|A|]$, $s_2 = [d|Q|]$ share the same public states.
  In this example, public state is larger than the common information sets.
  }
\label{fig:iig-subgame_public_state}
\end{figure}

\section{Distribution}
\second{
The set of states formed by a public sub-tree identifies the states we can reason about in isolation.
In perfect information, a strategy for all the states in a sub-tree was sufficient for value computation of these states.
That is not the case in imperfect information, where the transition dynamics and consequently the utilities depend not just on the strategy in the sub-tree, but also on the strategy above.
We have already seen this property in the best response calculation (Section \ref{sec:iig-perfect_info_distinction} and Algorithm \ref{alg:iig_best_response}).
Concretely, the transition dynamics in a sub-tree are fully identified by the distribution over the possible initial histories and a strategy for all the infostates of that sub-tree.
}

\subsection{Factorized Distribution}
\label{sec:iig-factorized_distribution}
\second{
Computing values in a sub-tree requires a strategy for all such states as well as a distribution over the initial histories.
But not any distribution over those histories can result from an agent's play as they have control only over actions in their respective infostate (rather than for histories).

Given an initial public state, the distribution over the histories is a function of i) fixed chance factor and ii) players' strategies above the public state.
The chance factor is fixed and can be different for any history.
Agent strategies on the other hand are not fixed and are the same for all histories grouped in a single infostate.

For a history $h$, the reach probability is a product of the chance contribution $P_c(h) = \Pi_{h'aw \sqsubseteq h} \mathcal{T}(h', a, w)$ and players' reach probabilities: $P^{\pi}(h) = \Pi_{i \in \mathcal{N}} P^{\pi_i}_i(h)$.
A player's reach probability $P^{\pi_i}_i(h)$ is then a product of the strategies in their infostate sequence leading to the that history: $P^{\pi}_i(h) = \Pi_{h'aw \sqsubseteq h} \pi_i(s_i(h), a)$.
As this quantity is equal for all the histories in the infostate $s_i(h)$, we have $P^{\pi}_i(s_i(h)) = P^{\pi}_i(h)$.

We can thus compute $P^\pi(h) = P_c(h) \Pi_{i \in \mathcal{N}} P^{\pi_i}_i(s_i(h))$, factoring this term to each player's infostate reach probabilities.
In order to compute a distribution over the histories in a public state, we only need to know reach probabilities for all the players' infostates in that public state.
The importance of this observation is that the number of infostates is often substantially smaller.

Consider Texas Hold'em poker as an illustrative example.
The number of player's infostates in a public state is $\binom{52}{2}$ (player is dealt two private cards), while the number of individual histories is $\binom{52}{4}$ (two private cards are dealt to each player).
To represent the distribution over all these histories, we only require reach probabilities for the infostates of both players: $2\binom{52}{2}$, which is more than $100$ times smaller from $\binom{52}{4}$.
}

\section{Public \Subgame{}}
\label{sec:iig-public_subgame}
\second{
We now have all we need for the most critical concept for search --- generalization of \subgame{}s in imperfect information.
While in perfect information, a sub-game was identified simply by a sub-tree, imperfect information sub-games are identified by:

\begin{enumerate}
    \item Public state $s_{pub} \in \mathcal{S}_{pub}$
    \item Per-player distribution over their info-states $\Delta(\mathcal{S}_i(s_{pub}))$.
\end{enumerate}

As we will see, this sub-game allows for all the same decomposition concepts as in perfect information.
Namely, it allows for:
\begin{itemize}
    \item Solving for optimal policies in \subgames{}.
    \item Defining value functions for \subgames{}.
    \item Decomposition.
    \item Construction of provably safe search methods.
\end{itemize}
}

\subsection{Public Sub-Game is a Game}
\second{
Simple construction reveals that the public \subgame{} forms a well-defined imperfect information game.
The ideas is that we construct a game as if it started in the public sub-tree with the appropriate distribution over the initial states.
We simply construct a game rooted in a chance node that ``deals out'' the initial histories/states , with probability of each history being the (normalized) product of each player's contribution (including chance) --- Definition \ref{defn:iig-public_subgame}.

\begin{defn}[Public \Subgame{}]
\label{defn:iig-public_subgame}
    Public sub-game defined by a tuple $(s_{pub} \in \mathcal{S}_{pub}, \Delta(\mathcal{S}_1(s_{pub})), \Delta(\mathcal{S}_2(s_{pub})))$ is a game rooted in a chance state $w_0$ with actions dealing out initial histories: $\mathcal{A}_0(w_0) = \mathcal{H}(s_{pub})$, and the strategy (distribution):  $\mathcal{T}(w_0, h) \propto \Delta(\mathcal{S}_1(h)) \, \Delta(\mathcal{S}_2(h)) \, P_c(h)$ 
\end{defn}
}

\section{Value Functions}
\label{sec:iig-value_functions}
\second{
Analogously to value functions for perfect information games, value functions map \subgames{} to values.
And as a public \subgame{} is a game, we also use the value under an optimal policy.

An important generalisation is that in perfect information, the value of a \subgame{} was a single scalar.
In imperfect information, it is a vector of values, a value for each initial infostate (Table~\ref{tab:iig-value_functions}).
As nothing limits the imperfect information case to contain a single state of both players, we also get Observation \ref{obs:iig_value_functions_generalization}.
}

\begin{table*}
\centering
\begin{tabular}{@{}lll@{}}
\toprule
 &  \multicolumn{1}{c}{Input} &  \multicolumn{1}{c}{Output} \\ \midrule
 Perfect Information & $s \in \mathcal{S}$ & $V(s)$  \\ 
 Imperfect Information & $s_{pub} \in \mathcal{S}_{pub}, \Delta(\mathcal{S}_1(s_{pub})), \Delta(\mathcal{S}_2(s_{pub})) $ & $V(s) \,\forall s \in \mathcal{S}_i(s_{pub})$ \\
\bottomrule
\end{tabular}
\caption{Value function in perfect and imperfect information settings.}
\label{tab:iig-value_functions}
\end{table*}

\begin{observation}
Perfect information value functions are a degenerate case of imperfect information value function with a single initial state
\label{obs:iig_value_functions_generalization}
\end{observation}

\subsection{Example}
\second{
Table~\ref{tab:iig-value_function_rps} illustrates the value function in \rps{}.
Consider the \subgame{} irooted in the public state just after the first player acted and the second player is to act.
That public state contains a single infostate of the second player and three infostates of the first player (after the rock, paper or scissors action).
The \subgame{} and the value function thus requires a distribution over the three possible initial states of the first player, and the probability of the second player being in their single state is $1$.
The output then consists of a single value for each of the players' states (three values for the first player and one value for the second player).
}

\begin{table*}
\centering
\begin{tabular}{@{}lllllll@{}}
\toprule
    \multicolumn{3}{c}{Input} & \multicolumn{3}{c}{Output} & \\
    \cmidrule{1-2} \cmidrule{4-5} 
    $\Delta(\mathcal{S}_1(s_{pub}))$ &$ \Delta(\mathcal{S}_2(s_{pub}))$  && $V_1$ & $V_2$ && Optimal Policy \\ \midrule
    $(0.2, 0.2, 0.6)$ & $(1) $  && $(0, 1, -1)$ & $(0.4)$ && $\pi_2 = (1, 0, 0)$ \\
    $(0.4, 0.3, 0.3)$ & $(1) $  && $(-1, 0, 1)$ & $(0.1)$ && $\pi_2 = (0, 1, 0)$ \\ 
    $(\nicefrac{1}{3}, \nicefrac{1}{3}, \nicefrac{1}{3})$ & $(1) $ && $\pi_2 \begin{pmatrix} 0 & 1 & -1\\ -1 & 0 & 1\\ 1 & 0 & -1\\ \end{pmatrix}$ & $(0)$ && any policy $\pi_2$ \\
\bottomrule
\end{tabular}
\caption{Value function in rock-paper-scissors for the \subgame{} rooted in the second public state.
Note that as the \subgame{} is zero-sum, the game values are in balance $\Delta(\mathcal{S}_1(s_{pub})) \T{V_1} + \Delta(\mathcal{S}_2(s_{pub})) \T{V_2} = 0$.
Value for the first player in the \subgame{} is maximized when the initial beliefs are uniform $(\nicefrac{1}{3}, \nicefrac{1}{3}, \nicefrac{1}{3})$.
}
\label{tab:iig-value_function_rps}
\end{table*}

\subsection{Unique Game Value, Multiple State Values}
\second{
While the game value is unique and thus the same under any optimal policy, individual state values can differ.
Different optimal policies can produce different infostate values.
In the \rps{} example, any policy of the second player is optimal under the uniform initial distribution of the first player.
This in turn corresponds to different state values of the first player (Table \ref{tab:iig-value_function_rps}).

But why do we care about these individual state values at all?
Why is not the game value sufficient?
Its importance will become clear once these value function are used within search, where we will need the individual state values.
}

\subsection{Structure}
\second{
As the value function in imperfect information is more complex due to the beliefs, there are some important questions related to the structure of the function. 
Do small perturbations of the input distribution lead to drastic changes in the output?
Is there some structure?
The structure of the function becomes even more relevant once we get to function approximation and learning (e.g. with neural networks), as learning/generalization with no structure is not possible \citep{wolpert1996lack}.

The game value itself ($\Delta S_i(s_{pub}) \T{V_i}$) has a particularly nice, piece-wise linear structure.
See Section \ref{sec:iig_graphical_values} and Figure \ref{fig:values_graphical_interpretation} for some graphical intuition, or e.g. \cite{vcermak2017algorithm, seitz2019value, kovarik2020value, brown2020combining} for more analysis and properties.
But how about the individual state values $V_i$?
In general, small input perturbations of the input distribution can lead to large changes of the output state values (Table \ref{tab:iig-value_function_rps_small_perturbtions}).
In other words, small perturbation in the definition of the game can lead to drastic changes in how the game is played optimally, but only to small change in the game value.

\begin{table*}
\centering
\begin{tabular}{@{}lllllll@{}}
\toprule
    \multicolumn{3}{c}{Input} & \multicolumn{3}{c}{Output} & Optimal Policy\\
    \cmidrule{1-2} \cmidrule{4-5} 
    $\Delta(\mathcal{S}_1(s_{pub}))$ &$ \Delta(\mathcal{S}_2(s_{pub}))$  && $V_1$ & $V_2$ && \\ \midrule
     $(\nicefrac{1}{3}+2\epsilon, \nicefrac{1}{3}-\epsilon, \nicefrac{1}{3}+\epsilon)$ & $(1)$ && $(0, 1, -1)$ & $(\epsilon)$ && $\pi_2 = (0, 1, 0)$ \\
    $(\nicefrac{1}{3}-2\epsilon, \nicefrac{1}{3}+4\epsilon, \nicefrac{1}{3}-2\epsilon)$ & $(1)$ && $(-1, 0, 1)$ & $(2\epsilon)$ && $\pi_2 = (0, 0, 1)$ \\
    $(\nicefrac{1}{3}-3\epsilon, \nicefrac{1}{3}-3\epsilon, \nicefrac{1}{3}+6\epsilon)$ & $(1)$ && $(1, -1, 0)$ & $(3\epsilon)$ && $\pi_2 = (1, 0, 0)$ \\
\bottomrule
\end{tabular}
\caption{Small perturbations in the input distribution can lead to large changes in the underlying optimal policy and thus output state values, but to only small changes in the game value.
}
\label{tab:iig-value_function_rps_small_perturbtions}
\end{table*}

Fortunately, this sensitivity happens only at the critical points of the otherwise piece-wise linear value function (Figure \ref{fig:values_graphical_interpretation}).
Furthermore, it is possible to bound the state value changes when the underlying policy is not optimal, but rather $\epsilon$-optimal (Theorem \ref{thm:iig_state_values_epsilon_smooth}).

\begin{thm}
    Let $(V_1^a, V_2^a)$ be the state values of a \subgame{} $G^a = (s_{pub}, d_1^a \in \Delta(\mathcal{S}_1(s_{pub})), d_2^a \in \Delta(\mathcal{S}_2(s_{pub})))$ under an optimal policy $\pi^a$.
    Le $d_1^b, d_2^b$ be $\epsilon$-perturbed distributions ($|d_i^a-d_i^b|_\infty < \epsilon$).
    There is an $\epsilon_2$-optimal policy $\pi^b$ for the perturbed \subgame{}  $G^b = (s_{pub},  d_1^b, d_2^b)$ producing $\epsilon_2$-perturbed state values $(V_1^b, V_2^b)$ where $\epsilon_2 < \epsilon \Delta_R$.
\label{thm:iig_state_values_epsilon_smooth}
\end{thm}
\begin{proof}
Following the original policy $\pi^a$ in the perturbed game is an $\epsilon_2$-optimal policy producing $\epsilon_2$ perturbed values, as the change in distribution over the terminal states is bounded by $\epsilon$.
\end{proof}

}
\chapter{Offline Solving I}
\label{chap:iig-offline_solve}

\second{
We will now describe offline solving methods for imperfect information games.
Because we include multiple distinct approaches, we split the methods into three separate chapters.

Chapter \ref{chap:iig-offline_solve} deals with optimization methods and self-play methods.
Section \ref{sec:iig-matrix_games_direct_optimization} starts with matrix games, and we analyse the structure of the worst case value as a function of the agent's policy (Section \ref{sec:iig_graphical_values}).
This is then used in construction of a linear optimization problem that produces an optimal policy (Section \ref{sec:iig-matrix_games_lp}) and allows us to prove the minimax theorem.
Section \ref{sec:iig-sequence_games_lp} generalizes the construction of the linear optimization problem to the sequence games and formalizes the sequence form.
Section \ref{sec:iig-selfplay_irl} discusses the application of (single agent) reinforcement learning methods, showing that naive application of such methods in self-play fails to converge even in simple imperfect information games.
Section \ref{sec:iig-multi_agent_rl} touches on some sound multi-agent reinforcement learning methods and Section \ref{sec:iig-fictitious_play} introduces fictitious play.
Double oracle and its related methods are then mentioned in Section \ref{sec:iig-double_oracle}.

Chapter \ref{sec:iig-offline_solve_regret} moves to the most important offline solving methods of this thesis --- methods based on the regret minimization framework that are central for the final search techniques.
Section \ref{sec:iig-regret_setup} presents the online convex learning concept of regret minimization, and Section \ref{sec:iig-regret_algorithms} includes some regret minimization algorithms.
Section \ref{sec:iig-regret_to_nash} then shows the striking connection between regret minimization and optimal policies, where we prove the Folk Theorem (\ref{thm:iig:folk_theorem}) that shows that if both players use a regret minimizer in self-play settings, the average policies converge to an optimal policy profile.
Section \ref{sec:iig-cfr} describes how to minimize regret in sequential settings --- counterfactual regret minimization (CFR).
CFR decomposes the full regret into a sum over individual regrets at each information state, allowing one to independently minimize these partial regrets.
Section \ref{sec:iig-cfr_plus} then mentions some modern members of the CFR algorithm family, such as CFR+ and DCFR.
Section \ref{sec:iig-mccfr} describes Monte Carlo variants of CFR (MCCFR) that sample trajectories rather than traversing the entire game tree at each iteration.
As the sampling introduces noise and slows down convergence, Section \ref{sec:iig-mccfr} describes VR-MCCFR, modern variance reduction methods that drastically speed up the convergence (a contribution of this thesis).

Chapter \ref{chap:iig-abstraction} deals with large games where we cannot directly use tabular methods.
Section \ref{sec:iig-abstraction_methods} discusses abstraction approaches that were for a long time state of the art for large imperfect information games (until the dawn of search).
Section \ref{sec:iig-non_tabular_methods} then briefly mentions some other non-tabular methods that can tackle large games.
}
\section{Direct Optimization for Matrix Games}
\label{sec:iig-matrix_games_direct_optimization}
\second{
An optimal policy maximizes its best response value $\max_{\pi_i} \brv(\pi_i)$.
For matrix games, the $\brv(\pi_i)$ function has a particularly simple form (\ref{eq:iig-minmax_opt_function}).

\begin{align}
\label{eq:iig-minmax_opt_function}
\brv(\pi_i) =  \min_{\pi_{-i}} \pi_i A \pi_{-i} 
\end{align}

Since \ref{eq:iig-minmax_opt_function} is a point-wise minimum of a affine function, the function is concave \citep{boyd2004convex}.
}

\subsection{Graphical Interpretation}
\label{sec:iig_graphical_values}
\second{
To build more intuition about the concave best response value function, we will graphically analyze the function on a small matrix game.
Table \ref{tab:iig:game_convexoptim} shows a game with $2$ actions for the row player, and $3$ actions for the column player.

\begin{table}[ht]
\centering
\begin{tabular}{llll}
  & X & Y & Z  \\ \cline{2-4}
A & 8 & 3 & 2  \\
B & 2 & 6 & 8  \\
\end{tabular}
\caption{Simple game for the analysis of the best response value function.}
\label{tab:iig:game_convexoptim}
\end{table}

Visualising the game in a particular way (Figure \ref{fig:iig:convex_optim1}) provides nice insight into the structure of the best response value function.
We can see that the function is concave and piece-wise linear --- key observation for algorithms that directly maximize the function.
This graphical reasoning can be traced to \cite{GTK51}.

\begin{figure}[ht]
    \centering
    \begin{subfigure}[b]{0.3\textwidth}
        \includegraphics[width=\textwidth]{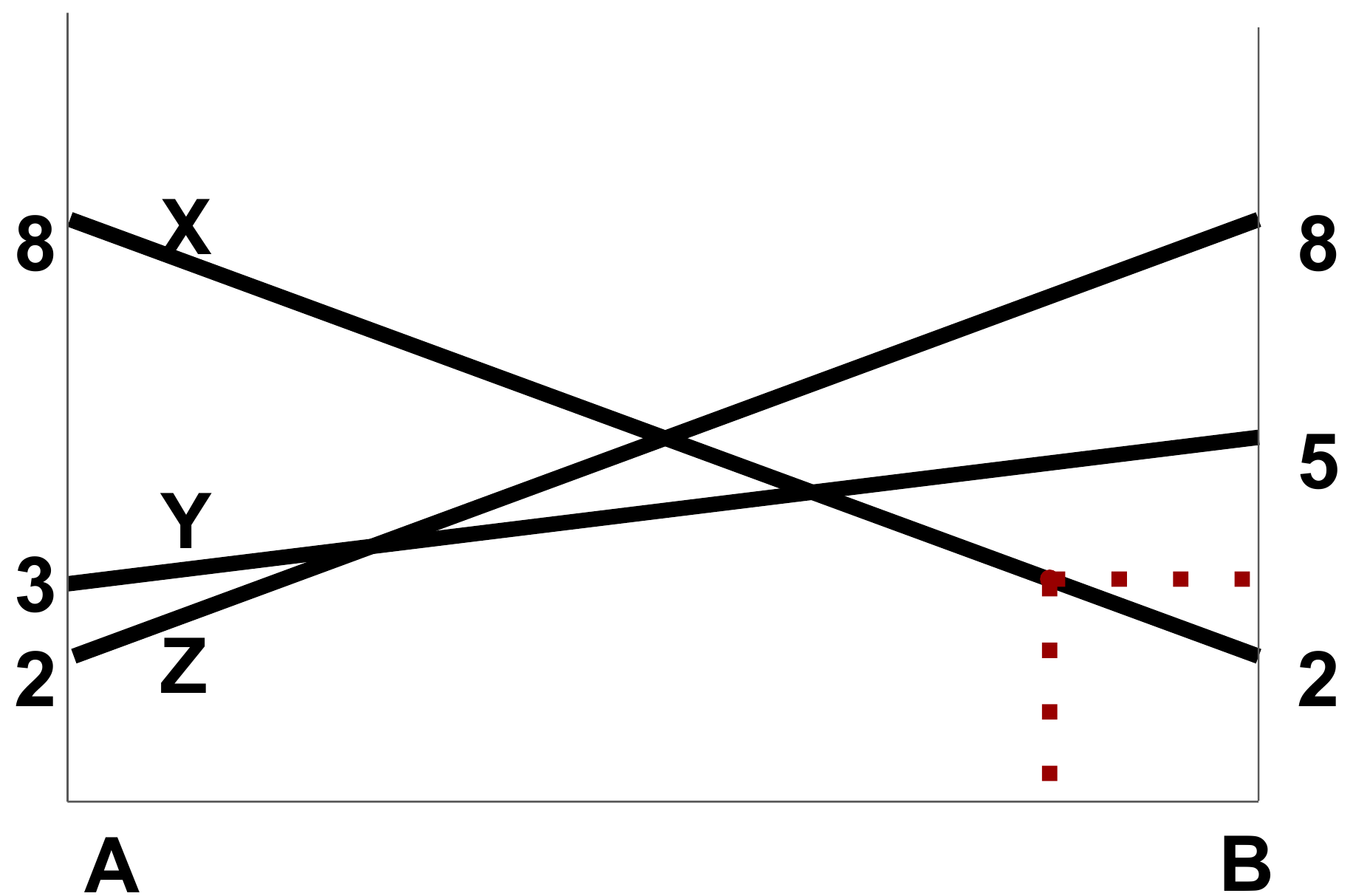}
        \caption{Visualization of Game \ref{tab:iig:game_convexoptim}.}
        \label{fig:iig:convex_optim1}
    \end{subfigure}
    ~ 
    \begin{subfigure}[b]{0.3\textwidth}
        \includegraphics[width=\textwidth]{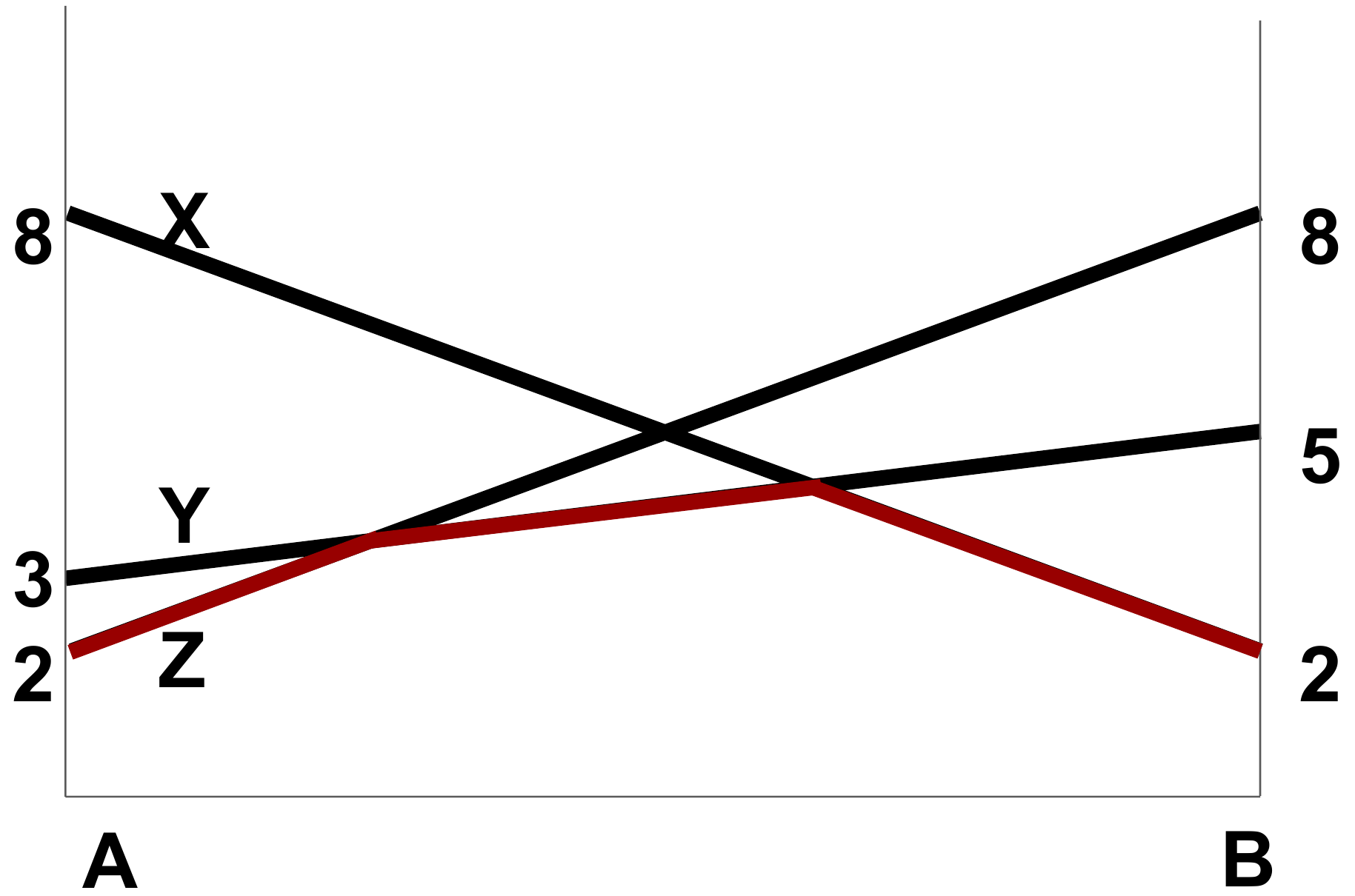}
        \caption{Worst case function we try to maximize.}
        \label{fig:iig:convex_optim2}
    \end{subfigure}
     ~ 
    \begin{subfigure}[b]{0.3\textwidth}
        \includegraphics[width=\textwidth]{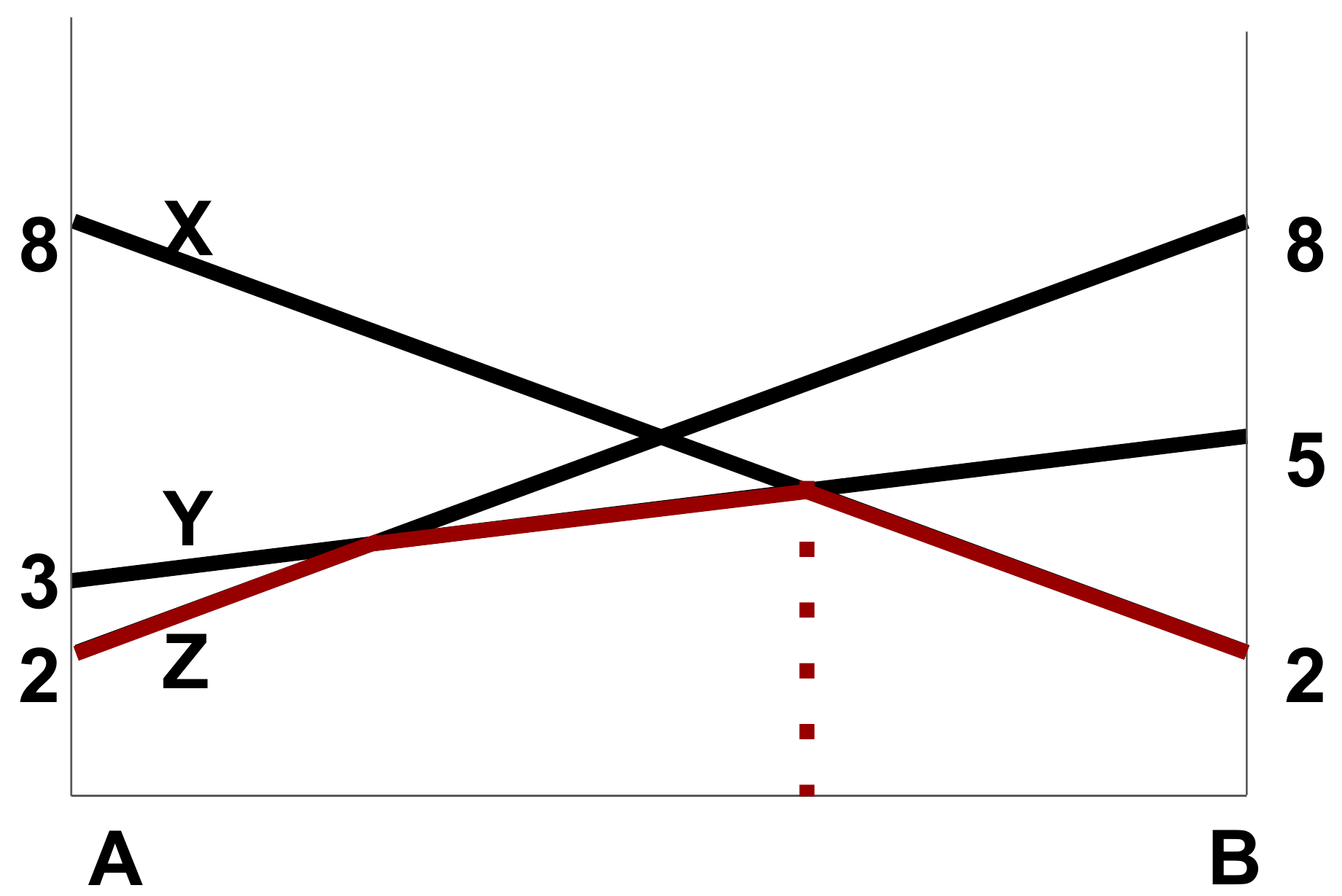}
        \caption{Strategy that maximizes the worst case function.}
        \label{fig:iig:convex_optim3}
    \end{subfigure}
    \caption{
    a) We can see the three lines corresponding to the value of the row play when the opponent responds with one of the three available actions.
    The x-axis corresponds to the strategy of the row player.
    As an example, the depicted red point is the return for the row player
    when they play $\pi_1 = (0.2A, 0.8B)$ and the opponent responds with action Z.
    b) Best response value function is then simply minimum of the lines (worst case best response).
    c) Strategy maximizing the worst case function.
    }
    \label{fig:values_graphical_interpretation}
\end{figure}

}

\section{Matrix Games Linear Programming}
\label{sec:iig-matrix_games_lp}
\second{
The piece-wise linear structure of \ref{eq:iig-minmax_opt_function} can be leveraged to create a closed-form linear program that optimizes the function in question.
We will construct a linear optimization for both the row and the column players, and then use the strong duality to prove the minmax theorem.
}

\subsection{Row Player}
\second{
First, notice that \maximin{} solution for the row player is:
\begin{align}
\label{eq:iig:minmax0}
    \max_{x \in \Pi_i} \min_{y \in \Pi_{-i}} x^\top A y
\end{align}

Since (\ref{eq:iig:minmax0}) is bi-linear, we need to decouple the $x$ and $y$ by introducing new variable.
Given a strategy $x$ for the row player, the best-responding opponent simply chooses the column with the smallest utility.

\begin{align}
\label{eq:iig:minmax1}
    \min_{y \in \Pi_{-i}} x^\top A y
\end{align}

This can be re-formulated as:

\begin{align}
\label{eq:iig:minmax2}
\begin{split}
    \max_{u \in \mathbb{R}} u \\
    x^\top A \geq u
\end{split}
\end{align}

We now just use this formulation in (\ref{eq:iig:minmax0}), resulting in the linear program (\ref{eq:iig:minmax_row_lp}) and Theorem \ref{thm:iig-matrix_lp}.

\begin{align}
\label{eq:iig:minmax_row_lp}
\begin{split}
    \max_{u \in \mathbb{R}, x \in \Pi_i } u \\
    x^\top A \geq u
\end{split}
\end{align}

\begin{thm}
\label{thm:iig-matrix_lp}
\citep[Theorem 1.11]{nisan2007algorithmic}
    Linear program \ref{eq:iig:minmax_row_lp} produces a \maximin{} strategy for the row player ($x$) and the corresponding \maximin{} value ($u$).
\end{thm}
}
\subsection{Column Player}
\second{
Analogical reasoning for the column player starts with $\min_{y} \max_{x} x^\top A y$ and results in
(\ref{eq:iig:minmax_column_lp}) and Theorem \ref{thm:iig-matrix_lp}.

\begin{align}
\label{eq:iig:minmax_column_lp}
\begin{split}
    \min_{v \in \mathbb{R}, y \in \Pi_{-i}} v \\
    A y \leq v
\end{split}
\end{align}

\begin{thm}
\label{thm:iig-matrix_lp2}
    Linear program \ref{eq:iig:minmax_column_lp} produces a \maximin{} strategy for the column player ($y$) and the corresponding \maximin{} value ($-v$).
\end{thm}

}

\subsection{Minimax Theorem}
\label{sec:iig-minmax_theorem}
\second{
We are now ready to prove the crucial minimax theorem.

\restminmaxtheorem*{}

The key observation is that the linear programs for the \maximin{} strategies of the row player \ref{eq:iig:minmax_row_lp} and the column player \ref{eq:iig:minmax_column_lp} form a pair of dual linear programs.
And since both are feasible, strong duality holds and thus the optimal
values of these programs are equal $v = -u$ \citep{bertsimas1997introduction}.

It is not a coincidence that we have used linear programming and duality to find optimal policies in two-person zero-sum games.
There is a one-to-one connection between linear programming and solving optimal policies, as it is possible to convert one problem to another \citep{boyd2004convex}.
}

\section{Sequential Games Linear Programming}
\label{sec:iig-sequence_games_lp}
\second{
This section generalizes the linear programming formulation of optimal policies from matrix games to sequential decision making.
Even though it is possible to convert a sequential game to matrix form, this is not practical as the resulting matrix game can be exponentially large compared to the sequential representation (Section~\ref{sec:igg-converting_formalism}).
}

\subsection{Sequence Form}
\label{sec:iig:sequence_form}
\second{
While a behavioral strategy defines how to act in individual information states, this representation is not directly useful for linear/convex optimization due to its non-convexity (probability of reaching a state $s$ is a product of the behavioral probabilities on its infostate path $\Pi_{s'as'' \sqsubseteq s} \pi(s', a)$)
Fortunately, there is a convex representation of a strategy --- the sequence form.
The idea is to represent a probability of a player reaching a state as the product of the sequence of a player's actions leading to that state.
One then must make sure that these reach probabilities are well formed, that the probability mass of reaching all children state sum to their parent.

In this section, we present a sequence form representation in the FOSG formalism by adapting the notation similar to that for extensive form games in \cite{nisan2007algorithmic}.
}

\subsubsection{Realization Plan}
\second{
We use $\mathcal{B}_i$ to denote the set of all state-action sequences of a player $i$:
$\mathcal{B}_i = \{ sa | s \in \mathcal{S}_i, a \in \mathcal{A}(s) \}$, 
and $b_i(ha) = s_i(h)a$.
The sequence form realization probability of a state-action sequence $sa \in \mathcal{B}_i$ becomes $\prod_{s'a' \sqsubseteq sa} \pi_i(s', a')$.
For game with perfect recall, Kuhn's theorem \citep{kuhn1953extensive, aumann1961mixed} states that
realization plan is equivalent to behavioral strategy.

We use $x$ and $y$ to denote the vectors of these quantities for all such sequences of player $1$ and $2$
respectively. 
Recall that the terminal utility (return) of reaching $z \in \mathcal{Z}$ is simply the accumulated reward on that trajectory:

\begin{align}
    u_i(z) = R_i(z) = \sum_{haw \sqsubseteq z} R_i(h, a, w)
\end{align}

We are now ready to define the sequence form payoff matrix $A \in \mathbb{R}^{|\mathcal{B}_1| \times |\mathcal{B}_2|} $,
with $A_{\sigma \tau}$ for indexes $\sigma \in \mathcal{B}_1, \tau \in \mathcal{B}_2$ defined as the chance-reach-weighted sum over the corresponding terminal states:

\begin{align}
    A_{\sigma \tau} = \sum_{(haw) \in Z \, : \, b_1(ha) = \sigma, b_2(ha) = \tau } P_{\mathcal{T}}(z) u_1(z)
\end{align}

With this notation, the expected value given realization plans of both players becomes simply:

\begin{align}
    x^\top A y
\end{align}

Not just any positive real valued vectors form a realization plan though.
A realization plan fulfills two important properties (ensuring that the probability mass of children sequences sum to their parent):

\begin{align}
    x_{\emptyset} &= 1 \\
    \sum_{a' \in \mathcal{A}(sao)} x_{saoa'} &= x_{sa}
\end{align}

These conditions can be compactly formulated using zero-one matrices $E, F$ and zero-one vectors $e, f$ \citep{nisan2007algorithmic}:
\begin{align}
    Ex = e, \, x & \geq \boldsymbol{0} \\
    Fy = f, \, y & \geq \boldsymbol{0}
\end{align}
}

\subsection{Optimal Policies}
\label{sec:iig-sequence_form_lp}
\second{
We will now use the sequence form to construct a linear optimization problem for computing an optimal policy.
Given a fixed realization plan $y$ of the second player, we can compute the best response for the first player as:

\begin{align}
\begin{split}
    \max  x^\top A y \\
    Ex = e, \, x & \geq \boldsymbol{0}
\end{split}
\end{align}

This can be turned into the following LP that solves for Nash equilibrium (for more details see \cite{nisan2007algorithmic}).

\begin{align}
\begin{split}
    & \min_{u, v} e^\top u  \\
    & Fy = f, E^\top u -Ay \geq \mathbf{0}, y \geq \mathbf{0}
\end{split}
\end{align}
}

\section{\Selfplay{} via Independent Reinforcement Learning}
\label{sec:iig-selfplay_irl}
\second{
We now consider another approach for producing optimal offline policies --- methods rooted in the \selfplay{} paradigm.
The first such method is \selfplay{} via independent reinforcement learning, where the agents best-respond to each other at each time step.
In perfect information, this sequence\footnote{\Subgame{} perfect best response.} converged to an optimal policy (Section~\ref{sec:pig-selfplay_methods}).
The same algorithm fails to produce optimal policies in imperfect information games.
Table~\ref{tab:iig:game_rps_shifted} presents a simple matrix game with optimal policies $\pi_1 = (\frac{2}{3}, \frac{1}{3})$ and $\pi_2 = (\frac{2}{3}, \frac{1}{3})$ for the row and column player respectively.
A best-responding sequence in this game $p_1 = \{A, B, A, B \hdots\}$, $p_2 = \{B, A, B, A \hdots\}$ then does not converge as the strategy at each time $p^t_i$ is highly exploitable.
Furthermore, note that the average strategies then are $\pi_1 = (\frac{1}{2}, \frac{1}{2})$ and $\pi_2 = (\frac{1}{2}, \frac{1}{2})$ and thus also exploitable.

\begin{table}[ht]
\centering
\begin{tabular}{lll}
 & A & B \\
\cline{2-3}
A & 1 & -2 \\
B & -2 & 4 \\
\end{tabular}
\caption{Matrix game where best-responding sequence does not converge.}
\label{tab:iig:game_rps_shifted}
\end{table}

While the same argument could be made using the \rps{}, the example would be far less illustrative.
The problem is that the relative frequencies of the individual actions in the best responding sequence $p_1 = \{R, P, S, R, P, S, \hdots\}$, $p_2 = \{R, P, S, R, P, S, \hdots\}$ match the optimal policy.
One could then conclude that average strategy converges, but that is not generally the case.
}

\subsection{Deterministic Optimal Policies}
\second{
As optimal policies under imperfect information might have to be stochastic (Corollary~\ref{cor:iig-minmax_stochastic}), the failure of the selfplay method to converge could be attributed to the fact that it produces a deterministic best response policy at each time step.
But a best-responding sequence can fail to converge even in a game with an optimal deterministic policy.
Table~\ref{tab:iig:game_rpswater} presents a modified \rps{} game with added action (W)ater.
Water draws against all the actions, and both players playing water is an optimal policy in this game (there are other optimal policies).
Sequence $p_1 = \{R, P, S, R, P, S, \hdots\}$, $p_2 = \{R, P, S, R, P, S, \hdots\}$ is then a best-responding sequence failing to converge in this game.

\begin{table}[ht]
\centering
\begin{tabular}{lllll}
  &  R & P & S & W \\ \cline{2-5}
R & 1  & 0 & -1  & 0 \\
P & -1 & 1  & 0  & 0 \\
S & 0 & -1   & 1  & 0 \\
W & 0 & 0   & 0 & 0 \\
\end{tabular}
\caption{Rock-paper-scissors-water: game where deterministic optimal policy exists (both player playing water), but best-responding selfplay sequence can fail to converge.}
\label{tab:iig:game_rpswater}
\end{table}

}
\section{Multi-agent Reinforcement Learning}
\label{sec:iig-multi_agent_rl}
\second{
The message of the convergence failure should be that we just cannot naively take a single-agent reinforcement learning algorithm and run it independently in multi-agent settings through \selfplay{}.
Single-agent reinforcement learning methods lack theoretical guarantees in these settings, often producing highly exploitable policies.
While we have simplified the analysis by allowing the algorithm to fully converge to best response at each \selfplay{} step, this approach fails to produce optimal policies even if the agents improve their policies less aggressively (e.g. using few steps of policy gradient or Q-learning).
For more insights and analysis see e.g. \cite{bailey2018multiplicative}.
Despite the lack of theoretical guarantees, this approach is relatively popular and it has recently been applied to e.g. Dota 2 \citep{berner2019dota}.
As such approach fails to produce optimal policies for \rps{}, it is a questionable choice for large two-player imperfect information game.

Rather than naively applying single-agent algorithms, one can properly design the algorithms for multi-agent settings.
There are many modern methods that take the complex multi-agent dynamics into account, providing strong convergence guarantees \citep{muller2019generalized, vinyals2019alphastar, perolat2020poincar}
None of these methods allow for search nor provide any building blocks for the next chapters and thus we do not dive any deeper into details.
}

\section{Fictitious Play}
\label{sec:iig-fictitious_play}
\second{
Simple modification to the best-responding sequence leads to the \gls{fictitious_play} algorithm.
\Gls{fictitious_play} best-responds to the average strategy of the opponent rather than the last one.

Let $\overline{\pi^t_i}$ denote the average strategy of the player $i$ up to time $t$.
The fictitious self-play sequence is then $\pi^{t}_i \in \brset{}(\overline{\pi^{t'-1}_{-i}})$.
Best-responding to the average policy has important theoretical consequences, as the average strategy profile then converges to optimal.
This algorithm dates all the way back to \citep{brown1951iterative}, with a convergence proof followed shortly after \citep{robinson1951iterative}.

\Gls{fictitious_play} is very simple to implement and often serves as a building block for other algorithms.
For example in large games, one can use single-agent reinforcement learning methods to approximate the individual best responses \citep{heinrich2015fictitious}.
Weakened fictitious play then replaces the best-response to the average strategy with a strategy that is only an increasingly better one \citep{van2000weakened, leslie2006generalised}
The problem of this algorithm is its convergence speed, where the number of iterations required can grow exponentially \citep{harris1998rate}.

} 

\section{Double Oracle Methods}
\label{sec:iig-double_oracle}
\second{
Rather than using best response to produce the self-play sequence, it can be used to iteratively build and expand the matrix game to be solved.
Double oracle methods \citep{mcmahan2003planning} iteratively expand the subset of actions to be considered by both the row and column player at time $t$ :  $\mathcal{A}^t_1 \subseteq \mathcal{A}_1$, $\mathcal{A}^t_2 \subseteq \mathcal{A}_2$.
This restricted matrix game $\mathcal{A}^t_1, \mathcal{A}^t_2$ is then solved e.g. via linear optimization.
Next, both players compute a best response against the resulting optimal policy by considering all the actions in the original unrestricted game $br^t_i \in \mathcal{A}_i$.
The restricted game is then expanded $\mathcal{A}^{t+1}_i = \mathcal{A}^t_i \, \cup \, br^t_i$ and the process repeats.
Double oracle methods terminate when the restricted game does not expand any more ($\mathcal{A}^{t+1}_i = \mathcal{A}^{t}_i$) and the key result is that the optimal policy for the restricted game is then also optimal for the original game.
This approach often terminates when the restricted game is substantially smaller than the original game.
But in the worst case, the restricted game fully recovers the original game (even when the support size is smaller).

Policy-Spaced Response Oracles (PSRO) then extends this idea for large games where the best response calculation is intractable, and uses reinforcement learning methods to approximate the best response \citep{lanctot2017unified}.
The restricted game consists of individual policies represented by neural networks and is referred to as the meta-game.
There are now multiple parallel and efficient variants of this approach \citep{mcaleer2020pipeline}.
Furthermore, AlphaStar (an AI for the game StarCraft II) builds on similar methods and the resulting meta-game with $888$ rows and columns has been released \citep{vinyals2019alphastar, vinyals2019grandmaster}.

Finally, double oracle methods can be efficiently generalized to sequential settings \citep{bosansky2014exact}.
Combination with PSRO results in efficient PSRO methods in the sequential representations \citep{mcaleer2021xdo}.
}

\chapter{Offline Solving II - Regret Minimization}
\label{sec:iig-offline_solve_regret}

\second{
Currently, the most successful algorithms for solving imperfect information games are based on the regret minimization framework.
Regret minimization is a general, online convex learning concept where an agent repeatedly makes decision against an unknown environment \citep{zinkevich2003online}.
Regret then measures the difference between the accumulated reward and the reward that a best time-independent action would receive in hindsight (e.g. always playing rock).
And while the environment can be adversarial as the reward vector at each time step can depend on the selected action, it is still possible to provide strong guarantees with regard to hindsight performance of regret.
An algorithm is said to be Hannan consistent if the regret grows sub-linearly --- the average regret converging to zero.

Regret minimization has an elegant and important connection to games when used in self-play.
In self-play, the reward vector can be computed using the opponent's policy and both players then select a next strategy using a regret minimizer.
It is then possible to show that the average strategy of the players is $\epsilon$-optimal for $\epsilon$ no greater than the sum of players' average regrets. 
And as the average regret converges to zero, the average strategy converges to optimal.

Finally, minimizing regret in sequential settings can be decomposed into individual regret minimizers in the information states --- counterfactual regret minimization.
This is particularly important as it will be used as a key building block of decomposition and search methods.
}

\section{The Setup}
\label{sec:iig-regret_setup}
\second{
While there are multiple concepts of regret, the one used in this book is often referred to as external regret\footnote{External regret with a comparison class of all the actions.}.
The agent repeatedly makes a decision against an unknown environment and observes a reward vector (reward for each action).

Formally, the agent's set of actions is $\mathcal{A}$ and at each time step $t$ the agent chooses a policy $\pi^t \in \Delta(\mathcal{A})$.
The agent then receives a reward vector $x_t \in \mathcal{R}^{|\mathcal{A}|}$ and the agent's value is the weighted sum $v^t = \sum_{a \in A} \pi^t(a) \, x^t(a) =  \pi^t \, {x^t}^\intercal$.
The cumulative reward up to time $T$ is simply $X_{\pi}^T = \sum_{t=1}^T \pi^t \, {x^t}^\intercal$.
External regret of action $a$, $R_a^T$ then contrasts this to a cumulative reward that would have been received if we rather followed 
action $a$ at each time step $R_a^T = \sum_{t=1}^T x_t(a) - X_{\pi}^T$.
Finally, external regret is then $R^T = \max_{a \in \mathcal{A}} R^T_a$.
In other words, how much we regret not taking the best action in retrospect.
}

\section{Regret Minimization Algorithms}
\label{sec:iig-regret_algorithms}
\second{
We now present some popular regret minimization algorithms.
To simplify the notation and analysis, we assume that the action reward $x^t(a) \in [0, 1]$.
No generality is lost as for $x^t(a) \in [a, b]$, the resulting bounds are simply scaled by a constant factor $|b-a|$}

\subsection{``Always Rock''}
\second{
Selecting the same action $a_1$ at each time step leads to worst case $R^T = T$, as we can assign $x^t(a_1) = 0, \, x^t(a') = 1 \, \forall t$.
}

\subsection{Simple Greedy}
\label{sec:iig-regret_simple_greedy}
\second{
Instead of selecting the same action, we might prefer the action with the highest accumulated regret $\pi^t = \argmax_{a \in \mathcal{A}} R^T_a$ (and selects the action with the smallest index if there are multiple maxima).
It is easy to show that the regret can be bounded by (\ref{eq:iig:greed_rm_bound}).

\begin{equation}
\label{eq:iig:greed_rm_bound}
   R^T_{\pi_{greedy}} \le \frac{T}{|\mathcal{A}|}
\end{equation}

This is not much of an improvement as we need sub-linear regret for Hannan consistency.
Turns out that in order to guarantee sub-linear regret, it is necessary to produce stochastic policies $\pi^t$ \citep{nisan2007algorithmic}.
}

\subsection{Hedge}
\second{
Arguably the most popular regret minimization algorithm --- Hedge \citep{freund1997decision} ---
is a parametric algorithm based on the weighted majority algorithm \citep{littlestone1994weighted}, selecting an action exponentially proportional to its regret (\ref{eq:iig-hedge}).

\begin{equation}
\label{eq:iig-hedge}
    \pi_t = \frac{e^{\beta R_{t-1}}}{\sum_{a \in A} e^{\beta R_{t-1}}}
\end{equation}

Given the sequence length $T$ ahead of time and setting $\beta = \sqrt{\frac{2 \log(|\mathcal A|)}{T}}$, Hedge provides sub-linear upper bound on the regret (\ref{eq:iig-hedge_bound}).

\begin{equation}
\label{eq:iig-hedge_bound}
   R^T(\pi_{hedge}) \le \Delta \sqrt{2 T \log(|\mathcal A|)} + \log(|\mathcal A|)
\end{equation}

If we do not know T, we can use the standard doubling trick and the asymptotic bound of (\ref{eq:iig-hedge_bound}) remains the same.
In practice, the $\beta$ parameter can be tuned for the game and can significantly improve performance.
Furthermore, there are parameter-free variants e.g. NormalHedge \citep{chaudhuri2009parameter}.
}

\subsection{Regret Matching}
\second{
\label{sec:iig-regret_matching}
One particularly simple algorithm is \gls{regret_matching}.
One simply plays an action proportionally to the positive regret of that action \citep{blackwell1956analog, hart2000simple}.
Let $R(a)^+ = \max(R(a), 0)$, regret matching then sets the action probability as follows:
\begin{equation}
\label{eq:iig:regret_matching}
    \pi^t_{rm}(a) = \frac{R(a)^+}{\sum_{a' \in \mathcal{A}} R(a')^+} 
\end{equation}

When $\sum_{a' \in \mathcal{A}} R(a')^+ = 0$, we simply select a uniform random policy $\pi^t_{rm}(a) = \frac{1}{|\mathcal{A}|}$.
Note that from the equation \ref{eq:iig:regret_matching}, we see that unlike Hedge (and many others), regret matching is not parametric.
This is a powerful property, since in practice we don't have to find parameters that work well for our particular problem.
Regret matching enjoys the following bound (\ref{eq:iig:regret_matching_bound}).

\begin{equation}
\label{eq:iig:regret_matching_bound}
  R^T(\pi_{rm}) \leq \sqrt{ |\mathcal{A}| T}
\end{equation}
}

\section{Connection To Nash Equilibrium}
\label{sec:iig-regret_to_nash}
\second{
We are now ready to show the striking connection from regret to Nash equilibria.
Namely, we will connect the average regret $\frac{R^T(\pi)}{T}$ to $\epsilon$-Nash equilibrium.
The idea is to use regret minimization in self-play, where we iteratively update policies of the players.
\Selfplay{} utilities are then used to define the reward vector, which in turn forms the regrets to be minimized (Figure~\ref{fig:iig-simultaneous_updates} and Algorithm~\ref{alg:iig-selfplay_regret_matrix}
 ).

At time $t$, a player (regret minimizer) produces their strategy $\pi^t_i$ and the reward vector is computed as the utility against a strategy of the opponent  $\pi^t_{-i}$.
For a matrix game, the reward vector of the player $i$ is $x^t_i = A \pi^t_{-i}$ and the current regret is: 
$r^t_i = A {\pi^t_{-i}} - \pi^t_i A {\pi^t_{-i}}$ .
}

\begin{algorithm}
\caption{\Selfplay{} Regret Minimization in Matrix Games}
\label{alg:iig-selfplay_regret_matrix}
\begin{algorithmic}[1]
\Function{RegretMatching}{regrets}
    \State positive\_regrets = regrets$^+$
    \State pos\_regret\_sum = $\sum_i$ positive\_regrets$_i$
    \If{pos\_regret\_sum = 0}
      \Return uniform\_policy
    \EndIf
    \State
    \Return positive\_regrets $/$ pos\_regret\_sum
\EndFunction
\State
\Function{RunSelfplay}{$A \in \mathcal{R}^{n \times n}$, steps}
    \State \Comment{Initialize regrets and average policies}
    \State $r_1$ = $\vec{0} \in \mathcal{R}^{1 \times m}$
    \State $r_2$ = $\vec{0} \in \mathcal{R}^{1 \times n}$
    \State $\overline{\pi}_1$ = $\vec{0} \in \mathcal{R}^{1 \times m}$
    \State $\overline{\pi}_2$ = $\vec{0} \in \mathcal{R}^{1 \times n}$
    \For t in steps
        \State \Comment{Compute the current policy}
        \State $\pi_1$ = RegretMatching($r_1$)
        \State $\pi_2$ = RegretMatching($r_2$)
        \State \Comment{Compute the current reward vector}
        \State $x_1 = \T{(A \T{\pi_2})}$
        \State $x_2 = \pi_1 A$
        \State \Comment{Compute the received value}
        \State $v_1 = \pi_1 \T{x_1}$
        \State $v_2 = \pi_2 \T{x_2}$
        \State \Comment{Compute current regret and add it to the cumulative regret}
        \State $r_1 += x_1 - v_1$
        \State $r_2 += x_2 - v_2$
        \State \Comment{Cumulate average policies}
        \State $\overline{\pi}_1 += \pi_1$
        \State $\overline{\pi}_2 += \pi_2$
    \EndFor
    \State \Comment{Return average policies}
    \Return $\{ \overline{\pi}_1 / steps, \overline{\pi}_2 / steps \}$
\EndFunction
\end{algorithmic}
\end{algorithm}

\subsection{Convergence Analysis}
\second{
There is a crucial connection between the average regret $\frac{R^T}{T}$ and the optimality of the average strategy of the \selfplay{}.
As the regret essentially measures by how much we could have improved and the regrets in \selfplay{} are a function of opponent's policy, there is a direct connection to $\epsilon$-Nash equilibrium (Theorem~\ref{thm:iig:folk_theorem}).

\begin{thm}
\label{thm:iig:folk_theorem}
Let the accumulated regret in \selfplay{} of the players be $R^T_1$ and $R^T_2$.
The averaged strategy profile $(\pi_1, \pi_2)$ is then $\epsilon$-Nash for
$\epsilon = \frac{R^T_1 + R^T_2}{T}$.
\end{thm}

\begin{proof}
WLOG we show that player $1$ cannot improve by more than $R^T_1 + R^T_2$ by switching to an arbitrary
strategy.
First, notice that $\sum_t \bar{\pi_1} A \pi^t_2 = \sum_t \bar{\pi_1} A \bar{\pi_2} = \sum_t \pi^t_1 A \bar{\pi_2}$.
\begin{align*}
    \max_{\pi^*_1} (\sum_t \pi^*_1(A) \bar{\pi_2}) -  \sum_t \pi^t_1 (A) \pi^t_2 = \max_{\pi^*_1} (\sum_t \pi^*_1(A) \pi^t_2) -  \sum_t \pi^t_1 (A) \pi^t_2  \le R^T_1 \\
    \sum_t \bar{\pi_1}(-A) \bar{\pi_2}  -  \sum_t \pi^t_1 (-A) \pi^t_2  \le \max_{\pi^*_2} (\sum_t \bar{\pi_1}(-A) \pi^*_2)  -  \sum_t \pi^t_1 (-A) \pi^t_2 \le R^T_2 \\
    \sum_t \bar{\pi_1}(A) \bar{\pi_2}  -  \sum_t \pi^t_1 (A) \pi^t_2 \ge -R^T_2 \\
    \max_{\pi^*_1} (\sum_t \pi^*_1(A) \bar{\pi_2}) - \sum_t \bar{\pi_1}(A) \bar{\pi_2} \le  R^T_1 + R^T_2 \\ 
    \max_{\pi^*_1} \pi^*_1(A) \bar{\pi_2} -  \bar{\pi_1}(A) \bar{\pi_2} \le \frac{R^T_1 + R^T_2}{T}
\end{align*}
\end{proof}

Theorem~\ref{thm:iig:folk_theorem} connects the average regret and the quality of the average strategy.
Thus, all we need is a regret minimizer with a sub-linear regret growth (Hannan consistent) and we have a method that provably converges to Nash equilibria.
}

\subsection{Possible Bounds}
\second{
Regret matching and many other regret minimizers provably enjoy $\frac{1}{\sqrt{T}}$ convergence speed of the average regret.
This matches the lower bound of the general adversarial case --- for any algorithm, it is possible to construct an adversarial sequence where the average regret converges no faster \citep{nisan2007algorithmic}.

When the environment is not strictly adversarial, it is possible to leverage the structure in the predictive regret framework.
A particularly useful structure is when the reward vector does not drastically change over time --- gradual variations \citep{chiang2012online}.
The idea of gradual variations can be leveraged in \selfplay{},  as we get to ``controll'' both sides.
The reward vector is then unlikely to shift drastically as the agents' policies change only gradually.
This structure is then used in the predictive regret framework, where we predict the next reward vector.
The better the prediction, the faster the convergence and there are predictive regret algorithms with $\frac{1}{T}$ convergence speed \citep{syrgkanis2015fast, farina2019optimistic, farina2020faster}.
This speed is then optimal for the \selfplay{} settings, as it matches the corresponding lower bound \citep{daskalakis2011near} .
}

\section{Alternating Updates}
\label{sec:iig-alternating_updates}
\second{
When we applied regret minimization in \selfplay{}, we updated both players at the same time.
The players used regret minimizer to produce their policy $(\pi^t_1, \pi^t_2)$ which is then used to compute the current regret.
This is referred to as simultaneous updates (Figure \ref{fig:iig-simultaneous_updates}).

There is another option --- alternating updates --- where we update players one by one.
We use $(\pi^t_1, \pi^t_2)$ to run a regret minimization step for one agent, producing $\pi^{t+1}$.
That policy is then used for a regret minimizaton step of the other agent, where the current regret is based on $(\pi^{t+1}_1, \pi^t_2)$ (Figure \ref{fig:iig-alternating_updates}).

For simultaneous updates, the regret updates of both players are based on the $(\pi^t_1, \pi^t_2)$ sequence.
For alternating updates, one player uses  $(\pi^t_1, \pi^t_2)$ while the other one  $(\pi^{t+1}_1, \pi^t_2)$. 
Alternating updates in \selfplay{} is thus not symmetric --- the reward vectors are not being computed symmetrically as in simultaneous updates.
An important caveat is that Theorem \ref{thm:iig:folk_theorem} holds only for simultaneous updates as black box regret minimization with alternating updates might fail to converge \citep{farina2019online}.
One must use additional assumptions about the regret minimizer to recover the convergence guarantess.
Fortunately, regret matching is one such regret minimizer \citep{burch2019revisiting}.
}

\begin{figure}[ht]
    \includegraphics[width=\textwidth]{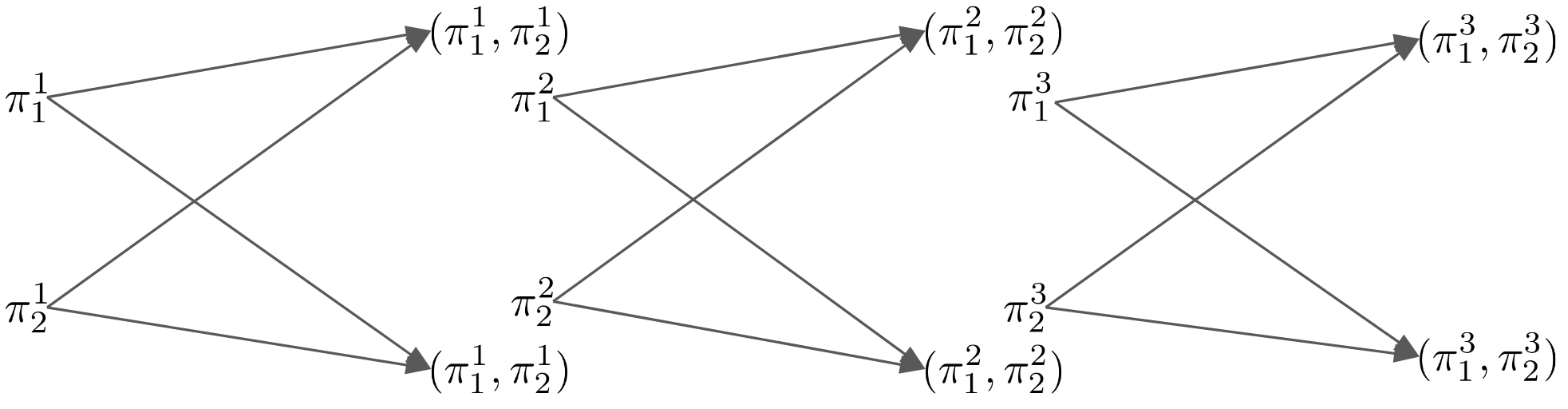}
    \caption{Regret minimization in \selfplay{} --- simultaneous updates.
    Both players update their policy at the same time and the regret update is symmetrical as it uses the same strategy pair for both player updates.}
    \label{fig:iig-simultaneous_updates}
\end{figure}

\begin{figure}[ht]
    \includegraphics[width=\textwidth]{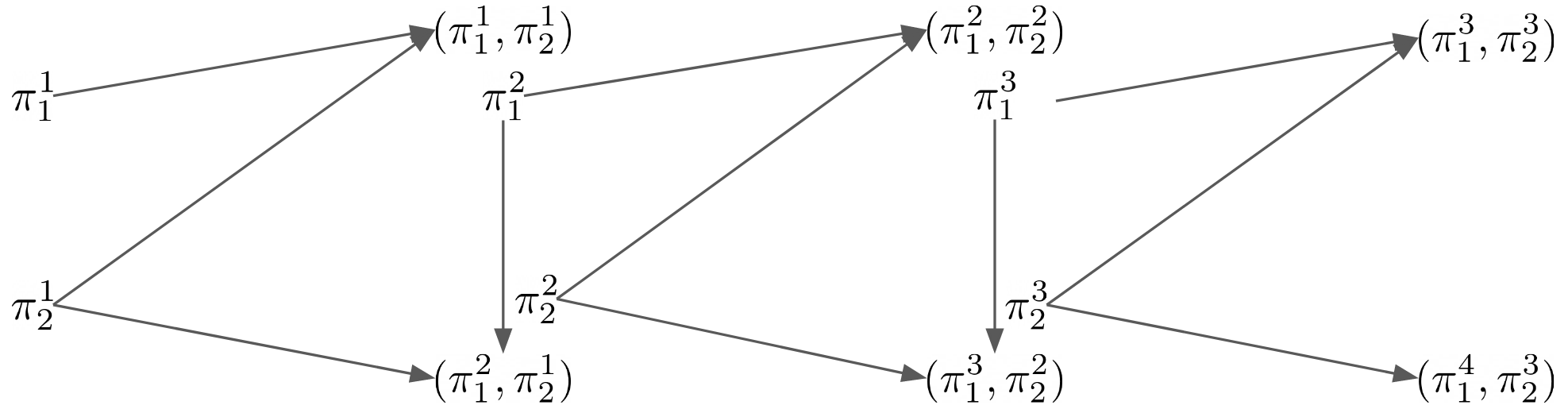}
    \caption{Regret minimization in \selfplay{} --- alternating updates.
    Players update their policy one by one and the regret update of one player uses different strategy pair than the opponent does.}
    \label{fig:iig-alternating_updates}
\end{figure}


\section{Counterfactual Regret Minimization}
\label{sec:iig-cfr}
\second{
In matrix games, we computed the regret against the best action in retrospect.
In other words, we compared to all possible pure strategies.
But what does that mean in sequential settings?
As we know from Section \ref{sec:igg-converting_formalism}, we can always convert to a matrix game.
While possible, this is extremely impractical due to the exponential size explosion.

Fortunately, it is possible to minimize regret directly in the sequential representation of the game.
\gls{cfr_g} (CFR) decomposes the full regret into small individual regrets in each information state (counterfactual regrets).
One can then bound the full regret by a sum of all these partial regrets, allowing one to minimize these partial regrets independently \citep{zinkevich2007regret}.
It is then possible to generalize the idea from sequential games to a more general case of convex sequential decision making \citep{farina2018online}.
}

\subsection{Counterfactual Values}
\second{
We first define state and state-action values.
Those terms (analogous to their reinforcement learning counterparts) will allow for a particularly easy view and definition of CFR.

The state value $v^\pi_i(h)$ is the expected future reward under policy $\pi$ to player $i$ given the history $h$.
The state-action value $q^{\pi}_i(h,a_i)$ is defined analogously, except that the player $i$ first takes the action $a_i$.
Counterfactual reach $P^{\pi}_{-i}(h)$ of a given history $h$ for player $i$ is the reach probability of that history when player $i$ attempts to reach $h$.

\begin{align*}
P^{\pi}_{-i}(h) := P_{\mathcal{T}}(h) \prod \nolimits_{j \in \mathcal N \setminus \{i\}} P^{\pi}_j(h)
\end{align*}

Now we can define counterfactual-weighted state and state-action values.
The counterfactual-weighted state-action value for player $i$ of action $a \in \mathcal{A}_i$ at state $s\in \mathcal{S}_i$ is (\ref{eq:iig-cf_q_values}).

\begin{align}
\label{eq:iig-cf_q_values}
q^{\pi}_{i,c}(s,a) := \sum\nolimits_{h \in \mathcal{H}(s)} P^{\pi}_{-i}(h) q^{\pi}_i(h,a)
\end{align}

Having these counterfactually-weighted action-values, we define the corresponding state-values as (\ref{eq:iig-cf_v_values}).

\begin{align}
\label{eq:iig-cf_v_values}
v_{i,c}^\pi(s) := \sum \nolimits_{a\in \mathcal{A}_i(s)} \pi_i(s,a) q^\pi_{i,c}(s,a)
\end{align}

Note that this value is not conditional on reaching $s$, as is standard in \rl{}, but instead depends on the counterfactual probability $P^{\pi}_{-i}$ of reaching $s$.
}

\subsection{Counterfactual Regret}
\second{
Given a strategy sequence $\pi^0, \dots, \pi^{t-1}$, we can use counterfactual state and state-action values to define the counterfactual regrets as follows:
\begin{align}
R^t_i(s,a) &= \sum\nolimits_{k=0}^{t-1} \left( q^{\pi^k}_{i,c}(s,a) - v^{\pi^k}_{i,c}(s) \right) \label{eq:iig-counterfactual_regret}\\
R^t_i(s) &= \max_{a \in \mathcal{A}(s)} R^t_i(s,a)
\label{eq:iig_cfr_regret}
\end{align}

Setting $ R^t_i(s)^{+} = \max(R^t_i(s), 0)$, the key result is that one can bound the full regret by the sum of the positive counterfactual regrets (Theorem \ref{thm:iig-cfr})  \citep{zinkevich2007regret}.

\begin{thm}
\label{thm:iig-cfr}
\citep[Theorem 3]{zinkevich2007regret}
\begin{align*}
R^t_i \leq \sum_{s \in \mathcal{S}_i}   R^t_i(s)^{+}
\end{align*}
\end{thm}

Theorem \ref{thm:iig-cfr} has critical consequences.
It decomposes the full regret into a sum of immediate regrets and allows us to independently minimize the immediate state regrets \ref{eq:iig_cfr_regret} (e.g. using regret matching).
Furthermore, computing the counterfactual values, regrets and updates can be done via a single tree traversal (Section \ref{sec:iig-cfr_public_tree_implementation}).
The same decomposition idea to individual partial regrets can then be generalized to sequential decision processes for convex sets \citep{farina2019online}.
}

\subsection{Average Policy}
\second{
Combining Theorem \ref{thm:iig-cfr} and Theorem \ref{thm:iig:folk_theorem} guarantees that if we use Hannan consistent regret minimizer in each infostate, the average strategy converges to optimal.
During the averaging, one must not forget to properly weight the behavioral strategy by the current reach probability (Section \ref{sec:pig-strategy_averaging}).
}

\subsection{Counterfactual Best Response}
\label{sec:iig-counterfactual_best_response}
\first{
Policy $\pi_i$ with non-negative counterfactual regret sum $\sum_{s \in \mathcal{S}_i} R_i(s)^{+} = 0$ is optimal and thus also a best response $\pi_i \in \brset$ (Lemma~\ref{lem:pig-nash_best_response}).
But not every best response $\pi'_i$ has zero regrets.
If a state $s \in \mathcal{S}_i$ has zero-reach probability $P_{\pi'_i}(s) = 0$, the policy might form a best response while $R_i(s)^{+} > 0$.
We thus define a refinement of best response policy --- the \gls{counterfactual_best_response}.
This concept can be thought of as a particular notion of \subgame{} perfect policy and  will be beneficial in later sections.

\begin{defn}[Counterfactual Best Response]
Counterfactual best response against a policy $\pi_i$ is a policy of the opponent with zero counterfactual regrets for all $s \in \mathcal{S}_{-i}$.
\end{defn}

We denote the set of such counterfactual best response policies as $\cfbrset(\pi_i)$ and we use $\cfbrv^{\pi_i}_{-i}(s)$ to denote the counterfactual best response value of a state $s \in \mathcal{S}_2$ --- counterfactual value given that the player $i$ follows $\pi_i$ and the opponent follows $\pi_{-i} \in \cfbrset(\pi_i)$.
}

\subsection{Public Tree Implementation}
\label{sec:iig-cfr_public_tree_implementation}

\begin{figure}[ht]
  \centering
\includegraphics[width=\textwidth]{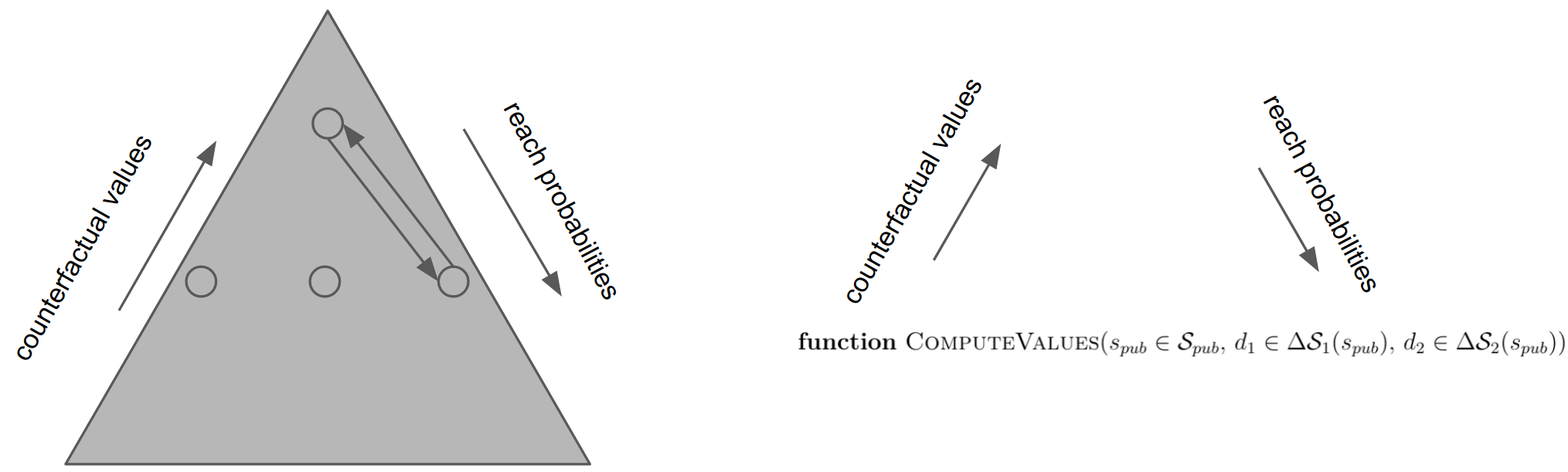}
  \caption{Public tree implementation of CFR traverses the public tree and send the reach probabilities down the tree and counterfactual values up the tree.}
\label{fig:iig-cfr_public_trees}
\end{figure}

\second{
The CFR algorithm iteratively i) produces policy $\pi^t_i(s)$ using counterfactual regret in individual states ii) computes the counterfactual values and updates the regrets iii) averages the policy.
One such iteration can be efficiently implemented via a single game tree traversal, where the traversal sends the reach probabilities down the tree (to be used as the $P^{\pi}_{-i}$ reach weights as in Equation~\ref{eq:iig-cf_q_values}) and values up the tee (propagating the values as in Equation~\ref{eq:iig-cf_q_values}).
Algorithm \ref{alg:iig_cfr} then shows an implementation where we traverse a public tree.
While there is no strong reason for this choice now and one could traverse the history or infoset trees, a public tree implementation will be beneficial once we will use value functions within the CFR algorithm.
Note that the public tree recursive traversal operates with reach probabilities and counterfactual values for all the states of a public state (Figure \ref{fig:iig-cfr_public_trees}), which will become particularly useful for the use of value functions.
}

\begin{algorithm}
\caption{Public Tree CFR}
\label{alg:iig_cfr}
\begin{algorithmic}[1]


\Function{PublicTreeCFR}{}

\For{$i : [1..T]$}
\State \Call{ComputeNextPolicies}{}
\State \Call{UpdateAveragePolicies}{}
\State \Call{ComputeValues}{$root_{pub}, 1, 1$}
\State \Call{UpdateRegrets}{}
\EndFor
\State \Call{NormalizeAveragePolicies}{}
\EndFunction
\State
\State \Comment{Recursively traverses the public tree, sending reach probabilities down and counterfactual values up.}
\Function{ComputeValues}{$s_{pub} \in \mathcal{S}_{pub}$, $d_1 \in \Delta \mathcal{S}_1(s_{pub})$, $d_2 \in \Delta \mathcal{S}_2(s_{pub})$}
\label{alg:func-compute_values}
  \State \Comment{Terminal public state.}
  \If{$s_{pub} \in \mathcal{Z}_{pub}$}
      \For{$i \in \mathcal{N}$}
        \For{$s \in \mathcal{S}_i(s_{pub})$}
            \State $P_{-i}(h) \gets d_{-i}(s) P_{\mathcal{T}}(h)$
            \State $v_{i,c}(s) = \sum_{h \in \mathcal{H}(s)} P_{-i}(h) R_i(h)$
        \EndFor
      \EndFor
      \Return
  \EndIf
  \State \Comment{Update the distribution for children.}
  \For{$i \in \mathcal{N}$}
    \For{$s \in \mathcal{S}_i(s_{pub})$}
        \For{$sao \in \mathcal{S}_i$}
            \State $d_i(sao) = d_i(s) \pi_i(a)$
        \EndFor
    \EndFor
  \EndFor
  
  \State \Comment{Recursion.}
  \For{$s_{pub}o_{pub} : \mathcal{S}_{pub}$}
    \State \Call{ComputeValues}{$s_{pub}o_{pub}$, $d_1(s_{pub}o_{pub})$}
  \EndFor
  
  \State \Comment{Compute state-action values.}
  \For{$i \in \mathcal{N}$}
    \For{$s \in \mathcal{S}_i(s_{pub})$}
        \State $v_{i,c}(s) = \sum_{sao \in \mathcal{S}_i} \pi(s, a) v_{i,c}(sao)$
        \For{$a \in \mathcal{A}_i(s)$}
            \State $q_{i,c}(s, a) = \sum_{sao \in \mathcal{S}_i} v_{i,c}(sao)$
        \EndFor
    \EndFor
  \EndFor
\EndFunction
\State
\Function{UpdateRegrets}{}
\For{$s \in \mathcal{S}$}
    \For{$a \in \mathcal{A}(s)$}
        \State $R(s, a) += q_{i,c}(s, a) - v_{i,c}(s)$
    \EndFor
\EndFor
\EndFunction
\State
\Function{ComputeNextPolicies}{}
\For{$s \in \mathcal{S}$}
    \State $\pi(s)$ = \Call{RegretMatching}{$R(s)$}
\EndFor
\EndFunction

\end{algorithmic}
\end{algorithm}

\section{CFR+}
\label{sec:iig-cfr_plus}
\second{
CFR+ is an algorithm with an interesting history and impressive empirical performance.
Initially discovered by a hobbyist scientist \citep{tammelin2014solving}, CFR+ drastically improved state of the art methods and has been used to solve limit Texas hold'em poker --- the largest imperfect information game played by humans to be solved to this day \citep{bowling2015heads}, and the resulting agent Cepheus has also been open sourced\footnote{\href{http://poker.srv.ualberta.ca}{http://poker.srv.ualberta.ca}}.
}

\second{
CFR+ modifies vanilla CFR\footnote{As originally introduced in \citep{zinkevich2007regret}.} in three ways: 
i) it uses regret matching plus 
ii) it uses alternating updates
iii) and it uses linear averaging of the policy.

\paragraph{i) Regret Matching Plus}
\second{
\Gls{regret_matching_plus} (RM$^+$) selects actions just like regret matching, except that it caps all the regrets to zero after each iteration.
That is, we set $R(a) \leftarrow R(a)^+$.
RM$^+$ enjoys the same regret bounds as RM does.
}

\paragraph{ii) Alternating updates}
See Section \ref{sec:iig-alternating_updates}

\paragraph{iii) Linear averaging}

The average strategy in Theorem~\ref{thm:iig:folk_theorem} is a uniformly weighted average of the individual (current) policies $\overline{\pi_i} = \frac{1}{T} \sum^T_{t = 1} \pi^t_i $.
CFR+ rather uses linearly increasing weight for the later iterations $\overline{\pi_i} = \frac{2}{(T)(T+1)} \sum^T_{t = 1} t \pi^t_i $.
This averaging has been proven to be sound for CFR+, but does not work in general as the analysis does not apply to CFR.
}

\subsection{Empirical Performance}
\second{
CFR+ has some exciting empirical performance.
It often greatly outperforms CFR and even many algorithms with a $\frac{1}{T}$ convergence bound.
Figure~\ref{fig:iig-cfr+_vs_cfr} compares the convergence speed of CFR+ and CFR on Leduc poker and \glasses{}.
}

\begin{figure}[ht]
    \centering
    \begin{subfigure}[b]{0.45\textwidth}
        \includegraphics[width=\textwidth]{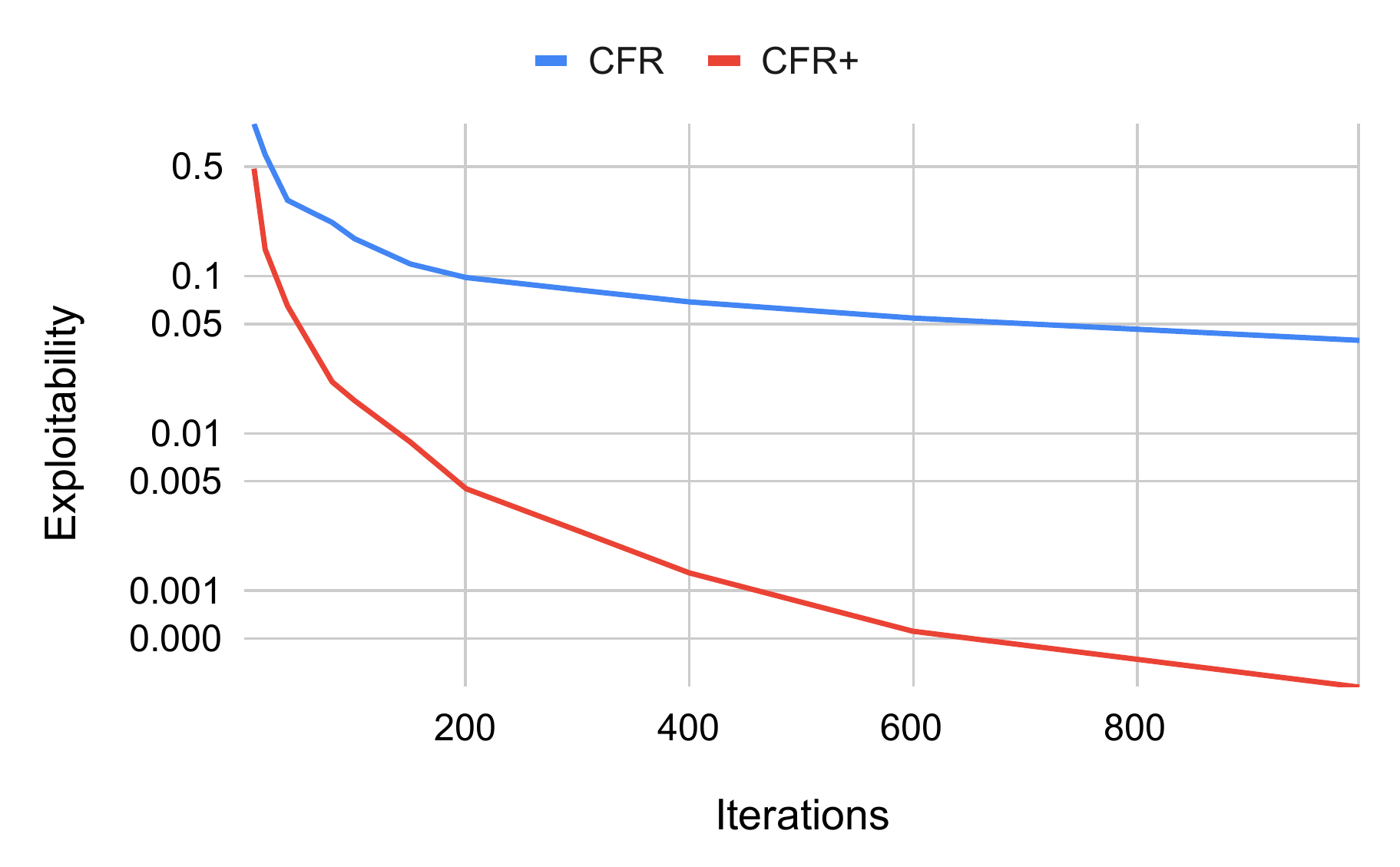}
        \caption{Leduc poker.}
    \end{subfigure}
    ~ 
    \begin{subfigure}[b]{0.45\textwidth}
        \includegraphics[width=\textwidth]{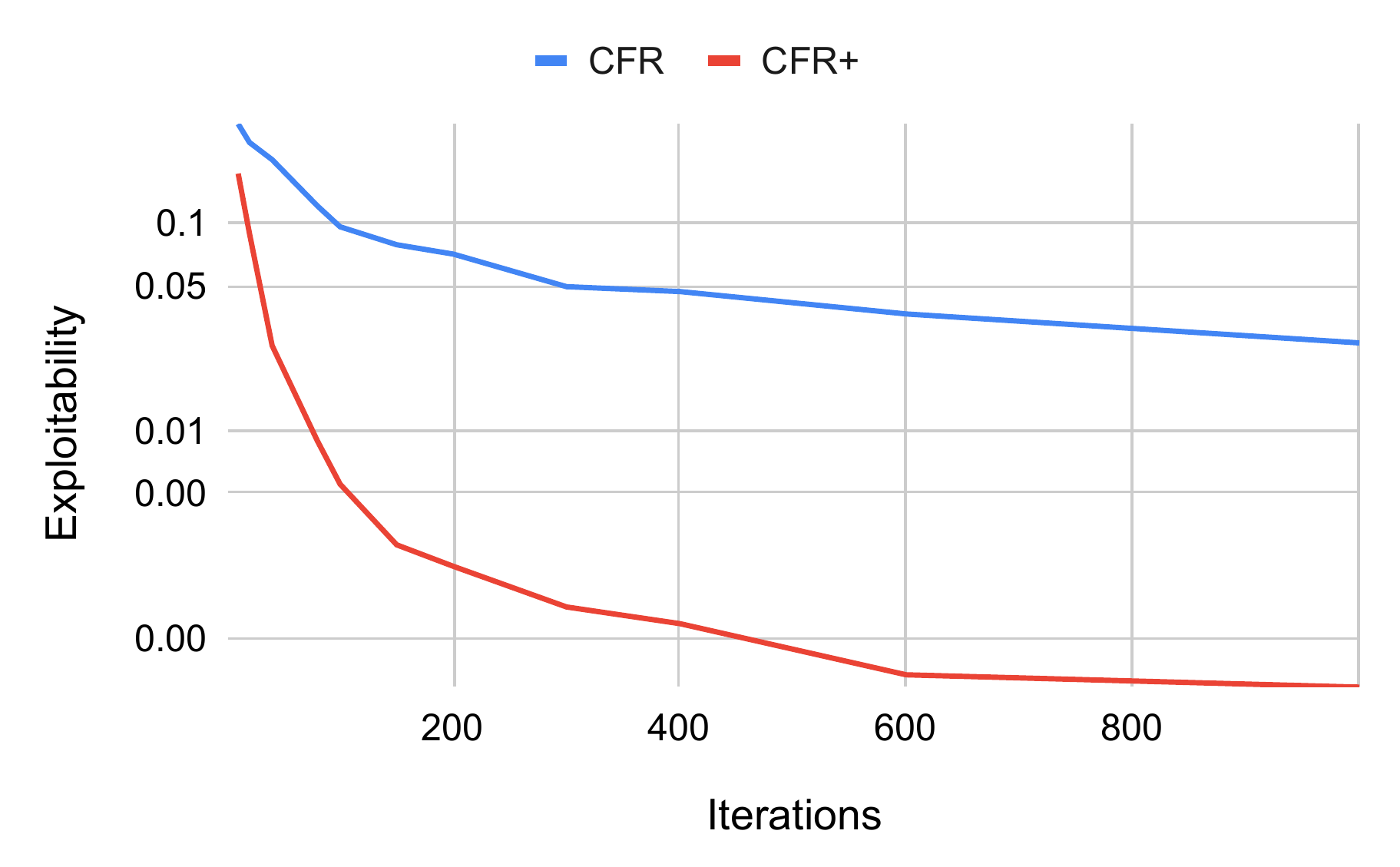}
        \caption{\Glasses{}.}
    \end{subfigure}
    \caption{
    Convergence of CFR and CFR+ tabular solvers.
    }
    \label{fig:iig-cfr+_vs_cfr}
\end{figure}

\second{
Why does CFR+ drastically outperform CFR and often converges much faster than its theoretical bound of $\frac{1}{\sqrt{T}}$?
The correct answer is that no-one really knows, though there are some appealing properties.
CFR+ has been initially analysed in \citep{tammelin2015solving}, where the authors show that the algorithm also guarantees a sub-linear tracking regret \citep{herbster1998tracking}.
Furthermore, alternating updates in the case of CFR+ are also soundly motivated \citep{burch2019revisiting}.

Unfortunately, it seems that the strong empirical performance of the algorithm is indeed just empirical.
While CFR+ often converges much faster than its bound would suggest, there are known cases where the convergence speed is much closer to the $\frac{1}{\sqrt{T}}$ bound, performing no better than CFR.
The common belief is that this algorithm's convergence rate is not $\frac{1}{T}$ in the worst case \citep{burch2018time, farina2019optimistic}.
}

\section{Monte Carlo CFR}
\label{sec:iig-mccfr}
\second{
While CFR methods traverse the full game tree at each iteration, Monte Carlo methods sample a subset of all such trajectories.
Technically, Monte Carlo CFR (MCCFR) is a family of algorithms defined by a partition $\mathcal{Q}$ of the terminal histories $\mathcal{Z}$, sampling a block $\mathcal{Q}_i \in \mathcal{Q}$ at each iterations and updating only the relevant information states for that block.
In combination with a proper importance sampling correction term, one can then guarantee that the values and updates are in expectation the same as in CFR, and derive a probabilistic bound for the regrets \citep{lanctot2009monte, Lanctot13phdthesis}.

This has a clear benefit of much faster per-iteration time, but the price one has to pay is that the values are noisy due to the variance introduced by the sampling.
Smaller blocks $Q_i$ provide faster iterations and higher variance.
Different sampling schemes thus essentially balance per-speed iteration and variance.
}

\subsection{Outcome Sampling}
\second{
We describe a particular member of the MCCFR family that samples a single history at each iteration.
That is, the partition $\mathcal{Q} = \mathcal{Z}$, making the notation particularly concise.
Furthermore, sampling a single history emphasizes similarities to some reinforcement learning algorithms and will simplify the reduced-variance variant of the algorithm presented in the next section.

Outcome sampling samples a single history $z \in \mathcal{Z}$, updating only the visited information states on the sampled trajectory.
To make sure each terminal history is sampled with non-zero probability, we restrict the sampling policy $\xi \in \Delta A(s)$ for all $s$ so that $\xi(s,a) > 0$ everywhere.
On each iteration $z \sim \xi$, the algorithm replaces counterfactual value with sampled counterfactual value $\tilde{v}^{\pi}_{i,c}(s | z)$.
\begin{align}
\label{eq:iig-sampled_cfv}
\tilde{v}^{\pi}_{i,c}(s | z) = \tilde{v}^{\pi}_{i,c}(h | z) = \frac{ P^{\pi}_{-i}(h) P^{\pi}(h,z) R_i(z)}{P^{\xi}(z)}, \mbox{ for } h \in s, h \sqsubseteq z, 
\end{align}
and $0$ for histories that were not played ($h' \not \sqsubseteq z$).
Note that the term is simply a counterfactual value computed from a single history $z$ with the importance sampling term $P^{\xi}(z)$.
Such estimate is then provably unbiased 
$\EX{}_{z \sim \xi}[\tilde{v}^{\pi}_{i,c}(s | z)] = v^{\pi}_{i,c}(s, I)$ \citep[Lemma 1]{lanctot2009monte}.
As a result, we can replace $v_{i,c}$ and $q_{i,c}$ in Equation~\ref{eq:iig-counterfactual_regret} with $\tilde{v}_{i,c}$ and $\tilde{q}_{i,c}$ respectively, resulting in the accumulate estimated regrets.

\begin{align}
 \tilde{R}^t(s,a) = \tilde{q}^{\pi}_{i,c}(s, a) - \tilde{v}^{\pi}_{i,c}(s)    
\end{align}

The corresponding regret bound then requires an additional term $\frac{1}{\min_{z \in \mathcal{Z}}P(z)}$, which is exponential in the length of $z$ (Theorem \ref{thm:iig-mccfr_regret_bound}).
Similar observations have been made in reinforcement learning \citep{arjona2018rudder}.

\begin{thm}
\label{thm:iig-mccfr_regret_bound}
\citep[Theorem 5]{lanctot2009monte}
For any $p \in (0,1]$, when using outcome-sampling MCCFR where $\forall z \in \mathcal{Z}$ either $P^{\pi}_{-i}(z) = 0$ or $P^{\xi}(z) \geq \delta > 0$ at every timestap, with probability $1-p$, average overall regret is bounded by $R^T_i \leq (1+\frac{\sqrt{2}}{\sqrt{p}}) (\frac{1}{\delta})\Delta_R \frac{\sqrt{|A_i|}}{\sqrt{T}}$
\end{thm}
}

\subsubsection{Recurrent Computation}
\second{
A particularly insightful formulation is to first recurrently formulate sampled state and state-action values $\tilde{q}^{\pi}_i(h, a | z) $ and $\tilde{v}^{\pi}_i(h | z)$, just like one would compute and propagate the values in a tree traversal.

\begin{align}
\label{eq:iig-mccfr_boostraped_q}
\tilde{q}^{\pi}_i(h, a | z) = \left\{ \begin{array}{ll}
     \tilde{v}^{\pi}_i(ha | z) / {\xi(h,a)}   & \mbox{if $ha \sqsubseteq z$} \\
     0 & \mbox{if $h \sqsubset z$, $ha \not\sqsubseteq z$} \\
     0 & \mbox{otherwise}\end{array} \right.
\end{align}
 
\begin{align}
\label{eq:iig-mccfr_boostraped_v}
\tilde{v}^{\pi}_i(h | z) = \left\{ \begin{array}{ll}
R_i(h) & \mbox{if $h = z$} \\
\sum_{a} \pi(h,a) \tilde{q}^{\pi}_i(h, a | z) & \mbox{if $h \sqsubset z$} \\
0 & \mbox{otherwise}\end{array}\right.
\end{align}

These values are then turned into corresponding sampled counterfactual values by

\begin{align}
\tilde{q}^{\pi}_{i,c}(h, a | z) = \frac{P^{\pi}_{-i}(h)}{P^{\xi}(h)} \tilde{q}^{\pi}_{i}(h, a | z)
\end{align}
}

\subsection{Convergence Speed}
\second{
Unfortunately, MCCFR methods usually converge slower than their full-tree traversal counterparts (even when the comparison takes into account the faster iteration time).
The main problem with the sampling variants is that they introduce variance that can have a significant effect on long-term convergence \citep{gibson2012generalized}.
Furthermore, CFR+ loses its magic and MCCFR+ is usually outperformed by MCCFR \citep{burch2018time}.
One way to look at this problem is that it is unlikely to get convergence near $\frac{1}{T}$ in the
sampling settings, as the central limit theorem tells us that the average value converges to the expected one at the rate of $\frac{\sigma}{\sqrt{T}}$.
Thus one would have to drastically decrease the sampling variance to get rid of this term.
}

\section{Variance Reduced MCCFR}
\label{sec:iig-vrmccfr}

\begin{figure}[ht]
  \centering
  \includegraphics[width=8cm]{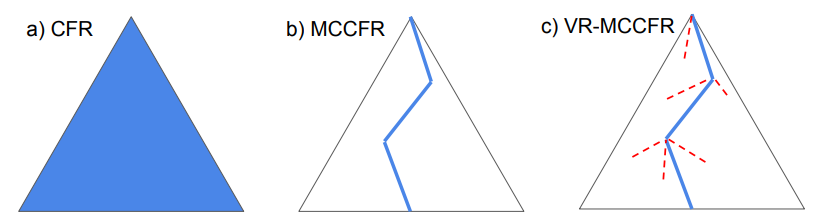}
  \caption{High-level overview of Variance Reduction MCCFR (VR-MCCFR) and related methods. 
  a) CFR traverses
the entire tree on every iteration. 
b) MCCFR samples trajectories and computes the values only for the sampled actions, while the off-trajectory actions are treated as zero valued. While MCCFR uses importance sampling weight to
ensure the values are unbiased, the sampling introduces high
variance. 
c) VR-MCCFR follows the same sampling framework as MCCFR, but uses baseline values for both sampled actions (in blue) as well as the off-trajectory actions (in red). 
These baselines use control variates and send up
bootstrapped estimates to decrease the per-iteration variance
thus speeding up the convergence.}
\label{fig:iig-vrmccfr}
\end{figure}

\second{
Variance Reduced MCCFR (VR-MCCFR) \citep{schmid2019variance} and its later variants \citep{davis2020low} are powerful methods that decrease the per-sample variance in MCCFR, resulting in orders of magnitude faster convergence.
VR-MCCFR is also another contribution of this thesis.

At the core of the method is the general idea of control variates.
The idea is to estimate the value of non-sampled actions and use the control variates technique to make sure the expectation of the action-values remains unbiased regardless of the estimates.
While any estimate leads to unbiased values, poor estimates can increase the variance rather than decrease it.
The closer the estimates are to the true value, the lower the variance.
Furthermore, VR-MCCFR recursively combines the estimates as it propagates values up the tree, propagating the benefits of the value estimates.
See Figure \ref{fig:iig-vrmccfr} for comparison of CFR, MCCFR and VR-MCCFR.
}

\subsection{Control Variates}
\label{sec:iig-control_variates}
\second{
Suppose we are trying to estimate a mean from samples $X = (X_1, X_2, \hdots, X_n)$.
The Monte Carlo estimator is then simply $\tilde{X}^{mc} = \frac{1}{n} \sum^n_{i=1} X_i$.
A control variate is a random variable $Y$ with a known mean $\EX[Y]$ to be paired with the original variable, resulting in a new random variable $Z_i = X_i + c(Y_i - \EX[Y])$ (with a corresponding Monte Carlo estimator $\tilde{Z}^{mc}$).
Since $E[Z_i] = \EX[X_i]$ for any $c$ we can use $\tilde{Z}^{mc}$ instead of $\tilde{X}^{mc}$
with variance $Var[Z_i] = Var[X_i] + c^2Var[Y_i] + 2cCov[X_i, Y_i]$ \citep{owen2016monte}.
So when $X$ and $Y$ are positively correlated and $c < 0$, variance is reduced when $Cov[X,Y] > \frac{c^2}{2} Var[Y]$.
}

\subsection{Baseline Enhanced Estimated Values}
\second{
VR-MCCFR constructs value estimates using control variates.
We first define an action-dependent baseline $b_i(s, a)$ to be used as a control variate to approximate or be correlated with $\EX[\tilde{q}^{\pi}_i(s, a)]$.
We also define a sampled baseline $\tilde{b}_i(s, a)$ for which $\EX[\tilde{b}_i(s, a)] = b_i(s, a)$.  
We can now readily use these terms in the control variate framework to construct a new baseline-enhanced estimate for the state-action values:

\begin{align}
\label{eq:baseline-enhanced-cfv}
\tilde{q}^{b, \pi}_{i,c}(s, a) = \tilde{q}^{\pi}_{i,c}(s, a) - \tilde{b}_i(s, a) +  b_i(s, a) 
\end{align}

First, note that $\tilde{b}_i$ is a control variate with $c = -1$ and thus the expectation remains unchanged (Lemma \ref{lemma:iig-vrmccfr_baseline_unbiased}).
Equation \ref{lemma:iig-vrmccfr_baseline_unbiased} essentially corresponds to a local adjustment of sampled counterfactual values in a single state.
Figure \ref{fig:iig-vrmccfr_comparison} illustrates such local computation for CFR, outcome sampling MCCFR and outcome sampling VR-MCCFR.
}

\begin{lemma}
\label{lemma:iig-vrmccfr_baseline_unbiased}
\citep[Lemma 1]{schmid2019variance}
For any $i \in \mathcal{N} - \{c\}, \pi_i, s \in \mathcal{S}_i, a \in \mathcal{A}_i(s)$, if $\EX[\tilde{b}_i(s, a)] = b_i(s, a)$ and $\EX[\tilde{q}^{\pi}_{i,c}(s, a)] = q^{\pi}_{i,c}(s, a)$, then $\EX[\tilde{q}_{i,c}^{b, \pi}(s, a)] = q^{\pi}_{i,c}(s, a)$.
\end{lemma}

\begin{figure}[ht]
    \centering
    \begin{subfigure}[b]{0.3\textwidth}
        \includegraphics[width=\textwidth]{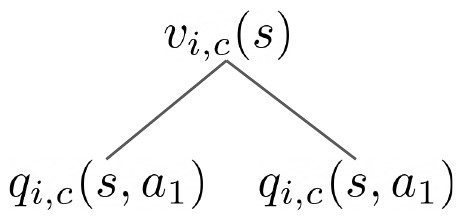}
        \caption{}
    \end{subfigure}
    \begin{subfigure}[b]{0.3\textwidth}
        \includegraphics[width=\textwidth]{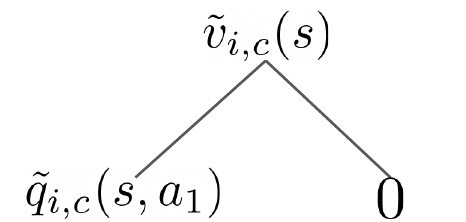}
        \caption{}
    \end{subfigure}
    \begin{subfigure}[b]{0.3\textwidth}
        \includegraphics[width=\textwidth]{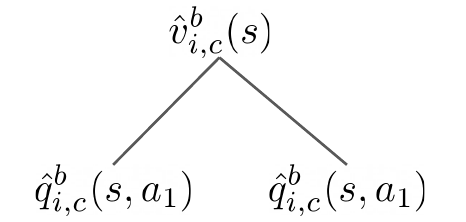}
        \caption{}
    \end{subfigure}
    \caption{a) CFR computes the exact value by traversing all the actions.
             b) Outcome sampling MCCCFR samples a single action and uses zero for the values of the non-sampled actions.
             c) VR-MCCFR uses baselines to adjust the estimates of all the actions. }
\label{fig:iig-vrmccfr_comparison}
\end{figure}

\subsubsection{Recursive Bootstrapping}
\second{
Given the recursive nature of the counterfactual values propagation, we can propagate the already baseline-enhanced
counterfactual values rather than the noisy sampled values.
This allows us to propagate the benefits up the tree rather than using the control variates trick in isolation.
Just as we recursively formulated sampled state and state-action values (Equation \ref{eq:iig-mccfr_boostraped_q} and \ref{eq:iig-mccfr_boostraped_q}) that were then used to formulate the sampled counterfactual values for outcome sampling, we recursively formulate the baseline-enhanced variants.

\begin{align}
\label{eq:bootstrapped-ub}
\tilde{q}^{b, \pi}_i(h, a | z) = \left\{ \begin{array}{ll}
 b_i(s(h), a) + \frac{\tilde{v}^{b, \pi}_i(ha | z) - b_i(s(h), a)}{\xi(h,a)} & \mbox{if $ha \sqsubseteq z$} \\
 b_i(s(h), a) & \mbox{if $h \sqsubset z$, $ha \not\sqsubseteq z$} \\
 0 & \mbox{otherwise}\end{array} \right.
\end{align}
and
\begin{align}
\label{eq:bootstrapped-ub-history}
\tilde{v}^{b, \pi}_i( h | z) = \left\{ \begin{array}{ll}
  u_i(h) & \mbox{if $h = z$} \\
  \sum_{a} \pi(h,a) \tilde{q}^{b, \pi}_i(h, a | z) & \mbox{if $h \sqsubset z$} \\
  0 & \mbox{otherwise} \end{array}\right.
\end{align}

Just like in outcome sampling, we then turn these estimates of state and state-action values into estimates of counterfactual values, which are in turn used to update the regrets.

\begin{align}
\label{eq:iig-vrmmcfr_bootstrapped_cfv}
\tilde{q}^{b, \pi}_{i,c}(s, a | z) = \tilde{q}^{b,\pi}_{i,c}(h, a | z) = \frac{P^{\pi}_{-i}(h)}{P^{\xi}(h)} \tilde{q}^{b, \pi}_{i}(h, a | z)
\end{align}

The resulting values can be then 
Note that these values collapse to the original outcome sampling for $b_i(s, a) = 0$ ($\tilde{q}^{b, \pi}_i(s, a)$ becomes $\tilde{q}^{\pi}_i(s, a)$).
Outcome sampling (and other MCCFR variants) can thus be viewed as VR-MCCFR algorithm with particular choice of zero baseline.

\begin{lemma}
\label{lemma:iig-vrmccfr_bootstrapped_unbiased}
\citep[Lemma 2]{schmid2019variance}
Let $\tilde{q}_{i,c}^{b, \pi}$ be defined as in Equation~\ref{eq:iig-vrmmcfr_bootstrapped_cfv}. 
Then, for any $i \in \mathcal{N} - \{c\}, \pi_i, s \in \mathcal{S}_i, a \in \mathcal{A}(s)$, it holds that $\EX_z[\tilde{q}_{i,c}^{b, \pi}(s, a | z)] = q_{i,c}^{\pi}(s, a)$.
\end{lemma}

}

\subsection{Choice of Baselines}
\second{
While any choice of baseline leads to unbiased estimates, the closer the estimate to the true value, the better (Theorem \ref{thm:iig-vrmccfr_baseline_variance}).
A typical choice of the estimates is to simply use values from previous iterations (e.g. with exponentially decaying weight), as this value is unlikely to drastically change.
This simple choice already leads to a drastic speed improvement over MCCFR.

\begin{thm}{\cite[Theorem 2]{gibson2012generalized}}
\label{thm:iig-vrmccfr_baseline_variance}
For some unbiased estimator of the counterfactual values $\tilde{v}_i$ and a bound on the difference in its value $\tilde{\Delta}_i = |\tilde{v}_i(\pi, I, a) - \tilde{v}^{\pi}_i(I, a')|$, with probability $1-p$
\[
\frac{R_i^T}{T} \leq \left( \tilde{\Delta}_i + \frac{\sqrt{ \max_{t, s, a} Var[r_i^t(s,a) - \tilde{r}_i^t(s,a)] }}{\sqrt{p}} \right) \frac{|\mathcal{S}_i| |\mathcal{A}_i|}{\sqrt{T}}.
\]
\end{thm}

Another key result is that there exists a perfect baseline that leads to zero-variance estimates at the updated information sets (Lemma~\ref{lemma:iig-vr_mccfr_oracle_baseline}).
Since the oracle baseline corresponds to the exact value, it is as hard to compute as the full tree traversal of CFR and thus impractical.
On the other hand, it can be used to analyze the performance of the algorithm and how our choice of baseline compares to the optimal one.
Note that even with the oracle baseline, VR-MCCFR is still not identical to CFR as VR-MCCFR only updates a subset of information states (albeit with the exact counterfactual values) while CFR updates the entire tree.

\begin{lemma}{\cite[Lemma 3]{schmid2019variance}}
\label{lemma:iig-vr_mccfr_oracle_baseline}
There exists a perfect baseline $b^*$ and optimal unbiased estimator $\tilde{v}^{*,\pi}_i(s, a)$ such that under a specific update scheme:  $Var_{h,z \sim \xi, s \in \mathcal{S}, h \sqsubseteq z}[\tilde{v}^{*, \pi}_i(h, a|z)] = 0$.
\end{lemma}
}

\subsection{Convergence Speed and Variance}
\second{
We evaluate the algorithm on Leduc poker with exponentially decaying weight for the average baseline (weight $\alpha = 0.5$), comparing MCCFR, MCCFR+, VR-MCCFR, VR-MCCFR+, and VR-MCCFR+ with the oracle baseline.
Variance reduced variants converge significantly faster than MCCFR and the relative speedup grows as the baseline improves during the computation. 
VR-MCCFR+ achieves two orders of magninute speedup compared to MCCFR.
VR-MCCFR+ with the oracle baseline initially outperforms VR-MCCFR+, but they get closer as time progresses and the learned baseline improves (Figure~\ref{fig:iig-vrmccfr_convergence}).
}

\begin{figure}[ht]
  \centering
  \includegraphics[width=8cm]{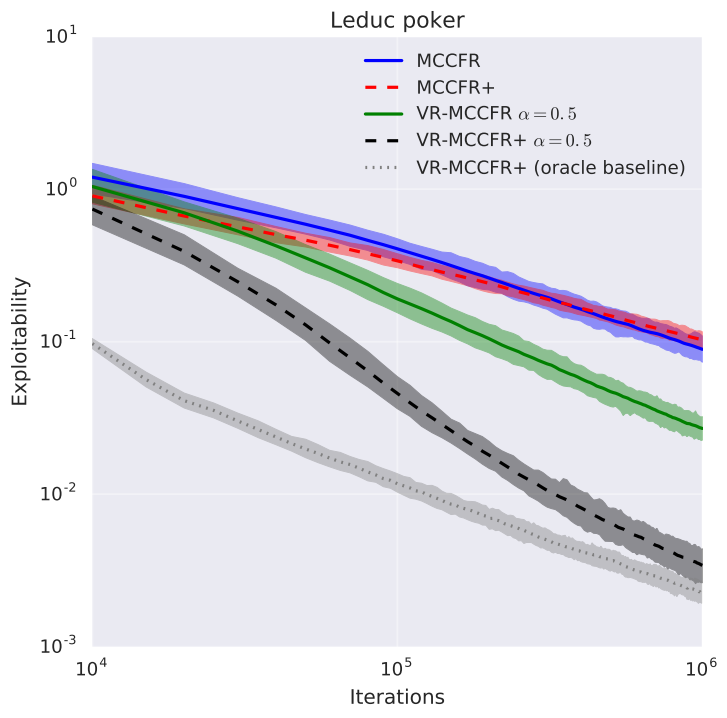}
  \caption{Convergence of exploitability for different MCCFR variants on logarithmic scale. 
VR-MCCFR converges
substantially faster than plain MCCFR. VR-MCCFR+ bring
roughly two orders of magnitude speedup. 
VR-MCC
with oracle baseline (actual true values are used as baselines) is used as a bound for VR-MCCFR’s performace to
show possible room for improvement. 
When run for $10^6$ iterations VR-MCCFR+ approaches performance of the oracle version. 
The ribbons show 5th and 95th percentile over 100 runs.}
\label{fig:iig-vrmccfr_convergence}
\end{figure}

\second{
To verify that the observed speedup is indeed thanks to the lower variance, we measured the variance of counterfactual value estimates for MCCFR+ and MCCFR (Figure~\ref{fig:iig-vrmccfr_variance}).
We sampled $1,000$ trajectories rooted in visited information sets, with each trajectory sampling a different estimate of the counterfactual value. 
While the variance of value estimates in MCCFR seems to be more or less constant during the computation, the variance of VR-MCCFR and VR-MCCFR+ value estimates is substantially lower and continues to decrease. 
}

\begin{figure}[ht]
  \centering
  \includegraphics[width=0.5\textwidth]{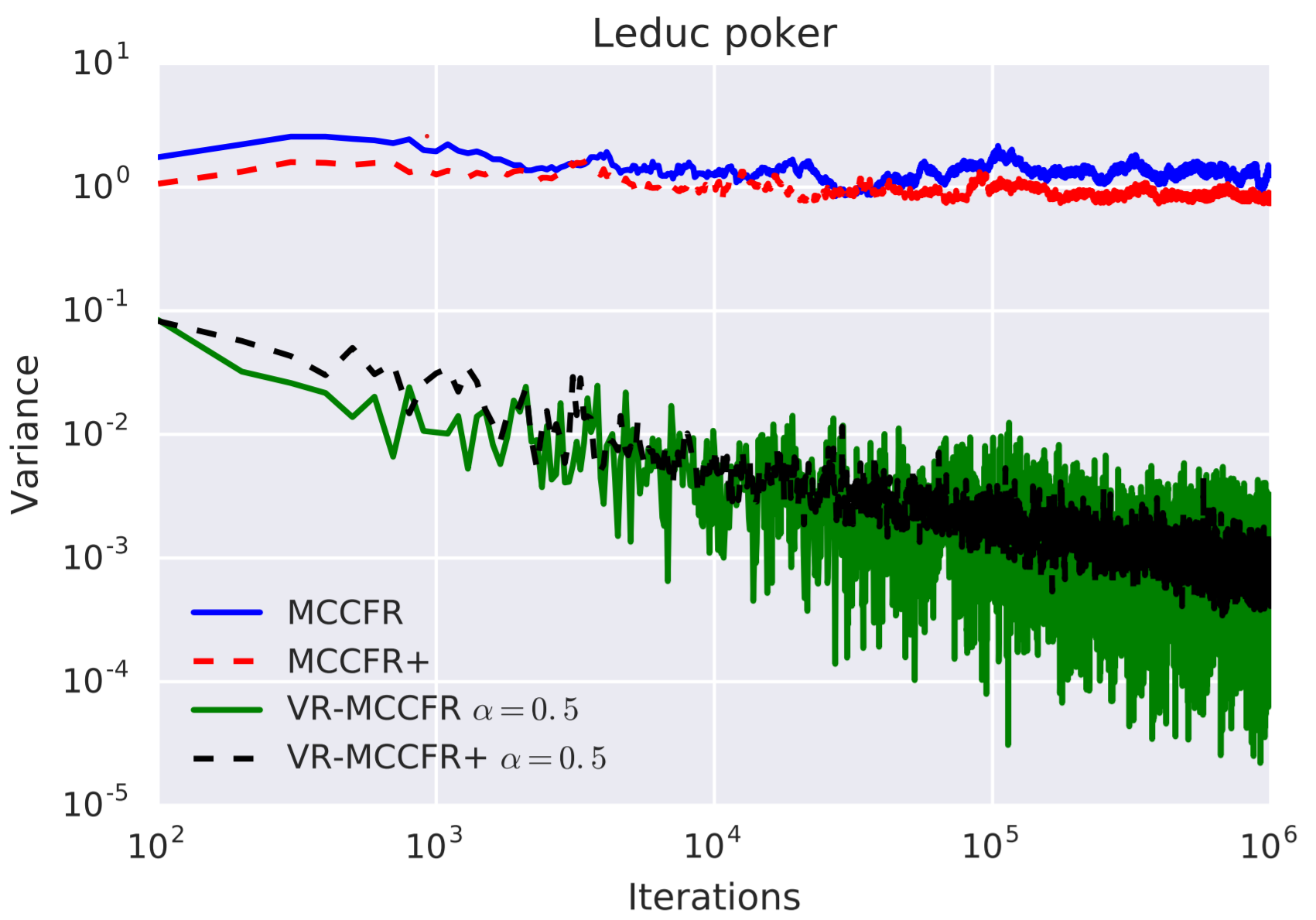}
  \caption{The empirical variance of the counterfactual values is significantly lower for VR-MCCFR and further decreases as the estimates improves over time. }
\label{fig:iig-vrmccfr_variance}
\end{figure}

\subsection{Connection to RL Baseline Methods}
\second{
VR-MCCFR has close connections to well-used baselines in reinforcement learning
methods (e.g. REINFORCE algorithm with baselines \citep{williams1992simple}).
That is simply because these baseline techniques of reinforcement learning are simply control variates in disguise.
An important difference in imperfect information settings of VR-MCCFR, state-action baselines perform significantly faster than state baselines.
In single agent environments of reinforcement learning, state-action baselines tend to perform no better than state baselines \citep{tucker2018mirage}.
This is likely due to the fact that the optimal policy is deterministic and one ends up sampling the action with the best value.
In imperfect information, this is not the case as we have to keep sampling all the actions at each iteration to guarantee convergence.
}

\chapter{Offline Solve III - Approximate and Abstraction Methods}
\label{chap:iig-abstraction}
\second{
Tabular methods are limited to games small enough for strategies to be stored explicitly.
In the same way there is no hope to use tabular methods to solve chess, we can not hope to use this approach for large imperfect information games.
For perfect information games, a common approach is then to use the online search methods.
But as this was for a long time thought to be impossible in imperfect information, there are some other non-search methods one can opt for. 
}

\section{Abstraction Methods}
\label{sec:iig-abstraction_methods}
\second{
First method is to simply use the tabular methods, but apply those on a smaller, abstracted game (Figure \ref{fig:iig_abstraction_translation}).
The hope is that if the abstracted game strategically resembles the full game, the resulting strategy will approximate the optimal solution.

\begin{figure}[ht]
  \centering
  \includegraphics[width=0.5\textwidth]{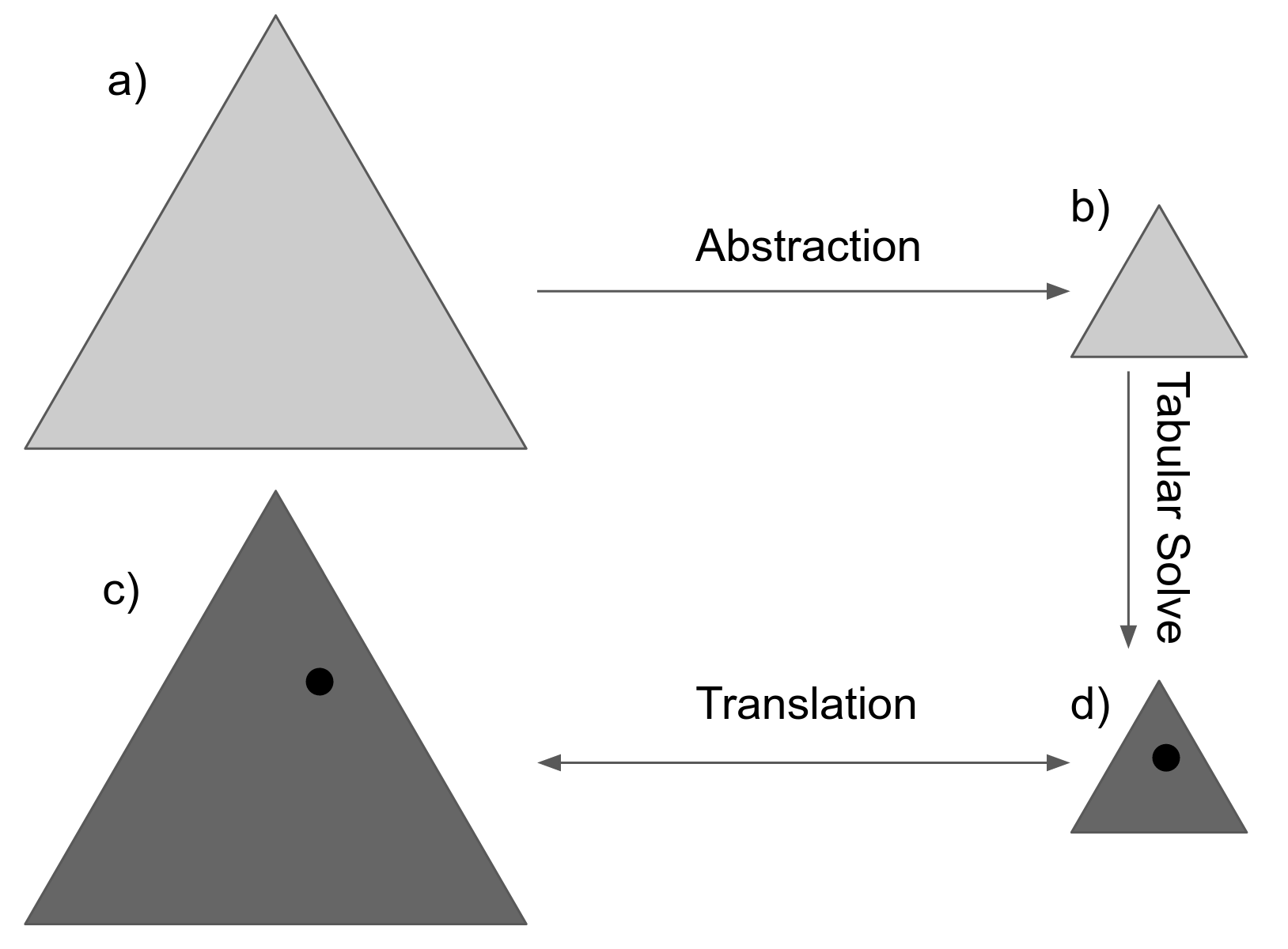}
  \caption{
  The full game (a) is first abstracted into a smaller, tractable abstracted game (b).
  The abstracted game is then solved via a tabular offline method, producing a policy (d).
  In order to play the original game, we need to translate a state from the full game back into the abstracted game.
  A translation issue arises as there are necessarily fewer states in the abstracted game.
  }
\label{fig:iig_abstraction_translation}
\end{figure}

This abstraction approach indeed was for a very long time state of the art method for large imperfect information games \citep{sandholm2010state, johanson2016robust}.
Considerable research was then devoted to larger and smarter abstraction techniques \citep{hawkin2011automated, johanson2013evaluating, ganzfried2014potential, schmid2015automatic}.

The ideal case is a lossless abstraction \citep{gilpin2007lossless}, but this is not possible for large games where one usually solves as large of an abstraction as possible.
But larger abstractions do not necessarily lead to a better exploitabality in the original game.
Even more surprisingly, this might not be the case even if one abstraction is strictly more finer-grained \citep{waugh2009abstraction}.
In practice though, larger abstractions tend to produce stronger policies \citep{johanson2016robust}
}

\subsection{Translation Issue}
\second{
The fundamental problem with abstraction methods is the translation.
This became apparent relatively early on in work on no-limit poker, where the agents are allowed to bet any amount of chips up to their stack size.
For a stack size of $20,000$ chips, there are about $20,000$ actions corresponding to different bet sizes, making this action tree intractably large.
Abstraction can thus only contain a subset of the actions, misunderstanding many of them --- e.g. the closest states for the opponent's bet of $12,000$ could be $10,000$ and $15,000$.
An agent then has no way to properly understand pot-odds in such states --- a key concept for optimal play in poker \citep{chen2006mathematics}.

Soft translation tries to interpolate the strategy between close states in the abstraction \citep{schnizlein2009probabilistic}.
But soft translation is in many ways just a Band-Aid on a bullet wound.
Local best response has been able to exploit even the best of the abstraction agents by huge margins \citep{lisy2017eqilibrium}.
It is fair to say that abstraction based methods are now mostly obsolete.
To see the naivety of the abstraction approach, consider the abstraction approach for a chess agent (Figure \ref{fig:iig-astraction_chess}).

\begin{figure}[ht]
  \centering
  \includegraphics[width=0.8\textwidth]{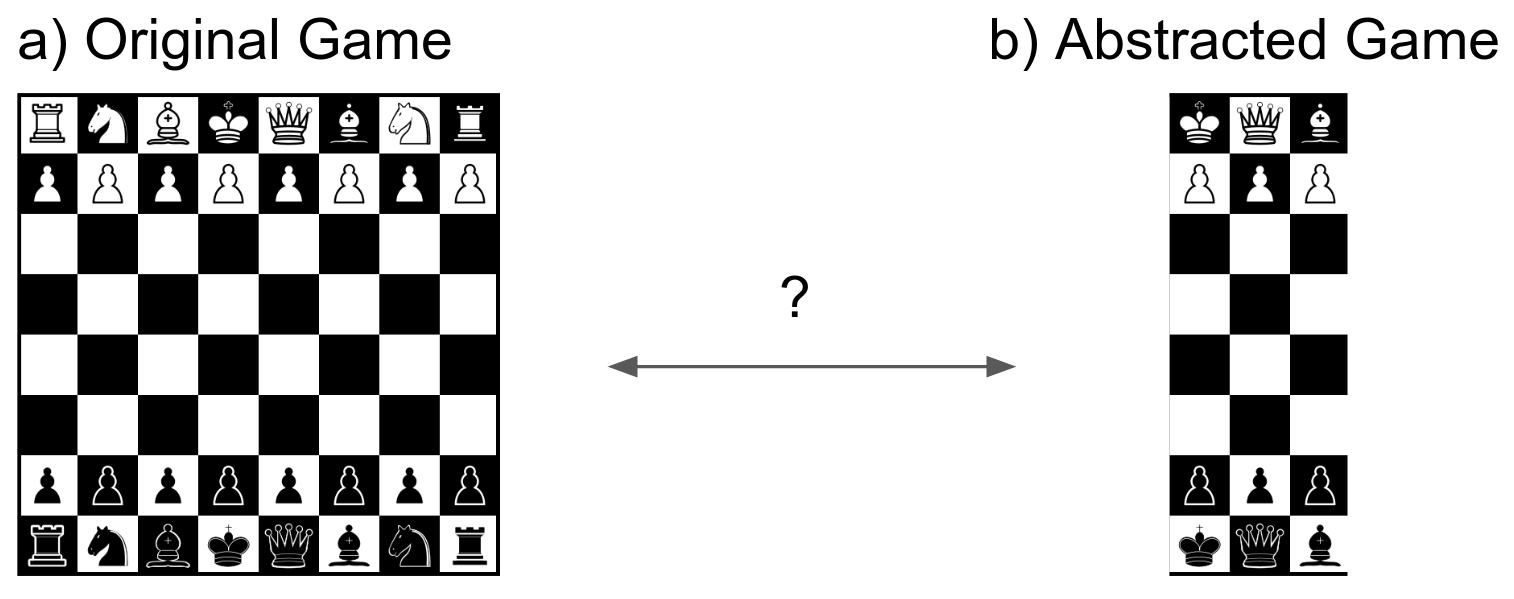}
  \caption{
  Abstraction method for chess requires constructing an abstracted game (b), and then to map the states and actions between the original game (a) and the abstracted game (b)
  }
\label{fig:iig-astraction_chess}
\end{figure}

}

\subsection{Local Best Response}
\label{sec:iig-lbr}

\second{
\Gls{lbr} (LBR) has shown that due to the translation issue, even the best abstraction-based no-limit poker agents are heavily exploitable by relatively simple techniques \citep{lisy2017eqilibrium}.
As exact best response value in the large game of no-limit poker is intractable, local best response approximates the value by approximating the best-responding policy.
By fixing the opponent, the environment becomes a single-agent environment where the optimal policy is best-responding \citep{bowling2003multiagent} and we can thus use any single-agent algorithm to find (or to approximate) such policy.
The performance against the resulting policy than serves as a lower bound on the agent's exploitability.

Local best response uses a small search tree with a poker-heuristic value function.
In its search tree, it only looks two actions forward (at most one action of the opponent) and considers only a subset of  the aviable actions (e.g. fold, call pot bet).
For its value function, it uses an expected rolout value when both players follow the call action until the end of the game (this value can be efficiently and exactly computed for Texas hold'em poker).

Table~\ref{tab:iig-translation_lbr_vs_bots} reports the performance of different poker agents against the local best response technique.
Interestingly, all the agents are substantially more exploitable than simply folding each hand (it is easy to verify that the exploitability of ``always fold'' policy is $750mbb$).

Furthermore, Table~\ref{tab:iig-translation_lbr_vs_bots} includes an agent that uses no card abstraction and uses the F,C,P,A action abstraction.
When LBR considers only these actions, there is no translation issue and LBR indeed fails to exploit this agent.
But once LBR considers more actions, it also exploits this agent and further confirms the translation issue of abstraction techniques.

\begin{table*}[ht]
\centering
\renewcommand{\arraystretch}{1.3}
\begin{tabular}{@{}lllll@{}}
\toprule
LBR Pre-flop Actions & F, C & C & C & C \\
LBR Flop Actions & F, C & C & C & 56bets \\
LBR Turn Actions & F, C & F, C, P, A & 56bets & F, C \\
LBR River Actions & F, C & F, C, P, A & 56bets & F, C \\
\midrule
Hyperborean (2014) & 721 $\pm$ 56 & 3852 $\pm$ 141 & 4675 $\pm$ 152 & 983 $\pm$ 95 \\
Slumbot (2016) & 522 $\pm$ 50 & 4020 $\pm$ 115 & 3763 $\pm$ 104 & 1227 $\pm$ 79 \\
Act1 (2016) & 407  $\pm$ 47 & 2597 $\pm$ 140 & 3302 $\pm$ 122 & 847 $\pm$ 78 \\
Full Cards [100 BB] & -424 $\pm$ 37 & -536 $\pm$ 87 & 2403 $\pm$ 87 & 1008 $\pm$ 68 \\
\bottomrule
\end{tabular}
\caption{
Local best response performance against strong abstraction-based poker agents in no-limit Texas hold'em poker [mbb/g].
The list of actions considered by the technique varies in different poker rounds, and we report four different configurations.
Actions are (F)old, (C)all (P)ot bet and (A)ll-in. 
For the full list of $56$ bets see \citep{lisy2017eqilibrium}.
Note that a policy that simply always folds each hand is exploitable by $750mbb$.
}
\label{tab:iig-translation_lbr_vs_bots}
\end{table*}

}

\section{Non-tabular Methods}
\label{sec:iig-non_tabular_methods}
\second{
Another option for solving games that are intractable for the tabular methods is to represent the strategies and other necessary values (e.g. regrets) implicitly (e.g. using neural networks), relying on the generalization power of the networks.
Together with methods that do not have to traverse the full game tree (e.g. MCCFR or VR-MCCFR), this approach allows for algorithms that approximately solve very large games.

The first non-tabular method, RCFR \citep{waugh2015solving} used regression trees and needed to keep track of a dataset that was substantially larger than the full game.
This limited the approach from being useful for exactly what one would hope to use non-tabular methods for --- large games.
The first two successful instances of non-tabular method in large games are then DeepCFR \citep{brown2018deep} and double\footnote{Not to be confused we double q-learning methods in reinforcement learning where the two networks are used to eliminated bias. Double here refers to separate networks for regrets and policies.} neural counterfactual regret minimization (DNCFR) \citep{li2018double}, both combine deep learning and MCCFR and being published around the same time.
Both algorithms use separate networks to represent i) regrets and ii) the average policy.
MCCFR is used to collect batches for training, using multiple trajectories to decrease the variance of the training data.

Single DeepCFR follows in these steps and shows how to replace the average policy network by storing a buffer of past regret networks \citep{steinberger2019single}.
The idea is that since the current strategy is a simple function of regrets, storing some of the previous networks amounts to storing these current strategies.
One can then approximate the average policy from previous strategies in the buffer.
Later work then combines modern VR-MCCFR methods into the DREAM algorithm \citep{steinberger2020dream}.
ARMAC is another MCCFR method that is closely related to DREAM \citep{gruslys2020advantage}.
Finally, there are recent methods that do not build on the CFR family such as properly modified follow the regularized leader algorithm \citep{perolat2020poincar}.
}

\chapter{Online Settings}
\label{chap:iig-online_settings}
\first{
This chapter re-visits the online setting as introduced in Section \ref{chap:pig-online_settings}.
There are only minor differences to account for the formalism of the imperfect information games, and most of  definitions remain intact.
The only notable difference is a particular global consistency connection, where Theorem \ref{thm:pig-local_consistency_subgame_perfect} no longer holds, making search in imperfect information more challenging.
}

\first{
Furthermore, optimal policies for imperfect information games might have to be stochastic (Corollary~\ref{cor:iig-minmax_stochastic}).
Repeated games then allow us to correctly capture the dynamics of online algorithm that an ``averaged'' offline policy never could.
Consider a simple online algorithm for \rps{} that simply produces the sequence (rock, paper, scissors, rock, paper, $\hdots$).
Despite its ``average'' strategy being optimal, this online algorithm is clearly highly exploitable.
}

\section{Repeated Game}
\first{
A repeated game with imperfect information is just like a repeated game with perfect information, except that the individual matches consist of imperfect information games.
The definition has to reflect this, and we will stick with the FOSG formalism.

Formally, The repeated game $p$ consists of a finite sequence of $k$ individual matches $p = (z_1, z_2, \hdots, z_k)$, where each match $z_i \in \mathcal{Z}$ is a sequence of world states and actions $\match_i = (w_i^0,\,a_i^0\,w_i^1,\,a_i^1\,\hdots,\,a_i^{l_i-1},\,w_i^{l_i})$ (ending in a terminal world state $w_i^{l_i}$).
For each visited world state in the match, there is a corresponding player state (information state) $s_i(w^t_i)$, i.e. their private perspective of the game.
For perfect information games, the notion of player state and world state collapsed as the player gets to observe the world perfectly.
}

\section{Online Algorithm}
\first{
Just as in Section \ref{sec:pig-online_algorithm}, an online algorithm $\search$ then simply maps a player's state observed during a match to a strategy (Definition \ref{defn:pig-online_algorithm}).

\onlinealgorithm*

Given two players $\search_1, \search_2$, we use $P^k_{\search_1, \search_2}$ to denote the distribution of game-plays when these two players face each other.
The average reward of $i$ is $R(p) = 1/k \sum_{i=1}^k u_p(z_i)$ and we denote $\EX_{p \sim P^k_{\search_1, \search_2}}[\mathcal{R}_p]$ to be the expected average reward when the players play $k$ matches. 
}

\section{Soundness}
\first{
Definition of sound algorithm also remains unchanged 

\soundness*
}

\section{Search Consistency}
\label{sec:iig-search_consistency}

\first{
Concepts local, global and strongly global consistencies remain unchanged.
An important discrepancy is that Theorem~\ref{thm:pig-local_consistency_subgame_perfect} does not hold in imperfect information.

\localsubgameperfect*

Local consistency is not sufficient in imperfect even if the algorithm is consistent with \subgame{} perfect equilibrium.
To illustrate this, we will construct a simple game where following is true.

\begin{itemize}
    \item The game has multiple \subgame{} perfect Nash equilibria and $\pi^a, \pi^b \in \nashset{}$.
    \item There are two states for the second player $\mathcal{S}_2 = \{s_1, s_2\}$
    \item Following $\pi^a$, $\pi^b$ in $s_1$ and $s_2$ respectively produces highly exploitable policy.
\end{itemize}

\begin{figure}[ht]
  \centering
  \includegraphics[width=0.5\textwidth]{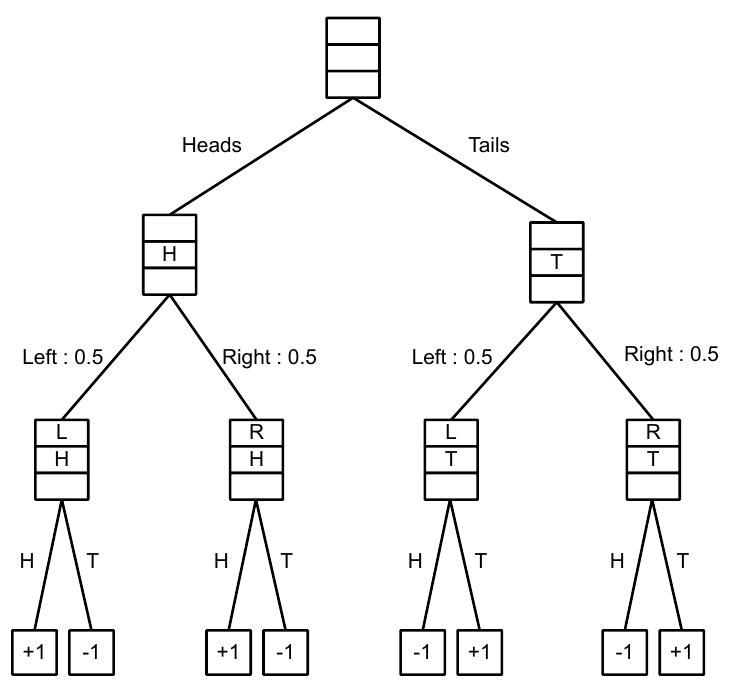}
  \caption{
    Coordinated matching pennies.
    First player makes a private action (H)eads or (Tails).
    Follows a chance player, selecting uniformly one of the publicly observable actions (L)eft or (R)ight.
    There are now two infostates for the second player $s_1 = [L||], s_1 = [R||]$ and four histories.
    Second player now chooses (L)eft or (R)ight.
    In their first infoset $s_1$, they win if they match the opponent's action.
    In their second infoset $s_2$, they win if they select the opposite action.
  }
\label{fig:iig_global_consistency_counterexample}
\end{figure}

Coordinated matching pennies (Figure~\ref{fig:iig_global_consistency_counterexample}) satisfies all three properties.
The two information states of player two are $s_1=[L||], s_1=[R||]$.
It is easy to verify that any optimal policy needs to satisfy $\pi_2(s_1)(H) = \pi_2(s_2)(T)$.
In other words, the player has to coordinate between these two information states.
Thus we can construct optimal policies $\pi^a_1, \pi^a_2$ as $\pi^a_2(s_1, L) = 1, \pi^a_2(s_1, L) = 0$ and $\pi^b_2(s_1, L) = 0, \pi^a_2(s_1, L) = 1$.
Now if the search technique follows $\pi^a_1$ in $s_1$ and $\pi^a_2$ in $s_2$, the resulting policy is highly exploitable.

While there are multiple notions of \subgame{} perfect policies in imperfect information settings, note that two information states $s_1, s_2$ have no predecessors and thus this counter-example holds regardless of the notion choice. 
}

\subsection{Online Outcome Sampling}

\begin{figure}[ht]
    \centering
    \begin{subfigure}[b]{0.44\textwidth}
        \includegraphics[width=\textwidth]{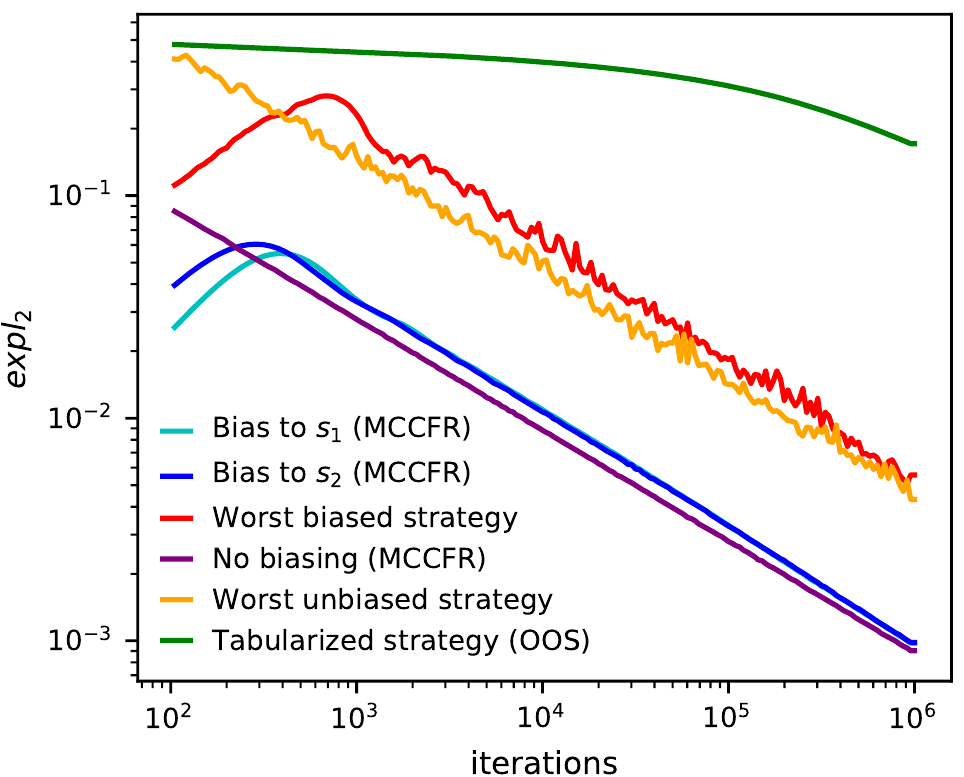}
        \caption{Coordinated matching pennies.}
        \label{fig:iig-oos1}
    \end{subfigure}
    \begin{subfigure}[b]{0.52\textwidth}
        \includegraphics[width=\textwidth]{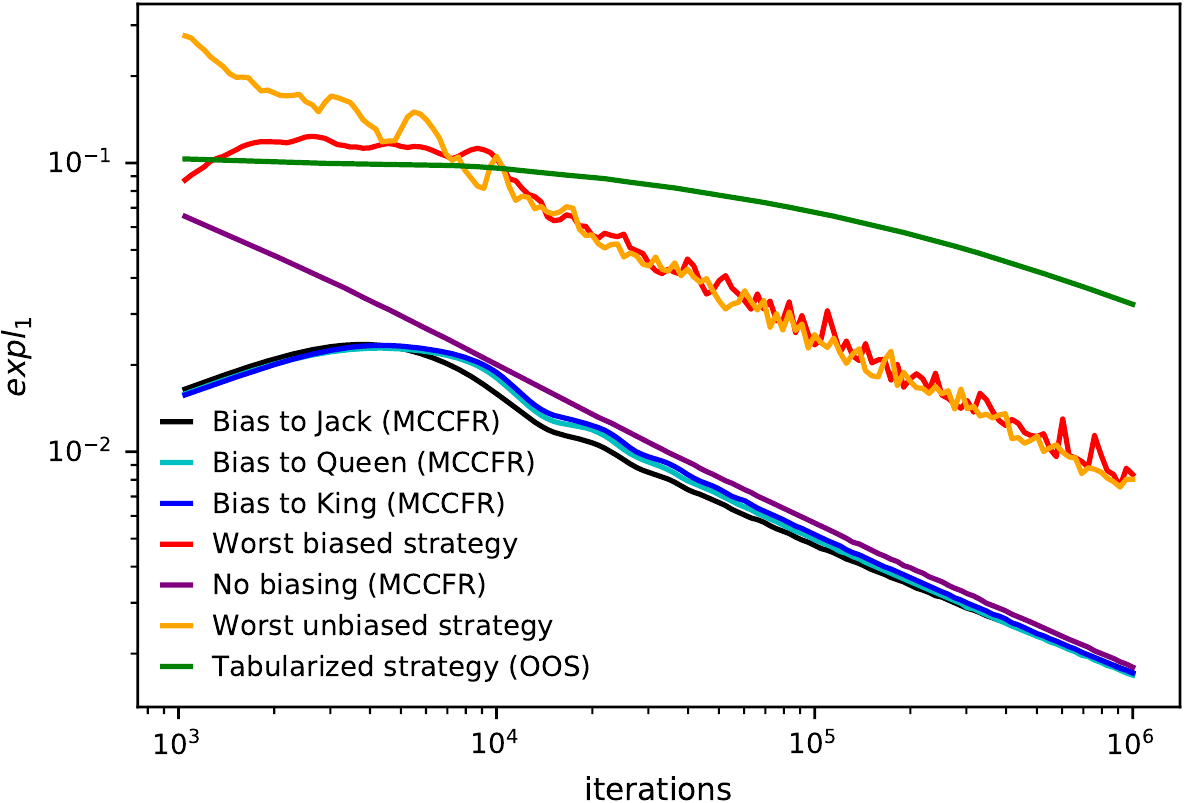}
        \caption{Kuhn poker.}
        \label{fig:iig-oos2}
    \end{subfigure}
    \caption{Exploitability of individual MCCFR policies and tabularized OOS.}
    \label{fig:iig-oos}
\end{figure}

\first{
A particularly interesting example of an algorithm that is only locally consistent is the Online Outcome Sampling 
(OOS) \citep{lisy2015online}.
At high level, OOS runs the offline MCCFR algorithm (Section~\ref{sec:iig-mccfr}
) in the full game (while also gradually building the tree), parameterized to increase the sampling probability of the current information state. 
The algorithm then plays based on the resulting strategy for that particular state. 
The problem is that these individual MCCFR runs can converge to different equilibria as the MCCFR is parameterized differently
in each information state. 
In other words, the OOS algorithm exactly suffers from the fact that it is only locally consistent.
Figure~\ref{fig:iig-oos} compares exploitability of the individual MCCFR policies $\pi^a_2(s_1), \pi^b_2(s_2)$ that OOS is locally consistent with, and the exploitability of the tabularized policy the OOS actually follows $\search(s_1) = \pi^a_2, \search(s_2) = \pi^b_2$.
The experiment is run on coordinated matching pennies and Kuhn poker, where the individual MCCFR runs were then parametrized to favor a particular Nash equilibrium --- see the Appendix of \cite{vsustr2020sound} for all the experimental details.
}
\chapter{Search}
\label{chap:iig-search}
\second{
This chapter introduces online search algorithms for imperfect information.
Just as in perfect information, \subgames{} and value function serve as the key building blocks.
And while these concepts are already more complex in imperfect information, the matter is farther complicated by the fact that strong global consistency is not trivial to achieve as we have to coordinate the policy across information states.
Furthermore, the base algorithm used for the search techniques is counterfactual regret minimization.
This algorithm is particularly suitable for value functions, as it already decomposes the full regret to the sum of individual partial infostate regrets.

\begin{figure}[ht]
    \centering
    \begin{subfigure}[b]{0.3\textwidth}
        \includegraphics[width=\textwidth]{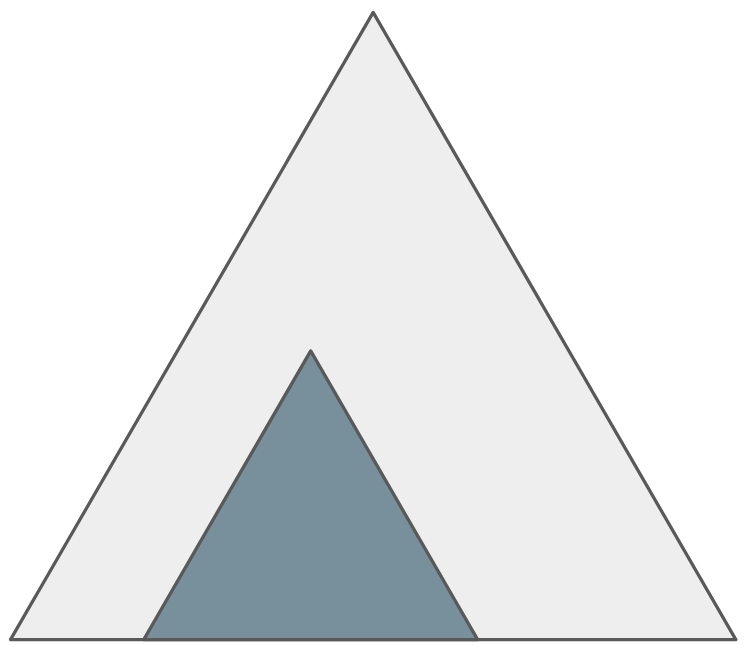}
        \caption{}
        \label{fig:iig-search_1}
    \end{subfigure}
    \begin{subfigure}[b]{0.3\textwidth}
        \includegraphics[width=\textwidth]{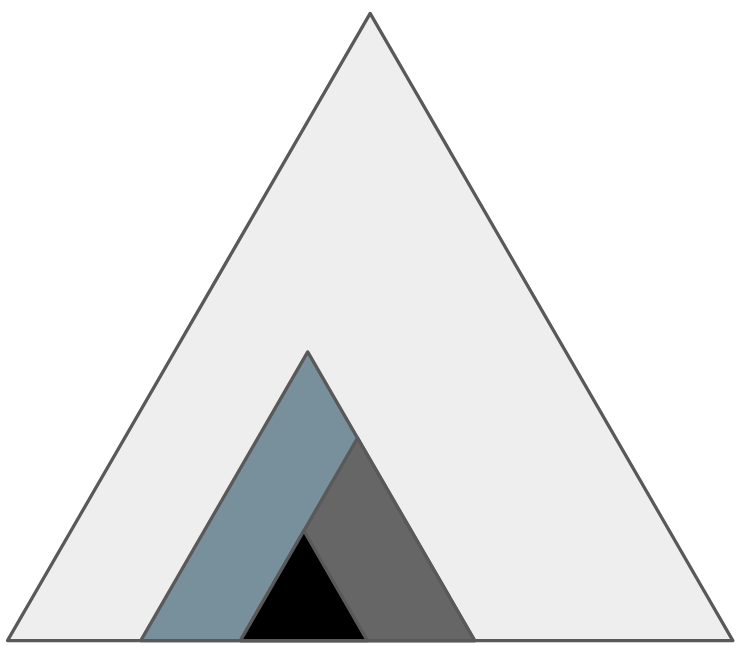}
        \caption{}
        \label{fig:iig-search_2}
    \end{subfigure}
    \begin{subfigure}[b]{0.3\textwidth}
        \includegraphics[width=\textwidth]{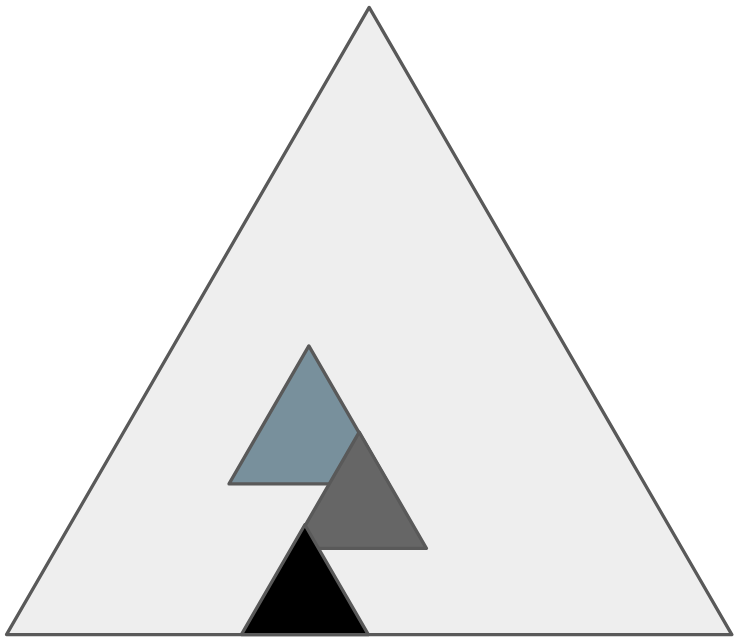}
        \caption{}
        \label{fig:iig-search_3}
    \end{subfigure}
    \caption{a) \Subgame{} (re)solve full lookahead.
            b) Continual \resolving{} with full lookahead.
            c) Continual \resolving{} with limited lookahead and value functions. }
    \label{fig:iig_search_intro}
\end{figure}

We first start with a single step search algorithm that reasons about the current \subgame{} (Figure~\ref{fig:iig-search_1}), analogically to full lookahead online minimax algorithm of Section~\ref{sec:pig-online_minmax}.
The full lookahead setting is where we will deal with the complication of consistency.
We begin with unsafe \resolving{}, an unsound approach that motivates the key ingredient in the future sound methods --- constrain counterfactual values that ensure the consistency of individual local policies.
These values are then used within another key concept --- \resolving{} games (and \resolving{} gadgets), where a particular gadget game guarantees to produce a \subgame{} satisfying the constraints.
Follows another contribution of the thesis, where a particular modification to the gadget allows to improve (refine) a \subgame{} policy while still ensuring the strong global consistency.

We then extend this algorithm to multiple steps, where we continually use safe \resolve{} method during the gameplay --- the continual \resolving{} algorithm (Figure \ref{fig:iig-search_2}).
Continual resolving with full lookaheads is analogous to the full looakehad minmax search algorithm. 
The main complication is that the algorithm keeps updating the constrains to be used for the safe resolving process.

We then combine this algorithm with generalized value functions, allowing to use limited lookaheads within the CFR algorithm.
This results in the final algorithm --- continual resolving with limited lookaheads and value functions (Figure \ref{fig:iig-search_3}).
The final algorithm is the culmination of all the building blocks introduced in this book, presenting a safe search algorithm for imperfect information games.

We finish with the final contribution of the thesis --- Deepstack.
DeepStack was the first to introduce the continual \resolving{} algorithm, combining sound search and learned value functions for poker \citep{moravvcik2017deepstack}.
The combination of continual \resolving{} and neural networks as value functions resulted in substantial improvement over the previous methods and DeepStack became the first program to beat professional human players in no-limit Texas hold'em poker.
}

\section{Full Lookahead Unsafe \Resolving{}}
\label{sec:iig-unsafe_resolving}
\second{
In imperfect information, the notion of the \subgame{} was defined as a combination of i) public state and ii) distribution over the players' infostates (Section~\ref{sec:iig-public_subgame}).
We first analyze what happens if we \resolve{} a \subgame{} where the distribution over the player's initial states is computed using an optimal policy of the full game.
That is we:

\begin{enumerate}
    \item Compute optimal policy profile $(\pi^*_i, \pi^*_{-i})$.
    \item Select a public state $s_{pub}$ as a root of a \subgame{}.
    \item Compute reach distribution $\Delta(\mathcal{S}_i(s_{pub})), \Delta(\mathcal{S}_{-i}(s_{pub}))$ using the agent's respective policies and the corresponding reach probabilities $P^\pi$.
    \item Solve the \subgame{} defined by $s_{pub}, \Delta(\mathcal{S}_i(s_{pub})), \Delta(\mathcal{S}_{-i}(s_{pub}))$, producing \subgame{} policy profile $(\pi^{sub*}_i, \pi^{sub*}_{-i})$.
\end{enumerate}

Is the resulting $(\pi^{sub*}_i, \pi^{sub*}_{-i})$ consistent with $(\pi^*_i, \pi^*_{-i})$?
Is such algorithm sound?
In perfect information, the answer is positive as the search minmax algorithm fits this framework.
And while this approach has has also been used in imperfect information \citep{ganzfried2015endgame} (referred to as the ``\subgame{} solving''), it both loses its theoretical guarantees and empirically produces highly exploitable policies.

Surprisingly, this idea breaks already in \rps{} \citep{ganzfried2015endgame, burch2014solving}.
Consider a \subgame{} rooted in the second public state (just after the first player has acted).
The public state contains a single decision infostate of the second player $\{[||]\}$, and three possible states for the first player $\{[|R|], [|P|], [|S|]\}$.
As the optimal strategy for the first player is to play uniformly in the first state of the game (prior to reaching this \subgame{}), distribution over their \subgame{} states is also uniform.
The resulting initial distributions of the \subgame{} are then $\Delta(\mathcal{S}_1(s_{pub})) = (\frac{1}{3}, \frac{1}{3}, \frac{1}{3})$ and  $\Delta(\mathcal{S}_2(s_{pub})) = (1)$.
See Figure~\ref{fig:iig_rps_unsafe_resolve_beliefs} for illustration.

It is easy to verify that any policy in this \subgame{} is optimal.
But while any policy is optimal in the \subgame{}, not any strategy is optimal in the original game (e.g. playing always rock is highly exploitable).
Thus simply solving a \subgame{} where the distribution is constructed from prior optimal policies is not even a locally consistent approach.
This approach is also referred to as the unsafe \resolving{}.
}

\begin{figure}[ht]
    \centering
    \begin{subfigure}[b]{0.48\textwidth}
        \includegraphics[width=\textwidth]{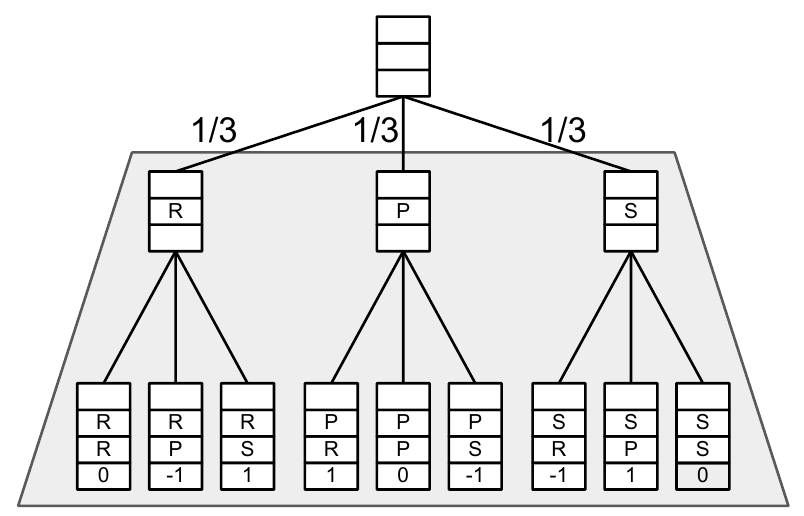}
        \caption{}
        \label{fig:iig_rps_unsafe_resolve_beliefs}
    \end{subfigure}
    \begin{subfigure}[b]{0.48\textwidth}
        \includegraphics[width=\textwidth]{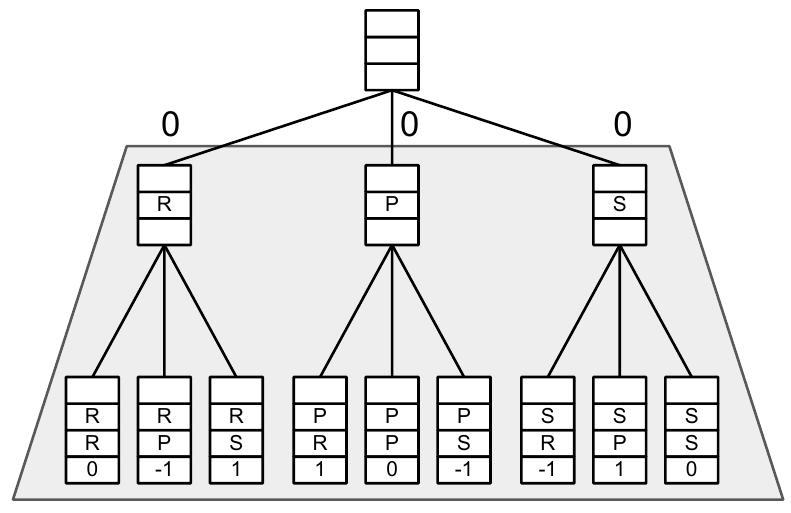}
        \caption{}
        \label{fig:iig_iig_rps_unsafe_resolve_values}
    \end{subfigure}
    \caption{a) Public \subgame{} with the initial distribution over the states based on a optimal policy in the full game. 
    b) Counterfactual values corresponding to that optimal policy.}
    \label{fig:iig_rps_unsafe_resolve}
\end{figure}

\subsection{Why Did It Break?}
\second{
After closer inspection of the counterexample, we can investigate why exactly this approach failed.
Understanding the issue will allow us to design a new, safe approach that will be crucial in constructing the final algorithms.

Why was the strategy optimal for the sub-game but not locally consistent?
Remember that the optimal policy is maximizing the worst case value.
When we constructed the \subgame{} and solved for the optimal policy, we allowed the opponent to best-respond
in the \subgame{}, but we fixed the reach probabilities (sub-game initial distribution) using an optimal policy prior to reaching this \subgame{}.
The opponent could not change their policy prior to this sub-game and alter their reach probabilities.
}

\second{
Why was it beneficial for the opponent to change their policy prior to this \subgame{}?
Consider their \subgame{} values $V_1$ at the beginning of the \subgame{}.
Under the optimal (uniform) policy of the second player, the values are $V_1 = (0, 0, 0)$.
But after the unsafe \resolving{} of the \subgame{}, the values for these states become $(0, 1, -1)$.
When weighted by the uniform reach probability, the \subgame{} value remains unchanged
$\T{(\frac{1}{3}, \frac{1}{3}, \frac{1}{3})}(0, 0, 0) = \T{(\frac{1}{3}, \frac{1}{3}, \frac{1}{3})}(0, 1, -1) = 0$.
The issue is that it is now beneficial for the opponent to change their policy prior to this \subgame{}, by reaching their second state more often and improve their total best-response value.
}

\subsubsection{Value Perspective}
\second{
Useful viewpoint at the issue is to realize that in perfect information, there is a single root state of the opponent and a single counterfactual value of the opponent that we care to minimize.
In imperfect information, there are multiple potential states of the opponent $\mathcal{S}_{-i}(s_{pub})$ and corresponding value vector $V_{-i}(s_{pub})$.
What policy should be then preferred, policy $\pi^a_i$ where $V_{-i} = (1, -1)$ or $\pi^b_i$ where $V_{-i} = (-1, 2)$?
Fixed initial beliefs provide a weighting over those individual values resulting in a single weighted value to be minimized.
But as the weights (reach probabilities) can change, so can the resulting value.

Note that if at any time an imperfect information $s_{pub}$ consists of only a single opponent's state, there is no issue and the same can be argued for a \subgame{} where there are no prior decisions of the opponent,
}

\section{Safe \Resolving{}}
\label{sec:iig-safe_resolve}
\second{
Safe \resolving{} methods make sure the opponent can't increase their \subgame{} value by altering their reach probabilities.
To prevent this, we need to find a policy in the \subgame{} for which the individual counterfactual values\footnote{Counterfactual best response values.} of the opponent remain the same as under the original optimal policy of the full game (or even better, are no greater than the original ones).
This then guarantees strong global consistency with a policy having the \resolved{} values.

Unsafe \resolve{} used i) public state $s_{pub}$ ii) player's reach $\Delta(\mathcal{S}_i (s_{pub}))$ iii) opponent's reach $\Delta(\mathcal{S}_{-i} (s_{pub}))$ and solved the related \subgame{}.
Safe \resolving{} on the other hand uses
\begin{enumerate}[label=(\roman*)]
    \item public state $s_{pub}$.
    \item player's reach $\Delta(\mathcal{S}_i (s_{pub}))$.
    \item opponent's counterfactual values $V^{bound}_{-i}(s_{pub})$.
\end{enumerate}

Safe \resolving{} then produces a policy for which $V_{-i}(s_{pub}) \leq V^{bound}_{-i}(s_{pub})$.
Note that this is always possible as the original optimal policy we are \resolving{} satisfies this constraint.
}

\subsection{Optimization Formulation}
\label{sec:iig-resolving_optimization}
\second{
The $V_{-i}(s_{pub}) \leq V^{bound}_{-i}(s_{pub})$ constraints is relatively straightforward to add to the linear programming formulation (Section \ref{sec:iig-sequence_form_lp}), as the vector $u$ already present in the optimization corresponds to the negative (counterfactual) best response values of the opponent.

\begin{align}
\label{eq:iig-safe_resolve_lp_formulation}
\begin{split}
    & \min_{u, v} e^\top u  \\
    & Fy = f, E^\top u -Ay \geq \mathbf{0}, y \geq \mathbf{0} \\
    & u_{s} \leq V_{-i}(s_{pub}, s) \, \forall s \in \mathcal{S}_i(s_{pub})
\end{split}
\end{align}

Linear optimization problem \ref{eq:iig-safe_resolve_lp_formulation} then produces \subgame{} policy for the player $i$ that is consistent with the original, full game optimal policy we used to derive $\Delta(\mathcal{S}_i (s_{pub}))$ and $V^{bound}_{-i}(s_{pub})$ \citep{burch2014solving, moravcik2016refining}.
}

\subsection{Gadget Formulation}
\label{sec:iig-cfrd_gadget}
\second{
But there is a more elegant and general approach, the gadget game formulation \citep{burch2014solving}.
The idea is to carefully construct a game for which the resulting optimal policy is guaranteed to satisfy the constraints.
Technically, we will construct a gadget game for which there exists a trivial mapping between an optimal solution of the gadget game and the desired policy.

The core of the construction is that we insert a gadget on top of the \subgame{} (often referred to as the ``CFR-D gadget'').
It starts with a chance node that first distributes the opponent into their root information states $\mathcal{S}_{-i}(s_{pub})$, and the opponent then gets to choose (in each respective infostate) whether to (t)erminate and receive the constraint values, or to (f)ollow and play the original sub-game.
This clever construction forces the player to play the \subgame{} so that the opponent's values are no greater than the corresponding (t)erminate values.
Figure \ref{fig:iig:cfrd_gadget} shows a high-level idea of the gadget construction and Figure \ref{fig:iig:cfrd_gadget_rps} shows a concrete construction for \rps{}.
Note that for the sake of clarity, the figures omit the initial chance state.

The beauty of the gadget game formulation is that it simply constructs a new game of similar size as the \subgame{}.
This allows us to use any algorithm for solving imperfect information games, e.g. any method from the CFR family.
}

\begin{figure}[ht]
    \centering
    \begin{subfigure}[b]{0.45\textwidth}
        \includegraphics[width=\textwidth]{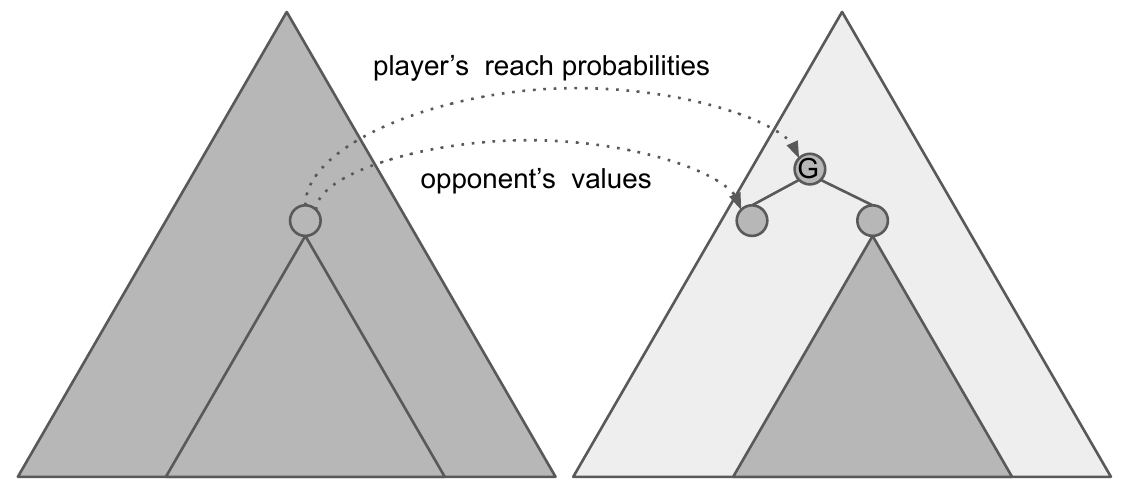}
        \caption{}
    \end{subfigure}
    \begin{subfigure}[b]{0.45\textwidth}
        \includegraphics[width=\textwidth]{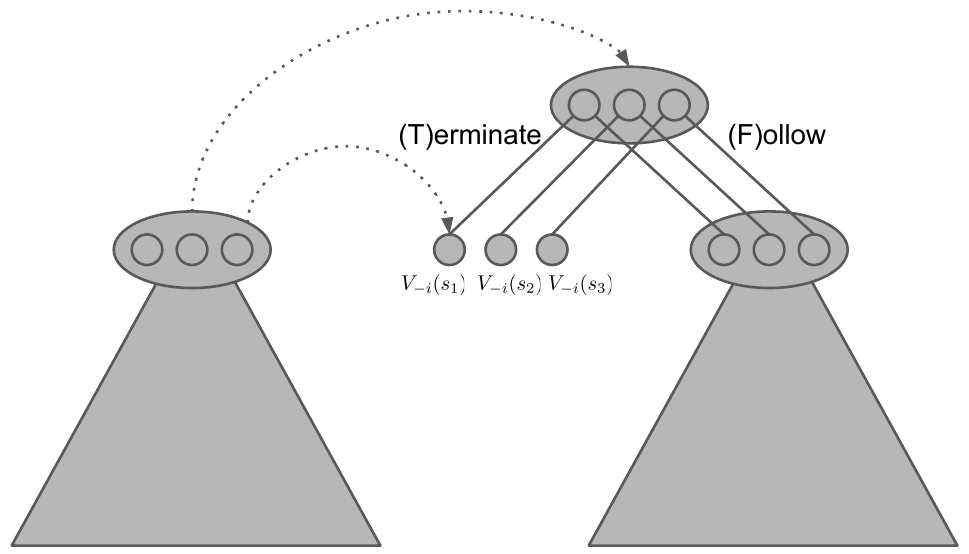}
        \caption{}
    \end{subfigure}
    \caption{
    a) The gadget game is constructed for a specific public \subgame{}, and the construction needs i) player’s reach probabilities and ii) the opponent’s counterfactual values. 
    b) The opponent gets to play first, and in each infoset chooses whether to
(t)erminate and receive the constraint values, or to (f)ollow and play the original
sub-game.
    }
    \label{fig:iig:cfrd_gadget}
\end{figure}

\begin{figure}[ht]
    \centering
    \begin{subfigure}[b]{0.42\textwidth}
        \includegraphics[width=\textwidth]{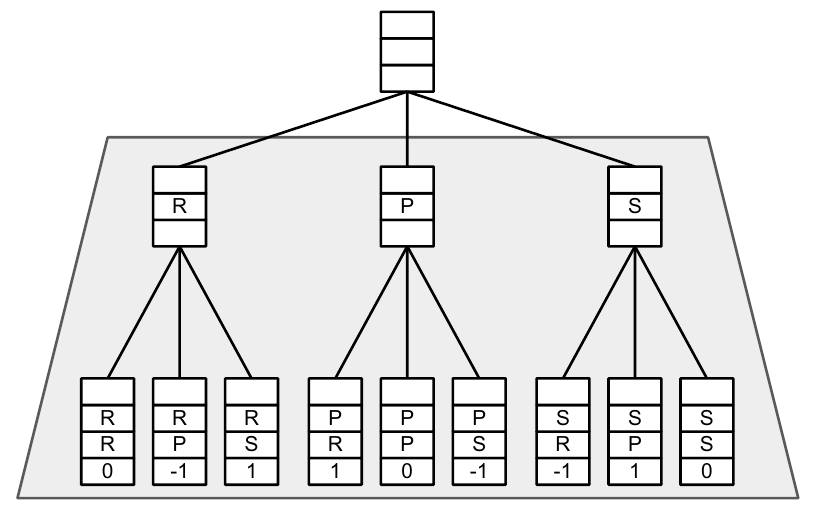}
        \caption{}
    \end{subfigure}
    \begin{subfigure}[b]{0.55\textwidth}
        \includegraphics[width=\textwidth]{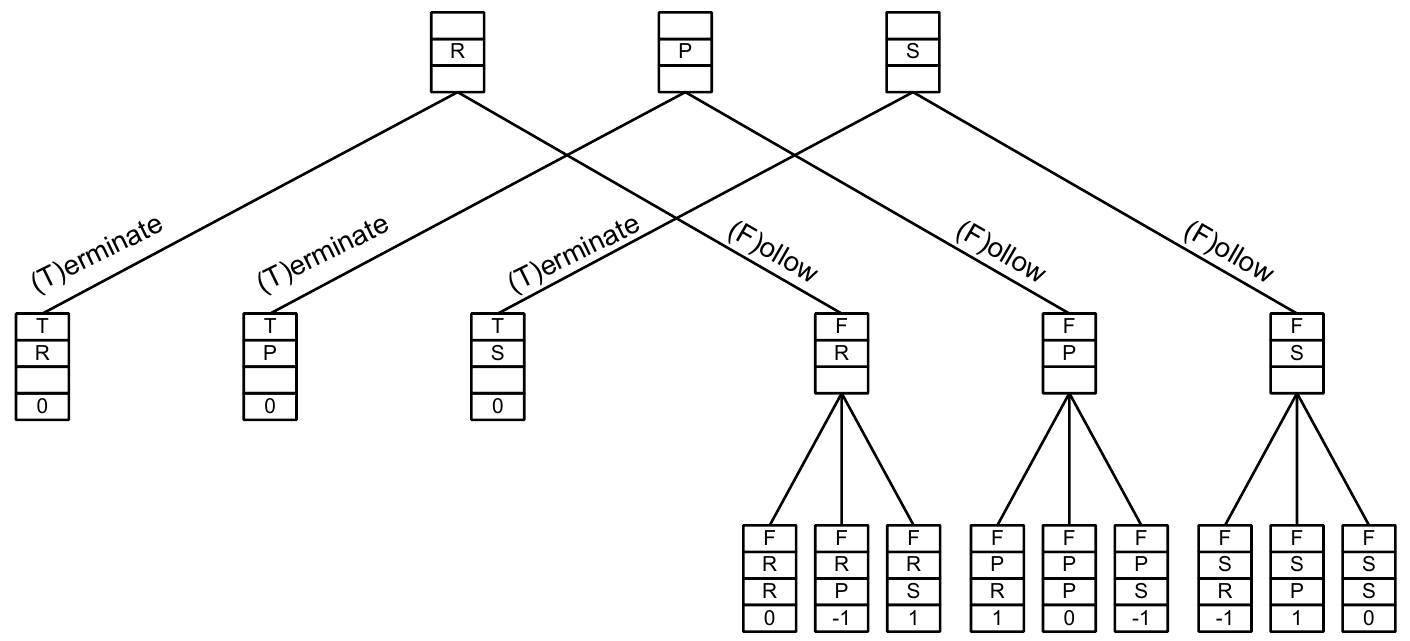}
        \caption{}
    \end{subfigure}
    \caption{a) \Subgame{} of \rps{}.
    b) CFRD gadget for \resolving{} the \subgame{}.
    Note that the zero values for the (T)erminate action come from the optimal solution.}
    \label{fig:iig:cfrd_gadget_rps}
\end{figure}

\subsubsection{Stackelberg Equilibria}
\second{
The same idea of gadget construction has recently been generalized for the
solution concept of Stackelberg equilibria \citep{ling2021safe}.
The paper follows two steps.
First, it shows that safe \resolving{} for Stackleberg equilibria can be achieved using a carefully constructed upper and lower bound constraints (compared to just upper bound).
Second, it introduces modified gadget construction for the lower bound constraints (and uses the ``CFRD gadget'' construction for the upper bound constraints).
}

\section{CFR-D}
\label{sec:iig-cfrd}
\second{
The notion of safe \resolving{} and the gadget formulation was first introduced by the CFR-D algorithm \citep{burch2014solving}.
The CFR-D algorithm decomposes the game tree into trunk and \subgames{}, storing and averaging the policy only for the trunk.
During each CFR iteration, individual \subgames{} (using the reach distribution defined by the current CFR policy) are solved for and strategies discarded, returning only the individual values.
The trunk computation thus uses a form of limited lookahead and value function (as described in Section \ref{sec:iig-depth_limited_solve}).
During playtime, CFR-D simply uses the stored trunk policy until it reaches one of the \subgames{}.
As the \subgame{} policies are never explicitly stored, CFR-D then re-solves a \subgame{} on-demand using the CFR-D gadget.

This decomposition was initially motivated by the need to overcome memory limitations, essentially trading memory for time.
The main ideas of the algorithm (decomposition and safe \resolving{}) now form the core building blocks of many modern search algorithms.
}

\section{\Subgame{} Refinement}
\label{sec:iig-subgame_refinement}
\second{

\begin{figure}[ht]
  \centering
  \includegraphics[width=0.4\textwidth]{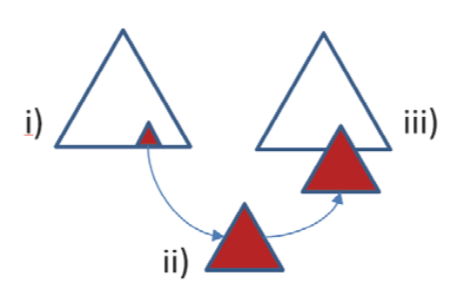}
  \caption{\Subgame{} refinement framework. 
  (i) the strategy for the game is pre-computed using coarse-grained abstraction 
  (ii) during the play, once we reach a node defining a sufficiently small subgame, we refine the strategy for that \subgame{} 
  (iii) this together with the original strategy for the
trunk creates a combined strategy. 
The point is to produce improved combined strategy}
\label{fig:iig-subgame_refinement}
\end{figure}

We now present another contribution of this thesis the --- the \gls{subgame_refinement} framework and max margin \resolving{}.
For a long time, state of the art methods simply pre-computed optimal solution for a simplified (abstracted)
version of the game (Section \ref{sec:iig-abstraction_methods}).
When the agent played the game, it simply acted based on this pre-computed policy.
But as the game was approaching its end, the remaining sub-problem
to reason about become either tractable or one could make a finer-grained abstraction of the \subgame{}.

An appealing idea is then to compute a new policy for the \subgame{}, refining the original strategy (Figure \ref{fig:iig-subgame_refinement}).
One can than use the CFR-D gadget to guarantee the soundness of such approach.
But even if one could do better in the finer-grained abstraction, CFR-D gadget has no strong reason to do so --- it is only guaranteed to recover the original solution.
That is, the resulting policy in the finer-grained abstraction is guaranteed to be no worse than the original policy based on the coarse abstraction.
}

\subsection{\Subgame{} Margin}
\second{
While the CFR-D gadget recovers the original policy, max margin aims to improve upon it and the resulting improvement is in a sense optimal \citep{moravcik2016refining}.
The safe \resolving{} constraints of CFR-D guaranteed that the opponent's counterfactual values of the resulting policy are no greater that counterfactual values of the original policy.
To be able to improve the values, we first introduce the \subgame{} margin (Definition~\ref{def:iig-subgame_margin}), corresponding to a slack between the values of the original policy and the values\footnote{Counterfactual best response values (Section~\ref{sec:iig-counterfactual_best_response}).} of the new strategy.
)

\begin{defn}[\Subgame{} Margin]
\label{def:iig-subgame_margin}
\citep[Definition 2]{moravcik2016refining}
Let $\pi_1, \pi_1'$ be a pair of strategies for \subgame{} $s_{pub}$.
A \subgame{} margin $SM_1$ is then
\begin{align*}
SM_1(\pi_1, \pi_1', s_{pub}) = \min_{s \in \mathcal{S}_2(s_{pub}) } \cfbrv_2^{\pi_1}(s) - \cfbrv_2^{\pi_1'}(s)
\end{align*}
\end{defn}

The \subgame{} margin has several useful properties. 
The worst case performance is closely related to the value of the margin.
If it is non-negative, the new strategy is guaranteed to be no more exploitable than original one. 
Furthermore, given that the opponent’s best response reaches the
\subgame{} with non-zero probability, the performance against this best response is improved and this improvement is at least proportional to the \subgame{} margin (and may be greater) (Theorem~\ref{thm:iig-maxmargin}).

\begin{restatable}{thm}{maxmargintheorem}
\label{thm:iig-maxmargin}
\cite[Theorem 1]{moravcik2016refining}
Given a strategy $\pi_1$, a \subgame{} $s_{pub}$ and a refined \subgame{} strategy $\pi_1^{s_{pub}}$, let $\pi'_1 = \pi_1[s_{pub} \leftarrow \pi_1^{s_{pub}}]$ be a 
combined strategy of $\pi_1$ and $\pi_1^{s_{pub}}$. 
Let the \subgame{} margin $SM(\pi_1, \pi_1', s_{pub})$ be non-negative. 
Then $BRV(\pi_1') - BRV(\pi_1) \geq 0$. 
Furthermore, if there is a best response strategy $\pi^*_2 \in \brset(\pi_1')$ such that  $P^{(\pi_1', \pi_2^*)}(s) > 0$ for some $s 
\in \mathcal{S}_2(s_{pub})$, then $BRV(\pi_1') -
BRV(\pi_1) \geq P_{-2}^{\pi_1'}(s) \, SM(\pi_1, \pi_1', s_{pub})$.
\label{margin_theorem}
\end{restatable}
\begin{proof}
See the Appendix of \citep{moravcik2016refining}.
\end{proof}

Theorem~\ref{thm:iig-maxmargin} has critical consequences for the \subgame{} refinement. 
Since we refine the strategy once we reach the \subgame{}, we are either
facing opponent's best response that reaches the $s_{pub}$, or they made a
mistake earlier in the game. 
Furthermore, the probability of reaching a \subgame{} is proportional to $P^{\pi}$. 

\subsection{Optimization Formulation}
\second{
It is relatively easy to modify the CFR-D optimization formulation to maximize the \subgame{} margin by introducing a slack variable.
The high level idea is that the CFR-D gadget construction simulates the ``no greater than'' constraints $u_{s} \leq V_{-i}(s_{pub}, s)$ (Section~\ref{sec:iig-resolving_optimization}).
Max margin then adds a slack/margin to the constraints that is then maximized $u_{s} + m \leq V_{-i}(s_{pub}, s)$ (Equation~\ref{eq:iig-maxmargin_constrants}).

\begin{align}
\label{eq:iig-maxmargin_constrants}
\begin{split}
    & \max_{u, v, m} m  \\
    & Fy = f, E^\top u -Ay \geq \mathbf{0}, y \geq \mathbf{0} \\
    & u_{s} + m\leq V_{-i}(s_{pub}, s) \, \forall s \in \mathcal{S}_{-i}(s_{pub})
\end{split}
\end{align}

}

\subsection{Gadget Game Construction}

\begin{figure}[ht]
    \centering
    \begin{subfigure}[b]{0.4\textwidth}
        \includegraphics[width=\textwidth]{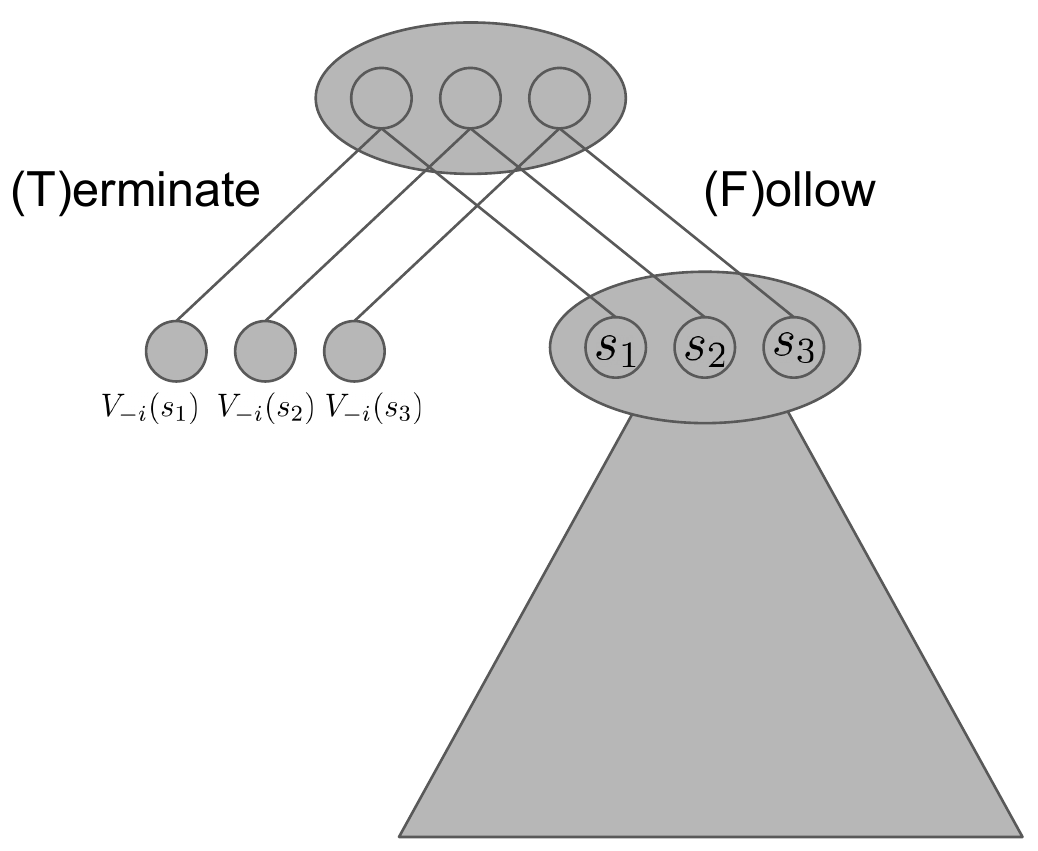}
        \caption{}
    \end{subfigure}
    \hspace{0.1\textwidth} 
    \begin{subfigure}[b]{0.25\textwidth}
        \includegraphics[width=\textwidth]{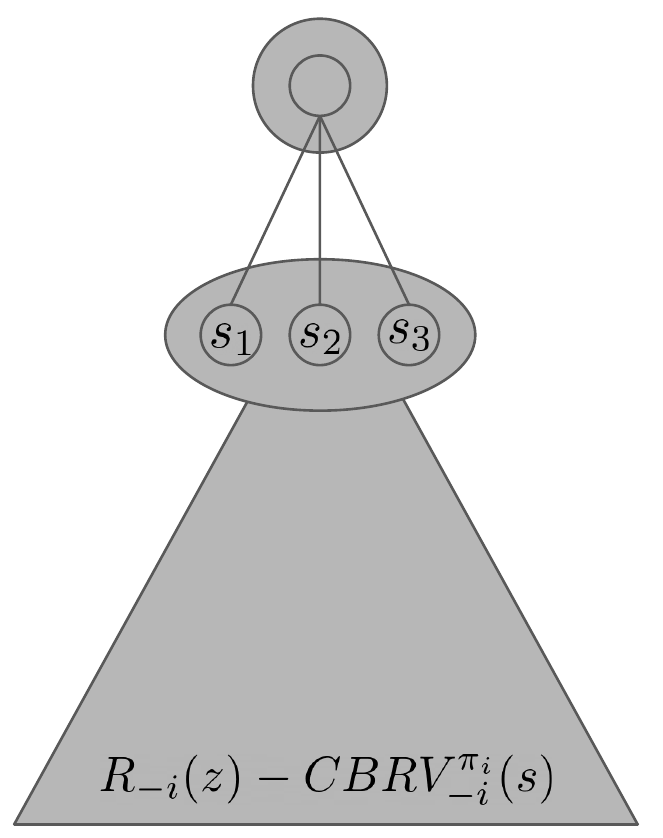}
        \caption{}
    \end{subfigure}
    \caption{
    Comparison of the CFRD and max-margin gadgets.
    a) CFRD gadget: the opponent at each root infostate $s \in \mathcal{S}_{-i}(s_{pub})$ chooses to either (T)erminate, receiving the bound value $V_{-i}(s_{pub}, s)$, or to (F)ollow.
    b) Max-margin gadget: the opponent chooses an infostate $s \in \mathcal{S}_{-i}(s_{pub})$ to start with, and the corresponding terminal values are shifted so that the values under the original policy are all zero.
    }
    \label{fig:iig-max_margin_gadget}
\end{figure}

\first{
Similarly to the gadget-game construction of the CFR-D constraints (Section~\ref{sec:iig-cfrd_gadget}), we can create a gadget game corresponding to the max margin optimization. 
We create the gadget game by making two modifications to the original \subgame{}.
i) we shift the opponent's utilities using the $\cfbrv{}_{-i}$, initializing all the state values to zero and 
ii) we add opponent's node followed by chance nodes at the top of the \subgame{}, allowing the opponent to pick any starting state.

We will distinguish the states, strategies, utilities, etc. for the gadget game by adding a tilde to corresponding notation.
The following is a detailed description of the steps (see also Figure~\ref{fig:iig-max_margin_gadget} that visualizes the constructed max margin gadget)

\paragraph{Utilities Shift}
    To compare the changes in the performance of each of opponents' (root) information states, it is necessary to give them a common baseline. 
    We use the original strategy $\pi_i^{s_{pub}}$ as the starting point.
    For every $s \in \mathcal{S}_{-i}(s_{pub})$, we subtract the opponent's original counterfactual best response value from all the terminal histories $z$ reachable from $s$, setting the return to
    $\tilde{R}_{-i}(z) = R_{-i}(z) - \cfbrv_{-i}^{\pi_i}(s)$.
    We also update $\tilde{R}_i(z) = -\tilde{R}_{-i}(z)$ since we need the game to remain zero-sum.
    The reason behind this utilities shift is that under the original policy $\pi_i$, opponent's values are now zero for all of their root infostates: $\cfbrv{}_{-i}^{\pi_i}(\tilde{s}) = 0 \, \forall \tilde{s} \in \tilde{\mathcal{S}}_{-i}(s_{pu})$.
    
\paragraph{Opponent's Node}
    The opponent is permitted to choose the distribution of their private states at the start of the \subgame{}, while the player retains their distribution from the original strategy $\Delta(\mathcal{S}_i(s_{pub}))$.  
    Since the opponent is aiming to maximize $\tilde{R}_{-i}$, they will always select the information state with the lowest margin.
    The minimax nature of the zero-sum game forces the player to find a strategy that maximizes this value.
    The construction adds a decision node $\tilde{d}$ for the opponent, where each action corresponds to choosing an information set $s \in \mathcal{S}_{-i}(s_{pub})$ to start with.
    Each action the leads to a new chance node $\tilde{c}_{s}$, where the chance player chooses the histories $h \in \mathcal{H}(s)$
    based on $\Delta(\mathcal{S}_i(s_{pub}))$.

}

\subsection{Experiments} 
\second{

\begin{figure}[ht]
  \centering
  \includegraphics[width=0.8\textwidth]{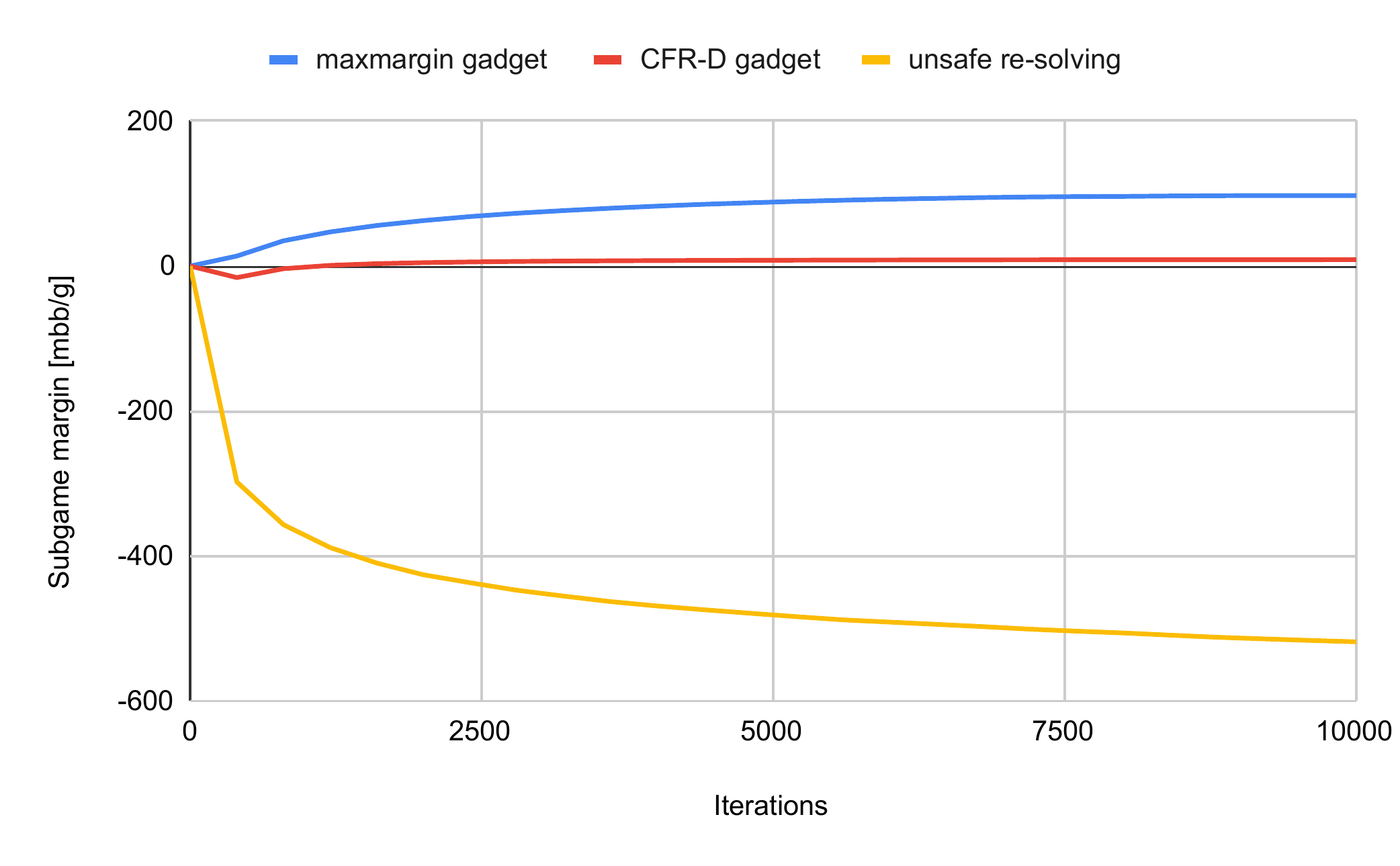}
  \caption{\Subgame{} margins of the refined strategies.
  One big blind corresponds to 100 chips.
  The max-margin technique produces the optimal value. 
  We see that the optimal value is
much greater than the one produced by either re-solving or
endgame solving (which produces even negative margins).
The 95\% confidence intervals for the results (after $10,000$
iterations) are: maxmargin $101.49 \pm  7.09$, \resolving{} $8.79 \pm 2.45$, unsafe solving $-518.5 \pm 49.19$.
}
\label{fig:iig-subgame_refinement_graph}
\end{figure}

We compare the unsafe \resolving{}, CFR-D re-solving and max-margin \subgame{} refinement in \subgames{} of no-limit Texas hold'em poker.
We use an improved version of the Nyx agent, the second strongest participant at the $2014$ Annual Computer Poker Competition (heads-up no-limit Texas hold'em total bankroll) as the baseline strategy to be refined in the \subgames{}.

All the three techniques tested start with the same abstractions and trunk strategy. 
Following \cite{ganzfried2015endgame}, we begin the \subgames{} at start of the last round (the river).
While we use card abstraction to compute the original (trunk) strategy, specifically \citep{schmid2015automatic} and \citep{johanson2013evaluating}, the fine-grained abstraction for the endgame is using no card abstraction.
We use the same actions in the refined \subgame{} as in the original strategy.

We refine only the \subgames{} that (after creating the fine-grained abstraction) are smaller 
than $1,000$ betting sequences --- this is simply to speed up the experiments.
The original agent strategy is used for both players in the trunk of the game.  
Once the self-play reaches the \subgame{}, we refine the strategy of the first player using each of the three techniques.
We ran $10,000$ iterations of the CFR+ algorithm in the corresponding gadget games and each technique was used to refine around $2,000$ \subgames{}.
Figure~\ref{fig:iig-subgame_refinement_graph} then visualizes the average margins for the evaluated techniques.
}

\first{
\textbf{Unsafe \Resolving}
The largely negative margin values suggest that the produced strategy may indeed be much more exploitable.
\textbf{CFR-D \Resolving{}}
The positive margin for \resolving{} shows that, although there's no explicit construction that forces the margin to be greater than zero, it does increase in practice.
Notice, however, that the margin is far below the optimal level.
\textbf{Max Margin Refinement}
This technique produces a much larger \subgame{} margin than the previous techniques.
The size of the margin suggests that the original strategy is potentially quite exploitable, and our technique can substantially decrease its exploitability (Theorem~\ref{thm:iig-maxmargin}).
}
}

\section{Full Lookahead Continual Resolving}
\second{
Safe re-solving produces a policy for a single (public) sub-game and thus corresponds to a single step of search (Figure~\ref{fig:iig-search_1}), producing a strong globally consistent strategy.
We now use safe re-solving as a building block in the continual re-solving algorithm.

As the name suggests, the algorithm continually performs re-solving step for the states visited during gameplay (Figure~\ref{fig:iig-search_2}).
Continual \resolving{} thus needs to use and update the invariants required by the individual re-solving step.
These invariants are updated on the game-play trajectory as the full lookahead rooted in the prior state includes the current game state and its required values.

\subsection{Updating Invariants}

\begin{figure}[ht]
  \centering
  \includegraphics[width=\textwidth]{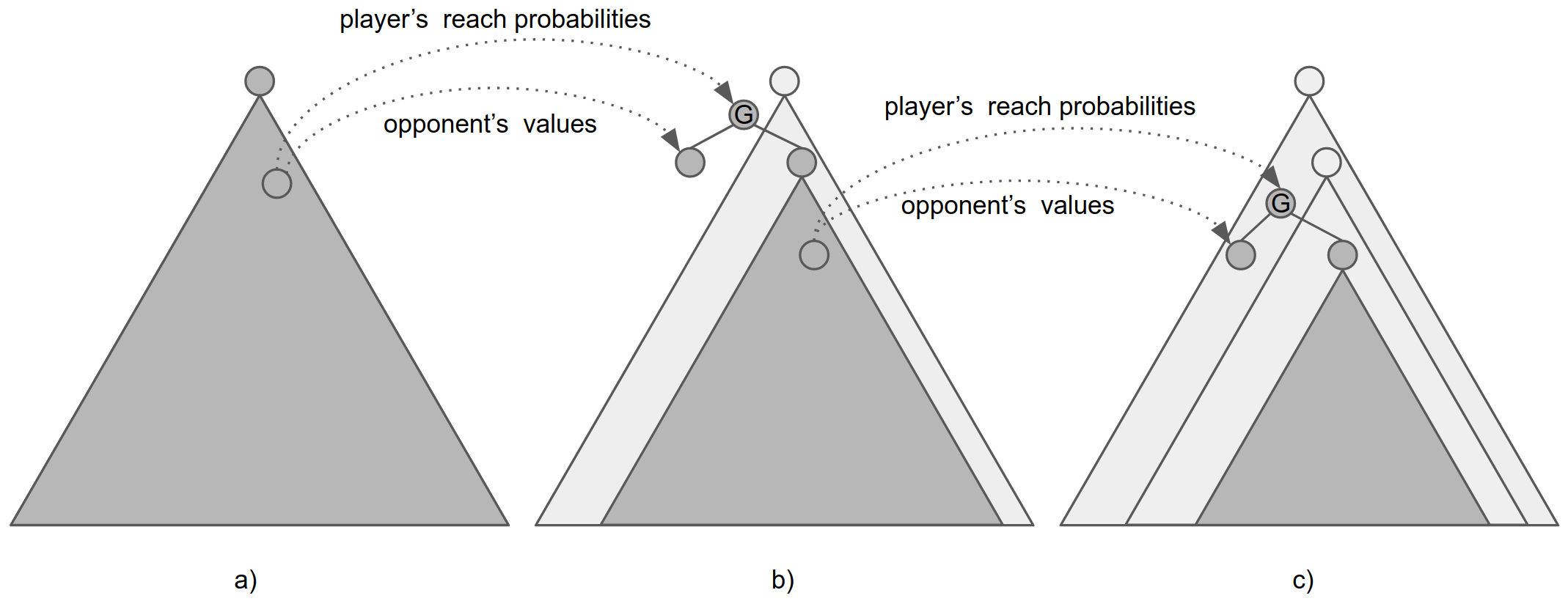}
  \caption{
  Continual \Resolving{}:
  a) Solve is run for the first state of the game.
  b) Re-solving step is performed for the next state the player is to act.
  To run re-solving, we construct the re-solving game using the opponent's value and player's reach probabilities.
  As the lookahead in the previous step included this state, the previous computation includes both of these quantities.
  c) We keep running re-solving for all the next nodes the player is to act in the same way.
  }
\label{fig:iig-continual_resolving_full_lh}
\end{figure}

Recall that the re-solving step requires the following invariants i) players' reach probabilities and ii) opponent's counterfactual values.
First observe that for the very first state of the game, no re-solving is necessary and we can simply solve the initial game and store the result (Figure \ref{fig:iig-continual_resolving_full_lh}a).
For a state visited during the game, we build a \subgame{} rooted in the current public state and use the corresponding invariants for the re-solve.
The invariants are retrieved from the previous lookahead tree as it must include the currently visited state.
This process is then repeated until the game is finished (Figure~\ref{fig:iig-continual_resolving_full_lh}bc).

\begin{algorithm}
\caption{Continual \Resolving{} with Full Lookahead}
\label{alg:iig-continual_resolving_full_lh}
\begin{algorithmic}[1]

\Function{Play}{$s \in \mathcal{S}_i$}
    \State $s_{pub} \gets \mathcal{S}_{pub}(s)$
    \State \Comment{Uses the average strategy from the last search tree.}
    \State $\Delta(\mathcal{S}_i(s_{pub})) \gets GetReachProbabilities(last\_search\_tree, s_{pub})$
    \State \Comment{Uses the averaged counterfactual values from the last search tree.}
    \State $V^{bound}_{-i}(s_{pub}) \gets GetCounterfactualValues(last\_search\_tree, s_{pub}, -i)$
    \State $search\_tree \gets$ \Call{BuildFullLookaheadTree}{$s_{pub}$}
    \State \Call{SafeResolve}{$search\_tree$, $\Delta(\mathcal{S}_i(s_{pub})$, $V^{bound}_{-i}(s_{pub}))$}
    \State $last\_search\_tree \gets search\_tree$
    \State \Return $a \sim \pi_i(s)$
\EndFunction
\State
\Function{NewGame}{}
    \State $last\_search\_tree \gets $ \Call{BuildFullLookaheadTree}{$s^{initial}_{pub}$}
    \State \Call{Solve}{$last\_search\_tree$}
\EndFunction
\end{algorithmic}
\end{algorithm}

}

\subsection{Resulting Policy}
\second{
As we need to reason about all the states in a public state, the algorithm produces a policy for all such states rather than the current infostate (i.e. strategy for all the poker hands we could be holding).
Although policies for the states other than the current one are not directly useful for action selection, they are need to make a sound \resolve{} search.
They can also be leveraged by variance reduction techniques (e.g. AIVAT, Section~\ref{sec:iig-aivat}).

As the re-solving step uses the gadget game, any game solving method can be used.
The common choice is to use a member of the CFR family, and continual \resolving{} has also been analysed with its Monte Carlo variants \citep{sustr2018monte}.
}

\subsection{Experiments}
\second{
While the full lookahead is clearly impractical for large games, one can still evaluate the performance of the continual \resolving{} algorithm on small game variants.
Figure \ref{fig:continual_resolving_full_lh} reports exploitability on Leduc poker and graph chase game (\glasses{}).
As continual \resolving{} is a stateless online search algorithm (Section~\ref{sec:pig-online_stateful_vs_stateles}), we can compute the exploitability of the tabularized policy (Section~ \ref{sec:pig-tabularized_policy}).
}

\begin{figure}[ht]
    \centering
    \begin{subfigure}[b]{0.48\textwidth}
        \includegraphics[width=\textwidth]{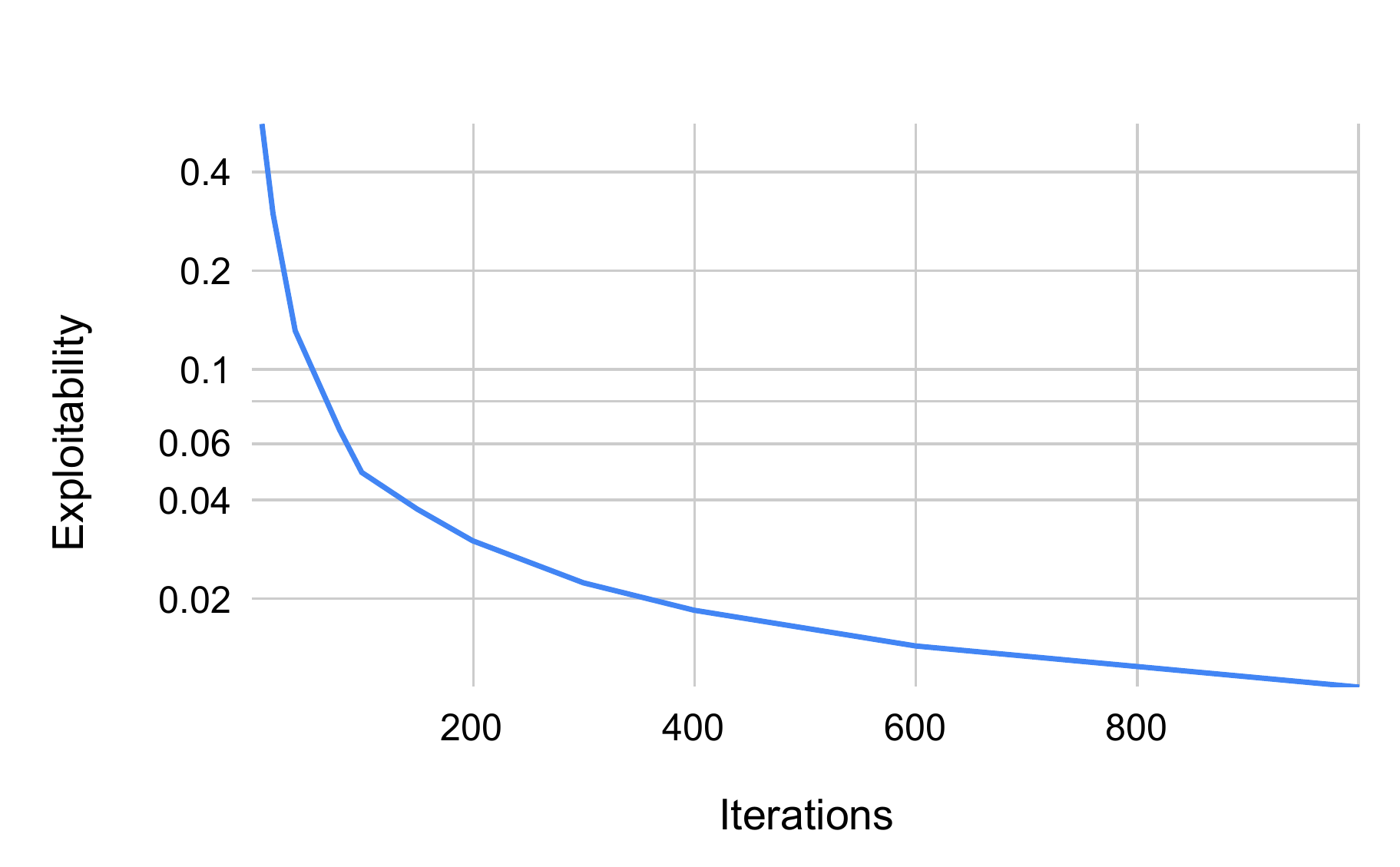}
        \caption{Leduc poker.}
        \label{fig:iig-continual_resolving_full_lh_leduc}
    \end{subfigure}
    \begin{subfigure}[b]{0.48\textwidth}
        \includegraphics[width=\textwidth]{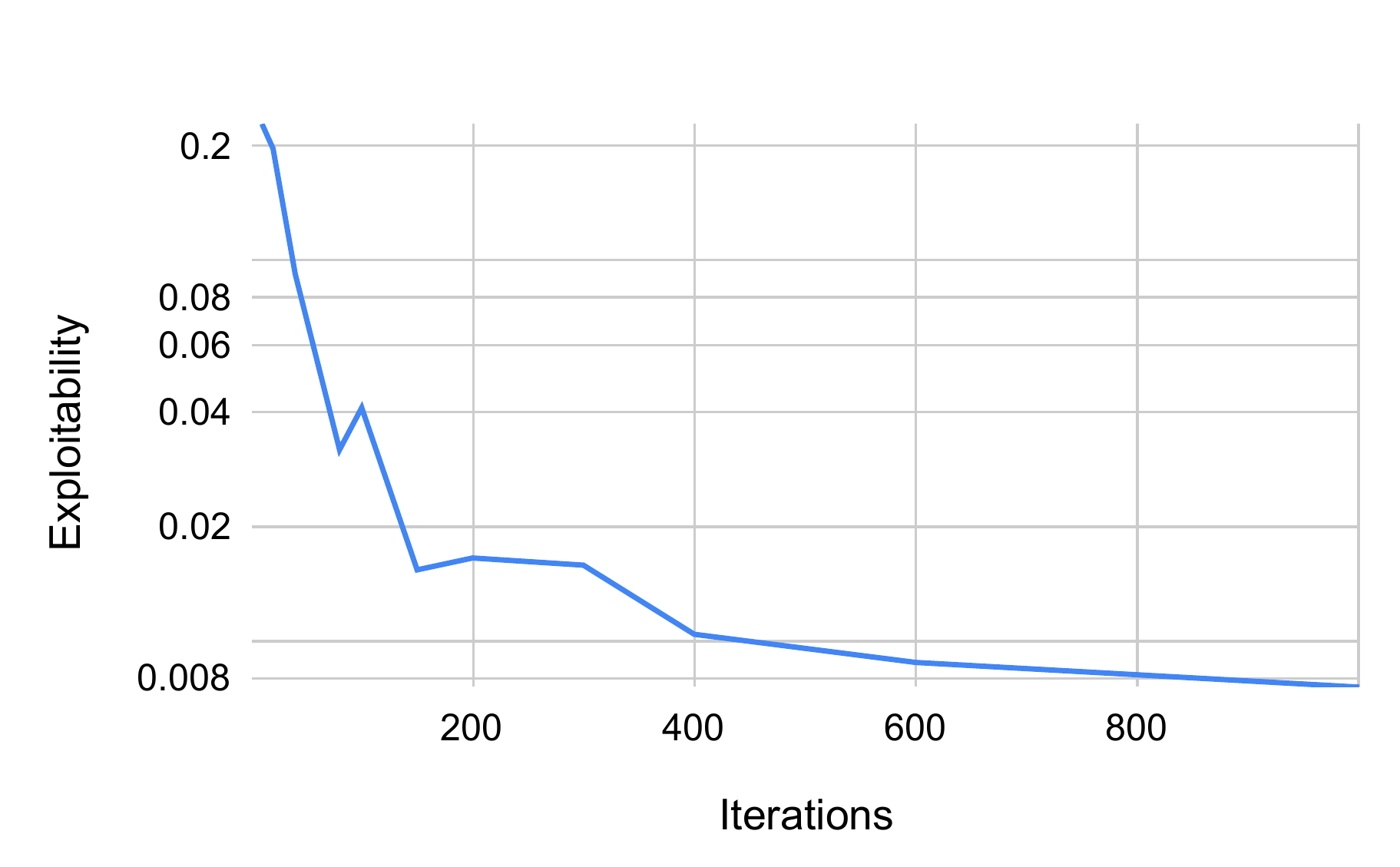}
        \caption{\glasses{}.}
        \label{fig:continual_resolving_full_lh_glasses}
    \end{subfigure}
    \caption{Continual \resolving{} with full-lookahead tree and increasing number of CFR iterations for the search / re-solve.}
    \label{fig:continual_resolving_full_lh}
\end{figure}

\section{Limited Lookahead and Value Functions}
\label{sec:iig-depth_limited_solve}
\second{
We are now ready to introduce a dept-limited solving.
Just like in perfect information games, this allows us to reason over only a limited number of steps forward and evaluate the \subgames{} at the end of our lookahead using a value function.
Fortunately, we already know how generalized value functions look for imperfect information games (Section \ref{sec:iig-value_functions}).
We just need to properly use them in the limited lookahead paradigm.

The separation to lookahead tree and \subgames{} is possible thanks to the notion of public tree and we will show a simple modification of the public tree CFR that leverages value functions for the \subgames{}.
}

\begin{figure}[ht]
  \centering
\includegraphics[width=0.3\textwidth]{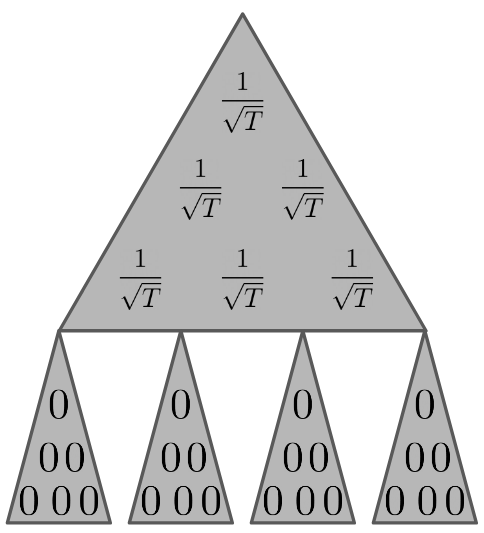}
  \caption{CFR with limited lookahead and value functions.
  During each iteration, regret minimizer is used to in the lookahead tree and zero-regret policies are played in the \subgames{}.}
\label{fig:iig-limited_lookahead_solve}
\end{figure}

\subsection{CFR and Value Functions}
\second{
CFR is particularly suited to be modified to use value functions.
This is because the CFR Theorem (Theorem~\ref{thm:iig-cfr}) bounds the overall regret by the sum of individual counterfactual regrets.
Furthermore, recall that the CFR algorithm (Algorithm~\ref{alg:iig_cfr}) essentially sends reach probabilities down the tree and sends values up the tree.
Consider what happens if we use a regret minimizer for all states in the lookahead tree, and use a zero-regret policy for all the states of the \subgames{} (Figure \ref{fig:iig-limited_lookahead_solve}).
The CFR Theorem then guarantees convergence of this approach.

As a zero-regret policy in a \subgame{} corresponds to a counterfactual best response, both players are best-responding to each other and thus playing optimally.
During each CFR iteration, we can thus simply use an optimal policy in all the \subgames{}.
Furthermore, the upward pass of CFR only requires values of those \subgames{}.
To run CFR in the lookahead tree, each iteration never explicitly requires one to compute the optimal (zero-regret) policies in the \subgames{}, it only requires the optimal values corresponding to  \subgame{} --- value functions.

Algorithm~\ref{alg:public_tree_cfr_limited_lh} then shows a minor modification to the ComputeValues Function (Algorithm~\ref{alg:iig_cfr}, Line \ref{alg:func-compute_values}) that results in version of CFR that uses value functions.
The reason why this is a particularly simple change is that our initial implementation of CFR was already operating on public states.
}

\second{
An important detail is that the counterfactual values returned by the value function need to be counterfactual best response values (Section~\ref{sec:iig-counterfactual_best_response}).
This is because value function returns values under an optimal policy, but we need to further restrict the corresponding optimal policy --- making sure the per-iteration regrets in the \subgame{} are all zero.
Fortunately, producing zero-regret optimal policies is not at all difficult and CFR algorithm already produces such policies.
}

\begin{algorithm}
\caption{Limited Lookahead Tree CFR}
\label{alg:public_tree_cfr_limited_lh}
\begin{algorithmic}[1]

\Function{ComputeValues}{$s_{pub} \in \mathcal{S}_{pub}$, $d_1 \in \Delta(\mathcal{S}_1(s_{pub}))$, $d_2 \in \Delta(\mathcal{S}_2(s_{pub}))$}
    \If{EndOfLookahead($s_{pub}$)}
        \State $v_{i,c}(s) \gets$ \Call{ValueFunction}{$s_{pub}, d_1, d_2$}
        \State \Return
    \EndIf
\EndFunction
\end{algorithmic}
\end{algorithm}

\section{Continual \Resolving{} with Value Functions}
\label{sec:iig-continual_resolving_vf}

\second{
It is time to combine all of the building blocks introduced up to this point, culminating in a practical sound search algorithm for imperfect information games.
We simply modify continual \resolving{} so that the individual \resolving{} steps use limited lookahead and generalized value functions (Figure \ref{fig:iig_continual_resolving}).
The combination of continual \resolving{} and value functions then results in the critical bound on the resulting policy (Theorem~\ref{thm:iig-deepstack}).

\begin{thm}
\label{thm:iig-deepstack}
\citep[Theorem 1]{moravvcik2017deepstack}
If the values returned by the value function used when the depth limit is reached have error less than $\epsilon$, and $T$ iterations of CFR are used to re-solve, then the exploitability of continual \resolving{} with value functions is less than $k_1 \epsilon + k_2 \sqrt{T}$, where $k_1$ and $k_2$ are game-specific constants.
\end{thm}
\begin{proof}
See the Appendix of \cite{moravvcik2017deepstack}.
\end{proof}
}

\begin{figure}[ht]
  \centering
  \includegraphics[width=\textwidth]{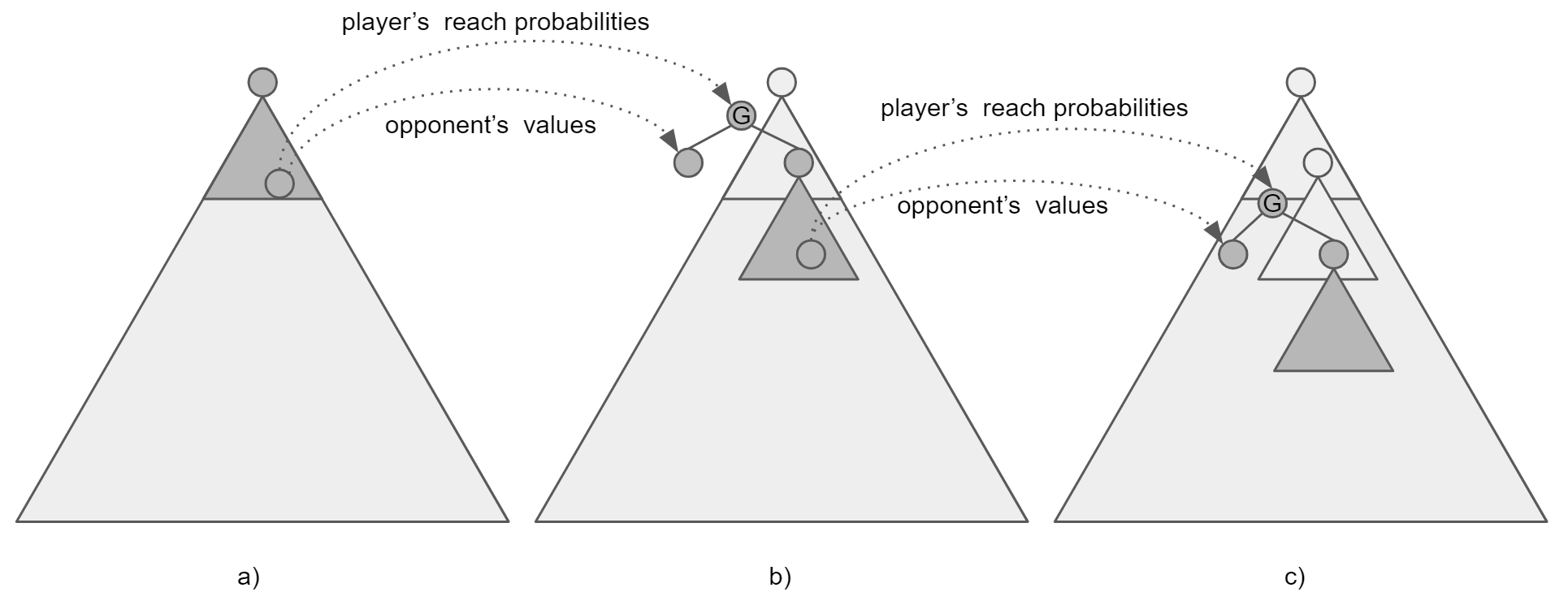}
  \caption{
  Continual Resolving:
  a) Limited lookahead solve is run for the first state of the game.
  b) Re-solving step is performed for the next state the player is to act.
  To run re-solving, we construct the re-solving game using the opponent's value and player's reach probabilities.
  If the lookahead in the previous step included this state,
  the previous computation includes both of these quantities.
  c) We keep running re-solving for all the next nodes the player
  is to act in the same way.
  }
\label{fig:iig_continual_resolving}
\end{figure}

\subsection{Value Functions}
\second{
Training generalized value functions is an open problem.
Deep learning methods are a natural choice to learn and represent the generalized value functions, but how should we train these value functions?
What training data and loss should one use?
The ultimate goal is to learn value functions so that the search results in a low exploitable policy.
But the exploitability is a complex function of the interaction of search and the value function.
As this question is far from answered in perfect information games, it is hardly a surprise that there is little known at this point about the imperfect information case.
}

\subsection{Experiments}
\second{
We again use Leduc poker and \glasses{} to evaluate our search algorithm, and use tabularization for the exploitability computation.
To see the effect of inexact value functions (i.e. functions producing value from an $\epsilon$-optimal policy with non-zero counterfactual regrets), we use the CFR algorithm with a varying number of iterations to serve as the value functions.
This is important since in practice, we will have access only to an approximate value functions.

Figure~\ref{fig:continual_resolving_vf} then shows the exploitability with a varying number of CFR iterations in search and in the value functions.
We can see that as we increase the number of CFR iterations within the search tree, the exploitability gets lower until it hits a limit of the imperfect value function.
Improving the quality of the value function (by increasing number of CFR iterations for the value function) then allows the search to converge closer to an optimal solution.
This is in line with the presented theorem, where the resulting exploitability is bound by the number of search iterations and the quality of the value function (Theorem~\ref{thm:iig-deepstack}).
}

\begin{figure}[ht]
    \centering
    \begin{subfigure}[b]{0.49\textwidth}
        \includegraphics[width=\textwidth]{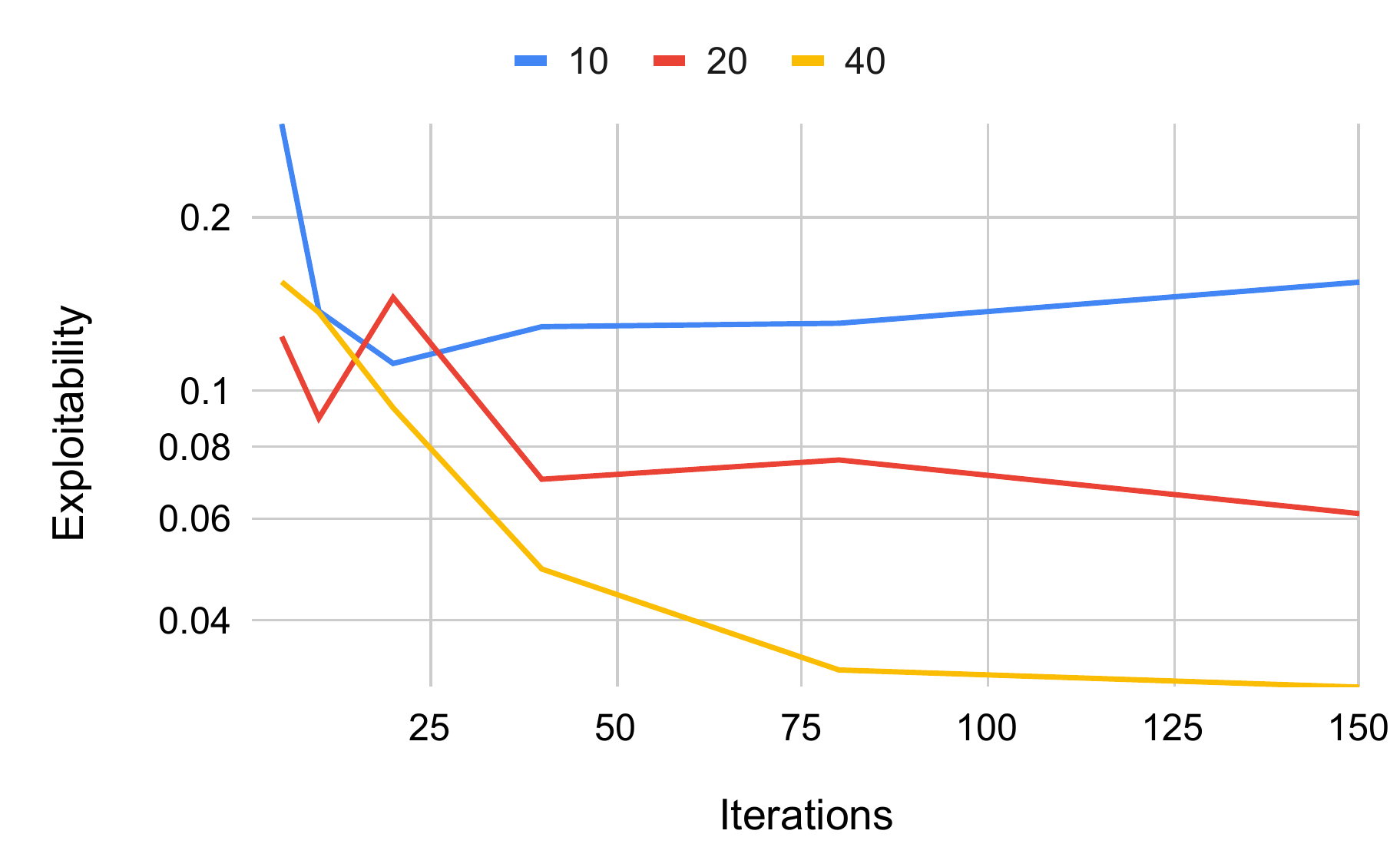}
        \caption{Leduc poker}
        \label{fig:continual_resolving_perfect_vf_leduc}
    \end{subfigure}
    \begin{subfigure}[b]{0.49\textwidth}
        \includegraphics[width=\textwidth]{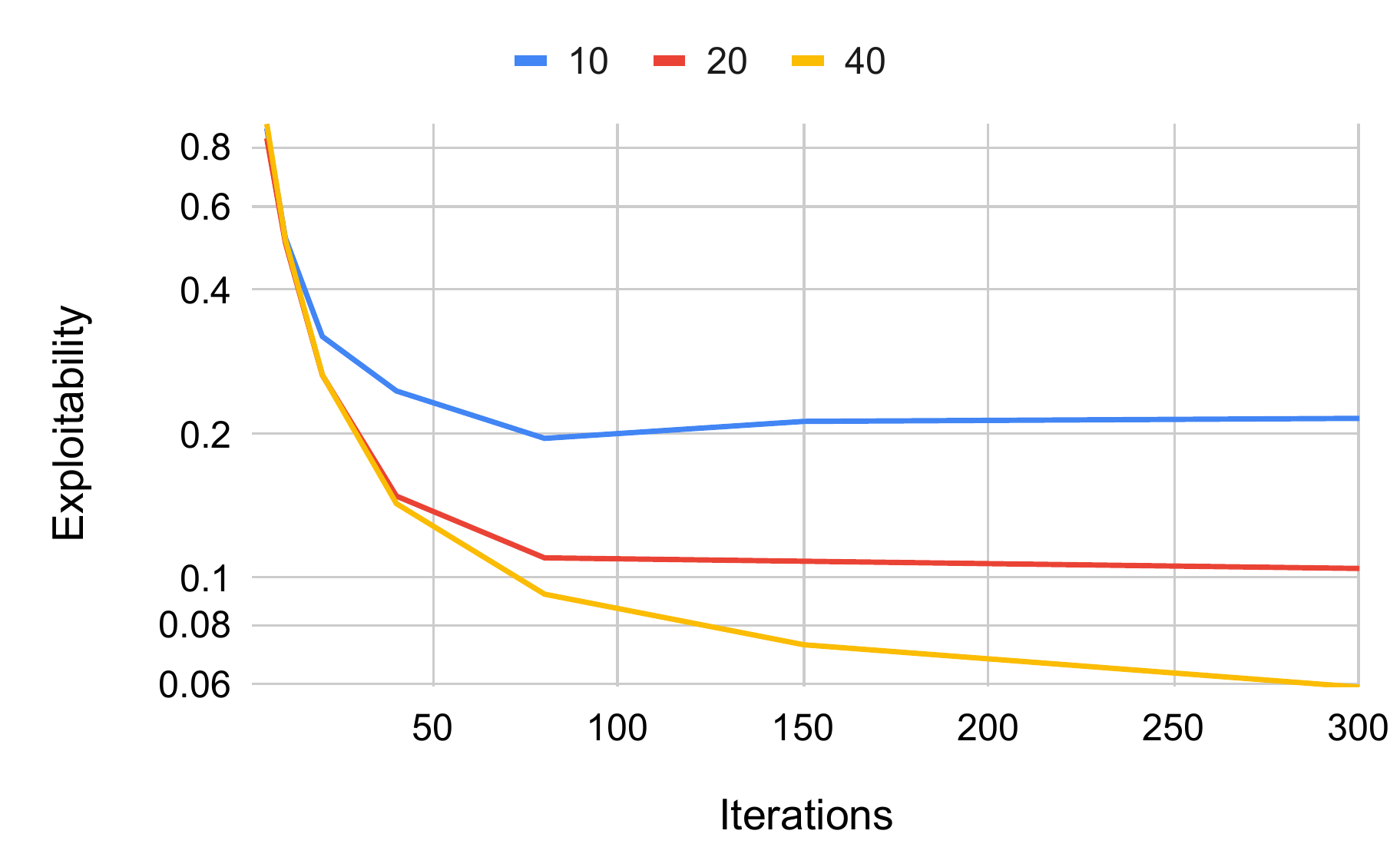}
        \caption{Graph chase game}
        \label{fig:continual_resolving_perfect_vf_glasses}
    \end{subfigure}
    \caption{Continual \resolving{} with limited lookahead tree.
    Lookahead tree is the smallest possible --- only one step forward.
    For the value function, we use CFR to compute the target values with varying number of iterations ($10$, $20$ and $40$).
    }
    \label{fig:continual_resolving_vf}
\end{figure}

\section{Limitations}
\second{
Currently, there are some inherent limitations of the presented sound search methods.
The \resolving{} steps as well as the value functions reason about all the states within a public state $\mathcal{S}_i(s_{pub})$.
While necessary for sound \resolve{} methods, there are games where number of such states is intractably large.
The methods also require explicit model of the environment (i.e. game tree) to run the search.
Finally, as the model is discrete it can not capture continuous action space.
}
\chapter{DeepStack}
\label{chap:iig-deepstack}

\second{
DeepStack --- the final contribution of the thesis --- was the first to introduce the combination of sound search and value functions in imperfect information game of poker \citep{moravvcik2017deepstack}.
The combination of continual resolving and neural networks as value functions led to a leap improvement over the prior methods.
First, unlike the abstraction based techniques, DeepStack was unexploitable by the local best response (Section~\ref{sec:iig-lbr}).
Second, DeepStack became the first program to beat professional human players in no-limit Texas hold'em poker.
}

\section{Search / Resolving Step}
\label{sec:iig-deepstack_search}
\begin{figure}[ht]
  \centering
  \includegraphics[width=1.0\textwidth]{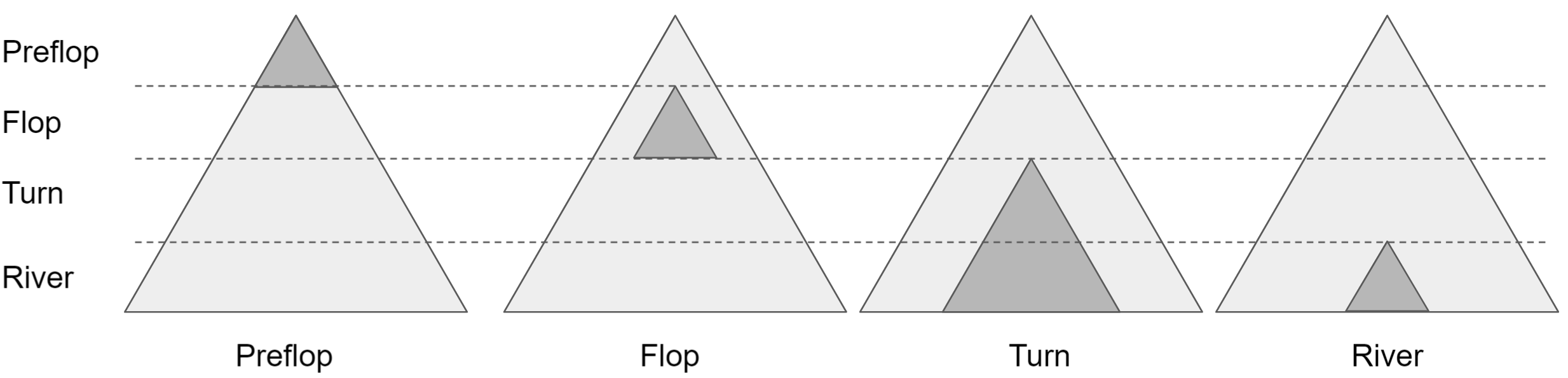}
  \caption{
  Lookahed used on different streets, see also Table~\ref{tab:iig-deeptack_lookaheads}}
\label{fig:iig_deepstack_lookaheads}
\end{figure}

\second{

DeepStack used continual resolving with limited lookaheads and value functions (Section~\ref{sec:iig-continual_resolving_vf}).
For the re-solving construction, it used a modified variant of the CFR-D gadget (Section~\ref{sec:iig-cfrd_gadget}) and fixed sparse lookahead tree (details in Section~\ref{sec:igg-deepstack_action_absraction}).
The lookahead tree was built all the way to the next round (for pre-flop and flop infostates) or until the end of the game (for turn and river).
This guaranteed that the lookahead included all the possible next states so that we could correctly update the invariants during the continual resolving process.

DeepStack then used a hybrid of CFR ad CFR+ for the policy computation --- combining regret matching plus, simultaneous updates and uniform  weighting of the average policy.
Furthermore, it skipped the first half of iterations when averaging the policy and counterfactual values.
The motivation is that the early iterations are often relatively poor (especially in the case of no-limit poker).
Skipping iterations has been proven to be sound and is a commonly used trick.
For more details on the search parameters, see Table~\ref{tab:iig-deeptack_lookaheads} and Figure~\ref{fig:iig_deepstack_lookaheads}.

\paragraph{Preflop}
Preflop used $2,000$ iterations, skipped $1,000$ of them and the lookahead was built until the beginning of flop.
During the initial skipped iterations, the auxiliary preflop network (at the end of the preflop) was used to speed up the computation.
This allows to use a single value function call for each state at the end of the preflop, rather than enumerating all the possible (isomorphic) $22,100$ flops.
For the final iterations that actually accumulate the the average strategy and counterfactual
values, we did the expensive enumeration and flop evaluations as we need the counterfactual values for continual resolving.

\paragraph{Flop}
Preflop used $1,000$ iterations, skipped $500$ of them and the lookahead was built until the beginning of turn.

\paragraph{Turn}
Preflop used $1,000$ iterations, skipped $500$ of them and the lookahead was built all the way until the end of the game.
To speed up the computation, we used a bucketed abstraction for the river \subgames{} within the lookahead tree (the same bucketing as used for neural net representation described in Section~\ref{sec:iig-deepstack_network_features}).

\paragraph{River}
Preflop used $1,000$ iterations, skipped $500$ of them and the lookahead was built all the way until the end of the game.
No abstraction other than the sparse lookaheads were used as these river \subgames{} can be resolved relatively quickly.

\begin{table*}[ht]
\centering
\renewcommand{\arraystretch}{1.3}
\begin{tabular}{@{}lllll@{}}
\toprule
Round & CFR Iterations & Skip Iterations & Lookahead & Value Functions \\
\midrule
Pre-flop & 1,000 & 500 & Until Flop & Flop/Aux \\
Flop & 1,000 & 500 & Until Turn & Turn \\
Turn & 1,000 & 500 & Full & - \\
River & 2,000 & 1,000 & Full &- \\
\bottomrule
\end{tabular}
\caption{
Detailed parameters of the re-solving step of DeepStack (see also Figure~\ref{fig:iig_deepstack_lookaheads}).
}
\label{tab:iig-deeptack_lookaheads}
\end{table*}

}

\section{Action Abstraction in Lookahead}
\label{sec:igg-deepstack_action_absraction}

\begin{figure}[ht]
    \centering
    \begin{subfigure}[b]{0.48\textwidth}
        \includegraphics[width=\textwidth]{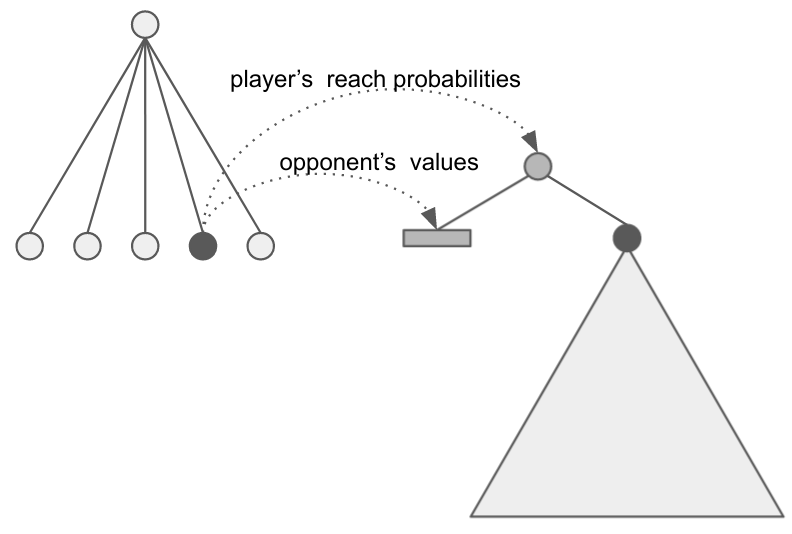}
        \caption{}
        \label{fig:iig-deepstack_maxtrick1}
    \end{subfigure}
    \begin{subfigure}[b]{0.48\textwidth}
        \includegraphics[width=\textwidth]{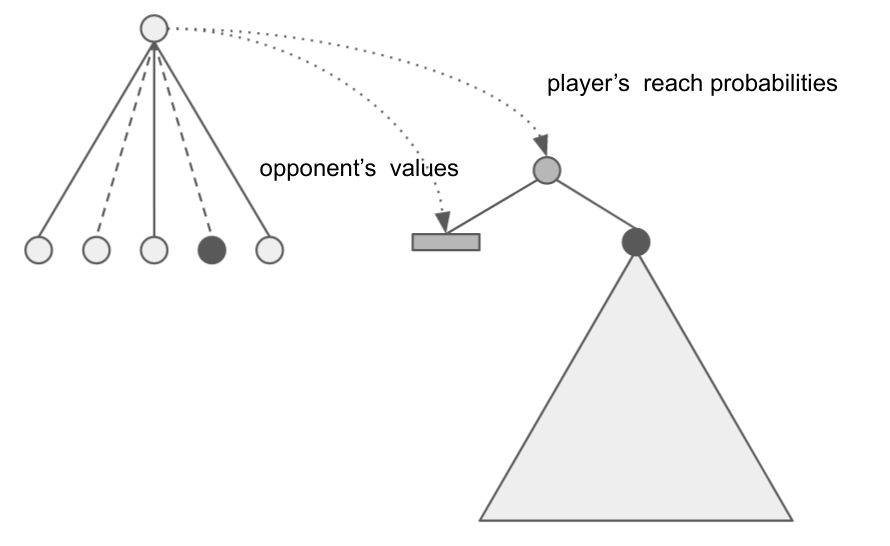}
        \caption{}
        \label{fig:iig-deepstack_maxtrick2}
    \end{subfigure}
    \caption{(a) Resolving a state included in the abstraction.
    We simply use the invariants for that state.
    (b) Resolving a state after the opponent made an action not included in the sparse lookahead.
    Due to structure of the game and the lookaheads used by DeepStack, we can use invariants from the previous state.
    }
    \label{fig:iig-deepstack_maxtrick}
\end{figure}

\begin{table*}[ht]
\centering
\renewcommand{\arraystretch}{1.3}
\begin{tabular}{@{}llll@{}}
\toprule
Poker Round & First Actions & Second Actions & Remaining Actions \\
\midrule
Pre-flop & F, C, \nicefrac{1}{2}P, P, A & F, C, \nicefrac{1}{2}P, P, 2P, A & F, C, P, A \\
Flop & F, C, \nicefrac{1}{2}P, P,  A & F, C, P, A & F, C, P, A \\
Turn & F, C, \nicefrac{1}{2}P, P, A & F, C, P, A & F, C, P, A \\
River & F, C, \nicefrac{1}{2}P, P, P, A & F, C, \nicefrac{1}{2}P,P, P, A & F, C, P, A \\
\bottomrule
\end{tabular}
\caption{
Sparse lookahead actions.
(F)old, (C)all/(C)heck, (P)ot bet and fractional (P)ot bets, (A)ll-in.
}
\label{tab:iig-deepstack_action_abstracion}
\end{table*}

\second{
DeepStack used sparse lookahead trees, where it considered only a subset of all the possible actions (there are almost $20,000$ actions in each infostate).
Sparse lookahead trees are now a common choice in online search agents for the no-limit poker \citep{brown2018superhuman, brown2020combining, zarick2020unlocking}.
For the full list of actions included in DeepStack's search see Table~\ref{tab:iig-deepstack_action_abstracion}.

There is an important distinction to the abstraction techniques described in Section~\ref{chap:iig-abstraction}.
DeepStack never needs to translate the current (real) state into an abstraction, 
regardless of the actions taken by the opponent.
The resolving step always starts in the exact state of the game, and DeepStack thus perfectly understands the current situation and the pot-ods.
This is confirmed by our LBR analysis in Section~\ref{sec:iig-deepstack_lbr}.

But to start the safe resolving step for $s_{pub}$, we still require the counterfactual values of the opponent $V^{bound}_{-i}(s_{pub})$ and agent's reach distribution  $\Delta(\mathcal{S}_i(s_{pub}))$ (Section~\ref{sec:iig-safe_resolve}).
DeepStack leveraged a particular structure of the game to correctly compute these invariants even if the current state $s_{pub}$ was not included in the sparse search tree.

\paragraph{Counterfactual Values}
First, consider the counterfactual values. 
As there is always a single player to act and all the actions are publicly observable, we can assign a unique player to each $s_{pub}$ (agent/opponent/chance).
Under any best responding policy, the following then holds for any $s_{pub}$ of the opponent
$\forall s \in \mathcal{S}_{-i}(s_{pub}) \forall a \in \mathcal{A}(s) \, : \, v_{-i,c}^{\pi}(s) \geq q_{-i,c}^{\pi}(s, a)$.
Furthermore, there is a single state $sa$ after an action $a$ in a state $s$.
We thus have $v_{-i,c}^{\pi}(s) \geq v_{-i,c}^{\pi}(sa)$.
This allows us to use counterfactual values of the state $s$ as the constraint values for $sa$ --- that is, we only need the values for the state prior to the action of the opponent.
DeepStack's lookahead tree then guarantees that this is the only off-tree case.
This is because the lookahead included all the possible next chance states and DeepStack will never take an actions not included in its abstractions.

\paragraph{Reach Distribution}
Agent's distribution over their infostates $\Delta(\mathcal{S}_i (s_{pub}))$ is not a function of the opponent's policy and we can thus compute the reach distribution for the current state $s_{pub}$ after we observe the state (i.e. after the opponent made an action).
}

\subsection{Action Abstraction Analysis}
\second{
The choice of sparse lookahead tree abstraction thus only effects 
i) the actions the agent can make and
ii) the quality of the counterfactual values to be used in the next re-solving step.
While we know that it is sound to use the values of the parent state $s$ for resolving a state $sa$, we can only do so because we assumed the values for $s$ are based on a best responding policy.
As we limit the actions in the lookahead, the resulting values might be considerably different.

The action set used by DeepStack was constructed using a careful analysis, where experimented with different abstraction sizes and measured the quality of the resulting counterfactual values (Table~\ref{tab:iig-deeptack_action_abstraction_cfvalues}).
Additional experiments proved that it is critical to include fold, call, all-in and a bet around the pot size --- and the same choice is then made by other search agents (Section~\ref{sec:iig-other_agnets}).

\begin{table*}
\centering
\renewcommand{\arraystretch}{1.3}
\begin{tabular}{@{}lllll@{}}
\toprule
Betting & Size & $L_1$ & $L_2$ & $L_{\infty}$ \\
\midrule
F, C, Min, \nicefrac{1}{4}P, \nicefrac{1}{2}P, \nicefrac{3}{4}P, P, 2P, 3P, 10P, A & 555k & 18.06 & 0.891 & 0.2724  \\
F, C, P, A & 61k & 25.51 & 1.272 & 0.3372 \\
F, C, 2P, A & 48k & 64.79 & 2.672 & 0.3445 \\
F, C, \nicefrac{1}{2}P, A & 100k & 58.24 & 3.426 & 0.7376 \\
F, C, \nicefrac{1}{2}P, P, A & 126k & 41.42 & 1.541 & 0.2955 \\
F, C, P, 2P, A & 204k & 27.69 & 1.390 & 0.2543 \\
F, C, \nicefrac{1}{2}P, P, 2P, A & 360k & 20.96 & 1.059 & 0.2653 \\
\bottomrule
\end{tabular}
\caption{
Errors in counterfactual values when solving a river game with different action abstraction.
The counterfactual values computed with $1,000$ and averaged over $100$ random river situations.
As the ``ground truth'' targets, we used the largest set of actions (F, C, Min, \nicefrac{1}{4}P, \nicefrac{1}{2}P, \nicefrac{3}{4}P, P, 2P, 3P, 10P, A) with $4,000$ iterations.
}
\label{tab:iig-deeptack_action_abstraction_cfvalues}
\end{table*}
}

\section{Value Function}
\second{
DeepStack used deep fully connected neural networks with seven hidden layers to represent the value functions.
Each layer contained 500 neurons and used the parametric rectified linear units \citep{he2015delving}.
The output of the final linear layer then predicts individual counterfactual values and thus contains as many neurons as there are infostates.
Since the game is zero-sum, it must be the case that the reach-weighted sum of counterfactual values is in balance (Equation~\ref{eq:iig-deepstack_net_gadget1}).
To make sure this is the case, we included additional custom layer that redistributes the ``excess'', ensuring the network produces only zero-sum values.

\begin{align}
\label{eq:iig-deepstack_net_gadget1}
\sum_{s \in \mathcal{S}_i(s_{pub})} \Delta(\mathcal{S}_i(s_{pub})(s) v_{i,c}(s) + \sum_{s \in \mathcal{S}_{-i}(s_{pub})} \Delta(\mathcal{S}_{-i}(s_{pub})(s) v_{-i,c}(s) = 0
\end{align}

We trained separate networks for flop and turn that were used during preflop and flop search respectively (as the lookahed reached all the way to the next street).
Furthermore, an auxiliary preflop network was used to speed up the computation on preflop.
}

\subsection{Networks Training}
\second{
All the networks were trained using the supervised training paradigm.
The training data was generated by sampling random\footnote{See supplementary material of DeepStack for details on the distribution.} \subgames{} (inputs) and solving them to obtain the counterfactual values (outputs/targets).
These input-output pairs were then used as the supervised dataset.
As the lookahead was built all the way to the next street (Section~\ref{sec:iig-deepstack_search}), the sampled \subgames{} always corresponded to the initial public states of the respective streets (with the exception of the auxiliary preflop network where the public states rather corresponded to the last states of preflop).

The solving used $1,000$ iterations of CFR+ and the FCPA action abstraction.
The training was implemented using the Torch7 libraries \citep{torch}.
The training loss was the average Huber loss \citep{huber1992robust} over the counterfactual values, minimized using the Adam stochastic gradient descent procedure \citep{kingma2014adam}.
We used batch size of $1,000$ and a learning rate $0.001$, which was decreased to $0.0001$ after the first $200$ epochs. 
Networks were trained for approximately $350$ epochs over two days on a single GPU, and the epoch with the lowest validation loss was then selected.

\paragraph{Turn}
The training data consisted of ten million randomly sampled turn \subgames{} that were then solved.

\paragraph{Flop}
We sampled one million flop \subgames{} and used limited lookahead solving (Section~\ref{sec:iig-depth_limited_solve}) with the already trained turn network as the value function --- thus bootstrapping the networks.

\paragraph{Aux Preflop}
We sampled ten million preflop situations and the target values were obtained by enumerating all $22,100$ possible flops
and averaging the counterfactual values from the flop network’s output (another level of bootstrapping).
This averaging is equivalent to the limited lookahead solving used for the flop network, as there are only chance states in between the last preflop state and the flop network.
}

\subsection{Networks Features}
\label{sec:iig-deepstack_network_features}
\second{
To simplify the generalization task of the networks, 
we map the distribution over the individual poker hands (combinations of public and private cards) into distribution over buckets/clusters.
This essentially compresses the state space of $\binom{52}{7}$ card combinations down to the number of buckets used.
For both the turn and flop networks, we used
$1,000$ clusters generated using k-means clustering with earth mover’s distance over hand-strength features \citep{ganzfried2014potential, johanson2013evaluating}.
The auxiliary preflop network used no bucketing as there are only $169$ isomorphic hands.
}

\section{Human Evaluation}
\second{
In collaboration with the International Federation of Poker (IFP, now the International Federation of Match Poker IFMP) \citep{ifmp},
we recruited $33$ players from $17$ countries.
Each player was expected to play $3,000$ games in between November 7th and December 12th, 2016.
Cash incentives were given to the top three performers
($\$5,000$, $\$2,500$ and $\$1,250$ CAD).

Poker is inherently game of high variance and even if one agent is substantially stronger than the opponent, one might have to play hundreds of thousands of games to get statistically significant result.
We thus evaluate the performance using AIVAT (Section~\ref{sec:iig-aivat}), where the counterfactual values produced by the continual resolving served as the value estimates required by AIVAT --- providing an excellent estimates and resulting in impressive 85\% reduction of standard deviation.

A total of $44,852$ games were finished by the players, with $11$ players completing the full 3,000 games. 
DeepStack won 492 $\pm$ 220 mbb/g when the raw data was used, and when AIVAT correction terms were used the performance was estimated at 486 mbb/g $\pm$ 40 mbb/g (note the significantly smaller confidence interval).
AIVAT also allows us to derive statistically significant individual results for all but one of the players who finished the required $3,000$ games.
DeepStack is beating all of such player, with 10 out of 11 results being significant and only for the best performing player is the result not significant.
The detailed results can be found in Table~\ref{tab:deepstack-results} and Figure~\ref{fig:iig_deepstack_winrate_graph}
}



\begin{figure}[ht]
  \centering
  \includegraphics[width=1.0\textwidth]{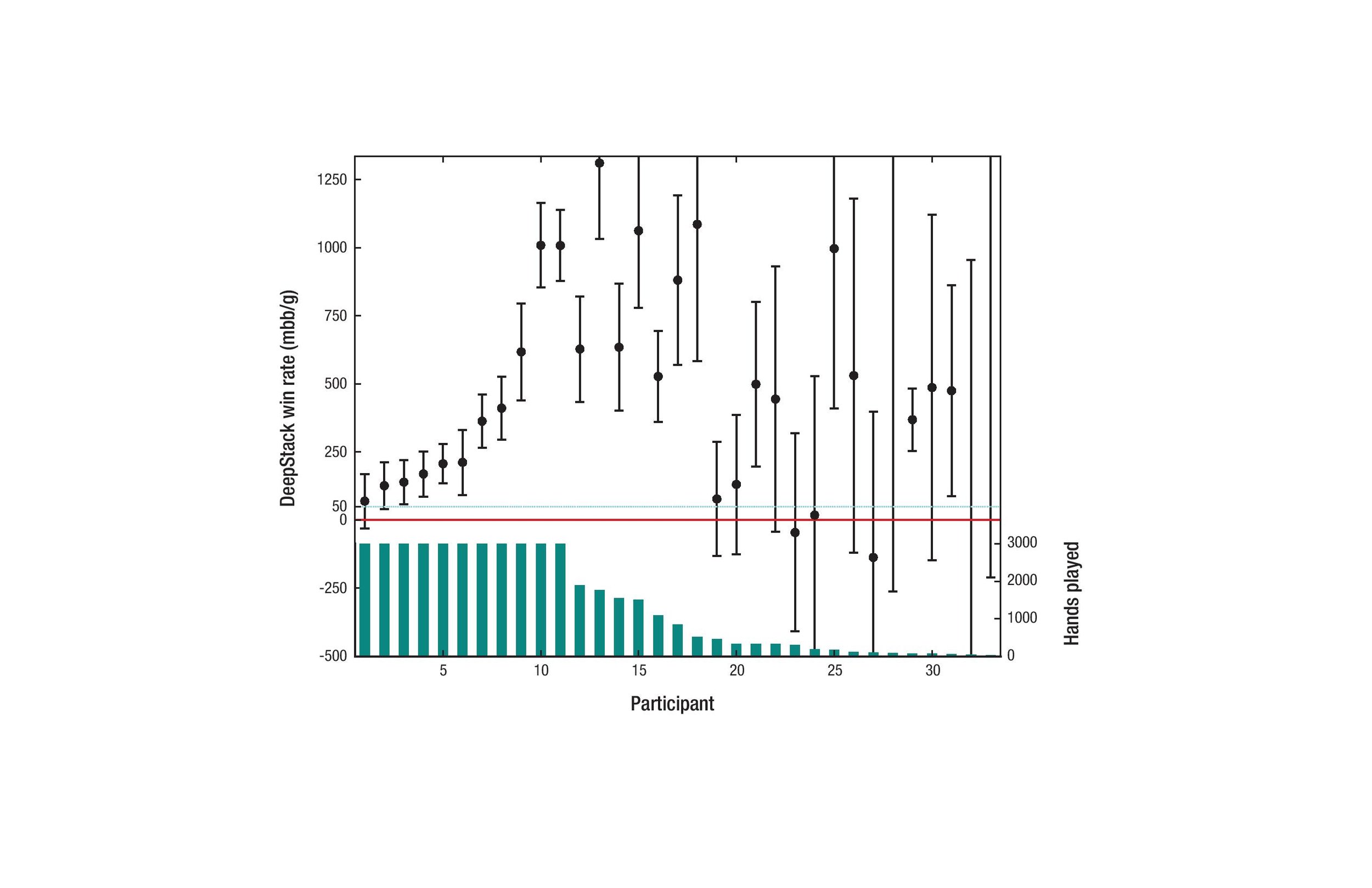}
  \caption{Win rate of DeepStack against the individual human players:
  AIVAT value with 95\% confidence interval and number of hands per player.
  Thanks to AIVAT variance reduction, the result is significant for all but one individual player who finished the required $3,000$ hands.
  See also Table~\ref{tab:deepstack-results} for more analysis and aggregate results.
  }
\label{fig:iig_deepstack_winrate_graph}
\end{figure}

\begin{table*}\centering
\renewcommand{\arraystretch}{1.3}
\begin{tabular}{@{}llr@{}}
\toprule
Player & Games & Win Rate (mbb/g) \\
\midrule
Martin Sturc & 3000 & $70 \pm 119$ \\
Stanislav Voloshin & 3000 & $126 \pm 103$ \\
Prakshat Shrimankar & 3000 & $139 \pm 97$ \\
Ivan Shabalin & 3000 & $170 \pm 99$ \\
Lucas Schaumann & 3000 & $207 \pm 87$ \\
Phil Laak & 3000 & $212 \pm 143$ \\
Kaishi Sun & 3000 & $363 \pm 116$ \\
Dmitry Lesnoy & 3000 & $411 \pm 138$ \\
Antonio Parlavecchio & 3000 & $618 \pm 212$ \\
Muskan Sethi & 3000 & $1009 \pm 184$ \\
Pol Dmit & 3000 & $1008 \pm 156$ \\
Tsuneaki Takeda & 1901 & $627 \pm 231$ \\
Youwei Qin & 1759 & $1306 \pm 331$ \\
Fintan Gavin & 1555 & $635 \pm 278$ \\
Giedrius Talacka & 1514 & $1063 \pm 338$ \\
Juergen Bachmann & 1088 & $527 \pm 198$ \\
Sergey Indenok & 852 & $881 \pm 371$ \\
Sebastian Schwab & 516 & $1086 \pm 598$ \\
Dara O'Kearney & 456 & $78 \pm 250$ \\
Roman Shaposhnikov & 330 & $131 \pm 305$ \\
Shai Zurr & 330 & $499 \pm 360$ \\
Luca Moschitta & 328 & $444 \pm 580$ \\
Stas Tishekvich & 295 & $-45 \pm 433$ \\
Eyal Eshkar & 191 & $18 \pm 608$ \\
Jefri Islam & 176 & $997 \pm 700$ \\
Fan Sun & 122 & $531 \pm 774$ \\
Igor Naumenko & 102 & $-137 \pm 638$ \\
Silvio Pizzarello & 90 & $1500 \pm 2100$ \\
Gaia Freire & 76 & $369 \pm 136$ \\
Alexander B\"{o}s & 74 & $487 \pm 756$ \\
Victor Santos & 58 & $475 \pm 462$ \\
Mike Phan & 32 & $-1019 \pm 2352$ \\
Juan Manuel Pastor & 7 & $2744 \pm 3521$ \\
\bottomrule
\end{tabular}
\caption{DeepStack's win rate against professional poker players, with $95\%$ confidence interval.}
\label{tab:deepstack-results}
\end{table*}

\subsection{Thinking Time}
\second{
To minimize the thinking time, both search and network evaluation were ran on GPU (single NVIDIA GeForce GTX 1080) and we batched the network evaluations.
To further speed up the play, we cached the resolving result for every observed preflop situation.
When the same information state was reached again, we simply retrieve the result from the cache.
Table~\ref{tab:iig-deeptack_thinking_time} then reports thinking times for DeepStack and humans.
Note that DeepStack acted considerably faster than our human players in all rounds, and that the pre-flop round was by far the fastest due to the cache.

\begin{table*}[ht]
\centering
\renewcommand{\arraystretch}{1.3}
\begin{tabular}{@{}lllll@{}}
\toprule
 & \multicolumn{2}{c}{Humans} & \multicolumn{2}{c}{DeepStack} \\
Round & Median & Mean & Median & Mean \\
\midrule
Pre-flop & 10.3 & 16.2 & 0.04 & 0.2 \\
Flop & 9.1 & 14.6 & 5.9 & 5.9 \\
Turn & 8.0 & 14.0 & 5.4 & 5.5 \\
River & 9.5 & 16.2 & 2.2 & 2.1 \\
Per Action & 9.6 & 15.4 & 2.3 & 3.0 \\
Per Hand & 22.0 & 37.4 & 5.7 & 7.2 \\
\bottomrule
\end{tabular}
\caption{
Thinking time [s] of DeepStack and humans.
}
\label{tab:iig-deeptack_thinking_time}
\end{table*}

}

\section{Local Best Response Evaluation}
\label{sec:iig-deepstack_lbr}
\second{
Not only DeepStack was the first to beat professional human in no-limit Texas hold'em poker.
Unlike the previous approaches based on the abstraction framework (Section~\ref{sec:iig-abstraction_methods}), suggesting that it's substantially harder to exploit and closer to optimal policy.
Furthermore, we can evaluate how the choice of the actions in the sparse lookahead affect the LBR performance.
Table~\ref{tab:iig-deeptack_lbr_deepsack_action_sets} reports the LBR results with three different actions abstractions.
Regardless of the abstraction in place, LBR is unable to exploit the agent.
}

\begin{table*}[ht]
\centering
\renewcommand{\arraystretch}{1.3}
\begin{tabular}{@{}lllll@{}}
\toprule
LBR Pre-flop Actions & F, C & C & C & C \\
LBR Flop Actions & F, C & C & C & 56bets \\
LBR Turn Actions & F, C & F, C, P, A & 56bets & F, C \\
LBR River Actions & F, C & F, C, P, A & 56bets & F, C \\
\midrule
Hyperborean (2014) & 721 $\pm$ 56 & 3852 $\pm$ 141 & 4675 $\pm$ 152 & 983 $\pm$ 95 \\
Slumbot (2016) & 522 $\pm$ 50 & 4020 $\pm$ 115 & 3763 $\pm$ 104 & 1227 $\pm$ 79 \\
Act1 (2016) & 407  $\pm$ 47 & 2597 $\pm$ 140 & 3302 $\pm$ 122 & 847 $\pm$ 78 \\
Full Cards [100 BB] & -424 $\pm$ 37 & -536 $\pm$ 87 & 2403 $\pm$ 87 & 1008 $\pm$ 68 \\
DeepStack & -428 $\pm$ 87 & -383 $\pm$ 219 & -775 $\pm$ 255 & -602 $\pm$ 214 \\
\bottomrule
\end{tabular}
\caption{
Local best response performance against strong abstraction-based poker agents in no-limit Texas hold'em poker [mbb/g].
The list of actions considered by the technique varies in different poker rounds, and we report four different configurations.
Actions are (F)old, (C)all (P)ot bet and (A)ll-in. 
For the full list of $56$ bets see \citep{lisy2017eqilibrium}.
Note that a policy that simply always folds each hand is exploitable by $750mbb$.
}
\label{tab:iig-deepstack_lbr_vs_bots}
\end{table*}

\begin{table*}[ht]
\centering
\renewcommand{\arraystretch}{1.3}
\begin{tabular}{@{}ll@{}}
\toprule
First level actions & LBR performance\\
\midrule
F, C, P, A & $-479 \pm  216$ \\
Default & $-383 \pm 219$ \\
F, C, \nicefrac{1}{3}P, P, \nicefrac{1}{2}P, 2P, A & $-406 \pm 218$ \\
\bottomrule
\end{tabular}
\caption{
LBR against DeepStack with different action sets in DeepStack's lookahead tree.
}
\label{tab:iig-deeptack_lbr_deepsack_action_sets}
\end{table*}

\section{AIVAT Variance Reduction Technique}
\label{sec:iig-aivat}

\second{
No-limit poker is inherently noisy game, where one needs to collect a large number of games to draw statistical conclusions.
Consider the first human-machine event in no-limit Texas hold'em poker, where the Claudico agent faced a team of four professional poker players \citep{ganzfried2016reflections}.
The match involved $80,000$ hands of poker, where each human played dozens of hours over the span of seven
days.
Despite this significant investment of human time and Claudico losing by over $90mbb$, the result was still not statistically conclusive (it was on the very edge of the statistical test).
It is wort mentioning that Claudico was an instance of abstraction based agent (Section~\ref{sec:iig-abstraction_methods}) combined with unsafe resolving method (Section~\ref{sec:iig-unsafe_resolving}) used for \subgame{} refinement.

Another large scale human match took place from January 11 to January 31, 2017.
Libratus (Section~\ref{sec:iig-libratus}) faced four top-class human poker players (Jason Les, Dong Kim, Daniel McAulay and Jimmy Chou). 
While this time the results allowed for statisticaly significant conclusion, the match involved a total of $120,000$ hands (50\% increase compared to Claudico).
And since this time Libratus was (mostly) a search based agent, it bested its human opponents.
}

\subsection{Previous Variance Reduction Methods}
\second{
Duplicate (or mirror) hands is a simple idea aimed to decrease the role of luck introduced by the chance events  \citep{davidson2014baseline}.
Mirror hands methods replicate the chance events for both positions of the players.
Namely in poker, the agents get to play both positions (small and big blind seats) with the identical chance deals (both private and public cards).
The motivation is that if one player gets ``lucky'' by favourable card deal, the variance will be reduced as the agent also gets to play from the less favourable position (note that this method also halves the number of datapoints as the paired outcome then forms a single measurement).
This is of course problematic when human play is involved, as we need to reset the memory of the agent to make sure they can not predict the chance events.
Typical workaround is then to play against human pairs, where the paired humans are dealt the duplicate hands while playing in a separate room (room one: computer vs player one, room two: computer vs player two) \citep{hedbergman}.
This pairing method was indeed used during both Claudico and Libratus events.
While simple, this method provides only a modest improvement.
}

\second{
Second, the technique of importance sampling over imaginary observations \citep{bowling2008strategy} uses the knowledge of player's policy to replace the single terminal state return $R_i(z)$ with expectation over the compatible terminal states in $\mathcal{S}_{pub}(z)$ and uses the importance sampling correction.
In poker terms, this replaces the single hand outcome with the expected value over all the other hands the agent could be dealt and their respective reach probabilities to that terminal public state (often referred to as the range evaluation in poker literature).
}

\second{
Third, techniques like MIVAT \citep{white2009learning} use the method of control variates (Section~\ref{sec:iig-control_variates}) to reduce the variance caused by chance events.
Each chance action is paired with a value estimate (baseline), resulting in correction terms for each chance event on the trajectory
(just like in VR-MCCFR in Section~\ref{sec:iig-vrmccfr}).
In poker, these correction terms essentially aim to negate the chance ``luck'' factor, e.g. when the agent hits good/bad card deal.
}

\subsection{AIVAT}

\begin{figure}[ht]
    \centering
    \begin{subfigure}[b]{0.18\textwidth}
        \includegraphics[width=\textwidth]{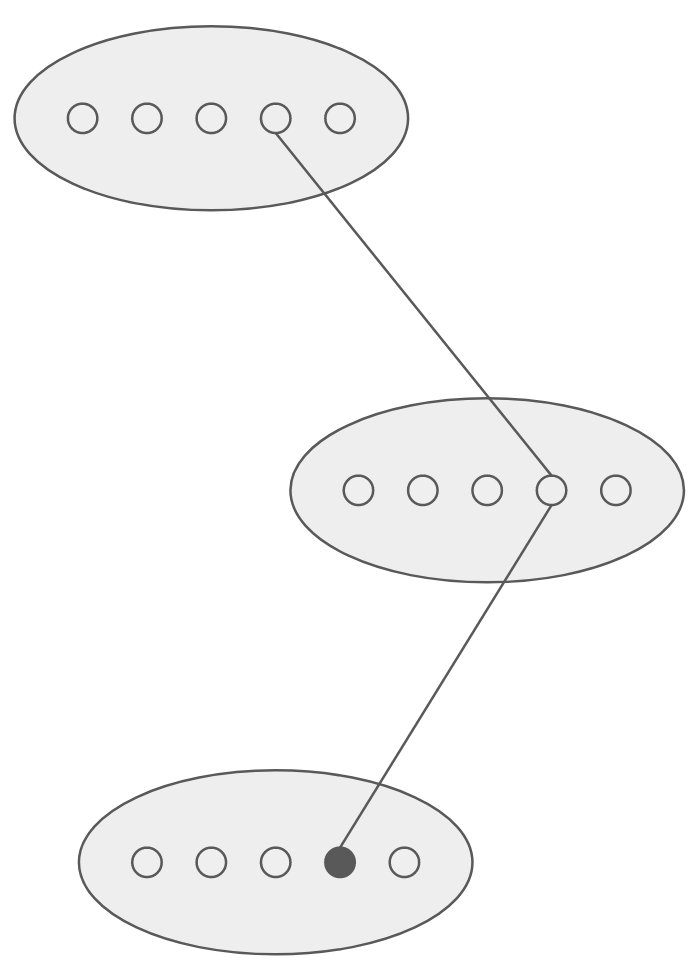}
        \caption{}
        \label{fig:iig-aivat1}
    \end{subfigure}
    \begin{subfigure}[b]{0.18\textwidth}
        \includegraphics[width=\textwidth]{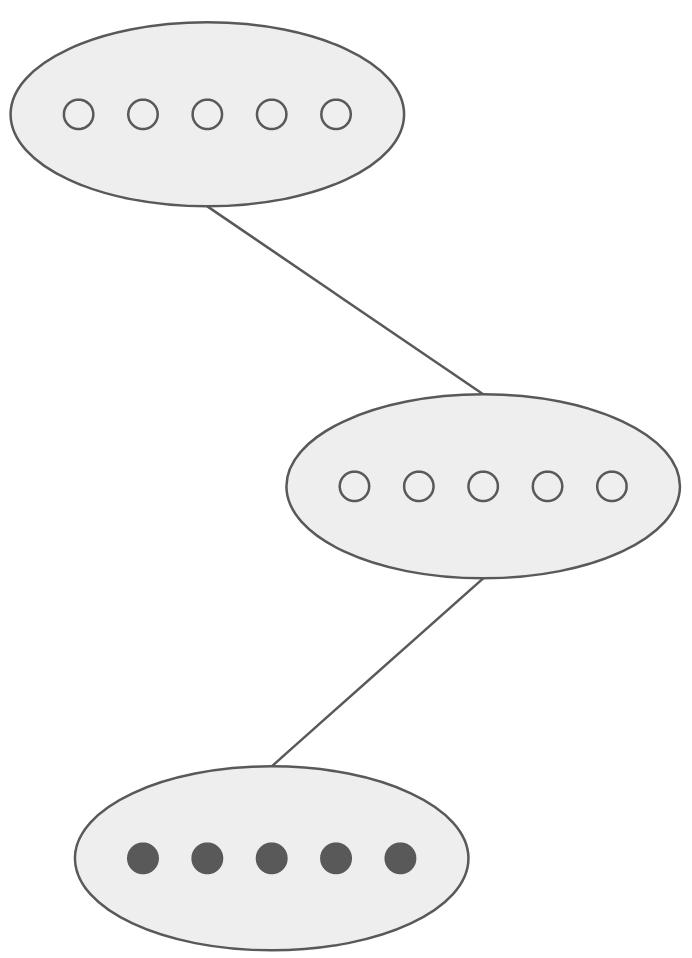}
        \caption{}
        \label{fig:iig-aivat2}
    \end{subfigure}
    \begin{subfigure}[b]{0.25\textwidth}
        \includegraphics[width=\textwidth]{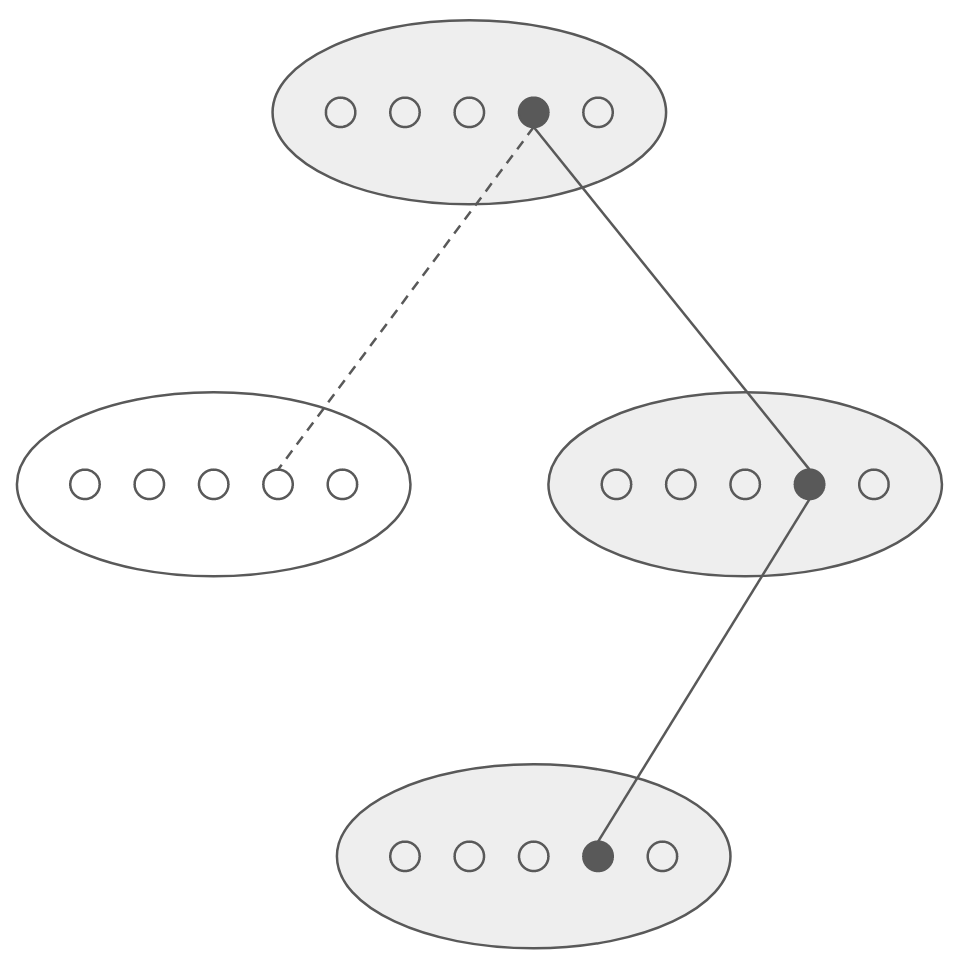}
        \caption{}
        \label{fig:iig-aivat3}
    \end{subfigure}
    \begin{subfigure}[b]{0.32\textwidth}
        \includegraphics[width=\textwidth]{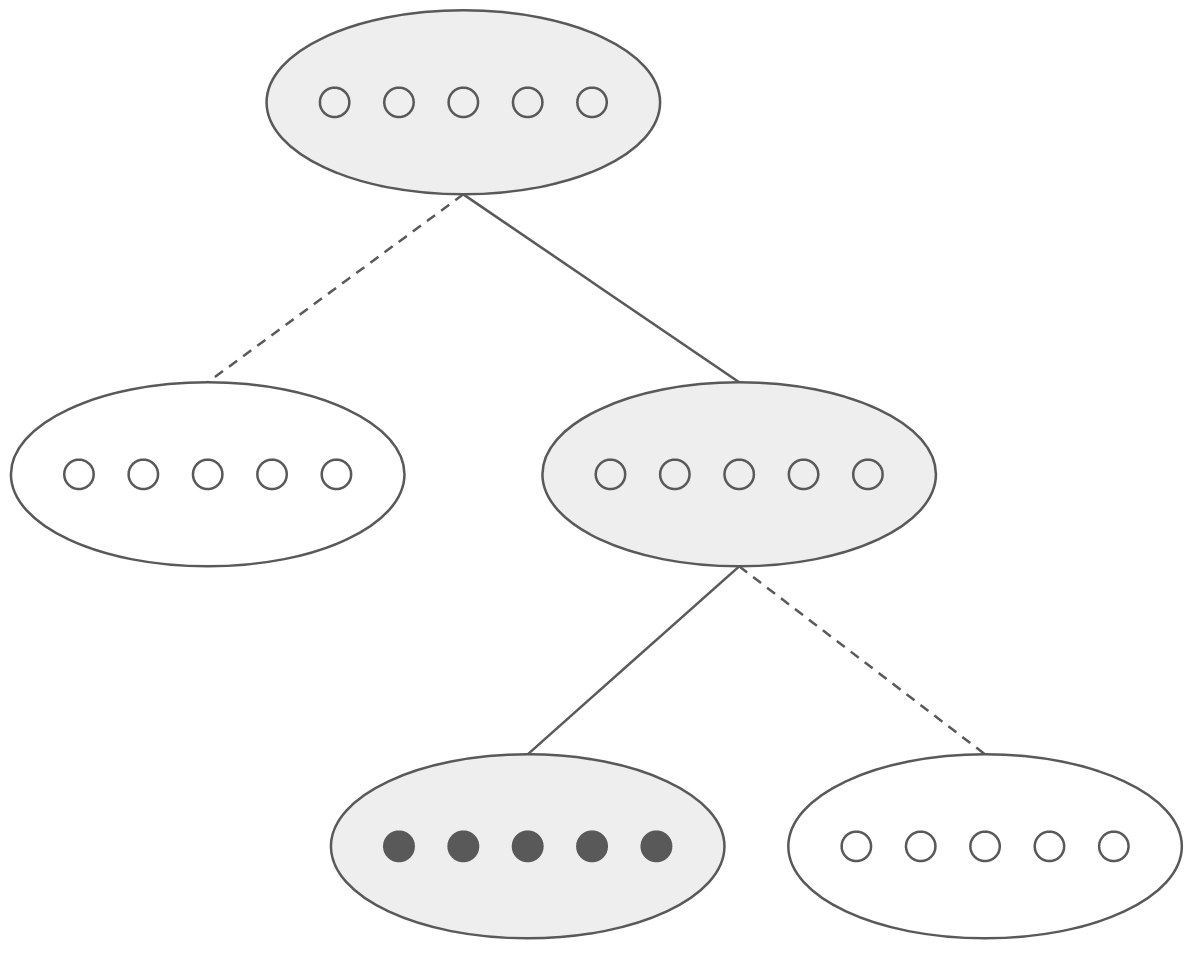}
        \caption{}
        \label{fig:iig-aivat4}
    \end{subfigure}
    \caption{(a) Monte Carlo estimate uses the single sampled outcome $R_i(z)$.
    (b) Importance sampling over imaginary observations considers the outcomes of all the possible private states.
    (c) MIVAT uses a single outcome and control variate with baselines to form correction terms for each chance event on the trajectory
    (d) AIVAT uses imaginary observations in combination with correction terms applied for chance events as well as the the agent's decitions. 
    }
    \label{fig:iig-aivat}
\end{figure}

\second{
\Gls{aivat} (action-informed value assessment tool, \citep{burch2018aivat}) combines the technique of imaginary observations and control variates.
Unlike MIVAT, it uses the control variate based correction terms both for the chance actions and the agent actions.
There are no correction terms for opponent's actions as we do not know their policy.

\begin{table*}[ht]
\centering
\renewcommand{\arraystretch}{1.3}
\begin{tabular}{@{}lll@{}}
\toprule
Estimator & $\bar{v}_x$ & $SD(v_x)$ \\
\midrule
Monte Carlo (chips) & 0.03871 & 25.962 \\
MIVAT & 0.02038 & 21.293 \\
MIVAT+Imaginary & 0.02596 & 16.073 \\
AIVAT & 0.00186 & 8.095 \\
\bottomrule
\end{tabular}
\caption{
Comparison of variance reduction techniques in no-limit Texas hold'em poker.
Data are computed from self-play of a small abstraction agent for a 100 big blind variant.
}
\label{tab:iig-aivat1}
\end{table*}

We compare the techniques in no-limit Texas hold'em poker, using a relatively small abstraction agent trained with MCCFR (Section~\ref{sec:iig-mccfr}) for the underlying policy.
The resulting counterfactual values are then used as the baselines for the correction terms.
The data are generated from self-play of the agent, using 1 million games.
Table~\ref{tab:iig-aivat1} compares the variance of Monte Carlo estimator, MIVAT and AIVAT.

AIVAT leads to substantially lower variance compared to prior techniques, resulting in  68\% decrease in standard deviation.
Just like better value estimates lead to lower variance for VR-MCCFR, so do they in the case of AIVAT.
While this is substantial improvement over the prior techniques, the same variance reduction technique in DeepStack proved even more effective.
This suggests that valued produced by the continual \resolving{} are much closer to the true values despite the fact that DeepStack faced human players. 
}

\section{Other Search Agents}
\label{sec:iig-other_agnets}

\second{
After DeepStack, there are now multiple agents that successfully employ search techniques in imperfect information games.
}

\subsection{Libratus}
\label{sec:iig-libratus}
\second{
Shortly after DeepStack, Libratus used search to beat top human professional players \citep{brown2018superhuman}.
Libratus has essentially two phases.
When the \subgame{} is small enough to be reasoned about, Libratus runs continual resolving with full lookahead and sparse action abstraction (the authors refer to the continual resolving as ``nested resolving'' as they developed the technique independently in parallel \citep{brown2017safe}).
In the earlier rounds though, Libratus did not use limited lookaheads combined with a value function.
Rather than doing search, it falls back to the prior abstraction based techniques (Section~\ref{chap:iig-abstraction}), where the pre-computed abstraction policy is referred to as the ``blue print'' policy.
}

\subsection{Modicum}
\second{
Following Libratus, the authors presented Modicum \citep{brown2018depth}.
Modicum replaces the ``blue print'' policy and uses limited lookahead search with value functions.
Rather than learning the value function, a hand-crafted rollout-based evaluation function is used.
Advantage of this evaluation function is that it is particularly fast and easy to evaluate.
On the other hand, it has very limited guarantees and one has to construct this evaluation function for each game manually
}

\subsection{Supremus}
\second{
Supremus is essentially a re-implementation of DeepStack with larger neural networks, larger search tree and more training data \citep{zarick2020unlocking}.
Supremus was then involved during the $2021$ public match between Doug Paulk (considered to be the best heads-up no-limit Texas hold'em specialist in the world) and Daniel Negreanu (recognized as the best poker player of the decade in $2014$ and having won over $42,000,000$ USD).
Supremus was used by Doug Paulk to help him to formulate the strategies, and he ended up winning around $1,200,000$ USD   \citep{dougsupremus}.
Thanks to the continual \resolving{} and learned value functions, state of the art agents in few years improved from being easily exploitable by simple techniques, to beating even the best human players.
}

\subsection{DeepRole}
\second{
DeepRole is an agent for the game ``The Resistance: Avalon'' --- a multi-agent hidden role game.
Its search algorithm combines CFR, search and value functions operating on public states, trained similarly to DeepStack \citep{serrino2019finding}.
}

\subsection{ReBeL}
\second{
Rebel is a recent algorithm that can be though of as ``simulating'' CFR-D (Section~\ref{sec:iig-cfrd})
\citep{brown2020combining}.
While the algorithm does not have to use gadget during the \resolving{}, it comes at a price.
It makes very strong assumptions about the \resolving{} policy, where it assumes an exact match between its training and evaluation phase.
Namely, the \resolving{} policy in a \subgame{} is assumed to match the one that produced the values during the trunk computation.
In other words, rather than using gadget, is uses the unsafe resolve (Section~\ref{sec:iig-unsafe_resolving}) and assumes the resulting policy to match the one that originally produced the \subgame{} values.
This strong assumptions also limits the algorithm to use different lookaheads and thinking time during the actual play (a crucial property of strong search based agents in perfect information games), as it introduces a mismatch between the training and evaluation phase.

This algorithm thus have fundamentally different properties and assumptions to the agents built on the notion of continual \resolving{}, where the only assumption is quality of the value function
}

\chapter*{Conclusion}

\first{
Thanks to the combination of search and (learned) value functions, computers bested even the strongest of human players in Chess, Go, Backgammon, Arima and many other perfect information games.
Generalization of the underlying concepts to imperfect information then allows for search methods in imperfect information,  resulting in agents that outplay humans in poker --- challenging game of imperfect information.

Imperfect information games allow for modeling much wider class of problems and interactions.
As many real world problems do not allow for perfect information, the concepts introduced in this thesis allow for potentially interesting applications.
}


\bibliographystyle{apalike}    

\renewcommand{\bibname}{Bibliography}


\bibliography{bibliography}


%
%

\listoffigures

\listoftables


\printglossaries

\chapwithtoc{List of Publications}
\label{chap:list_of_publications}

\section*{Journal Papers}
Moravcik Matej$^*$\footnotetext{$^*$Equal contribution, alphabetical order.}, Martin Schmid$^*$, Neil Burch, Viliam Lisý, Dustin Morrill, Nolan Bard, Trevor Davis, Kevin Waugh, Michael Johanson, and Michael Bowling. "Deepstack: Expert-level artificial intelligence in heads-up no-limit poker." Science 356, no. 6337 (2017): 508-513.
\\
\\
Burch, N., Moravcik, M. and Schmid, M., 2019. Revisiting cfr+ and alternating updates. Journal of Artificial Intelligence Research, 64, pp.429-443.

\section*{Conference Papers}
Sustr, M., Schmid, M., Moravcik, M., Burch, N., Lanctot, M., \& Bowling, M. (2020). "Sound search in imperfect information games" 20th International Conference on Autonomous Agents and Multiagent Systems
\\
\\
Davis, Trevor, Martin Schmid, and Michael Bowling. "Low-Variance and Zero-Variance Baselines for Extensive-Form Games." International Conference on Machine Learning. PMLR, 2020.
\\
\\
Schmid, Martin, et al. "Variance reduction in monte carlo counterfactual regret minimization (VR-MCCFR) for extensive form games using baselines." Proceedings of the AAAI Conference on Artificial Intelligence. Vol. 33. No. 01. 2019.
\\
\\
Burch, N., Schmid, M., Moravcik, M., Morill, D. and Bowling, M., 2018, April. Aivat: A new variance reduction technique for agent evaluation in imperfect information games. In Thirty-Second AAAI Conference on Artificial Intelligence.
\\
\\
Moravcik, M., Schmid, M., Ha, K., Hladik, M. and Gaukrodger, S.J., 2016, February. Refining subgames in large imperfect information games. In Thirtieth AAAI Conference on Artificial Intelligence.
\\
\\
Schmid, M., Moravcik, M. and Hladik, M., 2014, June. Bounding the support size in extensive form games with imperfect information. In Twenty-Eighth AAAI Conference on Artificial Intelligence.

\section*{Workshop Papers}
Schmid, Martin, et al. "Automatic public state space abstraction in imperfect information games." AAAI Workshop: Computer Poker and Imperfect Information. 2015.

\section*{ArXiv}
Gruslys, Audrunas, et al. "The advantage regret-matching actor-critic." arXiv preprint arXiv:2008.12234 (2020).
\\
\\
Sokota, Samuel, et al. "Solving Common-Payoff Games with Approximate Policy Iteration." arXiv preprint arXiv:2101.04237 (2020).
\\
\\
Timbers, F., Lockhart, E., Schmid, M., Lanctot, M., Bowling, M. (2020). Approximate exploitability: Learning a best response in large games. arXiv preprint arXiv:2004.09677.
\\
\\
Kovarik, V., Schmid, M., Burch, N., Bowling, M. and Lisy, V., 2019. Rethinking formal models of partially observable multiagent decision making. arXiv preprint arXiv:1906.11110.
\\
\\

\appendix

\openright
\end{document}